%% file: propre.tex
\documentclass[oneside,english]{article}

\input{macros/macros}

\usepackage{esint}

\makeatletter

\numberwithin{equation}{section}
\numberwithin{figure}{section}
\theoremstyle{plain}

  \theoremstyle{definition}
  
  \theoremstyle{definition}
  
  \theoremstyle{plain}
  
  \theoremstyle{plain}
  
  \theoremstyle{remark}
  
  \theoremstyle{plain}

\makeatother

\usepackage{babel}
  \providecommand{\corollaryname}{Corollary}
  \providecommand{\definitionname}{Definition}
  \providecommand{\examplename}{Example}
  \providecommand{\lemmaname}{Lemma}
  \providecommand{\remarkname}{Remark}
  \providecommand{\theoremname}{Theorem}
  \providecommand{\propositionname}{Proposition}

\usepackage{authblk}

\begin{document}

\title{Probabilistic supervised learning}

\author[1]{
Frithjof Gressmann
\thanks{\url{frithjof.gressmann.16@ucl.ac.uk}}
}

\author[1]{
Franz J.~Kir\'{a}ly
\thanks{\url{f.kiraly@ucl.ac.uk}}
}

\author[2]{
Bilal Mateen
\thanks{\url{b.mateen@warwick.ac.uk}}
}

\author[3]{Harald Oberhauser
\thanks{\url{oberhauser@maths.ox.ac.uk}}
}

\affil[1]{
Department of Statistical Science,
University College London,\newline
Gower Street,
London WC1E 6BT, United Kingdom
}

\affil[2]{Warwick Medical School,
University of Warwick,\newline
Coventry CV4 7AL, United Kingdom
}

\affil[3]{Mathematical Institute,
University of Oxford,\newline
Andrew Wiles Building,
Oxford OX2 6GG, United Kingdom
}

\thispagestyle{empty}
\maketitle


\begin{abstract}
Predictive modelling and supervised learning are central to modern data science.
With predictions from an ever-expanding number of supervised black-box strategies - e.g., kernel methods, random forests, deep learning aka neural networks - being employed as a basis for decision making processes, it is crucial to understand the statistical uncertainty associated with these predictions.\\

As a general means to approach the issue, we present an overarching framework for black-box prediction strategies that not only predict the target but also their own predictions' uncertainty. Moreover, the framework allows for fair assessment and comparison of disparate prediction strategies.
For this, we formally consider strategies capable of predicting full distributions from feature variables (rather than just a class or number), so-called probabilistic supervised learning strategies.\\

Our work draws from prior work including Bayesian statistics, information theory, and modern supervised machine learning, and in a novel synthesis leads to (a) new theoretical insights such as a probabilistic bias-variance decomposition and an entropic formulation of prediction, as well as to (b) new algorithms and meta-algorithms, such as composite prediction strategies, probabilistic boosting and bagging, and a probabilistic predictive independence test.\\

Our black-box formulation also leads (c) to a new modular interface view on probabilistic supervised learning and a modelling workflow API design, which we have implemented in the newly released skpro machine learning toolbox, extending the familiar modelling interface and meta-modelling functionality of sklearn. The skpro package provides interfaces for construction, composition, and tuning of probabilistic supervised learning strategies, together with orchestration features for validation and comparison of any such strategy - be it frequentist, Bayesian, or other.
\end{abstract}

\newpage
{\small
\tableofcontents{}
}
\newpage

\section{Introduction}

In many application scenarios one wishes to make predictions for an outcome which is inherently uncertain, to an extent where even perfect prior knowledge does not allow to make a point prediction - for example when predicting an event risk for an individual data point, e.g., a cancer patient's survival/death or a debitor's default. In other scenarios, a point prediction is sufficient, however knowing the predictions uncertainty is crucial. And generally, knowing the expected variation in observations is preferable about not knowing.

While the advantage of predictions with uncertainty certificates or, equivalently, probabilistic predictions, is probably universally acknowledged, the question of a suitable framework for model building and validation remains open. For instance, it needs to be able to fairly compare a frequentist supervised learning model with a Bayesian model applied to the task of predicting with an uncertainty guarantee.

Here, often the claim is heard that Bayesian type learning is incomparable with the ``classical'' supervised learning paradigm which is interpretationally more frequentist, or that in fact one is superior say on the grounds that one paradigm naturally predicts probability densities and can take into account a-priori ``beliefs'', or that the other paradigm is independent of model-specific assumptions and more ``natural'' due to absence of arbitrary prior beliefs.

Though we would argue that many of these points are moot, since from a scientific perspective, any method should be judged by how it solves the task it is supposed to solve - namely, predicting the outcome and its uncertainty - and by nothing else. Assuming, of course, that the task is indeed prediction (and not parameter estimation or inference, say).

In terms of the set-up, what is required of such a unified framework for the probabilistic prediction task, is a recognition of pertinent ideas in both domains (Bayesian and frequentist), while reconciling them within the setting where an accurate prediction with uncertainty guarantees is of primary interest. Of particular relevance are the following aspects:
\begin{itemize}
\item[(i)] Predicting a probability distribution, as often in the Bayesian framework and considerably less frequent in the frequentist one, is desirable as it is part of a preferable task. For example, the smarter predictor for the outcome of a fair game of dice would predict a uniform probability for each outcome, rather than a fixed outcome at random - since forcing to make the predictor making a point estimate is akin to forcing it to ``forget'' crucial information, namely knowledge of the result generating mechanism.
\item[(ii)] Avoiding model-specific assumptions, especially in comparison and assessment of models, as more frequent in the frequentist framework and less frequent in the Bayesian one (e.g., BIC and WAIC being model class dependent), is preferable above making such assumptions. Since validation of models by a procedure that is possibly specific to the type of model or model class considered is in constant danger of leading to circular argumentation.
\item[(iii)] Finally, the question of whether having a Bayesian prior in prediction is preferable or not (hence whether a Bayesian or frequentist type model is preferable) is not directly related to the task of predicting itself. In particular, as anything else, it is not relevant in the context of solving the prediction task, unless indirectly through the practical merit in prediction itself, which means accuracy of prediction, scientific parsimony/interpretability, reproducibility, and deployability, for instance.
\end{itemize}

For a synthesis of the above, consider the stylized thought experiment outlined in Figure~\ref{figure:faircomparison}.
We thus argue that adoption of a black-box view together with an external validation workflow is the most pertinent choice in the given task-centric setting.

\begin{figure}
\label{figure:faircomparison}
\centering
\includegraphics[width = 1.0\textwidth]{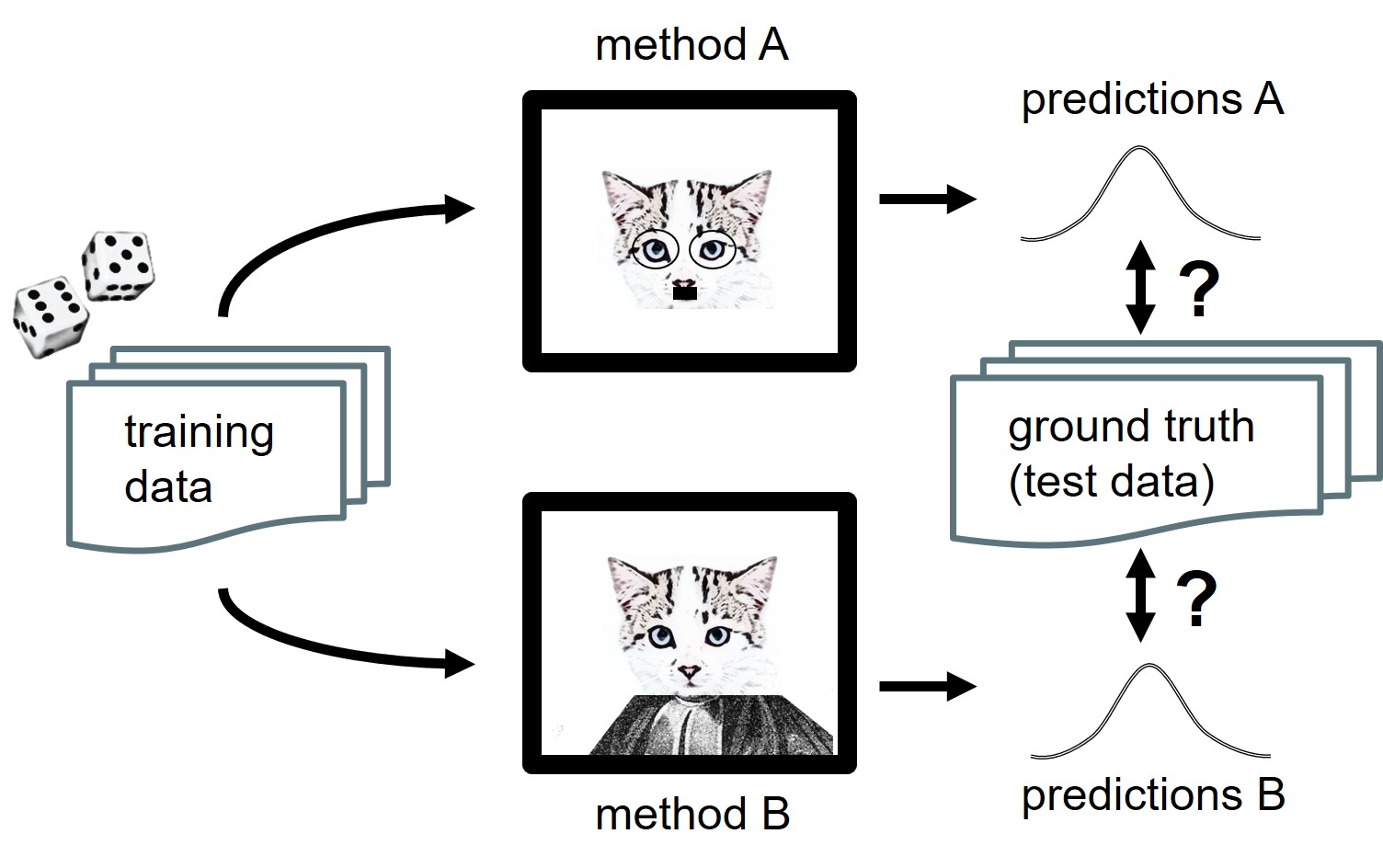}
\caption{Thought experiment illustrating the domain-agnostic black-box rationale.\newline
Suppose we query predictions from two different black-box prediction strategies, method A - symbolized by the frequentist cat in a black box - and method B - symbolized by the Bayesian cat in a black box. Both have access to the same training data, and are tasked with producing probabilistic predictions for the same test cases. Said predictions are scrutinized and evaluated against subsequent ground truth observations. The key point to observe is that for a fair comparison with regards to performance on this prediction task, knowledge of the specific workings inside the black box (e.g., does the cat believe in so-called ``priors'' and if yes which types of, or rather in the epistemiological validity of rejecting a null hypothesis; or, whether the cat is still alive) is not only unnecessary, but in fact irrelevant to external validity of either method, as it does not pertain to the task which is probabilistic prediction (and not model inference). Phrased more precisely, the means by which the predictions were obtained have no bearing on whether the predictions are accurate, hence the means by which the predictions were obtained should play no role whatsoever in a fair assessment of whether the predictions were accurate, nor in a fair comparison of the two methods with respect to their predictions' accuracy in the given task. The same argumentation applies for prediction strategies in any other cat-egory, e.g., Fisherian, deep learning aka neural networks, or any data specific, situational heuristics.}
\end{figure}

\newpage
Therefore, if one accepts the above arguments, a probabilistic supervised learning framework will:
\begin{itemize}
\item[1.] solve the task of predicting probability distributions,
\item[2.] allow model-agnostic validation and comparison for ``Bayesian'' and ``frequentist'' predictive models alike, and
\item[3.] be easily implemented in a modelling (e.g., software) toolbox that unifies both types of methodology.
\end{itemize}

In the following, we present such a framework that is inspired by motifs from both statistical paradigms and the common machine learning motif of methods as task-specific black-boxes to be studied from the ``outside''.

Our framework borrows substantially from prior art and is not new in terms of many of its components, but mainly in terms of its composition which we believe to be novel not only in an interpretative manner, but also in a way that allows easy collation say in form of a workflow toolbox for black-box learning and meta-learning with modular interface to both frequentist and Bayesian learning algorithms.

We would like to stress that unlike some of previous literature, we do not attempt to ``unify'' frequentist or Bayesian statistics in general:
our framework only applies in a predictive/supervised context where many concepts from both worlds can be shown (below) to coincide or overlap.

How a joint point of view may be achieved in general, say for parameter estimation or unsupervised learning, or whether it is even possible or desirable, is not entirely clear,
but also outside this manuscript's remit, to stress this again.

\subsection{Main ideas}
\label{sec:mainideas}
Central to our approach is conceptually separating (1) the task that is to be solved (prediction, probabilistic) from (2) the way it is actually solved (Bayesian or frequentist) from (3) philosophical considerations (what probabilities ``really mean'', or which empirically unvalidable beliefs are preferable). We think this is not only a good idea but necessary mental hygiene implied by the scientific method which separates (1) experiment, ``testing whether the task is solved'', from (2) predictions to be tested, ``method X will solve the task'', and (3) meta-empirical appraisal, ``I like method X because of the ideas it is based on''.\\

(1) The task to be solved, predicting a probability distribution from features, given a batch of examples to learn the model, naturally comes with the necessity of specifying a quantitative measure of whether and how well it has been solved - which is crucial to separate from the solution to avoid circular reasoning, say as in confirming beliefs that one already had to begin with (as in fully Bayesian Bayes factor validation). On a similar note, to enable assessment of any solution to the task it is necessary to avoid using an assessment method that applies only for certain types of models (such as BIC and WAIC for Bayesian models, or goodness-of-fit tests for specific frequentist models).
Rephrased, we argue for what is called an ``external'' validation as the process of checking a prediction should be independent from whether the model is frequentist or Bayesian, and from scientifically unvalidable beliefs (e.g., through a belief model). More concretely, among the popularly used Bayesian predictive validation criteria listed in the review paper of~\citet{vehtari2012Bayesian}, the only ones meeting these conditions are the $\calM$-open ones, in the terminology of the review, including the often-used Bayesian cross-validation via log-loss on the sample of posterior predictive distributions. The latter generalizes straightforwardly to performing out-of-sample (= independent test set) validation with a log-likelihood loss for any black-box method predicting a probability distribution (Bayesian or not); a special case is the log-loss commonly used in a frequentist classification scenario where class probabilities are predicted.\\

(2) Regarding the strategies solving the task, we adapt a pure black-box view where any algorithm that can predict probability distributions based on observation of example cases is admissible - similar to the (frequentist or Bayesian) supervised learning framework where non-parametric or entirely algorithmically specified models are also allowed; note how this includes both frequentist and Bayesian type models. For example, the former may include methods yielding a predictive variance such as frequentist formulations of Gaussian processes or feature-dependent mixture models; the latter may include returning the prior or posterior predictive distribution as the prediction output of the black-box. The black-box predictor may have so-called tuning parameters, which in this setting is defined as any modelling choice that is not made based on the data (or equivalently, something that the user may need to externally set in a software framework). In particular, the prior(s) of a fully Bayesian model are considered to be black-box tuning parameters, together with the hyper-parameters in Bayesian terminology. While that does not agree with the usual Bayesian terminology or some of its philosophical interpretations, from a scientific or framework design perspective it appears practically reasonable, because a choice that is made independently of the data should be clearly exposed as a choice that is independent of the data - independently of what that choice is, or how it is philosophically interpreted.\\

(3) The philosophical interpretation(s) and the belief of which one is most desirable we would leave with the reader, while we consider to see ourselves inspired by scientific empiricism in the Baconian sense which postulates validation-by-experiment and/or falsification-by-experiment. In particular, there should be a common (and, by statistical empiricism,) quantitative ``experiment'' which can tell us, given data, whether we should believe for all practical purposes that a reliable prediction has been made, yes/no/maybe and to which degree, independent of the method making the prediction (Bayesian, frequentist, or otherwise). With George Box, we acknowledge that ``all models are wrong'' (but some are useful, and some others are plain stupid), while any statistical model - or modelling process - that is ``not even wrong'' in the sense of Wolfgang Pauli (impossibility or absence of validation/falsification-by-experiment) should be rejected as pseudo-science.
In particular, we do not claim to propose the one-and-only framework following these principles without the need to embrace a particular modelling philosophy, nor do we claim to propose the one-and-only framework that ``finally unifies'' the ``opposing schools of thought'' of ``Bayesianism'' and ``frequentism'' - we only provide a (not ``the'') framework which supports reproducible, impartial scientific argumentation and empirical validation along the lines of~\citet{gelman2017beyond}, using currently available tools and concepts, supported by a clean and transparent consensus workflow design, for the specific task of probabilistic supervised learning.

\subsection{Technical contributions}

Our main contribution is providing a comprehensive framework for supervised learning in a probabilistic setting.
More precisely, it consists of:

\begin{itemize}
\item[(i)] An explicit mathematical formulation of the statistical set-up for probabilistic supervised learning, in Sections~\ref{sec:propreset} to~\ref{sec:probloss}.
While drawn mostly from motifs existing in literature, this seems to be the first instance in which the full picture is presented, in a model-independent black-box setting that allows for external validation and comparison of any predictive modelling strategy of the probabilistic type.

\item[(ii)] An identification of classical, or ``point predictive'' supervised learning under different mean losses,
with specific probabilistic model classes depending on the choice of classical mean loss, in Sections~\ref{sec:classprob} and~\ref{sec:classprob-badidea}.
While reminiscent of the least-squares-is-max-likelihood observation in linear regression, this seems to be almost fully novel (and surprisingly so).

\item[(iii)] A suggested practical workflow for comparing and assessing any type of predictive models in a probabilistic setting, in Sections~\ref{sec:modelvalidation.estim} to~\ref{Sec:classicbase}. While somewhat similar to Bayesian cross-validation, the workflow is agnostic of model type. It is further supported by quantification strategies
    which are shown to have frequentist and Bayesian interpretations at the same time, and a the definition of the concept of ``uninformed baseline'' which in
    the probabilistic setting is shown to be density estimation.\\
    We would like to re-iterate our point that we are not claiming the suggested workflow to be the only one that potentially makes sense, our aim is rather to provide one that is workable for the probabilistic black-box setting (which has not been addressed comprehensively in literature so far).\\
    However, through our set-up we are able to identify a number of problems and stylized failures of wide-spread workflow motifs, including frequentist permutation baselines (Section~\ref{Sec:badbaseline}), as well as AIC/BIC based model selection which may lead to probabilistic over-fitting (Section~\ref{sec:BIC}).

\item[(iv)] A number of theoretical results which expand upon motifs in classical supervised learning and establish links between learning theory for probabilistic prediction and information theory, including:

\begin{itemize}
\item[(iv.a)] a probabilistic bias-variance trade-off (Section~\ref{Sec:biasvariance});
\item[(iv.b)] an information theoretic formulation of probabilistic prediction (Section~\ref{sec:entropy}), including Theorem~\ref{Thm:entropy} which shows equivalence of statistical independence and probabilistic predictability;
\item[(iv.c)] a definition of probabilistic residuals through measure theoretic pullback, compatible with the expected generalization loss (Section~\ref{Sec:trafores}).
\end{itemize}

\item[(v)] Motivated by the above theoretical insights, a number of meta-learning approaches for probabilistic supervised learning, most notably:
\begin{itemize}
\item[(v.a)] semi-quantitative diagnostic plots for probabilistic models (Section~\ref{sec:propres});
\item[(v.b)] loss reduction techniques for probabilistic models (Section~\ref{sec:baggeraging}), by the bias-variance trade-off (iv.a) shown to contain, as special cases, classical bagging and Bayesian model averaging;
\item[(v.c)] boosting  of probabilistic models (Section~\ref{sec:boost}), enabled through the concept of residuals (iv.c) and their optimization;
\item[(v.d)] predictive type hypothesis tests for independence of paired multi-variate observations including two-sample tests, building on (iv.b) which yields a novel identification of probabilistic model validation as entropy estimation (Section~\ref{Sec:hypotest}).
\end{itemize}

\item[(vi)] A literature-based review of some influential probabilistic modelling paradigms from which probabilistic supervised learning strategies may be obtained (Section~\ref{sec:problearners}). Interestingly, most popular off-shelf algorithms (including Bayesian modelling and Gaussian Processes) appear to need at least minor modifications or adaptations before being applicable to the task of probabilistic prediction.

\item[(vii)] A workflow API design compatible with contemporary API design principles for machine learning toolboxes, suitable for a unified implementation of all the above (Section~\ref{sec:API}). The design is implemented in the newly released skpro toolbox for python, extending the widely used sklearn interface to probabilistic supervised learning, and also provides workflow orchestration functionality.

\item[(viii)] The framework, as implemented in the skpro package, is tested in a series of model validation and benchmarking experiments, which also shed some light on useful composite strategies built from point prediction primitives (Section~\ref{sec:experiments}).

\end{itemize}

We would like to stress that there are aspects of our proposed framework which have already been explicitly remarked or implicitly observed in prior literature.
In addition to the general appraisal in the subsequent Section~\ref{sec:priorart}, we discuss such occurrences in the context of the novel suggestions whenever we are aware of a relevant and specific connection, including a reference to the relevant source.
Ideas and results for which no source is cited are to our knowledge novel.

\subsection{Relation to prior art}
\label{sec:priorart}
We briefly discuss where relevant ideas have previously appeared in literature, and in which form.

We would like to give special credit the excellent review papers of~\citet{arlot2010survey} and~\citet{vehtari2012Bayesian} which, to the best our knowledge,
cover extensively the state-of-art in modern statistical model selection and model validation, and to which much of our familiarity with relevant literature
and existing techniques is due.

\subsubsection{Model assessment of probabilistic models via predictive likelihood}
\label{sec:priorart.assessment}
Predictive likelihood based model assessment has appeared
in four major branches of literature:

(a) In {\bf Bayesian literature}, the seminal paper of~\citet{geisser1979predictive} may be seen to contain, in parts of the initial discussion, important ideas
on how to jointly treat and compare probabilistic predictions arising from Bayesian and non-Bayesian models, including the use of the predictive likelihood
on a statistically independent (test) sample, although it pre-dates the advent of machine learning and hence supervised prediction.
The earlier publication by~\citet{roberts1965probabilistic} may be seen as a more qualitative precursor of such ideas.
Generally, studies where a predictive log-likelihood is estimated on held-out data seems to be rare in the Bayesian branch, in the cases known to us this is for
the purpose of variable selection in a classification setting~\cite{draper2000case,fearn2002bayesian}.
There seem to be no more recent instances of this idea.
None of the publications cited in Section~5.1.3, subsection ``k-fold-CV'', of the review~\cite{vehtari2012Bayesian} use
out-of-sample log-loss, instead opting to discuss a setting closer to frequentist model validation. However, it should be noted that in the same review (which is similar to our Section~\ref{sec:modelvalidation.estim}) the authors discuss how this may apply to the probabilistic setting, making it the primarily creditable source for this idea, despite the superficial appearance of only reviewing prior occurrences of the idea.
The rest of Bayesian literature emphasizes methodology which is mostly in-sample and bias-corrected.

\citet{gelman2014understanding} compare the log-loss to such popular Bayesian model selection criteria, similar to our discussion in Section~\ref{sec:BIC},
and mention the expected/empirical log-loss, claiming that in machine learning literature it ``often'' is called the ``mean log predictive density'',
though we were unable to find evidence for wider use except by the authors, though this may simply be due to such use taking place in actual applications of machine learning
rather than in scientific literature proper.

(b) {\bf Probabilistic forecasting} is a branch of statistics that has become quite popular with econometricians, environmental scientists and applied researchers interested in related topics~\cite{gneiting2014probabilistic}. The ``{\bf prequential approach}'' of~\citet{dawid1984present} is probably the origin of this branch of literature, mostly concerned
with (temporal) forecasting rather than (static) supervised prediction, and an empirical log-likelihood loss is among
possible criteria for assessing density predictions~\cite{gneiting2007strictly}. More recently~\citet{musio2013local} and subsequently~\cite{dawid2014theory, dawid2015bayesian} use arbitrary strictly proper scoring rules for Bayesian model selection.

Also found in this branch of research is an idea
known as ``calibration''~\cite{gneiting2007probabilistic} of which our suggestion of probabilistic residuals in Section~\ref{sec:propres} is reminiscent. Though, while much of our work draws from the above, it should be noted that much of the above-mentioned theory does not directly transfer to the probabilistic supervised prediction setting, as the latter introduces dependence between observed value and predicted distribution through the features.

(c) In the {\bf machine learning} literature, the log-loss (= log-likelihood out-of-sample) is used to measure predictive accuracy of probabilistic classification,
mostly of two-class type, see~\cite{rosasco2004loss,zhang2004statistical,masnadi2009design} for instances that specifically discuss the log-loss from an evaluation perspective.
In machine learning, it usually appears as one possible choice amongst many, with certain desirable smoothness properties.
It also frequently appears in software implementations and benchmarking toolboxes, the most prominent ones being scikit-learn~\cite{pedregosa2011scikit}
(\texttt{metrics.log\_loss} in version 0.19.1).
More generally, logistic or logarithmic losses are also ubiquitous in a somewhat more implicit manner in any machine learning or statistical literature concerned with logistic or logit type regression or classification models.

(d) The {\bf conditional density estimation} setting~\cite{rosenblatt1969conditional,silverman1986density} is mathematically and interpretationally quite close to ``probabilistic supervised learning'' as we present it. Modern literature on this topic is usually found in the machine learning domain and evaluates the goodness of the density estimate by out-of-sample or cross-validation estimates of the generalization log-loss~\cite{gray2003rapid,sugiyama2010conditional, holmes2012fast}.
In fact,~\citet{holmes2012fast} already note quite precisely how conditional density estimation may be interpreted in a density predictive sense (compare Section~\ref{sec:propreset.probsupl}), though without further establishing links to Bayesian or classical frequentist methodology.
It is worth noting that much of the conditional density estimation literature seems to assume some smoothness\footnote{usually through implicit or explicit juxtaposition of the word ``kernel'' with ``density estimation''}
in the feature or independent variables and exhibits a focus towards Nadaraya-Watson type estimators.
Conditional density estimation and evaluation via log-loss is implemented in more recent versions of Weka~\cite{hall2009weka}
(namely as \texttt{ConditionalDensityEstimator} and\\
\texttt{InformationTheoreticEvaluationMetric} since version 3.7).

\subsubsection{Connections of predictive model assessment to information theory}

The relationship between predictive model assessment and information theory has been studied previously and operationalised for use in model assessment. For example, one branch of the literature~\cite{robert1996intrinsic,bernardo2003intrinsic,peruggia2003model} has explored information theoretic distances in a fully parametric context, between (for example: generative) distributions which are not observed, as opposed to relating model performance expectations or asymptotics to such quantities in a non-parametric or black-box context.

In the context of the model-agnostic black-box approach which we are considering:\\
From a theoretical perspective, the connections between information theoretic quantities and supervised prediction are well understood for the probabilistic classification~\cite{reid2011information} and the temporal (hence sample-correlated) forecasting setting~\cite{dawid2007geometry}, though the general probabilistic supervised learning setting appears to be unstudied.

From a practical perspective, information theoretic quantities are often used in model evaluation without deeper theoretical consideration, e.g., as the ``cross-entropy loss'' in supervised classification.

\subsubsection{Bayesian/frequentist synthesis}

Reconciliation of the Bayesian and frequentist paradigms has been the topic of many (more or less influential) papers with wildly different ideas on how
to resolve the perceived ideological divide completely. The most prominent suggestion of this type is perhaps the ``Bayes-non-Bayes compromise'' spearheaded by~\citet{good1992bayes}.
As we do not attempt to carry out an overall synthesis and are only concerned with predictive model validation and comparison, we continue to discuss
only literature pertaining specifically to such a context.

The most relevant reference is possibly the work by~\citet{guyon2010model} which provides an overview which we believe to be very valuable, highlighting the crucial issues also discussed in our work and the need for a joint view in the predictive setting.

Though the suggestions in the later parts seem more similar to a least common multiple rather than a greatest common denominator\footnote{which, in our opinion, is not the ``ideal'' to strive for either}, that is, the authors
seem to attempt to make the black-box approach as all-encompassing as possible to accommodate for as many model selection and model validation strategies as possible.
The crucial issue of comparing frequentist to Bayesian models remains largely unaddressed - as opposed to a general argumentative direction discussing how
popular methodological approaches in both worlds are in essence the same while they are also different.

For the context of model comparison which in Sections~\ref{Sec:modelvalidation.tuning} and~\ref{Sec:hypotest} we relate to frequentist hypothesis testing and Bayes factor inference, it has been noted
by~\citet{berger1997unified} that in a relatively general setting these two may be shown to take an (almost) mathematically equivalent form,
yielding a jointly frequentist and Bayesian justification for the proposed types of assessment quantifiers.

On a more general note, we definitely agree with~\citet{efron2005bayesians} on the scientific premise that frequentist and Bayesian ideas need to be combined to address
the challenges of the 21st century (though not necessarily in the form of the empirical Bayes approach, and expressly without claiming any ``philosophical superiority''
to either approach~\cite{efron1986isn}, see also paragraph~3 of Section~\ref{sec:mainideas} above).

\subsubsection{Toolbox API designs for supervised learning and probabilistic learning.}
Major machine learning toolboxes such as Weka~\cite{holmes1994weka,hall2009weka}, scikit-learn~\cite{pedregosa2011scikit,sklearn_api}, caret~\cite{kuhn2008caret,caret2016} and mlr~\cite{bischl2011mlr,mlr2016} share, as central building principles of their API, the idea of object-oriented encapsulation of models or prediction strategies, and sometimes aspects of the validation process, with consistent/unified and modular interface. Data scientific workflow API design in itself is a relatively novel area of research which does not seem to be explicitly pursued in its own right, a notable and in our opinion seminal exception being the mentioned work of~\citet{sklearn_api}. This seeming lack of recognition in the literature appears to stand in crass discrepancy to its enormous practical importance, especially noting that from a scientific perspective\footnote{and, from an esthetic or inventive perspective, finding the right API is no less of a challenging task than finding the right statistical model or mathematical proof, necessitating a similar combination of systematic thinking, creative insight, critical appraisal, and, of course, hard work.\\
Probably for the same reason, i.e., lack of recognition in literature, we are unable to point towards specific scientific API design literature as a precursor and inspiration of our API design suggestions beyond the above, simply since API design is documented usually only implicitly in toolbox manuals, or as a footnote in papers focusing on toolbox use rather than API design. Hence we wish to give credit to these API designs by the above mention of concomitant scientific publications despite API design not necessarily being discussed there, in particular since our work may be also seen as an attempt in ensuring the guiding API principles common to all the mentioned toolboxes, as made precise by~\citet{sklearn_api}, for frequentist and Bayesian methods alike.}, a solid workflow and modelling API is a necessary practical precondition for fair quantitative assessment and comparison of any type of data analytics methodology\footnote{as long as it is carried out on a computer, and not, say, on a punchcard machine or within someone's imagination}, be it frequentist, Bayesian or otherwise.

The above, as modelling toolboxes for classical supervised learning, usually do not implement interfaces for probabilistic prediction (except classification or variance predictions); however, they provide method encapsulation and workflow modularization necessary for model-agnostic model building and comparison.

On the other hand, there are a number of abstract model-building environments for probabilistic modelling such as WinBUGS~\cite{lunn2000winbugs}, Stan~\cite{carpenter2016stan}, Church~\cite{wingate2011lightweight}, Anglican~\cite{wood2014new}, or Edward~\cite{tran2016edward}, and more generally a wealth of distinct approaches often subsumed more recently under the term ``probabilistic programming''~\cite{gordon2014probabilistic, carpenter2016stan}. These have in common that they usually implement a generic interface for Bayesian algorithms of a specific family (usually but not always: Markov Chain Monte Carlo type), thus are neither model nor domain agnostic. Also, by algorithm and interface design, they are usually not geared towards prediction, nor do they implement model validation and model comparison functionality\footnote{... or they implement ``broken'' model comparison functionality such as posterior predictive checks, see Section~\ref{sec:classprob-badidea}.} - however their design usually includes a high-level interface for fitting and sampling from some structured and hierarchical probability distribution models which contemporary supervised learning toolboxes are lacking.

In summary, there is to date no published modelling toolbox API design which is at the same time supervised/predictive, probabilistic, and domain-agnostic\footnote{With the exception of skpro, of course, which we consider linked to this publication.}, while there are some which are supervised and domain-agnostic but non-probabilistic, and some that are probabilistic but neither supervised nor domain-agnostic.

\subsection*{Acknowledgments}
We thank Andrew Duncan for pointing out the relevance of characteristic and universal kernels to the discussion of convolution (losses) in Section~\ref{sec:mixed.conv}.\\
We thank Ioannis Kosmidis for the main idea leading to Section~\ref{sec:mixed.split}.\\
We thank Yvo Pokern for very useful suggestions regarding notation (``law of'', conditioning, function and distribution spaces).\\
We thank Yiting Su (\begin{CJK*}{UTF8}{gbsn}苏怡婷\end{CJK*}) for reading the manuscript and pointing out mathematical typos, especially in Section~\ref{sec:adaptors.mixed}.\\
We thank Catalina Vallejos for pointing out to us the work of~\citet{geisser1979predictive}.\\

Parts of this manuscript contain probabilistic parallels to deterministic prediction specific material presented in the manuscript~\cite{burkart2017predictive} on which FK worked simultaneously. Parallels include choice of notation, proof strategies and central concepts such as the best uninformed prediction functional. Most parallel material was re-stated, since crucial concepts and phenomena in the deterministic case, e.g., elicitation and point prediction variance, are replaced by different phenomena in the probabilistic case, e.g., properness and information theoretic divergence. Some of these parallels intersect with a third, unpublished manuscript, which is the same that is referred to in the acknowledgments of~\cite{burkart2017predictive} as well.

FK acknowledges support by the Alan Turing Institute, under EPSRC grant EP/N510129/1.

\subsection*{Authors' Contributions}

FK conceived the project and is responsible for main approaches, ideas, proofs, manuscript structure, writing, and content, unless stated otherwise below.\\

BM contributed some of the main ideas in Section~\ref{sec:propres}, particularly the idea to use cdf evaluates as probabilistic residuals.\\

HO conceived many ideas and statements concerned with linking probabilistic prediction and loss evaluation with kernel discrepancy metrics, most notably in Section~\ref{sec:mixed.kernel} and the connection between predictive independence testing and kernel MMD based testing, in Section~\ref{Sec:hypotest}. HO pointed out a number of helpful areas of prior research, most notably the lecture notes of~\citet{singh2012infotheory}, the work of~\citet{reid2011information}, and the work of~\citet{lopez2016revisiting}. HO also generally contributed to the creative process through discussions in the mid-to-late stages of the project.\\

The suggested API for probabilistic learning and its implementation, as well as the respective experiments, are based on FG's MSc thesis, submitted August 2017 at University College London, and written under the supervision of FK. FK provided the central ideas for the high-level API design, FG contributed ideas for its practical realisation, including the representation of the predicted distributions (sec. \ref{sec:api-requirements}); the final API design was conceived in joint collaboration of FG and FK. The respective manuscript sections~\ref{sec:API} and~\ref{sec:experiments} are re-worked parts of FG's thesis manuscript, jointly done by FG and FK. The skpro python package has been designed and written by FG, with partial guidance by FK. Experiments were conducted by FG, with partial guidance by FK. FG made suggestions for the improvement of composite prediction strategies (bounding/capping and adaptors) based on the experimental results.\\

Final copy-editing and proof-reading was done by all authors.

\newpage
\section{The probabilistic prediction setting}
\label{sec:propreset}

This section introduces the mathematical setting for probabilistic supervised prediction. As mentioned above in Section~\ref{sec:priorart},
some of the central ideas and concepts introduced may be found in prior art.

\subsection{Notation and Conventions}
The natural logarithm (with basis $\mbox{e}=\exp (1)$ ) will be denoted $\log$. The non-negative reals will be denoted $\RR^+$, the natural numbers $\NN$, both sets contain $0$ by convention. \\
The cardinality of a set $S$ is denoted by $\card S$. Unless stated otherwise, all sets are assumed measurable, and endowed with a canonical measure, compatible with sub-setting\footnote{Intuitively, this is meant to imply that ``all integrals and expectations in the manuscript are well-defined''. We acknowledge that care needs to be taken for making statements theoretically precise in cases other than $\RR^n$ endowed with the Lebesgue measure and derived discrete product spaces, but we opt to avoid introduction of a full measure-theoretic framework for the purpose of readability and conceptual clarity.}.\\

For sets $\calX,\calY$, we will denote the set of functions mapping from $\calX$ to $\calY$ by $[\calX\rightarrow \calY]$. We will impose no further conditions of regularity on elements of $[\calX\rightarrow \calY]$ unless stated explicitly in addition.\\

For any measurable space $\calY$, we define $\Distr (\calY)$ to be the space of distributions on $\calY$ in the following sense: $\Distr (\calY)$ is the convex subset of probability measures, within the space of bounded, signed measures on $\calY$, which is Banach space under the total variation norm.\\
Under the further assumption on a $p\in \Distr (\calY)$ that is discrete or absolutely continuous, we will (by abuse of notation) identify $p$ with the corresponding function in $[\calY\rightarrow \RR^+]$, which is a probability mass function or probability density function in such a case.\\

For a random variable $X:\Omega \rightarrow \calX$, we say that $X$ takes values in $\calX$. Measurability of $X$ (and composites) will always be implicitly assumed. Instead of saying $``X\in [\Omega\rightarrow \calX]'',$ we will say ``$X$ takes values in $\calX$'', or abbreviatingly, ``$X$ t.v.i.~$\calX$''. This will be done without specifying the probability space $\Omega$, which will implicitly be assumed to exist and to be fixed (as usual, to the finest such space for all appearing variables).\\

For a random variable $X$ t.v.i.$\calX$, we will denote by $\calL(X)$ the canonically induced\footnote{For $X:\Omega \rightarrow \calX$, this is canonical in the measure theoretic sense, i.e., $\calL(X)$ canonically induced as the push-forward of the $\Omega$-measure to $\calX$.} distribution of $X$ in $\Distr(\calX)$, the ``law of $X$''. In the frequent statistical setting where $X$ is discrete or absolutely continuous, $\calL (X)$ may be identified with the probability mass function or probability density function of $X$, an element of $[\calX\rightarrow \RR^+],$ with which we identify $\calL (X)$ in such a case.\\

Considering the above, it is worth keeping in mind that ``probabilistic supervised learning'' will be concerned with elements of $[\calX \rightarrow \Distr (\calY)]$, i.e. functions mapping elements of $\calX$ to distributions over $\calY$, that will be interpreted as probabilistic prediction functionals. By the identification of $\Distr (\calY)$ with elements of $[\calY\rightarrow \RR^+]$ in the discrete or absolutely continuous case, such a prediction functional may be identified with an element $f$ of the set $[\calX \rightarrow [\calY \rightarrow \RR^+]],$ which is a set of (prediction) functions having (distribution) functions as an output. In particular, for $x\in\calX$, the output $f(x)$ of such a function is a function itself, an element of $[\calY \rightarrow \RR^+]$, hence for an additional $y\in\calY$, we may consider the output of $f(x)$ when given $y$ as input, an element or $\RR^+$ which we will denote by $[f(x)](y)$ or $f(x)(y)$, a notation consistent with the usual notation in typed lambda calculus or any other standard theories of types (such as Church's), or the iterated use of brackets in programming languages such as python or R.\\

On a somewhat less technical and more {\bf intuitive note}, we advise the reader familiar with supervised learning to keep in mind that the label predicted by a probabilistic supervised learning method is a distribution, hence (if a pmf or pdf exists) also a function, therefore a (trained) probabilistic supervised learning method is a function that predicts a function as a label instead of a single label. The slightly unusual notation $f(x)(y)$ will occur in situations where we evaluate the prediction rule $f$ at a test feature point $x$ which yields a distribution $f(x)$, and then this resulting distribution at an element of the label range $y$ which yields a probability or a density $f(x)(y)$; the occurrence of such terms and objects is a consequence of this practically relevant situation by which the slight mathematical overhead above is entailed.

\subsection{Statistical setting}
\label{sec:propreset.setting}
We will consider supervised learning with its standard data generating process: there is a generative random variable $(X,Y)$ t.v.i.~$\calX\times \calY$. For $z=(x,y)\in \calX\times \calY$, the projection $x$ onto $\calX$ is called features or independent variables (or predictors, a term we will not use to avoid confusion with prediction functionals below); the projection $y$ onto $\calY$ is called label or dependent variable or target variable. $z$ itself is called a feature-label pair.

Usually, but not necessarily, one has $\calX \subseteq \RR^n$ and $\calY\subseteq \RR$; for a binary classification task, one considers $\calY = \{-1,1\}$, and for a regression task, one considers $\calY = \RR$.

All data are assumed to be i.i.d.\footnote{by the i.i.d.~assumption, this setting hence does not include time series forecasting, spatial prediction, transfer learning, etc. In particular, the main CAVEAT of the i.i.d.~setting applies, which is limited generalizability outside the probabilistic sampling remit of $X$. In the probabilistic case this in particular also includes the uncertainty prediction itself. That is, outside the sampling remit of $X$, it is not only point predictions that may be wildly inaccurate, but also any estimate of uncertainty coming from a probabilistic method.},
 following the generative sampling distribution of $(X,Y)$. Note that $X$ and $Y$ are \emph{not} assumed to be independent of each other (otherwise labels would be independent of hence completely unrelated to the features).

Unless stated otherwise, it will be assumed that the conditional distribution $Y|X$ is \emph{either} discrete or absolutely continuous, that, always the same of the two, for all values which $X$ may take.
We will discuss the case of mixed distributions separately in designated sub-sections, as it will require additional care.

\subsection{Classical supervised learning}
\label{sec:classical-learning}

In classical supervised learning, the estimation task is to ``learn'' a prediction functional of the type $\varpi:\calX\rightarrow \calY$. This prediction functional is estimated (``learnt'') on training data $\calD = \{ (X_1,Y_1),\dots, (X_N,Y_N)\}$ where the $(X_i,Y_i)\sim (X,Y)$ are i.i.d.; the estimating process is modelled as a function-valued random variable $f$, t.v.i.~$[\calX\rightarrow \calY]$, which may depend on $\calD$ but is assumed to be independent of anything else.

The estimator $f$ is called a \emph{learning strategy}, in contrast to a \emph{prediction functional} such as $\varpi$. The dependence of $f$ on $\calD$ needs not be deterministic, but is in generally interpreted to be algorithmically given in practical instances.

A good learning strategy $f$ is supposed to yield a prediction functional which, on average (over $\calD$) is a good approximation of a putative relation of the type $Y\approx \varpi(X)$. Quantitatively, the expected prediction error should be small, on new data $(X^*,Y^*)\sim (X,Y)$ sampled from the generative distribution (independent of anything else), as quantified by a loss function $L:\calY \times \calY \rightarrow \RR^+$, comparing predictions and true values. The best prediction functional minimizes the expected generalization loss,
$$\varepsilon_L(f)= \mathbb{E}\left[L(f(X^*),Y^*)\right],$$
where expectations are taken over $f,\calD,$ and $(X^*,Y^*)$. The possibly most popular and most frequently used loss in this classical setting is the squared loss, $L_{sq}(\widehat{y},y_*) = \left(\widehat{y}-y_*\right)^2$.
More generally, a convex loss function is chosen, this means: $y\mapsto L(y,y_*)$ is convex, and the value of $L(\widehat{y},y_*)$ is minimized (usually to zero and possibly with $y_*$ considered fixed) for constant arguments $\widehat{y},y_*$ if and only if $\widehat{y} = y_*$, i.e., whenever the prediction is perfect.

(note: the setting above does not include, for the moment, losses which are not mean losses)

\subsection{Probabilistic supervised learning}
\label{sec:propreset.probsupl}
Probabilistic supervised learning accounts for the fact that even the most perfect supervised prediction strategy is unable to predict well if the conditional random variable $Y|X$ has a complicated (e.g., bimodal) or merely non-standard (e.g., non-normal) distribution. In such a case (which is many real world applications the empirical standard case), it is preferable to predict the conditional random variable $Y|X$ itself as opposed to a point estimate.

Hence we consider the estimation task of learning a prediction functional $\varpi: \calX \rightarrow \Distr (\calY)$, where $\Distr (\calY)$ is the set of distributions supported on the set of possible outcomes $\calY$. In the frequently considered probabilistic binary classification setting where $\calY = \{-1,1\}$, there is a one-to-one-correspondence $\Distr(\calY)\cong [0,1]$, where the right hand side number in $[0,1]$ is the predicted probability of the outcome $1$, or one minus the probability of the outcome $-1$. In a regression setting, one has $\Distr(\calY) = \Distr (\RR),$ i.e., we are ``predicting'' a distribution on the real numbers.
In either cases, it is important to note that the output of the prediction functional $\varpi$ is \emph{not} a random variable, but an explicit distribution.

As in the classical setting, a prediction functional $\varpi$ such as above will be estimated/learnt from training data $\calD = \{ (X_1,Y_1),\dots, (X_N,Y_N)\}$ where the $(X_i,Y_i)\sim (X,Y)$ are i.i.d.; thus, the estimation process may be modelled by a function-valued random variable $f$, t.v.i.~ $[\calX\rightarrow \Distr(\calY)]$, which may depend on $\calD$ and will be assumed to be independent of anything else.

As in the classical setting, we distinguish predictions functionals such as $\varpi$ from prediction strategies such as $f$, whose dependence on $\calD$ may be mediated through an abstract algorithm, and which formally are not functions but in fact function valued random variables.

Intuitively, the best prediction functional that can be possibly obtained from a probabilistic supervised learning strategy is the ``oracle'' $\varpi_{Y|X}: x \mapsto \calL(Y|X=x)$, i.e., the functional which for $x$ predicts the distribution of the random variable $Y$ conditioned on $X$ being $x$ (read $\calL(Y|X=x)$ as ``the law of $Y|X=x$'')

In order to quantify what makes a prediction functional that is distinct from $\varpi_{Y|X}$ good in approximating it, we will consider loss functions $L: \Distr (\calY)\times \calY\rightarrow \RR$ with the following desired ``minimality property'': the best prediction functional $\varpi_{Y|X}$ should minimize the expected generalization loss, i.e.,
$$\varepsilon_L(f) := \EE\left[L(f(X),Y)\right],$$
where the total expectation over $(X,Y,f)$ is taken. Note that while expectations are taken over $f$ as well, they should not and cannot be taken over the output of $f$ interpreted as a random variable, since the output is not a random quantity, but an explicit representation of the predictive distribution, which is crucial in assessing the prediction against ground truth observations.

A-priori, it is not entirely clear whether a loss function with the minimality property even exists (it does, as shown in Section~\ref{sec:probloss}), or why expectations should not be taken over the output random variables (prevents existence of such a loss function, see Section~\ref{sec:classprob-badidea}), these will be the first results in the subsequent sections which show that the so-defined setting of probabilistic supervised learning is (intuitively, mathematically, and empirically) reasonable.

\subsection{Probabilistic loss functionals}
\label{sec:probloss}

In this section, we introduce losses for the probabilistic supervised learning setting presented above. We formalize desirable properties such as properness, which states that the perfect prediction minimizes the expected loss. We highlight a selection that are proper in this sense. Definitions and exposition partly follow concepts and terminology in~\citet{gneiting2007strictly}.

\begin{Def}
A probabilistic loss functional is a function $L:\calP\times \calY\rightarrow \RR\cup\{\infty\}$,
where $\calP\subseteq \Distr (\calY)$ is a convex set of probability distributions\footnote{The requirement of $\calP$ being convex is equivalent to expectations $\EE(P)$ for random variables $P$ t.v.i.~$\calP$ being well-defined, or, phrased differently, arbitrary mixtures of measures contained in the set $\calP$ are still contained in $\calP$. From a practical viewpoint, it also implies that all expectations over elements of $\Distr(\calY)$ that are considered later in the manuscript exist and are well-defined.\\
Well-definedness of the convexity stipulation on $\calP$ is implied since $\Distr(\calY)$ is also a convex space. That is, for any Bochner measurable random variable $P$ t.v.i.~$\calP\subseteq\Distr(\calY),$ the expectation $\EE[P]$ is well-defined as an element of $\Distr (\calY)$, as a Bochner integral. This is because $\EE[|P|_{TV}]<\infty$, which a short calculation using the closedness of $\Distr(\calY)$ shows.\\
Note that a $P$ is Bochner-measurable iff it is Borel-measurable and tight - this is fulfilled in essentially all practically relevant cases. However, if needed, one could use Gelfand-Pettis or Dunford integration instead.}
The interpretation will be that $L$ compares a prediction, the first argument, to the observation, the second argument. Through the type of $L$, it is implicitly assumed that the prediction takes the form of a distribution in $\calP$, which encodes a predictive claim about how the second argument, the observation in $\calY$, is sampled.\\
In this manuscript, we will usually consider $\calP$ to be one of the following convex spaces:
\begin{enumerate}
\itemsep-0.2em
\item[(i)] the set of absolutely continuous distributions on $\calY$
\item[(ii)] the set of discrete distributions on $\calY$
\item[(iii)] the set of mixed distributions on $\calY$
\item[(iv)] the set of all distributions on $\calY$
\item[(v)] one of the above with additionally specified integrability or regularity conditions on elements of $\calP$ (e.g., squared-integrability)
\end{enumerate}
It will always be assumed that $\calL(Y|X=x)\in\calP$ for any $x\in\calX$, i.e., $L$ is able to assess the best possible prediction.\\
Depending on the type of distributions contained in $\calP$, which by construction implies an assumption on what form both the predictions and the ``true'' conditionals may take, we will respectively call $L$ an absolutely continuous probabilistic loss, discrete probabilistic loss, mixed probabilistic loss, and so on in analogy.
\end{Def}

\begin{Def}\label{Def:proper}
A probabilistic loss functional $L:\calP\times \calY\rightarrow \RR\cup\{\infty\}$ is called:
\begin{enumerate}
\itemsep-0.2em
\item[(i.a)] \emph{convex} if $L(\EE[P],y)\le \EE [L(P,y)]$ for any random variable $P$ t.v.i.~ $\calP$
\item[(i.b)] \emph{strictly convex} if $L(\EE[P],y)\le \EE [L(P,y)]$ for any random variable $P$ t.v.i.~$\calP$, and if equality holds if and only if $P$ is a constant random variable
\item[(ii.a)] \emph{proper} if $\EE[L(\calL(Y),Y)]\le \EE[L(p,Y)]$ for any random variable $Y$ t.v.i.~$\calY$ where $\calL(Y)\subseteq \calP$.
\item[(ii.b)] \emph{strictly proper} if $\EE[L(\calL(Y),Y)]\le \EE[L(p,Y)]$ for any random variable $Y$ t.v.i.~$\calY$ where $\calL(Y)\subseteq \calP$, and if equality holds, i.e., $\EE[L(\calL(Y),Y)] = \EE[L(p,Y)]$, if and only if $\calL(Y)$ and $p$ are equal.
\item[(iii.a)] \emph{local} if $L$ is a loss for continuous or discrete distributions, and $L(p,y) = h(y,p(y))$ for some function $h:\calY\times \RR\rightarrow \RR\cup\{\infty\}$, where $p(y)$ denotes the evaluation of the density or probability mass function (corresponding to) $p$ at the value $y$ (notational identification of $\Distr(\calY)$ with $[\calY\rightarrow \RR^+]$).
\item[(iii.b)] \emph{strictly local} if $L$ is a loss for continuous or discrete distributions, and $L(p,y) = g(p(y))$ for some function $g:\RR\rightarrow \RR\cup\{\infty\}$, where $p(y)$ denotes the evaluation of the density or probability mass function (corresponding to) $p$ at the value $y$ (notational identification of $\Distr(\calY)$ with $[\calY\rightarrow \RR^+]$).
\item[(iv.a)] \emph{global} if $L$ is not strictly local.
\item[(iv.b)] \emph{strictly global} if $L$ is not local.
\end{enumerate}
By the usual convention, strictly [convex/proper/local/global] losses are also [convex/proper/local/global]. Hence, in particular, strictly local losses are never strictly global.\\
We also note that convexity and properness may depend on the chosen domain $\calP$, restricting $\calP$ may make a non-[convex/proper] loss [convex/proper] and a [convex/proper] loss strictly [convex/proper].
\end{Def}

Intuitively, convexity implies that a certain prediction is better than one that is random.
Properness formalizes the desired minimality property discussed in the previous Section~\ref{sec:propreset.probsupl}, namely that the perfect prediction minimizes the loss in expectation.
Locality and globality describe whether predicted probabilities other than for the observed event are important. Strict locality is algorithmically attractive as $L$ can be easily computed for predictions obtained from a prediction functional $f:\calX\rightarrow \Distr(\calY)$ if the density which is the output of $f$ is easy to evaluate at an arbitrary value, and theoretically attractive since entirely dependent on density values and not on a particular choice of $\calY$. However, globality might be theoretically attractive as per losses which consider predicted probabilities of close-by values, and will be important for mixed distributions.

\begin{Def}\label{Def:losses}
We define two probabilistic loss functions for both the continuous and discrete case which will be of particular interest to us:
\begin{enumerate}
\itemsep-0.2em
\item[(i)] the log-likelihood loss (short: log-loss, sometimes also called: cross-entropy loss)
$$L_\ell:(p,y)\mapsto -\log p(y).$$
\item[(ii)] the integrated squared loss (also called, in the continuous case: Gneiting loss, in the discrete case: multi-class Brier loss)
$$L_{Gn}:(p,y)\mapsto - 2 p(y) + \|p\|_2^2.$$
\end{enumerate}
As in Definition~\ref{Def:proper}, $p(y)$ denotes the evaluation of the density or probability mass function (corresponding to) $p$ at the value $y$.
In the discrete case, $\|p\|_2^2 := \sum_{y\in \calY} p(y)^2$. In the continuous case, $\|p\|^2_2 := \int_\calY p(y)^2 \diff y.$
\end{Def}

It is not immediately straightforward what the losses in Definition~\ref{Def:losses} measure, and, more specifically, the loss of what exactly. A justification for the nomenclature is given by the following Lemma~\ref{Lem:losses}:

\begin{Lem}\label{Lem:losses}
Consider the log-loss and the integrated squared loss in Definition~\ref{Def:losses}. Let $Y$ be an absolutely continuous or discrete random variable t.v.in $\calY\subseteq \RR^n$ with pdf/pmf $p_Y$, let $p\in\Distr (\calY)$. Then it holds that:
\begin{enumerate}
\itemsep-0.2em
\item[(i)]$\EE[L_\ell(p,Y)] - \EE[L_\ell(p_Y,Y)] = \int_\calY p_Y(y) \log \frac{p_Y(y)}{p(y)}\;\diff y$
\item[(ii)] $\EE[L_{Gn}(p,Y)] - \EE[L_{Gn}(p_Y,Y)] = \int_\calY \left(p_Y(y)- p(y)\right)^2\;\diff y$
\end{enumerate}
where the $\int$ are w.r.t.~the Lebesgue measure on $\RR^n$ in the absolutely continuous case and w.r.t. the counting measure in the discrete case (that is, sums).
\end{Lem}
\begin{proof}
The statements follow from straightforward calculations after substituting definitions.
\end{proof}

Intuitively, Lemma\ref{Lem:losses} explains what the log-loss and the integrated squared loss measures and explains some of the nomenclature: with the notation of Lemma\ref{Lem:losses} if $\widehat{Y}$ is a random variable with pdf/pmf $p$, then the RHS expression in (i) is the cross-entropy of $\widehat{Y}$ w.r.t.$Y$, or the (information theoretic) Kullback-Leibler-distance of $p$ w.r.t.$p_Y$, which is known to be minimized iff $p=p_Y$ (e.g., see Proposition~\ref{Prop:losses}~(i) for a proof). The RHS expression in (ii) is the integrated squared distance between $p$ and $p_Y$, which is again minimized iff $p=p_Y$ (this is almost immediate, also see Proposition~\ref{Prop:losses}~(ii) for a proof).
The losses being, in expectation, well-known probabilistic distance measures implies that indeed the perfect prediction is the best one in terms of expected loss, as formalized through the properness property.

Further probabilistic losses, together with their properties, may be found in~\cite{gneiting2007strictly, dawid2007geometry}, including the Bregman type losses of which both log-loss and squared integrated loss are special cases, see~\cite[Section~5.3]{dawid2007geometry}. Many of the below results generalize to these, though in order to keep exposition simple, and since the algorithmic usefulness of the more general case is not entirely clear, we restrict discussion below to the two losses in Definition~\ref{Def:losses} above.

We continue with stating central mathematical properties of the log-loss and squared integrated loss, most importantly them being strict proper losses:

\begin{Prop}\label{Prop:losses}
\begin{enumerate}
\itemsep-0.2em
\item[(i)] the log-likelihood loss is strictly convex, strictly proper, and strictly local.

\item[(ii)] the integrated square loss is strictly convex, strictly proper, and global.
\end{enumerate}
\end{Prop}
\begin{proof}

(i) Locality follows by definition.
For convexity, consider a random variable $P$ which takes values in the domain $\calP\subseteq \Distr(\calY)$. Choose $y\in \calY$ and define $Z := P(y)$. Note that $Z$ is a real-valued random variable, and $\EE[P](y) = \EE[P(y)] = \EE [Z]$. Thus, by Jensen's inequality, $\log(\EE[Z])\le \EE[\log(Z)]$. Putting together all (in)equalities and definition of $L$, one obtains
$$L(\EE[P],y) = \log(\EE[Z]) \le \EE[\log(Z)] = \EE[L(P,Y)],$$
which proves convexity of $L$. Note that Jensen's theorem also states that $\log(\EE[Z])=\EE[\log(Z)]$ if and only if $Z$ is a constant random variable, which in the above implies strict convexity of $L$.
For strict properness, consider a random variable $Y$ t.v.i.~$\calY$ with probability mass or distribution function $p_Y:\calY\rightarrow \RR^+$. Gibbs' inequality states that
$$\EE[\log(p_Y(Y))]\le \EE[\log(p(Y))]$$
with equality if and only if the measures induced by $p_Y$ and $p$ agree.\\

(ii) Let $Y$ be an $\calY$-valued random variable with probability mass/density function $p_Y$.
For strict properness, an explicit computation (similar to the one in Lemma~\ref{Lem:losses}) shows that
$$\EE[L_{Gn}(p,Y)] = \int_\calY (p_Y(y)-p(y))^2 \diff y + \|p_Y\|^2_2,$$
which assumes its minimal value $\|p_Y\|^2_2$ if and only if $p_Y=p$ Lebesgue/discrete-almost everywhere, which is the same as $p_Y$-almost everywhere as $Y$ is assumed absolutely continuous or discrete.
\end{proof}

The logarithmic loss is furthermore the only local loss which is also strictly proper:

\begin{Prop}\label{Prop:logcanon}
\begin{enumerate}
\itemsep-0.2em
The log-likelihood loss is the only loss function, which is both strictly local and strictly proper for arbitrary choices of the prediction range $\calY$ and the predictive distribution space $\calP\subseteq \Distr(\calY)$.
\end{enumerate}
\end{Prop}
\begin{proof}
This is Theorem~2 in \cite{bernardo1979expected}. Note that the smoothness assumption is not necessary in its proof as the differential equation may also be phrased as the integral equation $g(x) = \int_{-\infty}^x \frac{1}{g(\tau)}\diff \tau$ (where $g:\RR\rightarrow \RR$ is the function such that $L(p,y) = g(p(y))$ completely determining any local loss). Following through the proof, note that the integral equation only needs to hold for $x$ in the range of possible pdf of $Y|X$, though this can be taken to be all of $\RR^+$ since locality and strict properness for any choice of $\calY$ and $\calP$ is required, hence the range of $Y|X$ may be chosen to be arbitrary intervals $[0,\alpha]$ whose union is $\RR^+$.
\end{proof}

The losses which are local, but not strictly local, yet strictly proper may also be characterized and give rise to certain Bregman losses~\cite{dawid2007geometry}.

One important yet subtle theoretical point is that properness of a loss is defined in terms of unconditional distributions. Hence this is not immediately equivalent to the minimality property as discussed in the previous Section~\ref{sec:propreset.probsupl} which is conditional on features from which the predictions are made. While this conditional variant of properness is more or less straightforwardly implied by the unconditional one, we state it explicitly as it establishes the relation to the task of probabilistic supervised learning:

\begin{Prop}\label{Prop:probmin}
Let $L:\calP\times \calY\rightarrow \RR \cup\{\infty\}$ with $\calP\subseteq \Distr (\calY)$ be a proper loss. Then, for any fixed $x\in\calX$, it holds that
$$\calL(Y|X=x) \in \argmin_p \EE \left[ L(p,Y) |X = x \right],$$
where the expectation is taken conditional on the event $X=x$, or equivalently w.r.t.~the conditional random variable $Y|X=x$. The minimizer $\calL(Y|X=x)$ is unique if $L$ is strictly proper.
\end{Prop}
\begin{proof}
This is easily obtained from conditioning.
Write $p_Y := \calL(Y|X=x)$. Properness of $L$ implies that
$$\EE[L(p_Y,Y|X=x)] \le \EE[L(p,Y|X=x)]$$
for any $p\in\calP$. Since $L$ is deterministic (i.e., not a random quantity), the conditioning may be pulled out yielding
$$\EE[L(p_Y,Y)|X=x] \le \EE[L(p,Y)|X=x].$$
The statement for strictly proper $L$ follows from following through the inequality condition.
\end{proof}

We will further expand upon implications of this seemingly simple observation of Proposition~\ref{Prop:probmin}, namely that proper losses also have the desired minimality property in the minimal setting, in Section~\ref{Sec:uninfbase}

We would also like to note that neither log-loss nor squared integrated loss, in general, retain from the classical set-up the property of being non-negative\footnote{The ``argument'' for non-negativity that $-\log p$ is always positive since $p\le 1$ is in general wrong, as densities can assume values larger than $1$ (a fact which beginners of statistics are sometimes unaware of). However, this argument is correct for discrete distributions where $p\le 1$ is in fact true, and following the discussion in Section~\ref{sec:mixed}, it can be shown to be correct in the general case in a certain sense.}.
However, as the crucial step in model selection and assessment is (additive) comparison, we will see that this is not an important property to lose, especially since non-negativity holds for \emph{loss differences}, as implied by the conditional version of Lemma~\ref{Lem:losses}.

Furthermore, losses that may not be split per test point, i.e., empirical losses that are not mean losses, may be considered, such as
out-of-sample log-likelihoods which are dependent (across test data points). We will not consider those in this manuscript since the right assumptions under which these are proper estimates for some generalization loss are unclear and plausibly technical\footnote{It is not immediately apparent how to take asymptotics over $N$ in a prediction strategy
$f:\calX^N\rightarrow \Distr(\calY^N)$ that predicts non-independent probabilities, note that $\Distr (\calY^N) \neq \Distr (\calY)^N$. If $N$ is considered a parameter of $f$, without further assumption $f$ could be anything and in particular without sensible loss asymptotic in $N$. On the other hand, it is not
possible, as in the independent setting, to evaluate a fixed $f$ at any number of test data points, only at the fixed number of $N$ test points. There are multiple additional assumptions by which either may be remedied, but neither we were able to see seem canonical or promising in adding to the framework. Though this may be an interesting future research avenue.}.
Furthermore, the minimizer of the expected generalization loss, in any plausibly reasonable formalization of that case, would be dominated in log-loss by a prediction which is independent across test data points, as Proposition~\ref{Prop:marginal} below shows. We defer its proof until a point where the technical prerequisites are in place.

On a historical note: the use of the log-loss is classical, an early (possibly the first) occurrence for inference purposes may be found in the work of~\citet{good1952rational}, while the squared integrated loss probably occurs first in~\cite{brier1950verification} for the case of a discrete label domain; for the above concept's idea history (in the non-supervised predictive setting) see~\cite{gneiting2007strictly}, for further related prior art see Section~\ref{sec:priorart}.

\subsection{Recovering the classical setting}
\label{sec:classprob}
The classical supervised setting with proper mean loss may be recovered as follows: note that in the classical setting, i.e., where $\calY\subseteq \RR^n$, one can build a probabilistic prediction functional $f:\calX\rightarrow \Distr(\calY)$ from a probability density/mass function $p:\calY\rightarrow \RR^+$ and a classical, deterministic prediction functional $g:\calX\rightarrow \calY$, namely by setting $[f(x)](t) = p(t-g(x))$. The functional $g$ prescribes the location of the prediction distribution whose shape is fixed by $p$.\\
While this choice seems somewhat arbitrary, there is an intimate correspondence between such predictors as measured by the log-loss, and combinations of point predictors $g$ and a mean loss $L$, where the choice of mean loss $L$ will directly translate into a choice of distribution shape $p$.\\
We begin with the perhaps most simple and illustrative instance to show the gist of this correspondence:

\begin{Ex}\label{Ex:classic}
Let $g:\calX\rightarrow \calY$ be a (classical/point) prediction functional. Let $\epsilon$ be a Gaussian random variable with zero mean and variance $\sigma^2$.\\
Define $f: \calX\rightarrow \Distr (\calY)\;;\, x\mapsto \calL (g(x) + \epsilon)$.
Then, a short explicit computation shows that
$$L_{\ell}(f(x),y) = \frac{1}{2\sigma^{2}}\cdot L_{sq}(g(x),y) - \frac{1}{2}\log(2\pi \sigma^2).$$
In particular, for this choice of $\epsilon$ the log-loss is the mean squared loss up to re-scaling and shift by a constant, hence both are equivalent for the purpose of assessing and comparing classical prediction functionals.
\end{Ex}

More generally, supervised learning with any (classically) proper mean loss can be understood as predicting a fixed family of distributions where only the location is allowed to vary:

\begin{Prop}\label{Prop:classic}
Let $g:\calX\rightarrow \calY$ be a (point) prediction functional, let $L:\calY \times \calY \rightarrow \RR^+$ be a point prediction loss (i.e., for the classical setting), such that the probability density
$$p(x) \sim \exp\left(-L(x,0)\right)\quad\mbox{, i.e., } p(x) =\frac{1}{C}\cdot \exp\left(-L(x,0)\right)\;\mbox{for suitable normalizing constant}\; C,$$
is well-defined. Let $\epsilon$ be a random variable with density $p$.
Define $f: \calX\rightarrow \Distr (\calY)\;;\, x\mapsto \calL (g(x) + \epsilon)$. Then,
$$L_{\ell}(f(x),y) = L (g(x),y) - C\quad\mbox{where}\; C = \log\int_\calY \exp\left(-L(t,0)\right)\;\diff t$$
\end{Prop}
\begin{proof}
This follows from substitution of definitions and a short elementary computation.
\end{proof}

Important example instances of the correspondence in Proposition~\ref{Prop:classic} are:
the choice of squared loss $L:(\widehat{y},y)\mapsto (\widehat{y}-y)^2$ corresponds to the choice of a Gaussian density $p$.
The choice of absolute loss $L:(\widehat{y},y)\mapsto |\widehat{y}-y|$ corresponds to the choice of a Laplacian density $p$.

Note that the computation in the example and the proof of Proposition~\ref{Prop:classic} are exactly the same when interpreting $\epsilon$ as an additive and probabilistic error term in a (frequentist) generative model, say with a Gaussian or Laplacian distribution. However, in the probabilistic setting, we do not interpret $\epsilon$ as an error, but instead as part of the prediction - since for one, the exact structure of $\epsilon$ is a (possibly unvalidated) assumption, and in fact, a smarter method could be expected to not only predict the mean but also output a predictive variance that may depend on the features (= the argument of $f$), a case which is not covered or even modelled in the classical setting.

More precisely, we do not assume any (frequentist) generative model beyond the observable, or any class of (Bayesian) belief model - since neither are directly observable hence, neither are empirically validable. Instead, we consider the in-practice observable (the data) as a ``ground truth'' in the empirical sense, i.e., a point of comparison for our prediction(s).
In fact, the above discussion shows that the very common frequentist assumption of the homoscedastic noise term in the generative model may easily (mis-)lead one to disregard any meaningful probabilistic prediction (while of course it is also not entirely wrong since it leads to useful special cases such as least squares regression).

As a further remark of interest, Proposition~\ref{Prop:classic} puts all classical mean losses, including the aforementioned mean squared error and the mean absolute errors, on the same scale by making them comparable through their identification with probabilistic supervised models, hence objects of the same kind that can be data-dependently compared on a common scale. This common scale is furthermore canonical since the log-loss is canonical, i.e., the only proper local loss by Proposition~\ref{Prop:logcanon}.

Interestingly, Example~\ref{Ex:classic} may already been found in lecture notes of~\citet[Section~13.4]{singh2012infotheory}, though this appears to be the only substantial part of the lecture notes making explicit reference to the supervised setting and we were unable to find any continuation of the discussion in the literature.

We would also briefly like to stress why considering prediction functionals which predict delta distributions supported at a point prediction is \emph{not} a valid way to recover the classical setting. While there is a 1:1 identification with classical prediction strategies, prediction of a delta distribution is incompatible with the more natural choices (i.e., the ones in Definition~\ref{Def:losses}) of probabilistic losses above. Thus considering delta predictions induces no correspondence between the whole setting, i.e., combinations of prediction strategies and losses on the classical and probabilistic side, only between different prediction strategies while breaking the connection to the loss functional.\\
Conversely, note how Proposition~\ref{Prop:classic} yields exactly the same loss, identifying a pair of classical loss and classical prediction functional with a pair of a probabilistic loss (the log-loss) and probabilistic prediction functional, where the identification on the level of the evaluated loss is algebraically exact.

\subsection{Probabilistic losses for mean-variance predictions}

The correspondence highlighted in Section~\ref{sec:classprob}, that is, between classical point predictions with a choice of loss, and probabilistic predictions with specific parametric form, may be further extended to a correspondence of certain probabilistic predictions to mean-variance predictions. Here, mean-variance predictions refer to predictions of only mean and variance rather than of a full distribution. We carry the correspondence out for the real univariate case.

\begin{Def}
We define the mean-variance likelihood-loss as
$$L_{mv}:(\RR\times \RR^+)\times \RR \rightarrow \RR\;,\; ((\mu,\nu),y)\mapsto \nu^{-1}\cdot(\mu-y)^2 +\log \nu$$
\end{Def}

As usual, the first argument of $L_{mv}$ is interpreted as a prediction, namely of a pair $(\mu,\nu)$ of a mean $\mu$ and a variance $\nu$. The observation is the second argument $y$, in the case of a perfect prediction it is sampled from a distribution with the predicted mean and variance, without further specifying the distributional form.

For illustrative purposes, we first prove a proper-like property for the mean-variance likelihood-loss:

\begin{Prop}\label{Prop:meanvar}
Let $Y$ be a random variable with finite expectation and variance. Then, 
$$(\EE[Y],\Var [Y])= \argmin_{(\mu,\nu)} \EE \left[ L_{mv}((\mu,\nu),Y)\right].$$
That is, the perfect prediction is the unique minimizer of the expected mean-variance likelihood-loss.
\end{Prop}
\begin{proof}
An elementary computation shows that
\begin{align*}
\EE \left[ L_{mv}((\mu,\nu),Y)\right] &= \nu^{-1}\cdot\EE[(\mu-Y)^2] +\log \nu\\
&= \nu^{-1} \Var[Y] + \nu^{-1}(\mu - \EE[Y])^2 + \log\nu. 
\end{align*}
For fixed $\nu$, the right hand side is minimized by $\mu = \EE[Y]$, i.e., $\EE \left[ L_{mv}((\EE[Y],\nu),Y)\right]\le \EE \left[ L_{mv}((\mu,\nu),Y)\right]$ for arbitrary $\nu$, and the minimizer is unique. The LHS is equal to $\nu^{-1} \Var[Y] + \log\nu =: g(\nu).$ Consider now the function $h:(0,\infty)\rightarrow \RR,\; z \mapsto z \Var[Y] -\log z$, and note that $g(\nu)=h(\nu^{-1})$. The function $h$ is smooth, with derivatives $h'(z) = \Var[Y] - z^{-1}$ and $h''(z) = z^{-2}$. As $h''$ is always positive on the domain, the unique critical point of $h$, which is the zero of $h'$, i.e., $z^{-1}=\Var[Y]$, is the unique minimum of $h$. Hence $\nu = \Var[Y]$ is the unique minimum of $\EE \left[ L((\EE[Y],\nu),Y)\right]$, i.e., $\EE \left[ L_{mv}((\EE[Y],\Var[Y]),Y)\right] \le \EE \left[ L_{mv}((\EE[Y],\nu),Y)\right]$ for any $\nu$, and the minimizer is unique.\\
Putting the two inequalities together proves the claim.
\end{proof}

The connection to probabilistic supervised learning may be established as below:

\begin{Rem}
Let $g:\calX\rightarrow (\RR,\RR+)$ be a mean-variance prediction functional. Let $\epsilon$ be a standard normal random variable, i.e., a Gaussian random variable with zero mean and variance one.\\
Define $f: \calX\rightarrow \Distr (\RR)\;;\, x\mapsto \calL (g(x)_1 + g(x)_2\cdot\epsilon)$.
Then, a short explicit computation shows that
$$L_{\ell}(f(x),y) = \frac{1}{2}\cdot L_{mv}(g(x),y) - \frac{1}{2}\log(2\pi).$$
In particular, for this choice of predictive distribution the log-loss is the mean-variance likelihood-loss up to re-scaling and shift by a constant, hence both are equivalent for the purpose of assessing and comparing classical prediction functionals.
\end{Rem}

In this context, Proposition~\ref{Prop:meanvar} may be seen as a loss analogue of the well-known statement on sample mean and variance being maximum likelihood estimators under Gaussianity assumptions. However, we would like to make two additional observations:

\begin{enumerate}
\itemsep-0.2em
\item[(i)] Proposition~\ref{Prop:meanvar} holds \emph{without any} assumptions on the distributional form of the true generative distribution (beyond finite expectation/variance). In particular, nowhere does Gaussianity have to be assumed.
\item[(ii)] In particular, a full distributional prediction of a Gaussian with mean and variance parameter, evaluated by log-loss, may equivalently be seen as a prediction of mean-variance only, as evaluated by the mean-variance likelihood-loss. This statement is entirely unaffected by the true distributional form of what is predicted, which may or may not be Gaussian.
\end{enumerate}

While on first glance it may seem strange that predicting a full distribution is supposed to be just as good as only predicting two key parameters, the reason is that we use likelihood-based losses for evaluation. The likelihood is closely linked to entropic measures, and the Gaussian is the most entropic distribution with fixed mean and variance. Further connections between likelihood, entropy, and prediction will be explored in Section~\ref{sec:entropy}.

\subsection{Short note: why taking (posterior) expectations is a bad idea}
\label{sec:classprob-badidea}
An observant reader may wonder whether the functional substitution (and hence the added level of complication) in defining the loss function with a functional and a scalar argument is really necessary; that perhaps defining a dummy random variable $Z_p$ with density $p$ and taking an expectation, say
$$L(p,y) := \EE_Z\left[ L_{sq}(Z_p,y)\right] = \EE_Z\left[ (Z_p-y)^2\right]$$
might be a better idea and simpler as one could now take any ``classical'' loss function in the expectation brackets. This is indeed what is in-practice used occasionally\footnote{what is presented is common usage in Bayesian machine learning, see~\cite{tran2016edward}, as opposed to the term's use in the Bayesian statistics community, see for example~\cite{gelman2014bayesian}, where PPC refer to a series of semi-graphical sanity checks rather than a quantitative, loss-based evaluation.} in a Bayesian context under the name ``posterior predictive checks'' (PPC), where $Z_p$ follows the sampling distribution of the posterior sample (and/or produces a finite posterior sample).

Unfortunately, this strategy is inappropriate for validating probabilistic predictive strategies as it is logically broken. This is because the deterministic, hence non-probabilistic, prediction of the mean $\EE[Z_p]$ always outperforms the probabilistic one, due to Jensen's inequality and convexity of the classical loss. For the use of the classical squared loss specifically, as in the usual bias-variance decomposition (compare Proposition~\ref{Prop:BVprob}), one obtains more explicitly
$$L(p,y) = \Var (Z_p) + (y - \EE_Z \left[Z_p\right])^2 \ge (y - \EE_Z \left[ Z_p\right])^2 = L_{sq}(\EE_Z(Z_p),y)$$
for arbitrary $p$. The main logical issue with PPC is that any non-trivial predictive distribution is dominated by predicting a point estimate (its mean), as measured by the PPC through $L$. Thus adding predictive expressivity through the distribution is not advantageous compared to the classical point predictions, and one could have restricted oneself to point predictions in the first place.

Notabene, this holds only as measured through PPC, thus the correct conclusion is that PPC is an unsuitable criterion to assess probabilistic predictions, rather than an issue with the task of probabilistic prediction itself\footnote{In particular, the argument sometimes heard in the context of machine learning, namely that point predictions are all that matters since they dominate probabilistic ones in PPC, is invalid as it implicitly, and falsely, assumes validity of PPC for assessing probabilistic predictions.}. On the contrary, certainly modelling the predictive distribution accurately is something that should be rewarded above not modelling it, if quantifying the uncertainty inherent to the prediction is important. While of course it should be measured by a method appropriate for the probabilistic setting, such as measured by a proper probabilistic loss functional.\\

Given the above, one may also wonder whether PPC may be modified to yield an appropriate method for assessing probabilistic predictions. For example, one could include moments of $Z_p$ instead the variation around $y$, but in order to distinguish arbitrary distributions from the optimal one, an infinity of moments are needed (since the set of possible distributions has infinite cardinality), which then indirectly leads back to substituting into the density (which is the Fourier transform of the characteristic function) by a detour which has more restrictive convergence assumptions.

Some parts of this discussion are found, in mathematically similar form, at the end of Section~3.1.2 of~\cite{vehtari2012Bayesian} (though the connection to the discussion in Section~\ref{sec:classprob} seems to be new, as well as the implied interpretation of this discussion with regards to the popular PPC methodology).

\newpage
\section{Model assessment and model selection}
\label{sec:modelvalidation}

In this section, we study model assessment and model selection strategies in the probabilistic supervised learning setting. More precisely, we showcase algorithms to estimate $\varepsilon(f)$ for a prediction functional or prediction strategy $f$, similar to those in the classical supervised learning setting. In this section, the loss $L$ is fixed, and the generalization error $\varepsilon = \varepsilon_L$ uses that fixed loss, unless stated otherwise.

\subsection{Out-of-sample estimation of the generalization loss}
\label{sec:modelvalidation.estim}

For a prediction functional $f$ t.v.i.~$[\calX\rightarrow \Distr (\calY)]$, possibly estimated from training data, the expected generalization loss $\varepsilon(f)$ can be straightforwardly estimated from an independent test sample $\calT = \{(X_1^*,Y_1^*),\dots, (X_M^*,Y_M^*)\}$, in a direct application of the usual out-of-sample/empirical estimation strategy of the generalization loss in the classical supervised setting. More precisely, in the out-of-sample setting one assumes that the test sample $\calT$ and the estimated prediction functional $f$ are independent.
The usual estimator for the generalization loss is the simple (conditional) mean estimator
$$\widehat{\varepsilon}(f|\calT) = \frac{1}{M}\sum_{i=1}^M L(f(X_i^*),Y_i^*)\quad\mbox{for estimating}\;\varepsilon(f)= \mathbb{E}\left[L(f(X),Y)\right].$$
As the feature-label-pairs in $\calT$ are i.i.d., the summands $L(f(X_i^*),Y_i^*)$ are also i.i.d.~after conditioning on $f$, hence the estimator $\widehat{\varepsilon}$ for estimating $\varepsilon$ is an instance of estimating a sample mean, namely that of the (conditionally i.i.d.) sample of test losses $L_i:=L(f(X_i^*,Y_i^*))$. Thus, the estimation asymptotics of $\widehat{\varepsilon}$ will be governed by the size of the test set $M$, and the concrete model $f$. Summarizing this discussion more formally:

\begin{Prop}\label{Prop:est}
Let $f$ be a $[\calX\rightarrow \Distr (\calY)]$-valued random variable (a prediction functional possibly depending on random training data through a learning strategy).
Consider a test set $\calT = \{(X_1^*,Y_1^*),\dots, (X_M^*,Y_M^*)\}$ with $(X_i^*,Y_i^*)~\sim (X,Y)$ i.i.d., and independent of $f$ (as a set, i.e., mutually).
Then,
$$\widehat{\varepsilon}(f|\calT) := \frac{1}{M}\sum_{i=1}^N L(f(X_i^*),Y_i^*)$$
is an estimator of the expected generalization loss $\varepsilon(f)= \mathbb{E}\left[L(f(X),Y)\right]$ (expectations taken over $f$ and $(X,Y)$), with the following properties:
\begin{enumerate}
\itemsep-0.2em
\item[(i)] $\widehat{\varepsilon}(f|\calT)$ is an unbiased estimator of $\varepsilon(f)$. I.e., $\EE \left[\widehat{\varepsilon}(f|\calT)\right] = \varepsilon(f)$
\item[(ii)] $\widehat{\varepsilon}(f|\calT) |f$ is a consistent estimator of $\varepsilon^* (f):=\mathbb{E}\left[L(f(X),Y)\right|f]$, i.e., when conditioning both estimator and estimated quantity on $f$. Important note: this is in general wrong without the conditioning on $f$.
\item[(iii)] $\sqrt{M}\left(\widehat{\varepsilon}(f|\calT) - \varepsilon^* (f)\right) \overset{d}{\rightarrow} \calN \left(0,\Var(L(f(X),Y)|f)\right)$, as $M\rightarrow \infty$, i.e., $\widehat{\varepsilon}(f|\calT)$ has $O(\sqrt{M})$ convergence asymptotic (conditional on $f$).
\end{enumerate}
\end{Prop}
\begin{proof}
These are standard properties of the i.i.d.~sample mean estimator, applied to $\varepsilon^*(f)$ which is a mean of the conditionally i.i.d.~losses $L(f(X_i^*),Y_i^*)$. (iii) is the classical central limit theorem for the i.i.d.~sample mean.
\end{proof}

Note that property~(ii) of Proposition~\ref{Prop:est} intuitively translates to saying that $\widehat{\varepsilon}$ is a good estimator for assessing goodness of a \emph{particular learnt prediction functional} $f$, though not necessarily (and in usual practical situations usually not) a good estimator for assessing goodness of the {\emph prediction strategy} which has led to obtain $f$, i.e., the prediction strategy whose formal counterpart is the unconditional random variable $f$.

An estimator of the latter kind can be obtained by replicates and/or re-sampling akin to the frequently used cross-validation or bootstrapping schemes. However, none of those are specific to the probabilistic losses or the probabilistic setting in general, nor are the properties in Proposition~\ref{Prop:est} which exposes a direct and straightforward correspondence to the classical setting. By virtue of this correspondence, the many open questions in cross-validation and bootstrapping also remain open for the probabilistic setting, and solving those in one setting will usually also solve them in the other. Though neither is the aim of this particular manuscript, thus we refer the reader to the respective meta-learning literature~\cite[e.g.,][]{arlot2010survey} instead of explicitly substituting a probabilistic $L$ into a number of well-known (or less well-known) statements for deterministic losses.

Finally, we would like to note that property~(iii) of Proposition~\ref{Prop:est} guarantees model selection in $O(\sqrt{M})$ asymptotic, as in the case of classical supervised learning - under the assumption that we have specified an algorithm for model selection.

\subsection{Model comparison and model performance}
\label{Sec:modelvalidation.tuning}

Estimating the generalization loss for a single model is worthless without a point of comparison, which can be a different model, or a baseline model - this is no different from the classical case, where measures such as the regression MSE or classification accuracy are also scientifically meaningless when given only for a single model without a baseline. Note that even though ``zero'' appears to be a canonical reference point in the classical setting, it is in general never reached even by a perfect model due to the residual noise which is a data-dependent quantity (compare the term $\Var(\varepsilon)$ in Proposition~\ref{Prop:BVclass} below), and anyway the crucial point of comparison is rather the ``uninformed'' model which is a surrogate for guessing, that is, for not using any information about the features. This is because it is of much higher practical interest to determine whether a model is better than guessing, than to determine whether it is worse than perfect (which it always will be anyway, except in the empirically unnatural noise-free setting). The performance of such an uninformed model in terms of say MSE, accuracy in the classical setting, or the probabilistic losses (e.g., log-loss or squared integrated loss) in the probabilistic setting, will always be data-dependent and hence needs to be estimated for an explicit comparison, or implicitly compared to.

Similarly, automated tuning of parameters or other meta-modelling schemes are also impossible without a proper way to compare the predictive goodness of different models, and to compare to an uninformed baseline.

The principle of model comparison in the probabilistic setting is simple and can be seen as a direct generalization of the classical setting.
Suppose we would like to compare models/modelling strategies $f_1,\dots, f_k$, one or multiple of which are possibly state-of-art baselines or uninformed baselines. We can do this on an independent test sample $\calT = \{(X_1^*,Y_1^*),\dots, (X_M^*,Y_M^*)\}$, by computing the samples of test errors
$$\calS_j = \{L(f_j(X_i^*),Y_i^*)\;:\; i \in \{1,\dots, M\}\}\quad\mbox{for}\; j = 1\dots k.$$
Recall that $\widehat{\varepsilon}(f_j|\calT)$ is the mean of the sample $\calS_j$ whose elements are i.i.d.~conditional on $f_j$ (due to the assumption that $\calT$ is an independent sample), thus follows the normal large-sample asymptotic of the mean, as stated in Proposition~\ref{Prop:est}.
Further notice that the samples $\calS_j$ are paired/grouped, by the index $i$ of $L(f_j(X_i^*),Y_i^*)$.

It is now interesting to observe how frequentist and Bayesian approaches to model selection coincide on the paired samples $\calS_j$. Namely, writing $\widehat{\mu}_j:=\widehat{\varepsilon}(f_j|\calT)$ for the mean of the sample $\calS_j$, one has:
\begin{enumerate}[(i)]
\itemsep-0.2em
\item the difference in means $\widehat{\mu}_j - \widehat{\mu}_i$ may be seen as a difference in evidence. In fact, if the comparison is by log-loss, i.e., $L=L_\ell$, then this coincides with the (Bayesian) logarithmic Bayes ratio for the independent test set, often used in Bayesian model comparison of $f_j$ and $f_i$ (though often the test set is not necessarily independent).
\item the samples $\calS_j$ and $\calS_i$ are paired i.i.d., and follow a normal asymptotic, hence a (frequentist) hypothesis test for difference of means (such as t-test) can be used for frequentist model comparison.
\item Note how such a (frequentist) hypothesis test assesses the plausibility of the (Bayesian) difference in evidence, in the case of the log-loss whether the logarithmic Bayes ratio is different from zero (though a stylized purist Bayesian would not accept the concept of ``significance'' as a measure for how non-zero the ratio is).
\end{enumerate}

Similarly, there is a correspondence of group-wise or portmanteau hypothesis tests and the quantifications based on the collection of all pairwise Bayes factors. In the forecasting setting, this dual interpretation has already been noted in~\cite[Section~7]{gneiting2007strictly} and in ~\cite[Section~3]{dawid2015bayesian}, while the normal asymptotic of the $\widehat{\mu}_j$ is characteristic of the supervised prediction setting, as it is tied to the i.i.d.-ness of the test sample.\\

More generally, even though $\widehat{\mu}_j$ has a (generative, frequentist) normal asymptotic, i.i.d.-ness of the test sample allows the application of non-parametric pairwise or portmanteau hypothesis test such as the Wilcoxon signed rank test for pairwise testing, and Cramer-von-Mises or Anderson-Darling for portmanteau testing. Non-parametric testing may even be preferable to application of Student family tests since in a finite sample setting the generative distribution may deviate arbitrarily from the normal. Algorithm ~\ref{alg:testing} gives pseudo-code for prototypical (single-split) validation.

Further, the samples $\calS_i$, may be used to compute (frequentist) $z$-approximated confidence intervals for $\widehat{\mu}_i$, as the standard error of the sample $\calS_i$. In our setting, such a confidence interval, would (asymptotically) coincide with a (Bayesian) credible interval for $\widehat{\mu}_i$ due to the normal asymptotic of $\widehat{\mu}_i$.
Algorithm~\ref{alg:meanerr} describes pseudo-code for prototypical (single-split) strategy to estimate the $\widehat{\mu}_i$ and frequentist standard errors $\widehat{\sigma}_\mu$ for the confidence intervals (a two-sided 95\% confidence interval would be $\widehat{\mu}_i\pm 1.96\cdot \widehat{\sigma}_\mu$).

Also note that the central limit asymptotic in Proposition~\ref{Prop:est} guarantees selection of the best prediction functional(s) in a (frequentist) generative setting if the comparison above is used for model selection (i.e., select the models for which there exist no significantly worse models).

Automated model tuning may be similarly done by the usual nested validation schemes, though parameter tuning may also be attempted through method-specific (e.g., likelihood type or Bayesian) methodology.

\begin{algorithm}[ht]
\caption{Validation meta-algorithm: Group-wise comparison of probabilistic predictors.\newline
\textit{Input:} Training data $\calD$, test data $\calT = \{(X_1^*,Y_1^*),\dots,(X_M^*,Y_M^*)\}$, predictors $f_1,\dots, f_k$,
a grouped sample test\newline
\textit{Output:} Which predictors are best/how they compare \label{alg:testing}}
\begin{algorithmic}[1]
    \State train $f_1,\dots, f_k$ on $\calD$
    \State use the trained $f_i$ to make predictions $p_{ij}:=f_j(X_i^*), i = 1\dots N, j = 1 \dots k$ for the labels in the test set
    \State Obtain grouped/paired samples $\calS_j := \{L(p_{ij},Y_i^*)\;:\; i \in \{1,\dots, M\}\}\quad\mbox{for}\; j = 1\dots k.$
    \State Run a paired/grouped test on the grouped/paired samples $\calS_j$.
       For example, if $k=2$ for a comparison of two methods, run a Wilcoxon signed-rank test on the pairs $(p_{i1},p_{i2}),i=1\dots N$
       and report the significance level.
\end{algorithmic}
\end{algorithm}

\begin{algorithm}[ht]
\caption{Validation meta-algorithm:\newline Error estimate and frequentist confidence intervals for the predictive loss.\newline
\textit{Input:} Training data $\calD$, test data $\calT = \{(X_1^*,Y_1^*),\dots,(X_M^*,Y_M^*)\}$, prediction strategy $f$\newline
\textit{Output:} an estimate $\widehat{\mu}$ for $\varepsilon(f)$ and an error estimate
       $\widehat{\sigma}_\mu$ for the standard error of this estimate \label{alg:meanerr}}
\begin{algorithmic}[1]
    \State train $f$ on $\calD$
    \State use the trained $f$ to make predictions $p_1:=f(X_1^*),\dots, p_M:= f(X_M^*)$ for the labels in the test set
    \State Return $\widehat{\varepsilon}:=\frac{1}{M}\sum_{i=1}^M L(p_i,Y_i^*)$ as the estimate
    \State Return $\widehat{\sigma}_\mu:= \left(\frac{1}{M(M-1)}\sum_{i=1}^M (L(p_i,Y_i^*) - \widehat{\varepsilon})^2\right)^{1/2}$
\end{algorithmic}
\label{alg:validation-2}
\end{algorithm}

\subsection{Uninformed baseline: label density estimation}
\label{Sec:uninfbase}

In scientific model assessment, a necessary argument for usefulness of any model is its improvement about ``uninformed guessing''. Since even if model A improves upon model B, both could be worse than an uninformed guess without further quantification, making the A vs B comparison meaningless without the comparison to such a ``guessing'' baseline.

For (both probabilistic and classical) supervised learning, we claim (and show) that the natural baselines are given by the class of prediction functionals that disregard the test features in making the prediction on the test set. It is not a-priori clear whether it should be able to make use of the training features, the training labels, or both; below we give a mathematical argument why the uninformed baseline is allowed to make use of the full training set.

We start with formalizing the concept of uninformedness and related properties.

\begin{Def}
\label{Def:uninformedprob}
A prediction functional $f:\calX \rightarrow \Distr(\calY)$ is called uninformed if it is a constant functional, i.e., if $f(x) = f(y)$ for all $x,y\in\calX$. We write $u_p$ for the uninformed prediction functional $u_\alpha:x\mapsto p$.\\
A prediction strategy, i.e., a $[\calX\rightarrow \Distr(\calY)]$-valued random variable, is called uninformed if all the values it takes are uninformed (prediction functionals).
\end{Def}

Note that uninformedness does \emph{not} imply that for fixed $x\in\calX$, the evaluation $f(x)$ is a constant random variable as $f$ may depend on the (random) training data $\calD$. However, the following characterizations hold:

\begin{Prop}
\label{Prop:uninformeddef}
In the above setting, the following are equivalent.
\begin{itemize}
\item[(i)] The strategy $f$ is uninformed.
\item[(ii)] $f(Z)$ is conditionally independent of $Z$, given $f$, for every $\calX$-valued random variable $Z$ independent of $f$.
\end{itemize}
In particular, statement (i) implies that (but is in general not equivalent to)
\begin{itemize}
\item[(ii')] $f(Z)$ is conditionally independent of $Z$, given $f$.
\end{itemize}
\end{Prop}
\begin{proof}
(i) $\Rightarrow$ (ii): Let $Z$ be an arbitrary $\calX$-valued random variable. Since $f$ is independent of $f$, by definition of independence it holds that, $f = f|Z$, and in particular $f(x) = f(x)|Z$ for an arbitrary constant $x\in\calX$. Since $f$ is uninformed, it holds that $f(x) = f(Z)$. Thus, $f(Z) = f(Z)|Z$ which implies (ii), because $Z$ was arbitrary.\\
(ii) $\Rightarrow$ (i): we prove this by contraposition. Suppose not(i), i.e., $f$ is not uninformed. Then, by definition, there are $x,x'\in\calX$ such that $f(x)\neq f(x')$ does not hold with probability one. Thus, a random variable $Z$ which assumes values $x,x'$ with probability $1/2$ each will be associated with $f(Z)$, which is not-(ii).
\end{proof}

We will show that the best uninformed prediction is given by the perfect density estimate:

\begin{Def}
We call the predictor $\varpi_Y := u_{\calL(Y)} = [x\mapsto \calL(Y)]$ (the constant predictor which always predicts the law of $Y$) the \emph{best uninformed predictor}.
\end{Def}

For this, we relate performance of prediction strategies to prediction functionals:
\begin{Def} \label{Def:detrain}
Let $f$ be a prediction strategy. We define $[\EE f] : x\mapsto \EE[ f(x)].$\\
For any random variable $Z$, we define $[\EE f |Z] : x\mapsto \EE[ f(x) |Z].$
\end{Def}

\begin{Lem}\label{Lem:functionjensen}
Let $L$ be a convex loss.
It holds that $\varepsilon_L(\EE f) \le \varepsilon_L(f)$.
If $L$ is strictly convex, equality holds if and only if $f$ is a prediction functional.
\end{Lem}
\begin{proof}
by Jensen's inequality and convexity of $L$, it holds that
$\EE[L(\EE [f(x)],Y)] \le \EE[L(f(x),Y)]$ for any prediction strategy $f$ and any $x\in\calX$. Thus, it also holds that
$$\varepsilon_L([\EE f]) = \EE[L([\EE f](X),Y)] \le \EE[L(f(X),Y)] = \varepsilon_L(f).$$
The equality statement is given by the converse statement in Jensen's inequality.
\end{proof}

\begin{Prop}
\label{Prop:uninformed}
Let $L$ be a convex, proper loss. The following are equal:
\begin{itemize}
\item[(i)] $\varepsilon_L(\varpi_Y)$
\item[(ii)] $\min \{\varepsilon_L(f)\;:\;f\mbox{ is an uninformed prediction functional}\}$
\item[(iii)] $\min \{\varepsilon_L(f)\;:\;f\mbox{ is an uninformed prediction strategy}\}$
\end{itemize}
If $L$ is also strictly proper, the minima in (ii) and (iii) are unique and only assumed if $f=\varpi_Y$ ($X$-almost everywhere).
\end{Prop}
\begin{proof}
We begin by proving equivalence of (i), (ii), and (iii).\\
Denote by $S_{(i)}, S_{(ii)}, S_{(iii)}$ the quantities defined in (i), (ii), (iii).\\
Since $\varpi_Y$ is a prediction functional and all prediction functionals are prediction strategies, it holds that $S_{(i)}\ge S_{(ii)}\ge S_{(iii)}$.\\
$S_{(i)}\le S_{(ii)}:$ By properness of $L$, it holds that
$\EE[L(\calL(Y),Y)] \le \EE[L(p,Y)]$ for any $p$ (in the considered domain), by definition of uninformed predictors this is equivalent to $\EE[L(\varpi_Y(x),Y)] \le \EE[L(u_p(x),Y)]$ for any $x\in\calX$. Thus,
$$\varepsilon_L(\varpi_Y) = \EE[L(\varpi_Y(X),Y)] \le \EE[L(u_p(x),Y)] = \varepsilon_L(u_p).$$
This implies $S_{(i)}\le S_{(ii)}$ since all elements of the set in (ii) are of form $u_p$ for some $p$.\\
$S_{(ii)}\le S_{(iii)}:$ Lemma~\ref{Lem:functionjensen} implies
$\varepsilon_L([\EE f]) \le \varepsilon_L(f).$ Since $[\EE f]$, by definition, is a prediction functional, this implies that $S_{(ii)}\le S_{(iii)}$.\\
All inequalities together prove equivalence.\\
Uniqueness in case $L$ is strictly proper follows from above observing that, in this case, inequalities are strict unless $\varpi_Y(x) = u_p(x)$ for ($X$-almost all) $x\in\calX$.
\end{proof}

The main argument why the best uninformed prediction is the correct mathematical object for a baseline comparison is mathematical. As we will show later, for a strictly proper loss, outperforming the best uninformed predictor $\varpi_Y$ is possible if and only if $Y$ and $X$ are not statistically independent:

\begin{Thm}
\label{Thm:entropypreview}
Let $L$ be a strictly proper, strictly convex loss. Then, the following are equivalent:
\begin{itemize}
\item[(i)] $X$ and $Y$ are statistically independent.
\item[(ii)] There is no prediction functional $f:\calX\rightarrow \Distr(\calY)$ such that $\varepsilon_L(f)\lneq \varepsilon_L(\varpi_Y)$.\\
     I.e., there is no ($L$-)better-than-uninformed prediction functional for predicting $Y$ from $X$.
\item[(iii)] There is no prediction strategy $f$ taking values in
$[\calX\rightarrow \Distr(\calY)]$ such that $\varepsilon_L(f)\lneq \varepsilon_L(\varpi_Y)$.\\
 I.e., there is no ($L$-)better-than-uninformed prediction strategy for predicting $Y$ from $X$.
\end{itemize}
\end{Thm}
\begin{proof}
This follows from the later Theorem~\ref{Thm:entropy} where the proof is carried out.
\end{proof}

Phrased more intuitively, Theorem~\ref{Thm:entropypreview} implies that $\varpi_Y$ is the best possible baseline certifying for ``uninformedness'' in the formal statistical sense, since any prediction that is better necessarily implies that non-baselines can exist - in both the self-referential sense, as well as through necessarily implying statistical dependence of $Y$ on $X$.

Or, more positively: if any prediction strategies are baselines, then certainly those that are uninformed prediction functionals, since their prediction is obtained without looking at any data - not at the training data by being fixed, and not at the test features by being constant per Definition~\ref{Def:uninformedprob}. On the other hand, Proposition~\ref{Prop:uninformed} in combination with Theorem~\ref{Thm:entropypreview} states that being able to outperform all such baselines is equivalent to statistical dependence of $Y$ on $X$.

We conclude with some remarks on the concept of uninformedness as given above.
First, Proposition~\ref{Prop:uninformed} may seem to show, on first glance, that, contrarily to what is claimed above, using the smallest class as uninformed predictors - namely the fixed predictors which are entirely independent of any data - is a better choice in practice. However, this conclusion is a non-sequitur and in fact faulty. Namely, if uninformed predictors are a baseline to which one wishes to compare, considering a wider, not smaller class of such prediction strategies will give a more solid baseline. This is especially true in the estimation sense discussed in Section~\ref{sec:entropy}.

Second, the above argumentation exposes a scientifically sufficient set of baselines: namely, methods estimating $\calL(Y)$ from the i.i.d.~sample $(X_1,Y_1),\dots, (X_N,Y_N)\sim (X,Y)$. It is interesting to note that while the class of uninformed baselines includes density estimators which perform this task by looking only at $Y_1,\dots,Y_N\sim Y$, density estimation is a strict subset, since the side information in the $X_i$ may be used, and may be useful even if the task is an unconditional estimation of $\calL (Y)$. This is for example easily seen in the situation where $X$ is identical to $Y$ up to some small noise terms.

On the topic of inference, we remark the following:
If the goodness of the prediction is measured by the logarithmic loss function $L_\ell$, goodness of prediction is measured by the likelihood of the predicted distribution.
Thus, if it is known that $Y$ follows a parametric distribution class, then (frequentist) maximum likelihood estimation, (Bayesian) parametric inference, or (hybrid) maximum a-posteriori estimation are possible uninformed estimators that may be baselines. For classification, i.e., a finite $\calY$, density estimation is particularly simple;
a loss-optimal uninformed prediction is achieved by predicting as probabilities the observed class frequencies (assuming all classes have been observed at least once in the training sample).

For the regression case and the case of general $\calY$, three characteristic differences to common inference should be observed:
First, even if a parametric model is known, the parameter itself is not of interest - since the distribution is the output. In particular, in frequentist inference, the prediction is not the estimate, but the distribution with the parameter substituted. In Bayesian inference, the prediction is not the posterior for the parameter(s), but the distribution with the parameter whose law is the posterior distribution. The probabilistic bias-variance decomposition in Proposition~\ref{Prop:BVprob} implies that taking a point-mixture, i.e., predicting a fixed distribution conditional on the training sample, namely the posterior predictive distribution for $Y$ (without modelling $X$) has lower expected loss than sampling from the parameter posterior.
Second, in general, no parametric model is known, hence the uninformed prediction will take the form of non-parametric density estimation.
As we will see in Section~\ref{sec:BIC}, non-parametric variants of classical maximum likelihood or Bayesian density estimation are both ill-defined when the set of possible distributions is completely unrestricted - due to potentially unbounded likelihood.

However, when combined with out-of-sample tuning as presented in Section~\ref{Sec:modelvalidation.tuning} and often carried out in the conditional density estimation literature (see Section~\ref{sec:priorart.assessment}), that problem can be avoided, as illustrated in Section~\ref{sec:oos}.

\subsection{Short note: why permutation of features/labels is not the uninformed baseline}
\label{Sec:badbaseline}

In the light of our suggestion in Section~\ref{Sec:uninfbase} to use constant predictors, hence density estimators, as uninformed baseline, a much more popular suggestion needs to be discussed: the permutation baseline. We hold that the permutation baseline is practically inappropriate and theoretically inadequate as a baseline, as we explain below.

The permutation baseline is found, in the earliest instance known to us, implemented in the scikit-learn toolbox~\cite{pedregosa2011scikit},
and has recently also been discussed in the context of hypothesis testing via prediction~\cite{lopez2016revisiting},
as a predictive baseline (for classical supervised classification).

Mathematically, the suggestion is to use as a pseudo-uninformed baseline the predictor $f:\calX\rightarrow \calY$ (in the classical setting) or
$f:\calX\rightarrow \Distr (\calY)$ (in the probabilistic setting) which has been obtained by learning on modified training data
$(X_{\pi(1)},Y_1),\dots,(X_{\pi(N)},Y_N)$ where $\pi$ is a uniformly sampled random permutation of the numbers $\{1,\dots, N\}$.
The baseline performance is then estimated by the performance of such an $f$, given a sensible choice of learning algorithm.
Alternatively, labels are permuted but not feature vectors, or both are permuted independently, both of which is equivalent to the above as long as
the sub-sampling for performance estimation is uniformly random.

The quite strong (enthymemous) argument in favour of using this $f$ as a baseline is that in jumbling the order of the features $X_i$, no $f$ can
make sensible use of the $X_i$ because the relation to the $Y_i$ is destroyed in the process.

We argue that this is a non-sequitur and hence a non-argument - the only certain way to prevent prediction of a test label $Y^*$ by a test feature $X^*$ (without preventing prediction altogether) is, by definition of what this means, actually removing access to the test feature $X^*$. Leaving access to $X^*$ may still enable $f$ to predict $Y^*$ well from $X^*$, irrespective of what happened to the training sample.

One argument which actually shows inadequacy of the permutation baseline is that a Theorem such as Theorem~\ref{Thm:entropypreview} cannot be true for a permutation version of ``uninformed'', since a prediction strategy not restricted in its image can always guess a good prediction functional.

Even worse, there are situations where one need not even resort to guessing, where the true prediction functional may even be perfectly recovered from the permuted data by the very same prediction strategy which also works on the unpermuted data. We give a stylized example where this is the case:

Let $(X,Y) := (1,1)\;\mbox{with probability}\; 2/3,\; (-1,-1)\;\mbox{with probability}\; 1/3,$ and consider, as usual, i.i.d.~samples $(X_1,Y_1),\dots, (X_N,Y_N)$ for training.
A naive algorithm which simply assigns the most frequent label to the most frequent feature - such as the one sketched as Algorithm~\ref{alg:countassoc} - will make correct predictions with high probability. This is true no matter whether the permutation is applied or not, and is independent of the size of the dataset, note for example that none of the variables or decisions in Algorithm~\ref{alg:countassoc} change.

\begin{algorithm}[ht]
\caption{Count association classifier, fitting. This is a theoretical counterexample to the permutation baseline, do not use in practice.\newline
\textit{Input:} Training data $\calD = \{(X_1,Y_1),\dots (X_N,Y_N)\}$, where $X_i$ and $Y_i$ take values in $\{-1,1\}$.
both trainable\newline
\textit{Output:} a classical predictor $f:\calX\rightarrow \calY$ \label{alg:countassoc}}
\begin{algorithmic}[1]
    \State Let $N_X\leftarrow \card\{X_i\;:\;X_i = 1\}$
    \State Let $N_Y\leftarrow \card\{Y_i\;:\;Y_i = 1\}$
    \State If $|N_X - N_Y| < |N- N_X + N_Y|$ output $f: x\mapsto x$
    \State Else output $f: x\mapsto -x$
\end{algorithmic}
\end{algorithm}

The counterexample may seem pathological, but in fact it is only stylized. Much less pathological examples may be easily constructed through replacing the numbers $-1,1$ well-separable clusters in whatever continuous or mixed variable space. Probabilistic versions of the examples can be obtained in straightforward ways (e.g., assigning small but non-zero probabilities to the ``other'' class).

Together with Theorem~\ref{Thm:entropypreview}, one may conclude that defining ``uninformedness'' through the data and not through properties of the learning algorithm may be theoretically insufficient\footnote{ Though the use of a spuriously good baseline is not too damaging in practice as it causes type II errors rather than type I errors. That is, methods compared with a too good baseline will be found less often to outperform the spuriously overperforming permutation baseline, than they would be found to outperform a fair baseline.}, especially since it does not define a valid class of baselines without further modifications or assumptions, as the above examples and considerations show.

In contrast to the permutation baseline criticized above, the ``constant prediction'' definition of uninformedness forbids $f$, by construction, to make use of the relevant information in prediction. The claim that a constant $f$ is indeed prevented from using any learnt information about the prediction rule is mathematically proven in Theorem~\ref{Thm:entropypreview}.

\subsection{Classical baseline: residual density estimation}
\label{Sec:classicbase}

The identification of classical supervised regression learning with certain homoscedastic probabilistic predictions, as described in Section~\ref{sec:classprob},
has direct algorithmic consequences in allowing to apply classical supervised learning algorithms as possible solution to the probabilistic setting.
These classical algorithms made probabilistic yield an important class of baselines in the probabilistic setting, since comparison to them certifies whether and,
if yes, to which extent the use of genuinely ``probabilistic'' models is necessary in addition.

The probabilistic form of classical prediction strategies is directly implied by the discussion in Section~\ref{sec:classprob} which shows that there is a one-to-one equivalence between:
a [choice of a probabilistic predictor predicting a location-variable density] and a
[choice of classical point-predictor and a choice of classical loss function].
For example, the probabilistic predictors always yielding a variance-1-Gaussian with variable mean are in one-to-one equivalence to the point-predictors evaluated by the mean squared error.
The probabilistic predictors always yielding a variance-1-Laplace distribution with variable mean are in one-to-one equivalence to the point-predictors evaluated by mean absolute error.

Note that in this correspondence, multiplying the classical loss function by a constant is equivalent to scaling the variance of the location-variable density,
hence the scaling parameter needs to be optimally determined to yield a fair baseline.

General transformation of a classical to probabilistic one may be done via the following meta-algorithm, which converts a density estimator and a classical prediction method into a
probabilistic predictor, by virtue of applying a density estimator to the residuals of a classical method, as
outlined in stylized pseudo-code as Algorithm~\ref{alg:classiprob}.

\begin{algorithm}[ht]
\caption{Meta-algorithm: obtaining a point prediction type baseline probabilistic regression strategy from a point prediction regression strategy and a distribution estimation strategy.\newline
\textit{Input:} Training data $\calD = \{(X_1,Y_1),\dots (X_N,Y_N)\}$, a classical prediction strategy $g$ t.v.i.~$[\calX\rightarrow \calY]$, a probabilistic (for baseline: uninformed) predictor $h$, t.v.i.~$[\calX\rightarrow \Distr (\calY)]$.\newline
\textit{Output:} a probabilistic predictor $f$ t.v.i~$[\calX\rightarrow \Distr (\calY)].$ \label{alg:classiprob}}
\begin{algorithmic}[1]
    \State Train $g$ on $\calD$.
    \State Obtain residuals $\rho_i:= Y_i - g(X_i)$ for $i=1,\dots, N$.
    \State Train $h$ on the sample $(X_1,\rho_1),\dots, (X_N,\rho_N)$.
    \State Return $f:= x\mapsto \left[y\mapsto [g(x)](y-f(x))\right]$.
\end{algorithmic}
\end{algorithm}

The idea in Algorithm~\ref{alg:classiprob} can of course be also directly applied to boost a classical prediction strategy with any (not necessarily uninformed) probabilistic one, though the resulting prediction strategy is only a point prediction type baseline if the probabilistic prediction strategy $h$ is a density estimation strategy, i.e., uninformed. Note that in that case, by the discussion in Section~\ref{Sec:uninfbase}, the density estimation strategy may look at the full training set, including the training features $X_i$.

If the classical loss is fixed in addition to the classical algorithm, parametric density estimators can be used to determine the scaling constant equivalent to the variance. For example, for the probabilistic version of ``classical predictor with mean squared error'', the probabilistic prediction strategy $h$ should predict a Gaussian with mean zero and with a variance estimated as the variance of the sample of residuals $\rho_i$.

We note that the above is reminiscent of some ideas suggested previously by~\citet{hansen2004nonparametric}, for example as part of the ``two-step method'' for conditional density estimation, though the full correspondence presented above appears to be a novel.

We also note that, interestingly, the correspondence allows to indirectly ``compare'' different choices of classical loss functions via equivalence to the shape of the location-variable density, which is immediate in the probabilistic setting but quite counterintuitive in the classical one.

The above discussion may be modified lightly to yield a classical baseline for classification, by an idea similar to Algorithm~\ref{alg:classiprob}, and recalling from Section~\ref{Sec:uninfbase} that the best uninformed baseline (hence the ``density estimator'') is predicting as probabilities the observed frequencies.
The classification version simply predicts probabilities of observed frequencies, conditioned on the point prediction of the classical algorithm.

\subsection{Bayes type information criteria and empirical risk asymptotics for the log-loss}
\label{sec:BIC}
In a probabilistic context, the Bayesian information criterion (BIC) is probably one of the most used model selection criteria.
To training data $\calD = \{(X_1,Y_1),\dots, (X_N,Y_N)\}$ t.v.i.~$\calX\times \calY$, a (usually parametric) model $f:\calX\rightarrow \Distr (\calY)$ is fitted.
The Bayesian information quantity for the model $f$ is then, formulated in our setting,
$$B_f := \left[ \frac{1}{N} \sum_{i=1}^N L_\ell(f(X_i),Y_i)\right]+k\cdot \log N,$$
where $k$ is the ``number of parameters'' in the model $f$. On the fixed batch of data $\calT$, among multiple possible models of type $f$, the one with smallest value $B_f$ is preferred; in the criterion, the term $k\cdot \log N$ acts as a penalizer for ``model complexity''.

Three important things should be noted about the Bayesian information criterion:
First, there is no independent ``test data'' as in classical supervised learning, or our proposed approach to model selection.

Second, the first term (in the square brackets) is nothing else than the training loss or in-sample loss or ``empirical risk'', as given in our proposed approach, i.e.,
$$B_f = \widehat{\varepsilon}(f|\calD) + \frac{k}{2}\cdot \log N.$$
Note that since $\calD$ is the training data from which $f$ was obtained as well (see the first note), $\calD$ and $f$ will \emph{not} be independent as it would be the case for an independent test data set, hence the discussion in Section~\ref{sec:modelvalidation.estim} will not apply (directly).

Third, and this is maybe also a misunderstanding about the BIC quite frequent among practitioners of statistics,
the criterion itself is only valid (or asymptotically valid) for models of a certain class,
usually similar to linear or polynomial regression models where $k$ will be the number of coefficients.
While for such models the folklore results hold, e.g., that BIC is an asymptotic estimate of the generalization error, they break completely down
for models as simple as Gaussian mixture models, an example we work out for illustrative purposes:

Our class of possible output distributions, in this example, consists of the parametric class
$$p(z|\mu,\sigma) = \frac{1}{K}\sum_{i=1}^K \frac{1}{\sqrt{2\pi\sigma^2_i}}\exp \left(-\frac{1}{2}\cdot \frac{(z-\mu_i)^2}{\sigma_i^2}\right),$$
with univariate argument $z\in \RR$, and the ``intuitive'' number of free parameters is $k=2K$ (namely, $K$ from the $\mu$-s and $K$ from the $\sigma$-s).
Under these parametric assumptions, we can write $f:x\mapsto [p(.| m(x),s(x)]$ with suitable $m,s:\calX\rightarrow \RR$.

It is now not too difficult to see that the term $k\cdot \log N$ is fixed for any such $f$, while each summand can become arbitrarily large by setting $\mu_i = X_j$ for some $i,j$ and letting $\sigma_j \rightarrow 0$. As this is true for any term while all terms are finite, $B_f$ is not bounded below in terms of the choice of $m$ and $s$. Note that there is no problem with making that choice algorithmically, as the training data have actually been seen, and all information to construct the predictions with unbounded (BIC) loss are readily available. Also, clearly (due to arbitrarily large discrepancy to any fixed ``correct'' prediction), such a prediction is nonsensical and can be shown to be arbitrarily bad in generalization.

The intuitive meaning of this example is, simply put, the phenomenon of \emph{overfitting} in the probabilistic setting - it is \emph{not} prevented by the BIC unless the model class is restricted, usually quite heavily. This is no different from the situation in classical supervised learning - empirical risk type regularizers (such as the analogue of BIC, training error plus the same regularisation term as in BIC) work well as long as one is within the remit of highly parametric regression models such as linear or polynomial regression, while this strategy completely breaks down for even slightly more advanced methods such as kernel regression, tree models or tree ensembles, or neural networks.

The same argument carries through, mutatis mutandis, for other popular model selection criteria, for which only the regularisation term changes. For Akaike's (information criterion) the regularisation term is $k$ instead of $\frac{k}{2}\log N$ which is as well constant throughout the counterexample. For the different versions of WAIC, note that the
log-loss is computed in-sample with a model class dependent regularization term that may be upper bounded in the pathological example.

Hence, in summary, and in absence of another reliable estimate of the probabilistic generalization error in a model-independent setting, we hold that the out-of-sample estimate discussed in Section~\ref{sec:modelvalidation.estim} is preferable already based on its elementary yet appealing (and, amongst all, proven) convergence guarantees.

Further discussion on the relation of the information criteria for predictive model evaluation may be found in~\cite{gelman2014understanding}, as stated above and demonstrated by the (counter-)example, assumptions on the model class need to be made for such relations to hold.

\subsection{No counterexample when out-of-sample}
\label{sec:oos}

Even though BIC and AIC break down with overfitting in a model-agnostic setting, that does not necessarily imply that a proposed alternative, such as the one presented in our framework, is free of shortcomings.
This means that even though the generalization error is reliably estimated, the selected model could, in-principle, still be nonsensical in any finite sample setting without any proof to the contrary.

However, we briefly show why the proposed out-of-sample validation is not misled in a non-parametric set-up, differently from the information criteria outlined above. This is due to a universal lower bound on the expected generalization loss, which depends only on the law of $(X,Y)$.

Suppose we are in the example of Section~\ref{sec:BIC}, or in any other situation where we would like to predict probabilistically.
Suppose that from training data, a prediction strategy $f$ which is a $[\calX\rightarrow \Distr (\calY)]$-valued random variable (through dependency on the training data) has been learnt.
By the bias-variance decomposition which we will prove later in Proposition~\ref{Prop:BVprob} and the positivity constraints therein, for the expected generalization log-loss $\varepsilon(f)=\EE \left[L_{\ell}(f(X),Y)|X\right]$, it holds that $\varepsilon(f) \ge \Ent (Y/X)$, where $\Ent(Y/X)$ is the conditional entropy of $Y$ w.r.t.~$X$. This is known to be a finite number for given $X,Y$ (assuming squared integrable densities). Since $\Ent(Y/X)$ does not depend on $f$, any expected generalization loss will be bounded below by it; for $Y|X$ following ``usual distributions'', as well as any that are compactly supported, the variance of $\varepsilon(f)$ is also finite.

Hence the pathology of unbounded expected generalization loss which may occur in-sample, as for the BIC in Section~\ref{sec:BIC}, will not occur out-of-sample.
This may seem slightly paradoxical since there are, similarly as in the in-sample case, predicted distributions with unbounded out-of-sample likelihood.
However, these are unlikely to occur and contribute only finitely in expectation, when distributions are obtained \emph{from a prediction strategy}, as soon as said prediction strategy is statistically independent of the data it is tested on - which is arguably a very weak requirement, in comparison to the pathology it removes.

Note that taking full expectations (i.e., over the training data) will \emph{not} remove the unboundedness in the in-sample case either, since then $f$ is again dependent on any training data point, hence the pathologically overfitting example in Section~\ref{sec:BIC} may be constructed around any training data point, and the pathology appears in expectation as well as in realization.

Thus the only way to remove the model selection pathology from appearing for log-loss type in-sample measures, such as BIC or AIC, is to restrict the choice of prediction strategies (by the discussion in Section~\ref{sec:BIC} effectively to variants of linear regression) - whereas the above reasoning shows that there is neither a pathology\footnote{Strictly speaking, it shows that there are no pathologies of the same type as they are inherent to BIC/AIC/WAIC in a black-box setting. Whereas the theorems in this manuscript show desirable properties, hence the absence of those pathologies which would be their negation. One feature of the proposed set-up which may be unusual is that empirical loss estimates on the same dataset may assume both negative and positive values, as $\Ent(Y/X)$ may be either, and is also unknown. However, this is strictly speaking not a pathology as methods are always compared to each other rather than to a putative optimum which is unknown, and practically never ``zero'' in any other classical setting either.} nor an artificial restriction on allowed prediction strategies in the proposed probabilistic supervised learning framework, which in many respects constitutes an extension of the state-of-art in validation of supervised learning strategies.

\newpage
\section{Mixed outcome distributions}
\label{sec:mixed}
Well-definedness and validity of the model assessment workflow outlined in Section~\ref{sec:modelvalidation} together with the guarantees given therein is less clear if mixed distributions are allowed as predictions or true conditionals $Y|X$. For example, the discussion in Section~\ref{sec:oos} which asserts lower boundedness of the expected generalization error is no longer correct if one allows mixed ``true'' distributions $Y|X$, even if predictions are restricted to be discrete or continuous.
Namely, in the case where conditionals $Y|X=x$ can be mixed, a prediction functional $f$ may put an arbitrarily high continuous density on a discrete point seen in the training set which can re-occur in the test set with positive probability - in this scenario, the out-of-sample loss is not lower bounded as the yet higher density on the already seen data point will lead to a yet lower expected generalization loss. This phenomenon is crucially due to the re-observations of points in the training set which may happen with probability, which cannot happen for absolutely continuous distributions, where the probability of re-observing any data point is zero.\\
Even worse, the logarithmic and squared integrated losses are ill-defined for prediction strategies which are also able to \emph{predict} mixed predictive distributions - such as Bayesian sampling from the predictive posterior, in which predictions take the form of said samples, which are sums of delta distributions through considering the posterior predictive sample as an empirical distribution.\\

To address the issues with mixed distributions, we summarize strategies to cope with true and predicted distributions which may be mixed on a domain $\calY\subseteq \RR^n$. Unlike in the continuous and discrete case, a good solution to the mixed case appears to be an open question - more precisely, it is unclear how to define an algorithmically simple\footnote{as in: no integration/approximation necessary, optimally evaluable by an analytic formula}, computationally tractable, strictly convex and strictly proper loss for the mixed case.

We showcase multiple possible approaches, together with algorithmic considerations and an overview of their advantages and disadvantages.

\subsection{Integrating a discrete loss applied to binary cut-offs}

The most frequently found strategy in literature applies to the case $\calY \subseteq \RR$ and obtains a proper loss by integrating up, over $x$, a proper binary loss for predicting correctly whether the target value is smaller than $x$. More formally:

\begin{Def}\label{Def:RPSes}
Let $L:\Distr(\BB)\times \BB\rightarrow \RR$ be a convex loss for boolean targets, i.e., for $\BB = \{\mbox{yes},\mbox{no}\}$. Let $Z$ be an $\RR$-valued random variable.\\
We define a loss for mixed distributions (parameterized through their cdf) in $\calP\subseteq \Distr(\RR)$ as follows:
$$L^Z:\calP\times \RR \rightarrow \RR\;\quad (F,y)\mapsto\EE[L(F_{\le Z}, y\le Z)],$$
where by $F_{\le \tau}$ we denote the pmf such that $F_{\le \tau}(\mbox{yes}) = F(\tau)$ and $F_{\le \tau}(\mbox{no}) = 1-F(\tau)$; and where by $y\le \tau$ we denote the corresponding boolean variable, i.e., ``yes'' if $y\le \tau$ and ``no'' if $y\gneq \tau$.
\end{Def}

\begin{Rem}\label{Rem:RPS}
An alternative, mathematically more verbose but equivalent definition is the one usually found in literature (see~\cite{gneiting2007strictly}):
$$L^Z: (F,y) \mapsto \int_{-\infty}^y w(\tau)\cdot L(F(\tau), 0)\;  \diff \tau + \int_{y}^\infty w(\tau)\cdot L(F(\tau), 1)  \;\diff \tau,$$
which is due to an alternative notation where $L\in [\RR\times \{0,1\}\rightarrow\RR]$, i.e., $L$ takes a number (not distribution) in the first argument, and a number (not ``yes'' or ``no'') in the second.\\
The ``continuous rank probability score'' (CRPS) is obtained for the boolean loss function $L:(p,y)\mapsto (1-p(y))^2$ in our notation, which is $L:(p_1,y)\mapsto y(1-p_1)^2 + (1-y)p_1^2$ in the more verbose alternative notation usually found in literature.\\
We also would like to note that a weight density $w$ is sometimes not specified or set to $w=1$, but in general $w$ is required to be a probability density for convergence of the integral/expectation. Lower boundedness of $L$ and $w$ being a probability density implies lower boundedness of $L^Z$.
\end{Rem}

\begin{Prop}
Assume the notation of Definition~\ref{Def:RPSes}. Assume that the $Z$ is absolutely continuous with everywhere positive (in particular: non-zero) probability density function $w:\RR\rightarrow \RR^+$.\\
Then the following hold:
\begin{itemize}
\item[(i)] If $L$ is not constant, then $L^Z$ (restricted to absolutely continuous predictions) is strictly global.
\item[(ii)] If $L$ is proper, then so is $L^Z$.
\item[(iii)] $L$ is strictly proper if and only if $L^Z$ is.
\end{itemize}
\end{Prop}
\begin{proof}
(i) By definition of strict globality, it suffices to specify two cdf $F,G$ and $y\in\RR$ such that $F(y)=G(y)$ but $L^Z(F,y)\neq L^Z(G,y)$. This is for example achieved by both $F,G$ putting mass $1/2$ on $0$, and $1/2$ mass on $\pm 1$, with a different sign for $F,G$.\\
(ii) Consider now a real random variable $Y$ with cdf $F_Y$.\\
Properness of $L$ implies $\EE[L(p_B,B)] \le \EE[L(p,B)]$ for any boolean/binary random variable $B$ with cdf $p_B$, and arbitrary pmf $p$. That implies
$$\EE[L(F_{Y,\le z}, Y\le z)]\le \EE[L(F_{\le z}, Y\le z)]$$
for any real-valued cdf $F$ and any $z\in\RR$. This (used as the inequality) implies
\begin{align*}
L^Z(F_Y,Y)& = \EE[L(F_{Y,\le Z}, Y\le Z)]\\
&= \EE_Z\EE_{Y|Z}[L(F_{Y,\le Z}, Y\le Z)|Z]\\
&\le \EE_Z\EE_{Y|Z}[L(F_{\le Z}, y\le Z)|Z]\\
&= \EE[L(F_{\le Z}, y\le Z)]\\
&= L^Z(F,Y)
\end{align*}
which implies the claim since $F$ was arbitrary.\\
(iii) We make two preparatory notes. First, (iii.1), by the proof in part (ii), we have that
$$L^Z(F,Y) - L^Z(F_Y,Y) = \EE_Z[\Delta(F_{\le Z},F_{Y,\le Z}, Y\le Z)|Z]\ge 0$$
for any cdf $F$ and $Y$ as in (ii), where $\Delta(F,F_Y,Z):= \EE_{Y|Z}[L(F_{\le Z}, Y\le Z) - L(F_{Y,\le Z}, Y\le Z)]$.
It also holds that $\Delta(F,F_Y, Z)\ge 0$.\\
Second (iii.2), if $F$ and $F_Y$ are unequal, there is a set $U\subseteq \RR$ of positive Lebesgue measure such that $F(u)\neq F_Y(u)$ for all $u\in U$. The set $U$ has also positive $Z$-measure since $Z$ was assumed absolutely continuous.\\
We now carry out the proof of (iii), two directions need to be shown.\\
$L$ s.p.$\Rightarrow L^Z$ s.p.: First assume $L$ is strictly proper, and $F,F_Y$ are distinct. By strict properness of $L$, it holds that $\Delta(F,F_Y,z) \gneq 0$ whenever $F(z)\neq F_Y(z)$. By (iii.2), it holds that $\EE[\Delta(F_{\le Z},F_{Y,\le Z}, Y\le Z)|Z\in U]\gneq 0$, thus $\EE[\Delta(F_{\le Z},F_{Y,\le Z}, Y\le Z)\gneq 0$ by (iii.1). Since the LHS of the latest expression is equal to $L^Z(F,Y) - L^Z(F_Y,Y)$ by (iii.1), and $F$ was arbitrary, this proves strict properness of $L^Z$.\\
$L^Z$ s.p.$\Rightarrow L$ s.p.: We prove this by contraposition. Assume $L$ is not strictly proper. Then there is a binary distribution $B$ with pmf $p_B$ such that $\EE[L(p_B,B)] = \EE[L(p,B)]$ while $p\neq p_B$. Let $B'$ be the random variable corresponding to $p$. An elementary computation shows that $\EE[L^Z(F_B,Y_B)] = \EE[L^Z(F_{B'},Y_B)]$ where $Y_B,Y_{B'}$ and $F_B,F_{B'}$ are the real valued random variables and their cdf obtained from identifying the values $\{\mbox{yes},\mbox{no}\}$, with $\{1,0\}$ (note that it is not a typo that $Y_{B'}$ does not occur in the equation). Since  $p_B,p$ were distinct, $F_B$ and $F_{B'}$ are as well, thus $L^Z$ is not strictly proper.
\end{proof}

The expectation in Definition~\ref{Def:RPSes} or the integral in Remark~\ref{Def:RPSes} cannot be evaluated analytically in general, but may be approximated by numerical techniques, such as trapezoid rule based or Monte Carlo based. An interesting choice is taking $Z$ to be equal to an independent copy of $Y$, and using the data samples as approximation samples. In the case where all predictions are samples, i.e., sums of delta distributions, the integrals may be turned into sums if integrals over $w = p_Z$ may be easily computed.

Summarizing the above discussion:\\
{\bf Advantages} of the integration approach are the ability to generalize any loss function for binary distributions to the mixed distribution case.\\
{\bf Disadvantages} are the restriction to the univariate case, the need to choose $Z$, and the lack of an analytical evaluation procedure for the integrated loss.

\subsection{Convolution of the target with a kernel}
\label{sec:mixed.conv}

Another possible strategy is based on observing that convolution with a continuous kernel, thus rendering mixed distributions continuous, for any $\calY\subseteq \RR^n$. Comparing a convoluted predicted distribution with a suitably noisy random variable thus removes the issue of having to deal with mixed random variables altogether.

\begin{Def}
Let $p\in\Distr(\calY)$, let $q$ be an absolutely continuous distribution in $\calY$.
We denote by $p\ast q$ the convolution distribution of $p$ with $q$, i.e.,
$$(p\ast q)(y) = \int_\calY p(\tau)q(y-\tau)\;\diff \tau.$$
We denote by $\ast q$ the operator of convolution with $q$, i.e.,
$$\ast q:\Distr(\calY)\rightarrow \Distr(\calY)\;:p\mapsto p\ast q.$$
For $\calQ\subseteq \Distr(\calY)$ we call the operator $\ast q$ injective if the natural restriction $\ast q:\calQ \rightarrow \Distr(\calY)$ is injective.
\end{Def}

\begin{Def}\label{Def:convol}
Let $L:\calQ\times \calY\rightarrow \RR$ be a convex loss for absolutely continuous predictions in $\calQ\subseteq \Distr(\calY)$. Let $Z$ be an absolutely continuous, $\calY$-valued random variable with pdf $p_Z$.\\
We define a loss for mixed distributions (parameterized through their cdf) in $\calP\subseteq \Distr(\calY)$ as follows:
$$L \ast Z:\calP\times \calY \rightarrow \RR\;\quad (p,y)\mapsto \EE[L(p\ast p_Z, y + Z)].$$
We call $L \ast Z$ a \emph{convolution loss}, the $Z$-convolved loss $L$ (e.g., standard normal convoluted log-loss).
\end{Def}

We make two remarks: $L\ast Z$ os well-defined, since convolutions of mixed with absolutely continuous distributions are absolutely continuous.\\
Also, note that $Z$ may be chosen by the experimenter, in particular an arbitrary number of values may be sampled from it. Samples from $Z$ do not appear explicitly in an evaluation $L \ast Z$, in which $Z$ figures only as auxiliary in construction.

The convolution process defines (strictly) proper losses, in case the absolutely continuous loss $L$ was:

\begin{Prop}\label{Prop:strprop}
Assume the notation of Definition~\ref{Def:convol}. The following hold:
\begin{itemize}
\item[(i)] If $L$ is proper, then so is $L\ast Z$.
\item[(ii)] Assume $L$ is strictly proper. Then, $L\ast Z$ is strictly proper if and only if the convolution operator $\ast p_Z$ is injective on mixed distributions.
\end{itemize}
\end{Prop}
\begin{proof}
(i) Let $Y$ be a mixed random variable t.v.i.~$\calY$ with distribution $p_Y$, let $p$ be any mixed distribution. By the definitions,
$$\EE[(L\ast Z)(p,Y)]= \EE[L(p\ast p_Z, Y+Z)],\quad\mbox{and}\;\EE[(L\ast Z)(p,Y)]= \EE[L(p_Y\ast p_Z, Y+Z)],$$
where $Z$ is independent of $Y$ and total expectations are taken over both on the RHS.\\
Since addition of random variables is convolution of densities, the density of $Y+Z$ is equal to $p_Y\ast p_Z$. Thus, by properness of $L$, it holds that
$$\EE[L(p\ast p_Z, Y+Z)]\ge \EE[L(p_Y\ast p_Z, Y+Z)].$$
Putting together all equalities and inequalities yields $\EE[(L\ast Z)(p,Y)] \ge \EE[(L\ast Z)(p_Y,Y)]$ which yields the claim since $p$ was arbitrary.\\
(ii) There are two directions to prove, both under the assumption that $L$ is strictly proper.\\
$\ast p_Y$ injective $\Rightarrow$ $L\ast Z$ strictly proper: Let $p$ be a distribution such that $\EE[(L\ast Z)(p,Y)] = \EE[(L\ast Z)(p_Y,Y)]$. It suffices to show that $p=p_Y$ since $p$ was arbitrary.. By definition, the equality implies
$$\EE[L(p\ast p_Z, Y + Z)] = \EE[(L\ast Z)(p,Y)] = \EE[(L\ast Z)(p_Y,Y)] = \EE[L(p_Y\ast p_Z, Y + Z)].$$
Since the distribution of $Y+Z$ is $p_Y\ast p_Z$, strict properness of $L$ implies that
$p_Y\ast p_Z = p\ast p_Z$. Injectivity of $\ast p_Z$ implies $p_Y=p$, which was the statement to prove.\\
$L\ast Z$ strictly proper $\Rightarrow$ $\ast p_Y$ injective: We show this by contraposition. Suppose $\ast p_Z$ is not injective, i.e., there are $p,q$ such that $p\neq q$ and $p\ast p_Z = q\ast p_Z$. Let $Y$ be a random variable with distribution $p_Y$. Then,
$\EE[L(p\ast p_Z, Y + Z)]=\EE[L(q\ast p_Z, Y + Z)]$ since $p\ast p_Z= q\ast p_Z$. But by definition, the LHS equals $\EE[(L\ast Z)(p,Y)],$ and the RHS equals $\EE[(L\ast Z)(q,Y)]$. Since $p$ is the distribution of $Y$ and $p\neq q$, this implies that $L\ast Z$ is not strictly proper, which was the statement to prove.
\end{proof}

We make two remarks about Proposition~\ref{Prop:strprop}: first, a natural question is whether examples of distributions whose convolution operator is injective exist. Indeed they do, examples are distributions which are at the same time characteristic, as a kernel functional, in the sense of~\cite{sriperumbudur2011universality}. For example, convolution with a Gaussian $Z$ (one of the running examples in~\cite{sriperumbudur2011universality}) is injective, as it corresponds to the heat diffusion kernel, which is universal\footnote{One more elementary way to show injectivity is through using that convolution in the original domain is the same as multiplication in the Fourier domain. As the Fourier transform of a Gaussian is Gaussian which is strictly positive, division by its Fourier transform is well-defined, hence the inverse is also well-defined as it may explicitly be obtained as $(\ast p_Z)^{-1}(f) = \calF^{-1}\left(\calF(f)/\calF(p_Z)\right)$.}. Second, one may ask whether $L\ast Z$ is not strictly proper if $L$ is not. This is not straightforward to answer, since a non-injective $\ast p_Z$ may cancel out the non-strictness in $L$, which is in turn seems to be the simplest characterization of this situation. Locus of non-injectivity and locus of non-properness need to agree, otherwise the arguments of Proposition~\ref{Prop:strprop}~(i) or~(ii) may be applied to the difference set.

Regarding algorithmic tractability, the only general way to evaluate $\EE[L(p\ast p_Z, y + Z)]$ appears to be numerical integration, for example of Monte Carlo type, methods which also come with their own uncertainty bounds (e.g., central limit theorem based), such as the adaptor algorithm presented later in Section~\ref{sec:adaptors.mixed} and natural extensions. Though, for specific $L$, specific $Z$, or specific forms of predicted distributions (such as empirical samples), less demanding and even analytic solutions may exist; however, the best practical approach does not seem clear. \\

Summarizing:\\
{\bf Advantages} of the convolution approach are the applicability to general (multi-variate real) $\calY$ and the ability to generalize any loss function for continuous distributions to the mixed distribution case.\\
{\bf Disadvantages} are the need to choose $Z$, and the lack of an analytical evaluation procedure for the convolved loss.

\subsection{Kernel discrepancy losses}
\label{sec:mixed.kernel}

Another option in the mixed distribution case are so-called kernel discrepancy losses first considered by~\citet{eaton1982method}.

\begin{Def}\label{Def:kernloss}
Let $k:\calY\times \calY\rightarrow \RR$ be a symmetric positive definite\footnote{Symmetric means: $k(y,y')=k(y',y)$ for all $y,y'\in\calY$. One possible definition of positive definiteness is: all possible symmetric kernel matrices formed with $k$, i.e., matrices of the form $K\in\RR^{N\times N}$ where $K_{ij}=k(y_i,y_j)$ for $y_1,\dots, y_N\in\calY$, are positive semi-definite. Note the discrepancy in nomenclature between a positive \emph{definite} kernel \emph{function} and a positive \emph{semi-definite} kernel \emph{matrix}, which is standard and uniform across literature.} kernel function.
We define a loss for mixed distributions in $\calP\subseteq \Distr(\calY)$ as follows:
$$L_k:\calP\times \calY \rightarrow \RR\;\quad (p,y)\mapsto -2 k(p,y) + k(p,p),$$
where for $\calP$-valued arguments, we define
$$k(p,y) := \int_\calY p(z) k(z,y) \;\diff z,\quad\mbox{and}\;k(p,p) := \int_{\calY^2} p(z) k(z,z') p(z') \;\diff z\;\diff z'.$$
We call $L_k$ the \emph{kernel discrepancy loss} (associated with the kernel $k$).
\end{Def}

Kernel discrepancy losses and convolved squared integrated losses are closely related. While in general, none is a special case of the other, there is a direct link between the two:

\begin{Lem}\label{Lem:kernconv}
Consider the squared integrated loss $L_{Gn}:\calP\times \calY\rightarrow \RR, (p,y)\mapsto -2 p(y) + \|p\|_2^2$ (see Definition~\ref{Def:losses}).\\
Let $Z$ be a continuous $\calY$ valued random variable with distribution function $p_Z:\calY\rightarrow \RR^+$.\\
Define $k_Z:\calY\times\calY\rightarrow \RR, (y,y')\mapsto \int_\calY p_Z(y)p_Z(\tau-y')\;\diff \tau$.\\
Then, $L_{Gn}\ast Z = L_{k_Z}$.
\end{Lem}

Lemma~\ref{Lem:kernconv} is interesting in two cases: in the case where $Z\sim \calN (0,\sigma^2)$ is Gaussian, so is $k_Z$ (also known as the squared exponential radial basis function kernel or the heat kernel). In the limiting case $\sigma\rightarrow 0$ of the convolution loss, or alternatively for taking the delta kernel $k(y,y') = \delta(y-y')$ with $\delta$ being the delta distribution in the kernel discrepancy loss, the squared integrated loss $L_{Gn}$ is recovered. However it should be noted that the limiting case of the squared integrated loss is, strictly speaking, not\footnote{unlike for example~\citet{dawid2007geometry} states in the respective section~6.3} a special case of the kernel discrepancy loss in Definition~\ref{Def:kernloss}, as the latter is a mixed distribution loss, whereas the former is a loss for continuous distribution, obtained as a limiting case but not as a special case (which would imply the squared integrated loss being valid for mixed distributions).

In literature, it is also sometimes incorrectly asserted\footnote{for example in section~6.3 of~\cite{dawid2007geometry}; interestingly, \citet{gneiting2007strictly} do not make the false statement in the parallel section~5.1.} that kernel discrepancy losses are strictly proper - this is untrue in general, for example for degenerate choices for $k$ such as constant kernels, e.g., $k:(y,y')\mapsto 42$ (which yields a loss that is proper but not strictly proper).
To characterize properness of kernel discrepancy losses, we introduce some results and concepts well-known in the kernel learning community. A foundational statement is the Moore-Aronszajn theorem:

\begin{Prop}\label{Prop:Moore-Aronszajn}
Let $k:\calY\times\calY\rightarrow \RR$ be a symmetric positive definite kernel function.\\
There is a map $\phi:\calY\rightarrow \calH$ where $\calH$ is a Hilbert space with inner product $\langle .,.\rangle_\calH :\calH\times \calH\rightarrow \RR$, such that $k(y,y') = \langle \phi(y),\phi(y')\rangle_\calH $ for all $y,y'\in\calY$.\\
The map $\phi$ and the Hilbert space $\calH$ are unique up to isomorphism, one possible choice has
$$\phi:y\mapsto k(y,.) \quad\mbox{and}\quad \left\langle k(y,.),k(y',.)]\right\rangle_\calH = k(y,y'),$$
where for $y\in\calY$, $k(y,.)$ is short-hand notation for the element $[z\mapsto k(y,z)]$ of $\calH$.
\end{Prop}
\begin{proof}
This is asserted as statement 2.4 of~\cite{aronszajn1950theory} which is usually given as a reference as it credits Moore, however the proof itself is only cited there and carried out in earlier work of~\citet{aronszajn1943theorie}. Another, more recent reference is the book of~\citet{scholkopf2002learning}.
\end{proof}

\begin{Def}
Keep the notation of Proposition~\ref{Prop:Moore-Aronszajn}.
The asserted properties of $\calH$ are called the \emph{reproducing property}. The Hilbert space $\calH$ the RKHS (reproducing kernel Hilbert space) of $k$, and the map $\phi$ is called the feature map of $k$.
\end{Def}

The correct notion corresponding to strict properness is characteristicness of the kernel, we follow~\citet{fukumizu2004dimensionality} and~\citet{sriperumbudur2008injective} for the definition:

\begin{Def}\label{Def:charac}
A symmetric positive definite kernel function $k:\calY\times\calY\rightarrow \RR$ with RKHS $\calH$ and feature map is called characteristic for $\calP\subseteq \Distr(\calY)$ if the so-called mean embedding map
$$\mu_k: \calP\rightarrow \calH,\; p\mapsto \EE_{Z\sim p}[\phi(Z)] = \EE_{Z\sim p}[k(Z,.)]\quad\mbox{ is injective.}$$
\end{Def}

We are now ready to state the properties of the kernel discrepancy loss.

\begin{Prop}\label{Prop:propkernel}
Assume setting and notation of Definition~\ref{Def:kernloss}. That is, consider a symmetric positive definite kernel function $k:\calY\times \calY\rightarrow \RR$ and the associated kernel discrepancy loss $L_k:\calP\times \calY \rightarrow \RR$ with $\calP\subseteq \Distr(\calY)$.\\
Then, it always holds that $L_k$ is a proper loss.\\
Furthermore, the following are equivalent:
\begin{itemize}
\item[(i)] $L_k$ is strictly proper.
\item[(ii)] $k$ is characteristic for $\calP$.
\end{itemize}
\end{Prop}
\begin{proof}
Let $\calH$ be the RKHS and $\phi$ the feature map for $k$. Let $p,p_Y\in\calP$ be any two distributions, let $Z,Z'\sim p, Y,Y'\sim p_Y$ be independent random variables. By definition of $L_k$ (Definition~\ref{Def:kernloss}) and $\mu_k$ (Definition~\ref{Def:charac}), one has
\begin{align*}
\EE[L_k(p,Y)] &= \EE[k(Z,Z') - 2 k(Y,Z)]\\
& = \EE\left[\left\langle\phi(Z),\phi(Z')\right\rangle_H - 2\left\langle\phi(Y),\phi(Z)\right\rangle_H\right]\\
&= \left\langle \EE\left[\phi(Z)\right],\EE\left[\phi(Z')\right]\right\rangle_H - 2\left\langle\EE\left[\phi(Y)\right],\EE\left[\phi(Z)\right]\right\rangle_H\\
&= \left\langle \mu_k(p), \mu_k(p)\right\rangle_H - 2\left\langle \mu_k(p_Y), \mu_k(p)\right\rangle_H\\
&= \left\langle \mu_k(p), \mu_k(p)\right\rangle_H - 2\left\langle \mu_k(p_Y), \mu_k(p)\right\rangle_H + \left\langle \mu_k(p_Y), \mu_k(p_Y)\right\rangle_H - \left\langle \mu_k(p_Y), \mu_k(p_Y)\right\rangle_H\\
 &= \left\langle \mu_k(p) - \mu_k(p_Y), \mu_k(p) - \mu_k(p_Y)\right\rangle_H  - \left\langle \mu_k(p_Y), \mu_k(p_Y)\right\rangle_H\\
&= \|\mu_k(p)-\mu_k(p_Y)\|_\calH^2 + \|\mu_k(p_Y)\|_\calH^2
\end{align*}
where the first equality is the definition of $L_k$; the second is using the reproducing property in Proposition~\ref{Prop:Moore-Aronszajn}; the third uses independence of $Y,Y',Z,Z'$ and linearity of $\EE$ and $\langle .,.\rangle_\calH$; the rest are algebraic manipulations,
where $\|.\|_\calH:\calH\rightarrow \RR^+$ is the canonical norm on $\calH$.
Since the second term $\|\mu_k(p_Y)\|_\calH^2$ does not depend on $p$, the expected loss $\EE[L_k(p,Y)]$ is minimized iff $\|\mu_k(p)-\mu_k(p_Y)\|_\calH^2 = 0$, which happens iff $\mu_k(p) = \mu_k(p_Y)$ since $\|.\|_\calH$ is a norm. In particular, $\EE[L_k(p,Y)]$ is minimized if $p=p_Y$, which proves properness of $L_k$.\\

We now prove the equivalence.\\
(ii)$\Rightarrow$ (i): as $L_k$ is strictly proper, $\EE[L(p,Y)]$ is minimized iff $p=p_Y$. By the above, $p=p_Y$ happens iff $\mu_k(p) = \mu_k(p_Y)$. Since both $p$ and $p_Y$ were arbitrarily chosen in $\calP$ which is the range of $\mu_k$, this is equivalent to stating that $\mu_k$ is injective, thus asserting that $k$ is characteristic for $\calP$.\\
(i)$\Rightarrow$ (ii): as $k$ is characteristic, $\mu_k(p) = \mu_k(p_Y)$ iff $p=p_Y$. Thus, by the above, $\EE[L_k(p,Y)]$ is minimized iff $p=p_Y$. This is equivalent to stating that $L_k$ is strictly proper, since both $p$ and $p_Y$ were arbitrarily chosen in $\calP$.
\end{proof}

In case there is an identification via Lemma~\ref{Lem:kernconv} of a convolution loss with a kernel discrepancy loss, say in the case of a Gaussian convolution and the Gaussian kernel, Propositions~\ref{Prop:strprop} and~\ref{Prop:propkernel} together relate injectivity of the convolution to characteristicness of the kernel.

Importantly, note that a number of popular choices for kernel functions, such as the Euclidean and polynomial kernels, are not injective, hence not characteristic, thus they don't give rise to a strictly proper kernel discrepancy loss.
On the other hand, possible choices\footnote{a proof is obtained either following~\cite{sriperumbudur2011universality} for characteristicness, or considering the convolution/Fourier transformation argument as in Section~\ref{sec:mixed.conv}} for a strictly proper kernel discrepancy loss are the Gaussian kernel $k(y,y') = \exp\left(-\frac{\|y-y'\|^2}{2\sigma^2}\right)$ and the Laplace kernel $k(y,y') = \exp\left(- \lambda \|y-y'\|\right)$. While these (and others) give rise to strictly proper losses, note that a kernel must be chosen, including a kernel parameter for which there does not seem to be a principled way of choosing.\\

More generally:\\
{\bf Advantages} of the kernel loss are the applicability to general (multi-variate real) $\calY$, and simple analytic form for predictions which are discrete/pure (e.g., empirical samples or posterior samples).\\
{\bf Disadvantages} are the need to choose a characteristic kernel, and the lack of an analytical evaluation procedure for the case of arbitrary mixed distributions.

\subsection{Realizing that all computing is discrete}
\label{sec:mixed.discrete}

It is a simple yet often overlooked fact that all arithmetic in computers is necessarily of finite accuracy, i.e., the set of representable states $y\in\calY$ is finite,
which, following this argument, one may assume without loss of generality.

Similarly, all integrals may be expressed as sums, though the main point which requires care is that the implementation enforces the sum-to-one constraint of probability densities.
We explain this in the example $\calY \subseteq \RR$, where there is a well-ordering on $\calY$, and due to finiteness of $\calY$ for each $y\in\calY$ a unique antecessor (= next smallest element) which
we will denote by $y^-$.
Thus, for a predicted $p\in \Distr (\RR)$, one may consider the left-continuous cumulative density function $F_p := z\mapsto  P(Z \le z)$ where $Z$ is a random variable with density $p$. We may obtain a distribution $p'\in \Distr (\calY)$ by setting $\Delta p(y):= F_p(y) - F_p(y^-)$ if $y$ is not the smallest element of $\calY$, and $\Delta p(y) := F_p(y)$ otherwise, yielding a function $\Delta: \Distr(\RR) \rightarrow \Distr(\calY)$.
The reduction to the discrete case is obtained by considering for any possibly occurring prediction functional $f\in[\calX \rightarrow \Distr(\RR)]$ the prediction functional
$\Delta\circ f \in [\calX\rightarrow \Distr(\calY)]$ instead.

Classical approximation results may now be invoked to justify this choice, but since the data are also necessarily supported in $\calY$, one may even argue
that at no point an approximation is actually made or is necessary to be made, hence no such results are necessary as a theoretical justification.\\

Unfortunately, most practical implementations of the above lead to numerical issues without a specific implementation of a probability type for mixed distributions. As an example, consider the case where one encodes distributions as probability mass functions on the $2^{64}\approx 1.6\cdot 10^{19}$ different numbers which may be represented by an IEEE 754 double, and the case of a mixed distribution with mass $1/2$ on 0, and otherwise being standard normal Gaussian. The successor of $0$ is the value of significant precision, ca.~$1.1\cdot 10^{-16}$, which will carry a mass of ca.~$0.5\cdot 10^{-16}$, which is below the value of significant precision of the double. While the precision issue may eventually be circumvented by smart application of the logarithm for the strictly local log-loss, it is less clear how the discretization principle would work for a strictly global loss such as the squared integrated loss (e.g., for computing the term $\|p\|_2^2$), since summation over $1.6\cdot 10^{19}$ evaluations is intractable.\\

Summarizing the above discussion:\\
{\bf Advantage} of the discretization approach is potential applicability to any $\calY$.\\
{\bf Disadvantage} is that it entails numerical issues for which the existence of a practicable solution is unclear.
(though if possible to solve the practical implementation issues, it might become the cleanest of all approaches)

\subsection{Decomposition into continuous and discrete part}
\label{sec:mixed.split}

Another approach is making use of the fact that mixed distributions, i.e., distributions without a singular part, may be separated into an absolutely continuous and a discrete part by the Lebesgue decomposition theorem. Prediction of these two parts may then be studied separately.

More precisely, any $\calY$-valued mixed random variable $Y$ with density $p_Y$ may be decomposed into:

\begin{itemize}
\item[(i)] the pure/discrete locus $\Pure (p_Y)$ of $p_Y$, which is a countable sub-set of $\calY$, and in most practically relevant cases finite;
\item[(ii)] a probability $\tau:= P(Y\in \Pure (p_Y))$ for observing a value in the pure locus;
\item[(iii)] a continuous random variable $Y_c$ t.v.i.~$\calY$, or equivalently an absolutely continuous distribution $Y_c$ on $\calY$;
\item[(iv)] a discrete random variable $Y_d$ t.v.i.~$\Pure (p_Y)$, or equivalently a positive probability $p_Y(y) := P(Y_d = y)$ for any $y\in \Pure (p_Y)$, with $\sum_{y\in \Pure (p_Y)} p_Y(y) = 1$
\end{itemize}

If one writes $Y_b$ for the random variable ``$Y\in\Pure (p_Y)$'', taking the value $1$ if it is and $0$ otherwise, it is directly seen that $Y_b$ is Bernoulli distributed with parameter $\tau$. The random variable $Y$ can be decomposed by conditioning on $Y_b$, in the sense that $Y|(Y_b = 1) = Y_d$ and $Y|(Y_b = 0) = Y_c$. This corresponds to a graphical model with three components, namely, $Y_b$, $Y_c$, and $Y_d$.

When comparing a prediction with an observation, both possibly mixed, goodness of the prediction may be assessed through each of the three components, more precisely:
\begin{itemize}
\item[(i)] whether the domain of $Y_d$, the pure locus $\Pure (p_Y)$ has been correctly identified
\item[(ii)] whether the parameter $\tau$ of $Y_b$ is close to the truth
\item[(iii)] whether $Y_c$ is a good prediction for the continuous part
\item[(iv)] whether $Y_d$ is a good prediction for the discrete part
\end{itemize}

In particular, if the pure locus $\Pure(p_Y)$ is known, meaning if task (i) is solved by an oracle or a-priori knowledge, models for tasks (ii), (iii), and (iv) may be estimated and assessed separately, and composite models may be obtained from combining any separate models solving the three tasks, and all of the discussion on models in the discrete or absolute continuous setting applies to the three tasks separately:

\begin{Prop}
Let $L_i:\calP_i\times \calY_i\rightarrow \RR, i\in\mbox{b,c,d}$ be convex loss functions, where:
$\calY_b = \{1,0\}$ and $\calP_b = \Distr(\calY_b)$; where $\calY_c = \calY$ and $\calP_c$ are the continuous distributions on $\calY$; and where $\calY_d$ is the (known) pure locus and $\calP_d = \Distr (\calY_d)$.\\
For a predicted distribution $p$, let $p_i,i\in\mbox{b,c,d}$ be the predicted distributions for the three components, as above. Let $\alpha_i\in \RR^+$ be positive.
$$\mbox{Define}\;L:(p,y)\mapsto \alpha_b L_b(p_b,[y\in \calY_d]) +
\left\{\begin{array}{cc} \alpha_c L_c(p_c, y)& \mbox{if}\;y\not\in\calY_d\\ \alpha_d L_d(p_d, y)& \mbox{if}\;y\in\calY_d\end{array} \right.$$
Considering $L$ as a loss for predicted and true distributions with fixed pure locus $\calY_d$, the following are true:
\begin{itemize}
\item[(i)] $L$ is convex.
\item[(ii)] If all $L_i,i\in\mbox{b,c,d}$ are proper, then so is $L$.
\item[(iii)] If all $L_i,i\in\mbox{b,c,d}$ are strictly proper, then so is $L$.
\end{itemize}
\end{Prop}
\begin{proof}
We fix a ``true'' random variable $Y$ with component variables $(Y_b,Y_c,Y_d)$, where $Y_b=\Bin(\tau)$, i.e., $P(Y\in \calY_d)= \tau$. Also consider a ``predicted'' distribution $p$ with component distributions $(p_b,p_c,p_d)$. By definition,
$$\EE[L(p,Y)] = \alpha_b L_b(p_b,Y_b) + (1-\tau)\alpha_c L_c(p_c,Y_c) + \tau\alpha_d L_d(p_d,Y_d).$$
(i) Let $P=(P_b,P_c,P_d)$ be a random variable taking values in distributions with pure locus $\calY_d$, decomposed such as above. From the definition, note that the decomposition of $\EE[P]$ is $(\EE [P_b],\EE [P_c],\EE [P_d])$. By definition and combining inequalities, positive linear combinations of convex losses are hence again convex.\\
(ii) Let $p_Y$ be the distribution of $Y$, with component distributions $(p_{Y,b},p_{Y,c},p_{Y,d})$. Properness of the $L_i$ implies that $p_{Y,i}$ is a minimizer of $p_i\mapsto \EE[L_i(p_i,Y_i)]$. Thus, $p$ is a minimizer of $p\mapsto \EE[L(p,Y)]$ as it is a positive linear combination of the $\EE[L_i(p_i,Y_i)]$ (with fixed coefficients, since $\tau$ does not depend on $p$ but on $p_Y$).\\
(iii) Let $p_Y$ be the distribution of $Y$, with component distributions $(p_{Y,b},p_{Y,c},p_{Y,d})$. Strict properness of the $L_i$ implies that $p_{Y,i}$ is the unique minimizer of $p_i\mapsto \EE[L_i(p_i,Y_i)]$. Thus, $p$ is the unique minimizer of $p\mapsto \EE[L(p,Y)]$ as it is a positive linear combination of the $\EE[L_i(p_i,Y_i)]$.
\end{proof}

Unfortunately, the pure locus of the true distribution $Y$ is in general unknown, unless in some practically relevant cases where it is clear through the predictive setting, e.g., if predictions are to be made on the interval $\calY = [0,1]$ and only the two boundaries $0,1$ may occur with positive discrete probability.\\
In the general case where the pure locus is unknown, it is not clear how assessment or estimation is to be achieved. The resultant problem is the same as the issue discussed in the introduction to this section, where in a putative training data set a sample from the pure locus may occur only once and re-appear in the test data set with positive probability. This allows for construction of pathological learning strategies with expected out-of-sample loss that is not bounded away from minus infinity (see Section~\ref{sec:oos}), unless the case of ``double observation'' is sensibly addressed. While a number of straightforward strategies we can think of seem to take immediate care of this, they seem rather like heuristics and even less principled than the full discretization approach presented in Section~\ref{sec:mixed.discrete}.\\

Summarizing the above discussion:\\
{\bf Advantages} of the decomposition approach include general applicability and a direct link to two well-understood cases, the discrete and continuous one.\\
{\bf Disadvantages} are the arbitrary choice of the coefficients $\alpha_i$, and the lack of applicability as-is to the case where the pure/discrete part of the true distribution is unknown.
(though if possible to solve this theoretical issue, it might become the cleanest of all approaches)

\subsection{Mixed-to-continuous adaptors}

A relatively simple solution to the problem of mixed distribution is to convert any prediction which is mixed into a prediction that is easier to treat, usually an absolutely continuous one.

Using notation of Section~\ref{sec:mixed.conv}, one may for example consider
$$\tilde{L}(p,y) = L(p\ast p_Z,y),$$
i.e., converting a mixed prediction $p$ into an absolutely continuous prediction $p\ast p_Z$.

However, the main thing to note is that $\tilde{L}$ will, in general, not be a proper or even strictly proper loss.\\
Thus, considering the convolution procedure above as the definition of a new loss $\tilde{L}$ is problematic. However, understanding the procedure as a modification to or a component of a prediction strategy producing $p$ resp.~$p\ast p_Z$ is perfectly valid. Though this means that as an approach it is best understood in the context of composite prediction strategies rather than in the context of losses for the mixed case, hence we defer the discussion to Section~\ref{sec:meta-algorithms} were meta-strategies for probabilistic prediction, including wrapped and composite, will be considered.\\

Summarizing the above discussion:\\
{\bf Advantage} of the adaptor approach is that it reduces evaluation of mixed density prediction to evaluation of continuous density prediction.\\
{\bf Disadvantage} is that this approach wraps a mixed strategy inside a continuous one which is whose goodness is assessed, instead of the wrapped one which is assessed indirectly.

\subsection{Non-mixed distributions}

A mathematically inclined reader will notice that the class of mixed distributions is still restrictive compared to all distributions, since not all distributions are mixed. It is hence a-priori unclear whether there is a yet more general and relevant case, that of general distributions, which needs to be studied as well. However, we claim that while it might be mathematically interesting, studying the more general case of all distributions (on an arbitrary domain) could be practically irrelevant, due to the following:

By Lebesgue's decomposition theorem, on any (measurable) domain $\calY$, the predicted label distributions may be decomposed in an absolutely continuous measure, a discrete/pure measure, and a singular measure. Empirical estimation or inference with singular measures appears to be a widely open topic, with unclear benefit in predictive practice; similarly, while absolutely continuous and discrete measures are well understood, there is to our knowledge no classification or explicit characterization of singular measures.

Thus, assuming the absence of a singular part may be sensible as long as there is no general theory of estimation for singular measures or a practical use to it.

\newpage
\section{Learning theory and model diagnostics}\label{sec:ltheory}

This section presents a number of theoretical statements on model performance and learning, establishing parallels to the deterministic setting and information theoretical results:

\begin{itemize}
\item[(i)] a probabilistic bias-variance trade-off, extending the trade-off(s) found in the deterministic setting
\item[(ii)] an information theoretical description of model performance, including results which equate unpredictability to statistical independence,
\item[(iii)] a discussion of residuals in the probabilistic setting.
\end{itemize}

\subsection{Variations on Jensen}

Before proving results in learning theory, we present some elementary yet helpful variations on Jensen's inequality for the case of conditioning. The statements in Lemma~\ref{Lem:geomar} are straightforward consequences, but we state them for easy reference.

\begin{Lem}
\label{Lem:geomar}
Let $X$ be a univariate real random variable taking only non-negative values. Let $\phi:\RR\rightarrow \RR$ be a convex function\footnote{For example, $[x\mapsto L(x,y)]$ where $x$ is a convex loss.}. Then:
\begin{itemize}
\item[(i)] $\phi [\EE (X)] \le \EE [\phi(X)].$
\item[(ii)] $\phi \left(\EE_{X|Z} [X|Z]\right) \le \EE_{X|Z} [\phi(X)|Z]$ for any random variable $Z$ t.v.i.~$\calZ$.
\item[(iii)] $\phi [\EE (X)] \le \EE_Z \left[\phi \left(\EE_{X|Z} [X|Z]\right)\right] \le \EE [\phi(X)]$ for any random variable $Z$ t.v.i.~$\calZ$.
\end{itemize}
Furthermore, if $\phi$ is in addition strictly convex, then:
\begin{itemize}
\item[(i')] Equality holds in (i) if and only if $X$ is constant.
\item[(ii')] Equality holds in (ii) if and only if $X = f(Z)$ for some function $f:\calZ\rightarrow \calX$.
\item[(iii')] The left inequality in (iii) is an equality if and only if $X = f(Z)$ for some function $f:\calZ\rightarrow \calX$ (defined $Z$-almost-everywhere). The right inequality in (iii) is an equality if and only if $X$ and $Z$ are statistically independent.
\end{itemize}
\end{Lem}
\begin{proof}
(i) and (i') is the usual statement of Jensen's inequality.\\
(ii) follows directly from (i) by conditioning on $Z$. More precisely, (i) implies (ii) for conditioning on any measurable $\calZ'\subseteq \calZ$, i.e., $\EE [\phi(X)|Z\in\calZ'] \le \phi \left(\EE [X|Z\in \calZ']\right)$. Since $\calZ'$ is arbitrary, (ii) follows.\\
(ii') follows from (i'): in case $Z$ is discrete, the proof is not technical: (i') implies that $\EE [\phi(X)|Z=z] = \phi \left(\EE [X|Z=z]\right)$ if and only if $X|Z=z$ is constant, i.e., takes value $f(z)$. For mixed or continuous $Z$, the proof logic is the same, though requires a full measure theoretic approach (which we do not carry out here).\\
(iii) the left inequality follows from taking total expectations (over $Z$) in (ii). The right inequality follows from applying (i) to the random variable $\EE_{X|Z} (X|Z)$, which is a function of $Z$.\\
(iii') for the left inequality follows from noting that equality holds if and only if the inequality in (ii) holds $Z$-almost always, and from (ii').\\
(iii') for the right inequality follows from noting that by (i'), equality on the right hand side holds if and only if this function is constant, which is equal to all conditionals $X|Z\in \calZ', \calZ'\subseteq \calZ$ being equal, which is equivalent to statistical independence of $X$ and $Z$.
\end{proof}

We continue with conditional variants of Definition~\ref{Def:detraincond} and Lemma~\ref{Lem:functionjensen} for use in marginalization statements.

\begin{Def} \label{Def:detraincond}
Let $f$ be a prediction strategy. We define $[\EE f] : x\mapsto \EE[ f(x)].$\\
For any random variable $Z$, we define $[\EE f |Z] : x\mapsto \EE[ f(x) |Z].$
\end{Def}

\begin{Lem}\label{Lem:functionjensencond}
Let $L$ be a convex loss, let $Z$ be a random variable.
\begin{itemize}
\item[(i.a)] It holds that $\varepsilon_L(\EE f) \le \varepsilon_L(f)$.
\item[(i.b)] If $L$ is strictly convex, $\varepsilon_L(\EE f) = \varepsilon_L(f)$ if and only if $f$ is a prediction functional.
\item[(ii.a)] It holds that $\varepsilon_L(\EE f|Z) \le \varepsilon_L(f)$.
\item[(ii.b)] If $L$ is strictly convex, $\varepsilon_L(\EE f|Z) = \varepsilon_L(f|Z)$ if and only if $f$ is a functional with parameter $Z$, i.e., ($Z$-almost) all conditionals $f|Z=z$ are constant prediction functionals.
\end{itemize}
\end{Lem}
\begin{proof}
All statements are direct applications of Lemma~\ref{Lem:geomar}.
\end{proof}

\subsection{Bayesics on predictions and posteriors}

We discuss some advanced topics related to Bayesian posteriors.
More precisely, in the Bayesian paradigm, prediction and inference results take the form of a (Bayesian belief) posterior distribution. Though in the predictive setting, it is not necessarily clear what the most appropriate choice for the object over which the distribution is taken would be in the probabilistic supervised setting. Namely, we consider the following issues:
\begin{itemize}
\item[(i)] When making a prediction for a label point $x\in \calX$, should the prediction be a posterior distribution over elements of $\calY$, that is, an element of $\Distr (\calY)$, usually called the ``predictive posterior''? Or should it be a posterior distribution over predicted distributions in $\Distr (\calY)$, that is, an element of $\Distr(\Distr(\calY))$?
\item[(ii)] When predicting labels for \emph{multiple} test points: should the predictive posterior be joint, i.e., over all test points and an element of $\Distr(\calY^N)$, or should there be one predictive posterior per test point (a marginal posterior), so all predictions together are an element of $\Distr(\calY)^N$?
\item[(iii)] Combining the above, one may also ask whether in the case of multiple test points, one should consider as Bayesian predictions elements of $\Distr(\calY^N)$, $\Distr(\calY)^N$,
    $\Distr(\Distr(\calY)^N),$ or $\Distr(\Distr(\calY))^N.$
\end{itemize}

Under the assumption that one accepts that external evaluation by proper losses is a reasonable quantifier of predictive quality, we argue, by a line of reasoning  similar to  Section~\ref{sec:classprob-badidea}, that for the purpose of prediction, marginal posteriors in $\Distr (\calY)$ are the choice which is empirically validable - i.e., we claim that the answer to question (iii) is to only consider (indepedent) predictions $\Distr(\calY)^N$. In particular, no posteriors over distributions in (i), or joint test set posteriors in (ii), should be considered.

It is noteworthy that answers to these are not uniform across Bayesian literature. For example, (i) is answered differently, depending on whether the prediction is for a classification task ($\calY$ is discrete), or for a regression task ($\calY = \RR$). In the classification case, one may find posterior distributions over the predicted probability, i.e., in essence a posterior over possible values of the conditional law $\calL (Y|X=x)$ parameterized through the probabilities $P(Y = c)$ for all $c\in \calY$. The posterior is hence an element of $\Distr(\Distr(\calY))$. On the other hand, in the regression case, the ``(belief) distribution of (conditional) distributions'' which is more natural in the Bayesian set-up - but usually not properly expressible in common Bayesian notation - is usually integrated out to yield the ``predictive posterior''. Mathematically, instead of the law of a random variable $P$ t.v.in $\calP\subseteq \Distr(\calY)$, which is an element of $\Distr(\Distr(\calY))$, the law of its expectation $\EE [P]$, an element of $\Distr(\calY)$, is considered. Bayesian notation, in a semi-parametric case, usually denotes the latter law as $\int p(y|\theta)p(\theta)\;\diff \theta$, while the former law is not notationally expressible (it would be the push-forward of $p(\theta)$ onto the set of all $p(.|\theta)$, which in our notation is $\calP$).

We proceed describing the line of reasoning leading to our claims regarding points (i) - (iii). We note that the answer to (iii) already follows from arguing (i) - no posterior distributions over predictive distributions - and (ii) - no joints over test point predictions. Our key assumption is that expected loss - rather than a posterior over the predicted loss - is the main endpoint of evaluation, as in Section~\ref{sec:modelvalidation}. One may of course dispute this assumption.

The first fact, i.e., that no posteriors over predictive distributions should be predicted follows directly from Corollary~\ref{Cor:bagging}, or Lemma~\ref{Lem:functionjensencond}~(ii.b) where $Z$ is taken to be the posterior over predicted distributions. It can also be seen to directly follow from the definition of convexity~\ref{Def:proper}. Namely, all of these statements imply, in slightly different form, that for a convex loss, predicting a single, fixed distribution will have lower or equal expected loss than predicting a posterior over losses and then obtaining the expected loss by integration over the posterior of the loss.

The second fact, i.e., that predictions of joint probabilities of test labels are dominated, in expectation, by predictions of independent probabilities of test labels, may be derived as an application of Proposition~\ref{Prop:marginal}. We would like to point out that this statement is specific to the log-loss.

\begin{Prop}
\label{Prop:marginal}

Consider a joint probabilistic prediction functional $f:\calX^N\rightarrow \Distr (\calY^N)$ and the corresponding marginal/independent prediction functional
$$f^\amalg: \calX^N\rightarrow \Distr (\calY)^N,\quad (x_1,\dots, x_N)\mapsto \prod_{i=1}^N\left(\int f(x_1,\dots, x_N) \;\diff x_1 \dots \diff x_{i-1} \diff x_{i+1}\dots \diff x_N\right).$$
In particular, assume that $f$ is independent of the data $(X_i,Y_i),i=1\dots N$.
Then, for the expected losses, it holds that
$$\EE\left[-\log f^\amalg(X_1,\dots, X_N)(Y_1,\dots, Y_N)\right] \le \EE\left[-\log f(X_1,\dots, X_N)(Y_1,\dots, Y_N)\right]$$
where $(X_1,Y_1),\dots,(X_N,Y_N)\sim (X,Y)$ i.i.d.~(as in the usual set-up).
\end{Prop}
\begin{proof}
Define notation for the set of test features $X_* := \{X_1,\dots, X_N\}$, test labels $Y_*:= \{Y_1,\dots, Y_N\}$, and the evaluation of $f$ at $X_*$, that is $P:= f(X_1,\dots, X_N) | X_*$, where conditioning is over $X_*$, equivalently over all $X_i$.
Define notation for $X_*$-conditional marginals and argument-conditional marginals of $P$ as follows:
\begin{align*}
P(Y_1,\dots, Y_{i-1}) & := \EE_{Y_i,\dots, Y_N|X_*} P(Y_1,\dots, Y_N)\\
P(Y_i|Y_1,\dots, Y_{i-1}) & := P(Y_1,\dots, Y_i|X_*)/P(Y_1,\dots, Y_{i-1}),\quad\mbox{or, more generally,}\\
P(Y_i,i\in I) & := \EE_{Y_i,i\not\in I|X_*} P(Y_1,\dots, Y_N)\\
P(Y_i,i\in I|Y_j,j\in J) & := P(Y_i,i\in I\cup J|X_*)/P(Y_j,j\in J)
\end{align*}
Not that all expectations are conditional on the set of test features $Z$.
Note that, by definition, one has
\begin{equation}\label{eqn:LHS}
\log f^\amalg(X_1,\dots, X_N)(Y_1,\dots, Y_N) = \sum_{i=1}^N \log P(Y_i)
\end{equation}
which is, up to the expectation and sign, the left hand side of the expression to prove.
Also note, that, by definition, it holds that
\begin{equation}\label{eqn:RHS}
\log f(X_1,\dots, X_N)(Y_1,\dots, Y_N) = \sum_{i=1}^N \log P(Y_i|Y_1,\dots, Y_{i-1}),
\end{equation}
which is, up to the expectation and sign, the right hand side.
Note that both expressions are conditional on $X_*$ by the above convention.

Now note that by the marginal-conditional-correspondence,
$$P(Y_i) = \EE_{Y_1,\dots, Y_{i-1}|X_*} P(Y_i|Y_1,\dots, Y_{i-1}).$$
hence, by Lemma~\ref{Lem:geomar}~(ii) we may conclude that
$$\log P(Y_i) \ge \EE_{Y_1,\dots, Y_{i-1}|X_*} \left[\log P(Y_i|Y_1,\dots, Y_{i-1})\right]$$
Summing over $i$, we therefore obtain that
$$\sum_{i=1}^N \log P(Y_i) \ge \sum_{i=1}^N \EE_{Y_1,\dots, Y_{i-1}|X_*} \left[\log P(Y_i|Y_1,\dots, Y_{i-1})\right],$$
and by taking further expectation over $Y_*$, we obtain that
$$\EE_{Y_*|X_*}\left[\sum_{i=1}^N \log P(Y_i)\right] \ge  \EE_{Y_*|X_*}\left[\sum_{i=1}^N \log P(Y_i|Y_1,\dots, Y_{i-1})\right],$$
and after flipping the sign and taking further expectation over $X_*$, the claim is obtained directly from substituting
Equations~\ref{eqn:LHS} and~\ref{eqn:RHS}.
\end{proof}

We would like to remark: Proposition~\ref{Prop:marginal} implies that it is sufficient to consider prediction strategies of the type $f: \calX^N\rightarrow \Distr (\calY)^N$ instead of strategies $f: \calX^N\rightarrow \Distr \left(\calY^N\right).$ While this probably cannot be further reduced to the situation $f:\calX\rightarrow \calY$
without additional assumptions, it is unclear how to obtain reliable generalization estimates for a function of type $f: \calX^N\rightarrow \Distr (\calY)^N$ that looks at all test data samples at once. Considering a situation similar to the permutation baseline counterexample in Section~\ref{Sec:badbaseline} may lead to the conjecture that allowing a method to look at too many (e.g., more than one) test data points at once may be form of ``cheating'', at least when it comes to evaluating the goodness of such a method.\\

It also should be pointed out that Proposition~\ref{Prop:marginal} covers only the setting in which the log-loss (the only strictly local strictly proper loss) is used for evaluating performance.
While unclear what the corresponding statement for other losses would be (potentially giving rise to Stein-like paradoxes), this is already due to the simple issue that it is unclear how a general loss would be defined or a sample of feature-label pairs as opposed to a single point. As long as the definition implies~\ref{eqn:RHS} in the proof (which is natural due to the likelihood interpretation of the log-loss), the straightforward generalization of Proposition~\ref{Prop:marginal} still holds.

\subsection{Bias and variance in the probabilistic setting}
\label{Sec:biasvariance}

An elementary yet central motif in the theory of classical supervised learning is the well-known bias-variance trade-off.

To briefly recapitulate in our notation, the classical bias-variance-trade-off decomposes the expected squared loss (i.e., for the deterministic loss function $L(\widehat{y},y) = (\widehat{y} - y)^2)$ of a prediction functional $f:\calX \rightarrow \calY$, or a prediction strategy $f$ that is a $[\calX\rightarrow\calY]$-valued random variable.\\
It is usually also assumed that the generative process $(X,Y)$ that takes values in $\calX\times \calY$ corresponds to a homoscedastic noise process, i.e., that $Y = \varpi(X) + \varepsilon$, where $\varpi:\calX\rightarrow\calY$ is a ``true'' functional, the random variable $\varepsilon$ is centered noise, and $X,\varepsilon, f$ are mutually independent (= independent noise, independent test set). The usual statement of the classical bias-variance-trade-off (valid for both prediction functionals and prediction strategies $f$) is as follows:

\begin{Prop}
\label{Prop:BVclass}
In the classical supervised learning setting as above, it holds that
$$\varepsilon(f)=\EE \left[L_{sq}(f(X),Y)\middle|X\right] =  \Bias(f|X)^2 + \Var (f|X) + \Var(\varepsilon),$$
where the expressions on the right hand side are defined as follows:
\begin{align*}
L_{sq}(\widehat{y},y) &= (\widehat{y} - y)^2\\
\Bias(f|X) &= \EE[ f(X)|X] - \EE[f(X)|X]\\
\Var(f|X) &= \EE[f(X)^2|X]- \EE[f(X)|X]^2,
\end{align*}
and all expectations are taken with respect to both $f$ and the conditional random variable $Y|X$.
All summands on the right hand side are non-negative numbers.
Furthermore, it holds that $\varepsilon(f) = \Var(\varepsilon)$ if and only if $f(X)=\varpi(X)$ almost always, i.e., the prediction is perfect.
In particular, the above is a proper additive decomposition of the expected generalization squared loss.
\end{Prop}
\begin{proof}
This follows from substitution of definitions, multiplying out, and using linearity of expectation (on both sides).
\end{proof}

We establish a form in which this generalizes to our probabilistic setting where we drop any assumption above and return to the full probabilistic setting in Section~\ref{sec:propreset}. In that case, we wish to decompose the expected generalization loss of a prediction strategy $f$. No further assumptions, e.g., on the data generative process, are made.

\begin{Def}\label{Def:biasvar}
Consider the probabilistic supervised learning setting with:\\
a convex loss $L:\calP\times\calY\rightarrow \RR, \calP\subseteq \Distr(\calY),$
a data generative process $(X,Y)$ t.v.i.~$\calX\times \calY$,\\
a prediction strategy $f$ t.v.i.~$[\calX\rightarrow\Distr(\calY)].$\\
We define quantities
\begin{align*}
\varepsilon_L(f|X) &= \EE\left[L(f(X),Y)\middle|X \right]\\
\Var_L(f|X) &=  \EE\left[ L(f(X),Y) \middle|X \right] - \EE\left[L([\EE f](X),Y) \middle|X \right],\quad \mbox{with}\;\EE f\;\mbox{as in Definition~\ref{Def:detrain}}\\
\Bias_L(f|X) &= \EE\left[L([\EE f](X),Y) \middle|X \right] - \EE\left[L(\varpi_{Y|X}(X),Y) \middle|X \right]\\
\DBias_L(f|X) &= \EE\left[L([\EE f](X),Y) \middle|X \right] - \EE\left[L([\EE f](X),Y - \alpha) \middle|X \right]\\
\PBias_L(f|X) &= \EE\left[L([\EE f](X),Y - \alpha) \middle|X \right] - \EE\left[L(\varpi_{Y|X}(X),Y) \middle|X \right]\\
\Err_L (Y|X) &:= \EE_{Y|X} \left[L(\varpi_{Y|X}(X),Y) \middle| X \right]\\
\mbox{where:} & \alpha = \argmin_{\widehat{\alpha}\in\RR}  \EE\left[L([\EE f](X),Y - \widehat{\alpha})\right]\\
&  \varpi_{Y|X} = [x\mapsto \calL(Y|X=x)],\quad\mbox{i.e., the ``true'' density}
\end{align*}
where all generalization errors are conditional on the feature variable $X$, and computed w.r.t.~the joint data generating process $(X,Y)$.\\
Note that, by definition, $\Bias_L(f|X) = \DBias_L(f|X) + \PBias_L(f|X),$ and
$$\varepsilon_L(f|X) = \Bias_L(f|X) + \Var_L(f|X) + \Err_L(Y|X),$$
hence the quantities induce a trade-off decomposition of the $X$-conditional expected generalization error.\\
We call the defined random variables (all deterministic functions of $X$):
\begin{align*}
\Err_L(Y|X) &\quad \mbox{the irreducible generalization error at $X$},\\
\Var_L(f|X) &\quad \mbox{the prediction ($L$-)variance of $f$ at $X$},\\
\Bias_L(f|X) &\quad \mbox{the prediction ($L$-)bias of $f$ at $X$},\\
\DBias_L(f|X) &\quad \mbox{the location specitic part of } \Bias_L(f|X),\\
\PBias_L(f|X) &\quad \mbox{the shape specific part of } \Bias_L(f|X).
\end{align*}
We also define variants $\Err_L(Y|X=x) = \EE\left[L(f(X),Y)\middle|X=x \right]$,\\
$\Bias_L(f) = \EE\left[L([\EE f](X),Y) \middle|X =x\right] - \EE\left[L(\varpi_{Y|X}(X),Y) \middle|X =x\right],$ etcetera, where the conditioning is explicit on an input $x\in\calX$.\\
Furthermore, we define unconditional variants $\Err_L(Y/X) = \EE[\Err_L(Y|X)]$, $\Bias_L(f) = \EE[\Bias_L(f|X)],$ etc., for which in the naming the qualifier ``at $X$'' is omitted (while the dependence on $L,X,Y$ and possibly $f$ remains)
\end{Def}

Note that the two-step definition of the expected generalization error $\varepsilon_L(f)$ (through conditioning on $X$ and then taking total expectations) coincides with our previous definition, due to the law of iterated expectation.

The bias term has a well-known form for the log-loss (aka cross-entropy loss) and the squared integrated loss, closely related to their names:

\begin{Ex}\label{Ex:BiasDiv}
Let $f:\calX\rightarrow \Distr (\calY)$ be a prediction functional (e.g., $f=\EE g$ for a prediction strategy $g$). Then:
\begin{itemize}
\item[(i)] For the log-loss $L: (p,y)\mapsto -\log p(y)$, it holds that\\
$\Bias_L(f|X) = D_{KL}(f(X)\|\varpi_{Y|X}(X)),$ where $D_{KL}(p\| q) := \int_\calY p(y) \log\frac{p(y)}{q(y)}\;\diff y$ is the Kullback-Leibler divergence.

\item[(ii)] For the squared integrated loss $L:(p,y) \mapsto -2 p(y) + \|p|_2^2$, it holds that\\
$\Bias_L(f|X) = \|f(X) - \varpi_{Y|X}(X))\|^2_2,$ where $\|p\|_2^2 := \int_{\calY} p(y)^2\;\diff y$ is the $\mbox{L}^2$-norm of functions.
\end{itemize}
In the more general terminology of~\cite{dawid2007geometry}, $\Bias_L(f|X)$ is the ($L-$)divergence function of $f(X)$ w.r.t.~$\varpi_{Y|X}(X)$, and $\Bias_L(f)$ is its expectation over $X$, which conversely could also be taken as a definition of a ``conditional divergence function'' for the case of a general loss.
\end{Ex}

Important properties from the classical bias-variance trade-off transfer to the probabilistic case:

\begin{Prop}
\label{Prop:BVprob}
Consider the probabilistic supervised learning setting with:\\
a convex loss $L:\calP\times\calY\rightarrow \RR, \calP\subseteq \Distr(\calY),$
a data generative process $(X,Y)$ t.v.i.~$\calX\times \calY$,\\
a prediction strategy $f$ taking values in $[\calX\rightarrow\Distr(\calY)].$\\
Consider the bias-variance decomposition in Definition~\ref{Def:biasvar}:
$$\varepsilon_L(f|X)=\EE \left[L(f(X),Y)|X\right] =  \PBias(f|X) + \DBias(f|X) + \Var (f|X) + \Err_L (Y/X).$$
The right hand side expression of the trade-off have the following properties:
\begin{itemize}
\item[(i.a)] All summands on the RHS which depend on $f$, i.e., $\DBias_L(f|X),\PBias_L(f|X), \Var_L(f|X)$, are non-negative random variables.
\item[(i.b)] In particular, $\DBias_L(f),\PBias_L(f), \Var_L(f)$ are non-negative numbers.
\item[(ii.a)] $\varepsilon_L(f|X=x) = \Err_L (Y|X=x)$ if $f(x)=\varpi_{Y|X}(x)$, i.e., if the prediction at $x$ is perfect.
\item[(ii.b)] If $L$ is strictly proper, then $\varepsilon_L(f|X=x) = \Err_L (Y|X=x)$ if and only if $f(x)=\varpi_{Y|X}(x)$.
\item[(iii.a)] $\varepsilon(f) = \Err_L (Y/X)$ if $f=\varpi_{Y|X}$ ($X$-almost everywhere).
\item[(iii.b)] If $L$ is strictly proper, then $\varepsilon(f) = \Err_L (Y/X)$ if and only if $f=\varpi_{Y|X}$ ($X$-almost everywhere).
\item[(iv.a)] $\Var_L(f) = 0$ if $f$ is a prediction functional (i.e., a constant function-valued random variable).
\item[(iv.b)] If $L$ is strictly convex, then $\Var_L(f) = 0$ if an only $f$ is a prediction functional.
\end{itemize}
In particular, the above is a proper decomposition of the expected loss (with desirable minimality properties).
\end{Prop}
\begin{proof}
Substitution and summation on the right hand side immediately yields the left hand side (by linearity of expectation).\\

(i.a) The claim for $\Var$ follows from Lemma~\ref{Lem:geomar}~(iii), applied to conditioning on $X$.\\
The claim for $\DBias$ follows from the definition of $\alpha$ and $\argmin$.\\
The claim for $\PBias$ follows from properness of $L$, noting that $\varpi_{Y|X}(X)$ is the law of $Y|X$. (i.b) is a direct consequence, by taking total expectations.\\

(ii.a),(ii.b) follow directly from (strict) properness of $L$, noting that due to the conditioning, the first arguments are all fixed elements of $\Distr(\calY)$.\\

(iii.a) is directly implied by (ii.a), integrating over $x$.\\
We prove (iii.b) by contraposition. Assume $f(x)\neq\varpi_{Y|X}(x)$ for all $x\in\calU$ with an $\calU\subseteq \calX$ of positive probability under $X$. Then, by (ii.b), $\varepsilon_L(f|X\in\calU)\neq \Err_L(Y|X\in\calU)$, and by (i.a), $\varepsilon_L(f|X\in\calU) - \Err_L(Y|X\in\calU)\gneq 0$. Again using (i.a), this implies $\varepsilon_L(f|X) - \Err_L(Y/X)\gneq 0$ which was the negated supposition to prove.\\

(iv.a) and (iv.b) are a rephrasing of Proposition~\ref{Lem:functionjensen}, by substituting definitions.
\end{proof}

Analogies between the probabilistic bias-variance trade-off given by Proposition~\ref{Prop:BVprob} and the classical one in Proposition~\ref{Prop:BVclass} may be obtained in the following ways:

First, substitution of a suitable parametric Gaussian density as the output of $f$, as in section~\ref{sec:classprob}, recovers the classical decomposition (up to scaling and constants, as well as the square in the bias term).

Second, noting that in both cases, the variance terms $\Var$ are characterized as being independent of the ground truth (true labels or true distribution); the bias terms are characterized by not depending on the randomness in the prediction strategy $f$ (e.g., the training set), but on $f$ only through its expectation $\EE f$ which is a prediction functional. Both trade-offs have a bias term that quantify dislocation, or additive distance of the average prediction from the ground truth, $\PBias$ and $\DBias$. Both trade-offs have a ``minimum achievable expected error'' $\Err$ corresponding to a perfect prediction, which is the generative noise term $\Var(\varepsilon)$ in the classical setting; in the probabilistic setting, the minimum error term corresponds to a conditional entropic quantity, as it will be shown in Lemma~\ref{Lem:Entpre}~(iv) below. However, the additional term $\PBias$, which quantifies the mismatch in the shape of the predicted density with the shape of the true density, appears only in, and is thus characteristic for, the probabilistic setting.

An important practical consequence of Proposition~\ref{Prop:BVprob} is Corollary~\ref{Cor:posterior} below which is the probabilistic equivalent of the classical variance reduction lemma.\\
For optimal readability in frequentist and Bayesian contexts alike, we give two (up to conditioning) equivalent formulations of the same statement.

\begin{Cor}
\label{Cor:posterior}
Consider the probabilistic supervised learning setting, as in Proposition~\ref{Prop:BVprob}.
Let $f$ be a prediction strategy t.v.i.~$[\calX\rightarrow \Distr (\calY)]$.
\begin{itemize}
\item[(i)] It holds that
$$\Var_L(\EE f|X) \le \Var_L(f|X),\; \Bias_L(\EE f|X) = \Bias_L(f|X),\;\mbox{in particular}\;\varepsilon_L(\EE f)\le \varepsilon_L(f).$$
\item[(ii)] Let $\theta$ be any random variable, let $\calD$ be the training data on which $f$ depends.
Let $f_{\setminus \theta} := \EE_{\theta|\calD}[f|\calD]$, i.e.~$f_{\setminus \theta}$, is $f$, with $\theta$ marginalized conditional on the training data $\calD$.
Then,
$$\Var_L(f_{\setminus \theta}|X) \le \Var_L(f|X),\; \Bias_L(f_{\setminus \theta}|X) = \Bias_L(f|X),\;\mbox{in particular}\;\varepsilon_L(f_{\setminus \theta})\le \varepsilon_L(f)$$
\end{itemize}
\end{Cor}
\begin{proof}
We prove the slightly more general formulation (ii); to obtain (i), remove the conditioning in the notation.
It holds that $\Bias_L(f_{\setminus \theta}|X) = \Bias_L(f|X)$ since $\EE[f_{\setminus \theta}|X] = \EE[f|X]$, by construction.
For $\Var_L(f_{\setminus \theta}|X) \le \Var_L(f|X)$, apply Lemma~\ref{Lem:functionjensencond}~(ii.a) to the expectation $\EE_{\theta | \calD}$.
The statement for $\varepsilon_L$ follows from the decomposition in Proposition~\ref{Prop:BVprob}.
\end{proof}

Intuitively, Corollary~\ref{Cor:posterior}~(ii) implies that when predicting using fully Bayesian models, it is better (or equally good), in terms of any convex loss, to marginalize over free parameters ($\theta$ in the Corollary), instead of sampling and/or leaving them in as additional randomness in the prediction functional $f$.
We will expand upon further algorithmic consequences of this for bagging and model averaging in the later section~\ref{sec:baggeraging}.

\subsection{Information, predictability, and independence}
\label{sec:entropy}

In this section, we study the dependency relation between two random variables $X,Y$, t.v.i.~ $\calX,\calY$, from the perspective of probabilistic predictability and relate this to information theoretic quantities.

In this section, we will consider a strictly proper, convex loss function $L:\calP\times \calY\rightarrow \RR$, applicable to the law of $Y$, i.e., we assume that $\calL(Y)\in \calP$. We note that quantities defined below may depend on the choice of $L$.

We will further write:
\begin{align*}
p_Y&:= \calL(Y)\quad\mbox{for the law of}\;Y\\
p_{Y|X}(x)&:= \calL(Y|X=x)\quad\mbox{for the law of}\;Y\;\mbox{conditional on}\;X=x
\end{align*}
As per the usual convention, the sub-indices $X,Y,Y|X$ are not considered values but part of the notation, thus expectations do not affect those sub-indices.

We now define key information theoretic quantities\footnote{The quantities in Definition~\ref{Def:entropies} are related to, but not identical with the key quantities in~\cite{dawid2007geometry}, even though they look similar on first glance. While~\citet{dawid2007geometry} consider divergences between two different distributions on $\calY$, our divergences are of an $\calX$-valued random variable to an $\calY$-valued random variable. Neither appear to be an obvious generalisation of the other, since~\citet{dawid2007geometry} consider distributions, while we consider potentially associated random variables, thus the former is not obtained for taking $\calX = \calY$.}:

\begin{Def}\label{Def:entropies}
Let $(X,Y)$ be a pair of random variables taking values as above. We define:
\begin{align*}
\Ent (Y) &= \EE\left[L(p_Y,Y)\right] \\
\Ent (Y|X=x) &:= \EE_{Y|X} \left[L(p_{Y|X}(x),Y) \middle| X=x \right]\quad\mbox{for}\;x\in\calX\\
\Ent (Y|X\in\calA) &:= \EE_{Y|X} \left[L(p_{Y|X}(x),Y) \middle| X \subseteq \calA\right]\quad\mbox{for}\;\calA\subseteq\calX\\
\Ent (Y|X) &:= \EE_{Y|X} \left[L(p_{Y|X}(X),Y) \middle| X \right]\\
\Ent (Y / X) &:= \EE_X\left[\Ent (Y|X)\right] = \EE_X \EE_{Y|X} \left[L(p_{Y|X}(X),Y) \middle| X \right].
\end{align*}
The value $\Ent (Y)$ is called the ($L-$)\emph{entropy} of $Y$.
The value $\Ent (Y|X=x)$ is called the \emph{conditional ($L$-)entropy} of $Y$ conditioned on the event $X=x$, similarly $\Ent (Y|X\subseteq \calA)$. The random variable $\Ent (Y|X)$ is called the \emph{feature-informed conditional ($L$-)entropy} of $Y$ conditioned on $X$. We call $\Ent (Y/X)$ the (total) \emph{conditional ($L$-)entropy} of $Y$ above $X$.\\
The implicit dependence on $L$ will not be made notationally explicit, only verbally where appropriate, but it will be understood.
\end{Def}

We note that if $L$ is the logarithmic loss, then Definitions~\ref{Def:entropies} define classically familiar\footnote{This is correct with a somewhat non-standard choice regarding $\Ent(Y|X)$, and the introduction of $\Ent(Y/X)$, which is to resolve a notational clash that readers familiar with both information theory and statistics may spot or will have observed. Namely:\\
In statistical notation, the conditional expectation $\EE[Y|X]$ is a random variable that is a function of the value of $X$. As $\Ent$ is an operator quite similar to the expectation $\EE$, in notational concordance, $\Ent(Y|X)$ should behave in analogy, i.e., it should be random variable which is a function of the value of $X$.\\
In information theoretic notation, however, the conditional entropy $\Ent (Y|X)$ is a number and not a random variable, which constitutes the notational clash. However, the information theoretic $\Ent (Y|X)$ is nothing else the expectation of the natural statistical definition of $\Ent(Y|X)$.\\
As the statistical variant will become important, we cannot just discard it but need to resolve the clash.
To this end, in our notation, $\Ent(Y|X)$ is a random variable which is a function of the value of $X$, and $\Ent(Y/X)$ is its expectation, which equals $\Ent(Y|X)$ in information theoretic notation.} entropic quantities;
for the choice of squared loss, we obtain the integrated squared error measures familiar from density estimation.

We will relate them two types of optimal predictions:

\begin{Def}\label{Def:bestuninf}
We define two ``best'' prediction functionals:
\begin{itemize}
\item[(i)] the \emph{best prediction functional} $\varpi_{Y|X}: x\mapsto \calL (Y|X=x) = p_{Y|X}(x)$ .
\item[(ii)] the \emph{best uninformed prediction functional} $\varpi_{Y}: x\mapsto \calL (Y) = p_Y$.
\end{itemize}
\end{Def}
We note that Definition~\ref{Def:bestuninf}~(i) agrees with the definition of the ``perfect prediction'' in Definition~\ref{Def:biasvar} of the bias-variance decomposition; and Definition~\ref{Def:bestuninf}~(ii) agrees with the definition of best uninformed prediction functional given in Section~\ref{Sec:uninfbase}.

We repeat further definitions from Section~\ref{Sec:uninfbase} on performance of predictors to which we will relate the entropic quantities:

\begin{Def}
A prediction functional $f:\calX\rightarrow \calP$ is called:
\begin{itemize}
\item[(i)] \emph{uninformed} if $f$ is constant, i.e., if $f(x)=f(x')$ for all $x,x'\in\calX$.
\item[(ii)] \emph{($L$-)better-than-uninformed} if $\varepsilon_L(f)\lneq \varepsilon_L(\varpi_Y)$
\end{itemize}
A prediction strategy, i.e., a random variable t.v.i.~$[\calX\rightarrow \calP]$, is called:
\begin{itemize}
\item[(i)] \emph{uninformed} if all realizations of $f$ are constant functionals, i.e., if $f(x)|f=f(x')|f$ for all $x,x'\in\calX$.
\item[(ii)] \emph{($L$-)better-than-uninformed} if $\varepsilon_L(f)\lneq \varepsilon_L(\varpi_Y)$
\end{itemize}
As before, all prediction strategies are assumed independent of $(X,Y)$.
\end{Def}

The justification for the terminology ``best'' (uninformed or overall prediction functional) is given by the following result, which also relates the generalization performances of the best predictors to the entropy $\Ent(Y)$ to the conditional entropy $\Ent(Y/X)$, and the latter two to each other in a classically familiar way:

\begin{Lem}\label{Lem:Entpre}
The following hold for prediction functionals:
\begin{itemize}
\item[(i)] $\Ent (Y) = \inf_f \varepsilon(f)$ where the minimum is taken over uninformed predictors $f:\calX\rightarrow \calP$.
The infimum is achieved as a minimum, by and only by the best uninformed predictor $\varpi_{Y}$.
\item[(ii)] $\Ent (Y/X) = \inf_f \varepsilon (f)$ where the minimum is taken over all prediction functionals $f:\calX\rightarrow \calP$.
The infimum is achieved as a minimum, by and only by the best predictor $\varpi_{Y|X}$.
\end{itemize}
In particular:
\begin{itemize}
\item[(iii)] $\Ent (Y) = \varepsilon(\varpi_{Y})$,
\item[(iv)] $\Ent (Y/X) = \varepsilon(\varpi_{Y|X}) = \Err_L(Y/X)$, and
\item[(v)]  $\Ent(Y/X)\le \Ent(Y)$.
\end{itemize}
\end{Lem}
\begin{proof}
(i) Let $f\in [\calX\rightarrow \calP]$ be uninformed. By definition of $\varepsilon$, one has
$\varepsilon(f) = \EE \left[L(f(X),Y)\right].$ By definition of uninformed, there is $q$ such that $q=f(x)$ for all $x\in\calX$, thus
$\EE \left[L(f(X),Y)\right] = \EE \left[L(q,Y)\right].$
Since $L$ is strictly proper, the last expression is minimized in $q$ for taking $q=p_Y$, and the minimal value is exactly $\Ent(Y)$, by definition of strict properness.\\
(ii) $\varepsilon (f)\ge \Ent(Y/X)$ is implied by Proposition~\ref{Prop:BVprob} and taking expectations over $X$. The fact that the value $\Ent(Y/X)$ is achieved by the choice $\varpi_{Y|X}$, i.e., that $\Ent(Y/X)=\varepsilon(\varpi_{Y|X})$, follows from substitution of definitions. Strict properness of $L$ implies that this is the unique minimizer.\\
(iii) is a direct consequence of (i), and (iv) is a direct consequence of (ii).\\
(v) follows from the fact that the set over which (ii) is minimizing includes the set over which (i) is minimizing.
\end{proof}

We would like to note that while the values $\Ent(Y),\Ent(Y/X)$ may depend on the choice of the loss function $L$, the best uninformed predictor and the best predictor do not depend on that choice as long as the loss is strictly proper, as it has been assumed in this section. The only concept in Definition~\ref{Def:bestuninf} that may depend on the choice on $L$ is that of being ($L$-)better-than-uninformed.

\begin{Thm}
\label{Thm:entropy}
The following are equivalent:
\begin{itemize}
\item[(i)] $X$ and $Y$ are statistically independent.
\item[(ii)] There is no ($L$-)better-than-uninformed prediction strategy for predicting $Y$ from $X$.
\item[(iii)] There is no ($L$-)better-than-uninformed prediction functional for predicting $Y$ from $X$.
\item[(iv)] $\varepsilon(\varpi_{Y|X}) = \varepsilon(\varpi_{Y})$.
\item[(v)] $\varpi_{Y|X} = \varpi_{Y}$.
\item[(vi)] $\Ent (Y/X) = \Ent (Y)$.
\end{itemize}
\end{Thm}
\begin{proof}
(i)$\Rightarrow$ (ii): Let $f$ be any prediction strategy, let $q:=\EE[f(X)]$, note that $q\in \calP$. By convexity of $L$, it holds that $\EE[L(q,Y)|f(X)] \le \EE[L(f(X),Y)|f(X)]$. Since $f$ is independent of $(X,Y)$ and $X$ is independent of $Y$ by assumption, it also holds that $f(X)$ is independent of $Y$. Thus, $\EE[L(q,Y)|f(X)] = \EE[L(q,Y)]$ which implies $\EE[L(q,Y)] \le \EE[L(f(X),Y)|f(X)]$. Taking total expectations implies $\EE[L(q,Y)] \le \EE[L(f(X),Y)] = \varepsilon (f)$. Strict properness of $L$ implies $\EE[L(p_Y,Y)]\le \EE[L(q,Y)]$. Lemma~\ref{Lem:Entpre} yields $\varepsilon(\varpi_{Y}) = \EE[L(p_Y,Y)].$ Putting all inequalities together (in reverse order) yields $\varepsilon(\varpi_{Y})\le \varepsilon (f)$, so $f$ is not ($L$-)better-than-uninformed. Since $f$ was arbitrary, (ii) is implied.\\
(ii)$\Rightarrow$ (iii): Any prediction functional is a (constant) prediction strategy, thus (ii) implies (iii).\\
(iii)$\Rightarrow$ (iv): This is implied by Lemma~\ref{Lem:Entpre}~(i) and~(ii), since the set of uninformed prediction functionals is contained in the set of prediction functionals.\\
(iv)$\Rightarrow$ (v): By Lemma~\ref{Lem:Entpre}~(i), $\varpi_{Y}$ is the unique prediction functional $f$ with $\varepsilon(f) = \Ent (Y)$. By Lemma~\ref{Lem:Entpre}~(ii), $\varpi_{Y|X}$ is the unique prediction functional $g$ with $\varepsilon(g) = \Ent (Y/X)$. Since $\Ent (Y/X)=\Ent (Y)$, the unique $f$ and the unique $g$ must coincide, hence $\varpi_{Y|X}=\varpi_{Y}.$\\
(v)$\Rightarrow$ (i): By definition of $\varpi_{Y|X}$ and $\varpi_{Y}$, the assumption (vi) implies that $p_{Y|X}(x) = p_Y$ for any $x\in \calX$. The latter is a possible definition of, or characterization of statistical independence of $X$ and $Y$.\\
(v)$\Leftrightarrow$ (vi): By Lemma~\ref{Lem:Entpre}, $\Ent (Y/X)=\varepsilon(\varpi_{Y|X})$, and $\Ent (Y)=\varepsilon(\varpi_{Y}),$ which implies the claim by substitution.\\
\end{proof}

Theorem~\ref{Thm:entropy} implies some practically interesting equivalences.
Comparing with the practical rationale in Section~\ref{Sec:uninfbase}, one may note that estimating $\varpi_{Y|X}$ is at the core of supervised learning, estimating $\varpi_{Y}$ is density estimation - thus the best prediction functional may be considered to be the best supervised prediction, and the best uninformed prediction functional may be considered to be the best density estimator.
Comparing the performance of the two, as in determining whether statement~(iv) of Theorem~\ref{Thm:entropy} is true, is the scientifically necessary act of baseline comparison.\\
Per the equivalences in Theorem~\ref{Thm:entropy}, this act of model comparison is equivalent to determining whether $X$ and $Y$ are independent; the sub-task of estimating the performance of the best supervised learner and the best density estimator becomes equivalent with estimating the entropy of $Y$, and estimating the conditional entropy of $Y$ w.r.t~$X$.

We recur to a particularly interesting implication of these equivalences, namely the identification of statistical independence testing with comparison of probabilistic models, in Section~\ref{Sec:hypotest}.

For the remainder of the section, we focus on further connections with the classical (Shannon/differential) entropies, in the case where $L$ in Definition~\ref{Def:entropies} is taken to be the log-loss and $(X,Y)$ are jointly absolutely continuous\footnote{Absolute continuity is not an entirely necessary assumption but makes exposition easier.}, t.v.i.~$\calX\times\calY$. For this case, we will in addition consider quantities directly related to $X$ and the respective marginal, joint, and conditional density functions $p_X,p_Y,p_{XY}, p_{Y|X}$ as per the usual notation.

\begin{Def}\label{Def:entropiesShannon}
In the situation above, we define:
\begin{align*}
\Ent (Y) &= \EE \left[- \log p{Y}(Y) \right] \\
\Ent (X) &= \EE \left[- \log p_{X}(X) \right] \\
\Ent (Y,X) &= \EE \left[- \log p_{X,Y}(X,Y) \right] \\
\Ent (Y|X) &:= \EE_{Y|X} \left[- \log p_{Y|X}(Y) \middle| X \right]\\
\Ent (Y / X) &:= \EE\left[\Ent (Y|X)\right] = \EE_X \left[\EE_{Y|X} \left[- \log p_{Y|X}(Y) \middle| X \right]\right]\\
&= - \int_{\calX}\int_{\calY} \log p_{Y|X}(y|x)\cdot p_{X,Y}(x,y)\;\diff x\;\diff y.
\end{align*}
The value $\Ent (X)$ is called the \emph{entropy} of $X$.\\
The value $\Ent (Y,X)$ is called the \emph{joint entropy} of $X$ and $Y$.\\
All other quantities are called the same as in Definition~\ref{Def:entropies}.
\end{Def}

The quantities occurring in both Definitions~\ref{Def:entropies} and~\ref{Def:entropiesShannon} agree for the choice of $L$ being the log-loss in Definitions~\ref{Def:entropies}; the quantities exclusive to Definition~\ref{Def:entropiesShannon} all include integration over a logarithm that involves a density over the value range of $X$.

Classically, one has the following well-known relations between marginal and joint entropies:

\begin{Prop}
\label{Prop:entropy}
The following inequalities hold classically for the (Shannon/differential) entropies defined above:
\begin{itemize}
\item[(i)] $\Ent (X)\le \Ent(X,Y)\;\mbox{and}\;\Ent (Y)\le \Ent(X,Y)$
\item[(ii)] $ \Ent(X,Y)\le \Ent(X)+ \Ent(Y),$ and equality holds if and only if $X$ and $Y$ are statistically independent (as random variables).
\end{itemize}
If furthermore $\calX$ resp.~$\calY$ is discrete, then
\begin{itemize}
\item[(iii)] $\Ent (X)\ge 0\;\mbox{resp.}\;\Ent (Y)\ge 0$
\end{itemize}
(this is not necessarily true if $\calX$ resp.~$\calY$ is not discrete)
\end{Prop}

It is interesting to note how the symmetric statements in Proposition~\ref{Prop:entropy} translate to statements about the $X$-conditional (and loss-general) quantities in Definitions~\ref{Def:entropies}:

\begin{Cor}
\label{Cor:entropy}
The following equalities and inequalities hold for the conditional entropy $\Ent(Y/X)$:
\begin{itemize}
\item[(i)] $\Ent (Y/X) = \Ent (Y,X) - \Ent(X)$
\item[(ii)] $\Ent (Y/X) = \Ent (Y)$ if and only if $X$ and $Y$ are statistically independent.
\end{itemize}
\end{Cor}
\begin{proof}
(i) Note that by definition (of conditional density), $p_{Y|X}(y|x) = p_{X,Y}(x,y)/p_X(x)$. Hence
\begin{align*}
\Ent (Y/X) &= - \int_{\calX}\int_{\calY} \log p_{Y,X}(y,x)\cdot p_{X,Y}(x,y)\;\diff x\;\diff y + \int_{\calX}\int_{\calY} \log p_{X}(x)\cdot p_{X,Y}(x,y)\;\diff x\;\diff y\\
&= - \int_{\calX}\int_{\calY} \log p_{X,Y}(x,y)\cdot p_{X,Y}(x,y)\;\diff x\;\diff y + \int_{\calX} \log p_{X}(x)\cdot p_{X}(x)\;\diff x\\
&= \Ent (Y,X) - \Ent(X).
\end{align*}
The statement and (ii) then follow from Proposition~\ref{Prop:entropy}~(ii).
\end{proof}

It is interesting to note how the equivalence of (i) and (vi) in Theorem~\ref{Thm:entropy}, i.e., equivalence of statistical independence and zero conditional entropy, relates to the classical entropic independence statement which is Proposition~\ref{Prop:entropy}~(iii). Corollary~\ref{Cor:entropy}~(iii) is an equivalent statement which is obtained through subtracting $\Ent (X)$ from both sides of the equality in Proposition~\ref{Prop:entropy}~(iii) while noting Corollary~\ref{Cor:entropy}~(i).\\
In this latter characterization of independence via the conditional entropy $\Ent(Y/X)$, the entropy of $X$ itself does not actually figure, in contrast to the Proposition~\ref{Prop:entropy}~(iii) via the joint entropy $\Ent(Y,X)$. This makes the conditional characterization of Corollary~\ref{Cor:entropy}~(iii) much more appealing from a practical perspective, since it avoids estimation of the entropy of $X$ altogether. The formulation via strictly proper losses in Theorem~\ref{Thm:entropy} may arguably be considered even more appealing, as estimation of entropies in a multivariate setting is replaced altogether by the much better understood task of prediction and predictive model validation, in the context of an arbitrary loss functional.\\
From a theoretical perspective, all links above are equally valuable, as they closely tie together predictive modelling, statistical independence testing, and information theory; methodological as well as theoretical insights about the one will hence immediately transfer to the other two.\\
We would also like to note that the results linking the log-loss and the Shannon/differential entropic quantities may be generalized to Bregman losses and Bregman entropic quantities, which includes the squared integrated loss as a special case; despite this being a straightforward generalization, we refrain from carrying it out it would make the discussion more technical and the link to classical quantities less clear.

\subsection{Transformations and probabilistic residuals}
\label{Sec:trafores}

Here, we study how a differentiable transformation of the target variable affects a model and its performance. For this, we consider the above setting of generative random variables $(X,Y)$ with the additional assumption that $Y$ takes values in the reals, and that $Y$ and $Y|X$ have densities $p_Y,p_{Y|X}$ and continuous, piece-wise differentiable cdf $F_Y,F_{Y|X}$.

Below, we briefly repeat some well-known facts about transformations of random variables:

\begin{Lem}\label{Lem:trafo}
Let $Y$ be an absolutely continuous variable taking values on a compact $\calY\subseteq \RR$ with density $p_Y$. Let $T:\calY\rightarrow\calZ$ be a (piece-wise) diffeomorphism\footnote{$T$ is a diffeomorphism if it is bijective, and $T$ and its inverse $T^{-1}$ are both (piece-wise) differentiable}, let $t:=\frac{\diff T(x)}{\diff x}$.\\
Consider the random variable $Z:=T(Y)$ with cdf $F_Z$. Then:
\begin{itemize}
\item[(i)] $Z$ is absolutely continuous, i.e., has a pdf $p_Z$, and $F_Z$ is continuous, piece-wise differentiable.
\item[(ii)] It holds that $p_Z(z) = \frac{p_Y (T^{-1}(z))}{\left|t(T^{-1}(z))\right|}$.
\item[(iii)] $Z$ is uniform on $[0,1]$ if and only if $T=F_Y$.
\end{itemize}
\end{Lem}

Next, we introduce notation for a transformed density such as in Lemma~\ref{Lem:trafo}~(ii), and prediction methods whose predicted density is transformed:

\begin{Def}
Let $p:\calY\rightarrow \RR^+$ be a probability density, let $T:\calY\rightarrow\calZ, U:\calZ\rightarrow\calY$ be (piece-wise) diffeomorphisms.
\begin{itemize}
\item[(i)] We write $T^\sharp p := \left[z\mapsto \frac{p (T^{-1}(z))}{\left|t(T^{-1}(z))\right|}\right]$, where $t:=\frac{\diff T(x)}{\diff x}$ and call $T^\sharp p$ the push-forward of $p$ (through $T$). The induced map $T^\sharp$ is called the push-forward operator induced by $T$.
\item[(ii)] For a prediction functional $f\in[\calX\rightarrow [\calY\rightarrow\RR^+]]$, we define in analogy the push-forward of $f$ (through $T$) as $T^\sharp f := [x\mapsto T^\sharp f(x)]$, i.e., the push-forward operator $T^\sharp$ is applied to every element of the image.
\item[(iii)] For a prediction strategy, we apply $T^\sharp$ to every prediction strategy in which it can take values in, as such the notation coincides with the standard meaning of $T^\sharp  f$ where $f$ may be a random variable.\\
\item[(iv)] We write abbreviatingly $U_\sharp :=\left(U^{-1}\right)^\sharp$ and call $U_\sharp$ the pull-back operator induced by $U$. The applicates $U_\sharp p$ and $U_\sharp f$ are called pull-backs of $p,f$ (through $U$).
\end{itemize}
\end{Def}

One verifies that the push-forward and pull-back operators have the naturally expected properties:
\begin{Lem}\label{Lem:Tsharp}
Let $T:\calY\rightarrow\calZ, U:\calZ\rightarrow \calZ'$ be (piece-wise) diffeomorphisms. Then:
\begin{itemize}
\item[(i.a)] $U^\sharp T^\sharp p = (U\circ T)^\sharp  p$ for any probability density function $p:\calY\rightarrow \RR^+$.
\item[(i.b)] $U^\sharp T^\sharp f = (U\circ T)^\sharp  f$ for prediction strategy $f$ t.v.i.~ $\calX\rightarrow[\calY\rightarrow \RR^+]$.
\item[(ii.a)] $T_\sharp U_\sharp p = (U\circ T)_\sharp  p$ for any probability density function $p:\calZ'\rightarrow \RR^+$.
\item[(ii.b)] $T_\sharp U_\sharp f = (U\circ T)_\sharp  f$ for prediction strategy $f$ t.v.i.~ $\calX\rightarrow[\calY\rightarrow \RR^+]$.
\item[(iii.a)] $T_\sharp T^\sharp  p = T^\sharp  T_\sharp p = p$ for any probability density function $p:\calY\rightarrow \RR^+$.
\item[(iii.b)] $T_\sharp T^\sharp  f = T^\sharp  T_\sharp f = f$ for any prediction strategy $f$ t.v.i.~$\calX\rightarrow[\calY\rightarrow \RR^+]$.
\item[(iv.a)] $T^\sharp  [\EE P] = \EE[T^\sharp P],$ and $T_\sharp  [\EE P] = \EE[T_\sharp P]$
    for any random variable $P$ t.v.i.~$[\calY\rightarrow \RR^+]$.
\item[(iv.b)] $T^\sharp  [\EE f] = \EE[T^\sharp f],$ and $T_\sharp  [\EE f] = \EE[T_\sharp f]$ for any prediction strategy $f$ t.v.i.~$\calX\rightarrow[\calY\rightarrow \RR^+]$.
\end{itemize}
\end{Lem}
\begin{proof}
(i.b), (ii.b) and (iii.b) follow directly from applying (i.a), (ii.a) and (iii.a) for every image of $f$.\\
(iii.a) follows from (i.a) for the choice $U = T^{-1}$ and/or switching $T$ with $T^{-1}$.\\
(i.a) follows from an elementary computation, using the chain rule.\\
(ii.a) follows from (i.a) after substituting $(T^{-1},U^{-1})$ from (ii) into $(U,T)$ from (i), using $(U\circ T)^{-1}= T^{-1} \circ U^{-1}$ and substituting definitions.\\
(iv.a) We prove the statement for $T^\sharp,$ as the statement for $T_\sharp$ follows from switching $T$ with its inverse, by (iii.a). By definitions,
\begin{align*}
T^\sharp [\EE P] = \left[z\mapsto \frac{[\EE P] (T^{-1}(z))}{\left|t(T^{-1}(z))\right|}\right]
= \left[z\mapsto \EE\left[\frac{P (T^{-1}(z))}{\left|t(T^{-1}(z))\right|}\right]\right]
= \EE[T^\sharp P],
\end{align*}
which is what was to prove. The parallel statements in (iv.b) follows from application to the range of $f$, i.e., by writing $P:= f(x)$ for any $x\in\calX$, applying (iv.a), and noting that $x$ was arbitrary.\\
A more generalizable proof of the above statements is obtained by observing that either push-forward/pull-back of $p$ is (as pdf of a measure) identical with the measure-theoretic push-forward/pull-back of (the measure induced by) $p$, in our case.
\end{proof}

The well-known properties in Lemma~\ref{Lem:trafo} have direct consequences for the log-loss:

\begin{Prop}\label{Prop:trafo}
Let $T:\calY\rightarrow\calZ$ be a (piece-wise) diffeomorphism with derivative $t$. Let $L:\calP\times \calY\rightarrow \RR; (p,y)\mapsto - \log p(y)$ be the log-loss (for absolutely continuous predictions $\calP$, see Definition~\ref{Def:losses}). Then the following hold:
\begin{itemize}
\item[(i)] $L(p,y) = L(p\circ T^{-1}, T(y))$ for any $p\in\calP, y\in\calY$.
\item[(ii)] $L(T^\sharp p,T(y)) = L(p, y) - L(|t|,y)$ for any $p\in\calP, y\in\calY$.
\item[(iii)] $L(F^\sharp p,F(y)) = 0$, where $F$ be the cdf to the pdf $p$.
\end{itemize}
\end{Prop}
\begin{proof}
(i) Substituting definitions, $L(p,y) = -\log p(y) = -\log (p\circ T^{-1}) (T(y)) = L(p\circ T^{-1}, T(y))$.\\
(ii) Substituting definitions,
\begin{align*}
L(T^*p,T(y)) &= - \log  \frac{p (T^{-1}(T(y)))}{\left|t(T^{-1}(T(y)))\right|} \\
&= - \log  p (T^{-1}(T(y))) + \log\left|t(T^{-1}(T(y)))\right| \\
&= - \log  \frac{p (y)}{\left|t(y)\right|} \\
&= - \log  p (y)+ \log\left|t(y)\right| \\
&= L(p,y) - L(|t|,y), \\
\end{align*}
which proves the claim.\\
(iii) is the interesting special case of (ii) where $T=F$ and $t=p$.
\end{proof}

There are two interesting implications of Proposition~\ref{Prop:trafo} for the supervised learning setting. The first states that the order of predictive performance is preserved under transformations, if measured by the log-loss:

\begin{Cor}\label{Cor:loglossmono}
Let $f,g$ be prediction strategies t.v.i.~$[\calX\rightarrow[\calY\rightarrow\RR^+]],$ as usual assumed independent of the generative process $(X,Y)$.\\
Let $T:\calY\rightarrow\calZ$ be a (piece-wise) diffeomorphism. Let $L:\calP\times \calY\rightarrow \RR; (p,y)\mapsto - \log p(y)$ be the log-loss.\\
Then, for any $x\in\calX,y\in\calY$, it holds that
$$L(g(x),y) - L(f(x),y) = L(T^\sharp g(x),T(y)) - L(T^\sharp f(x),T(y)),$$
i.e., paired log-loss differences preserved under a joint diffeomorphic transformation of the target and the prediction strategies. In particular, we have the same statement for generalization losses, i.e.,
$$\EE[L(g(X),Y)] - \EE[L(f(X),Y)] = \EE[L(T^\sharp g(X),T(Y))] - \EE[L(T^\sharp f(X),T(Y))],$$
and the following are equivalent:
\begin{itemize}
\item[(i)] $\EE[L(f(X),Y)]\le \EE[L(g(X),Y)]$.
\item[(ii)] $\EE[L(T^\sharp f(X),T(Y))]\le \EE[L(T^\sharp g(X),T(Y))]$.
\end{itemize}
The equivalence above also holds with equality instead of inequality, thus also with strict inequality instead of inequality.
\end{Cor}
\begin{proof}
The first statement
$$L(g(x),y) - L(f(x),y) = L(T^\sharp g(x),T(y)) - L(T^\sharp f(x),T(y)),$$
follows directly from substituting the equality in Proposition~\ref{Prop:trafo}~(ii) to the two terms on the right hand side.\\
The second statement
$$\EE[L(g(X),Y)] - \EE[L(f(X),Y)] = \EE[L(T^\sharp g(X),T(Y))] - \EE[L(T^\sharp f(X),T(Y))]$$
follows from the first statement by linearity of expectation and integration over $X,Y,f,g$, noting that by assumption $(X,Y)$ is independent of the pair $(f,g)$.\\
The equivalence of (i) and (ii) is directly implied.
\end{proof}

While intuitively, Corollary~\ref{Cor:loglossmono} implies that order of model performance is preserved under (diffeomorphic) transformations of the target, more is true: not only the order in terms of generalization loss is preserved, but also all the generalization losses up to an additive offset. Even more, the paired loss differences obtained from any test sample remain unaffected by any transformation.

We note that Proposition~\ref{Prop:trafo} and Corollary~\ref{Cor:loglossmono} seem to require the log-loss to hold. Both are, in fact, wrong for other losses such as the squared integrated loss. It seems an interesting question whether either consequence in Corollary~\ref{Cor:loglossmono} characterizes the log-loss together with the assumption of being strictly proper, i.e., whether strict locality is implied.

The second implication of Proposition~\ref{Prop:trafo} is obtained by noticing that in Corollary~\ref{Cor:loglossmono}, all statements remain valid when $T$ is chosen in dependency on the features, $X$.

\begin{Cor}\label{Cor:loglosstrafolearn}
All statements of Corollary~\ref{Cor:loglossmono} remain correct if $T$, instead of being a fixed element of $[\calY\rightarrow\calZ]$, is a random variable taking values in diffeomorphisms in $[\calY\rightarrow\calZ]$.\\
In particular, $T$ may depend on any training data $f$ or $g$ have access to, on the values of $f$ and $g$ themselves, and even on $X$ and $Y$.
\end{Cor}
\begin{proof}
This follows from observing that the validity of the first statement is unchanged for the fixed $T$, in Corollary~\ref{Cor:loglossmono}, being any value which the random variable $T$ in this corollary may take. In particular, the consequences obtained from taking equations are still valid.
\end{proof}

Corollary~\ref{Cor:loglosstrafolearn} in particular allows us to consider the effect of a transformation learnt on the training set that may depend on a feature input. To make a general statement, we define shorthand notation for pullback-composition of prediction strategies:

\begin{Def}\label{Def:compose}
Let $T$ be a random variable t.v.i.~$[\calX\rightarrow [\calY\rightarrow \calZ]]$. We write:\\
\begin{itemize}
\item[(i)] $[T^\sharp f] := [x\mapsto T^\sharp(x) f(x)]$ for any prediction strategy $f$ t.v.i.~$[\calX\rightarrow \Distr (\calY)]$.
\item[(ii)] $[T_\sharp f] := [x\mapsto T_\sharp(x) f(x)]$ for any prediction strategy $f$ t.v.i.~$[\calX\rightarrow \Distr (\calZ)]$.
\end{itemize}
For a probabilistic prediction strategy $g$ t.v.i.~$[\calX\rightarrow \Distr(\calY)]$, where $\calY\subseteq \RR$, so $\Distr(\calY)$ may be identified with cdf in $[\calY\rightarrow [0,1]]$, we abbreviatingly write $g^\sharp f$ or $g_\sharp f$ (if ranges/domains agree) instead of defining a variant of $g$ that returns cdf.
\end{Def}

\begin{Cor}\label{Cor:trafoboost}
Let $f$ be a prediction strategy t.v.i.~$[\calX\rightarrow[\calY\rightarrow\RR^+]]$ i.e., predicting pdf. Let $F$ be the corresponding random variable t.v.i.~$[\calX\rightarrow[\calY\rightarrow [0,1]]$ which predicts cdf, i.e., for any $x\in\calX$, it holds that $F(x)$ is the cdf to the pdf $f(x)$.\\
Let $L:\calP\times \calY\rightarrow \RR; (p,y)\mapsto - \log p(y)$ be the log-loss.\\
Then, it holds that:
\begin{itemize}
\item[(i)] the random variable $L([F^\sharp f](X),F(X)(Y))$ is constant $0$.
\item[(ii)] in particular, $\EE[L([F^\sharp f](X),F(X)(Y))] = 0,$ which is the same performance as of the uninformed predictor that always predicts the uniform distribution on $[0,1]$.
\item[(iii)] Let $g$ be a probabilistic prediction strategy for predicting the transformed target $F(X)(Y)$ from $X$, i.e., t.v.i.~$[\calX\rightarrow [[0,1]\rightarrow\RR^+]$ and independent of $(X,Y)$.\\
    Then, $-\EE[L(g(X),F(X)(Y))] = \EE[L(f(X),Y)] - \EE[L([F_\sharp g](X),Y)]$
\item[(iv)] In particular, in the situation of (iii), the learning strategy $f_\sharp g$ outperforms $f$ in terms of expected generalization log-loss if and only if the generalization log-loss of $g$ (for predicting $F(X)(Y)$) is negative; equivalently, by (ii), if and only if $g$ outperforms the uniform uninformed baseline.
\end{itemize}
\end{Cor}
\begin{proof}
(i) follows from Proposition~\ref{Prop:trafo}~(iii) applied to every joint value of $F,f,Y,X$.\\
The first statement (ii) follows from taking expectations in (i). The fact that $\EE[L(u(X),F(X)(Y))] =0$ for $u:x\mapsto[y\mapsto 1]$ being the uninformed prediction of uniform density on $[0,1]$ follows from $L(u,F(X)(Y)) = -\log u(X)(F(X)(Y)) = \log 1 = 0.$\\
(iii) follows from Corollary~\ref{Cor:loglosstrafolearn} applied to the second statement of Corollary~\ref{Cor:loglossmono} which implies
$$\EE[L([F^\sharp f](X),F(X)(Y))]-\EE[L(g(X),F(X)(Y))] = \EE[L(f(X),Y)] - \EE[L([F_\sharp g](X),Y)],$$
and then using (ii) which equates the first term on the LHS with zero.\\
(iv) is a re-formulation of the equality in (iii), in the special case where $\EE[L(g(X),F(X)(Y))]\lneq 0$.
\end{proof}

Corollary~\ref{Cor:trafoboost} may seem technical, but may be seen as a parallel to the following deterministic phenomenon: if one succeeds to predict the residuals of a first-pass prediction $f:\calX\rightarrow \RR$ better than a certain baseline, by using a residual predictor $g:\calX\rightarrow \RR$ to predict the residual $Y - f(X)$, the composite prediction strategy $g+f$ is better than $f$ alone. The principle which generalizes is first predicting via $f$, then correcting via the residual prediction $g$, and that this works if (and only if) $g$ is better than a certain prediction baseline (in the classical case with squared loss: always predicting the mean residual).\\
In the probabilistic case, the same idea is shown to work, though the additive correction has to be replaced by composition due to the probabilistic nature of prediction: $f$ predicts distributions, and $g$ is used to predict cdf applicates $F(X)(Y)$. Corollary~\ref{Cor:trafoboost}~(iv) shows that the composite prediction strategy $F_\sharp g$ is better than $f$ alone if (and only if) $g$ manages to predict the ``probabilistic residuals'' $F(X)(Y)$ better than the uniform uninformed predictor.\\

We note that somewhat similar computations have appeared in~\cite{menon2016linking}, some of which may be interpreted in terms of probabilistic residuals.

\subsection{Visual diagnostics of probabilistic models}
\label{sec:propres}
Residuals are a key motif in diagnostics of deterministic supervised models.
While Section~\ref{Sec:trafores} has introduced cdf applicates with properties similar to residuals, there are different objects in the probabilistic setting which share different aspects of classical residuals used in standard predictive model diagnostics. As such, it is no different from the classical deterministic setting for which also multiple types of residuals such as signed, unsigned and standardized residuals are known.\\

We collate a number of objects which may take place of these in the probabilistic regression setting (i.e., $\calY\subseteq\RR$) to diagnose predictive behaviour of a fixed method:

\begin{Def}
Let $(X_1^*,Y_1^*),\dots, (X_M^*,Y_M^*)$ be a test data batch t.v.i.~$(\calX\times \calY),$ and let $f:\calX\rightarrow \calP$ with $\calP\subseteq \Distr(\calY)$ be a prediction functional to be diagnosed. We define, for the rest of this section:
\begin{itemize}
\item[(i)] the sample of label pdf predictions $p_i:= f(X_i^*)$
\item[(ii)] a derived sample of point statistics, for example (but not necessarily) $Y'_i:=\EE[\widehat{Y}^*_i]$ where  $\widehat{Y}^*_i\sim p_i$
\item[(iii)] the sample of loss residuals $L_i:= L(p_i,Y_i^*)$
\item[(iv)] the sample of probability residuals $R_i:= F_i(Y_i^*)$, where $F_i$ is the cdf of the pdf $p_i$ define
\end{itemize}
\end{Def}

The natural analogies are as follows: loss residuals are analogous to the unsigned residuals, probability residuals analogous to signed and standardized residuals in the classical deterministic setting. The former analogy is given by the discussion in Sections~\ref{sec:classprob} and~\ref{sec:modelvalidation.estim}, while the latter analogy is given by the discussion at the end of Section~\ref{Sec:trafores}. More concisely, loss residuals are the natural basis for quantitative predictive model assessment, while the probability residuals are the natural basis for set-wise residual based model building strategies such as boosting.

More precisely, the predictions residuals above may be used for visual or quantitative diagnostics of residuals as follows:
\begin{itemize}
\item[(i)] A probabilistic true-vs-predicted plot, with bar/whiskers or quantile plot elements parallel to the $y$-axis. The $y$-axis indexes true values, the $x$-axis predicted values. The non-quantile point estimate $Y'_i$ may be additionally plotted. This is similar to the deterministic cross-plot in assessing model goodness and heteroscedasticity.
\item[(ii)] A probabilistic predicted-vs-loss-residual plot, as a scatter plot of pairs $(Y'_i,L_i)$, with $x$-axis indexing $\calY$ and $y$-axis the range of loss residual values. The $y$-axis may be zeroed by subtracting the paired loss values of the best uninformed baseline prediction.
\item[(iii)] A probabilistic predicted-vs-probability-residual plot, as a scatter plot of pairs $(Y'_i,R_i)$, with $x$-axis indexing $\calY$ and $y$-axis the range of probability residual values.
    As in the classical true-vs-residual plot, structure deviating from the expectation of a ($Y'_i$-conditional) uniform distribution indicates potential for residual boosting.
\end{itemize}

We would like to note that while it is tempting to replace the point statistic $Y'_i$ by the true value $Y^*_i$, the respective variants of (ii) and (iii) would be unclear to interpret due to interaction effects through the conditioning $L_i|Y^*_i$ resp.~$R_i|Y^*_i$. Especially the latter loses the property of $R_i|Y'_i$ being uniform for the perfect prediction.\\
It should be noted that this is not a property special to the probabilistic setting - the same issue arises in the case of deterministic point predictions where plotting residuals against true values instead of predicted values makes the plots difficult to interpret. For example, homoscedastic standard normal residuals in the signed residual plot (plotted against predicted values) are expected to form a band around the zero axis, while the analogue plot of signed residuals against true values will in general show a skew which in itself is both expected and unproblematic, thus uninformative.

\newpage
\section{Meta-algorithms for the probabilistic setting}
\label{sec:meta-algorithms}

Building on the setting in Section~\ref{sec:modelvalidation} and theoretical results in Section~\ref{sec:ltheory}, we provide applications in form of meta-algorithms derived from the probabilistic setting.
More precisely, in this section, we present meta-algorithms for:
\begin{itemize}
\item[(i)] using capped losses or mixture strategies to remove over-penalization of low-probability outliers when using the log-loss for assessment.

\item[(ii)] target adaption and target composition, for example to obtain probabilistic strategies from point prediction strategies, or to convert mixed predictions into continuous ones.

\item[(iii)] bagging or averaging probabilistic prediction strategies, based on properties of the probabilistic bias-variance trade-off in Proposition~\ref{Prop:BVprob}.

\item[(iv)] boosting probabilistic prediction strategies, based on a properties of the probability residuals considered in Section~\ref{Sec:trafores}.

\item[(v)] probabilistic predictive independence testing, i.e., testing of the hypothesis ``$X$ is independent of $Y$'' for two random variables $X,Y$, based on the equivalence of predictability and statistical dependence derived in Theorem~\ref{Thm:entropy}.
\end{itemize}

Meta-algorithms (i) and (ii) take the form of an output adaptor strategy, that is, a wrapping meta-strategy which creates a new prediction strategy by modifying solely predictions of a base strategy.\\
Meta-algorithms (iii) and (iv) take the form of an ensemble strategy, that is, a meta-strategy which creates a more powerful prediction strategy from a collection of weaker base strategies.\\
Meta-algorithm (v) takes the form of a workflow interface, that is, a meta-strategy which relates the workflow of statistical independence testing to the workflow of comparing of probabilistic supervised learning models through showing how any instance of the latter can be turned into a solution of the former.

\subsection{Target re-normalization and outlier capping for the log-loss}
\label{sec:logmix}

Before discussing more general meta-modelling strategies, we will discuss a specific meta-modelling strategy to address a peculiarity of the log-loss which can make successful adaptation of deterministic modelling strategies difficult.

Namely, if only for one test label point a very low probability/density is predicted, the loss residual for that test label point may be extremely high, up to a value of $\infty$ if a probability/density of zero is predicted. This property of heavily penalizing inaccurate predictions of low-probability events is not shared by all strictly proper losses, e.g., not by the squared integrated loss, however by Proposition~\ref{Prop:logcanon}, this issue cannot be solved by changing the loss without dropping strict locality which is a desirable property to have in at least one of the considered loss measures.

In this section, we outline an output adaptor meta-strategy, that is, a wrapping meta-strategy which acts by modifying predictions only, for the case of probabilistic regression, i.e., when the data generating process $(X,Y)$ takes values in $\calX\times \calY$ with $\calY\subseteq \RR$.

The main idea is to ensure that the predicted density is always above a certain value $\epsilon$, which implies an expected and empirical log-loss that is upper bounded by $-\log \epsilon$. For predictions on $[0,1]$ this can be achieved by mixing with a uniform, which can be interpreted both as a modification to the strategy and to the loss.

\begin{Prop}\label{Prop:lossapprox}
Let $f$ be a $[\calX\rightarrow \Distr(\calY)]$-valued prediction strategy, with $\calY = [0,1]$.
Consider the log-loss $L:(p,y)\mapsto -\log p(y)$
\begin{itemize}
\item[(i)] For the mixture strategy $f_\epsilon : x\mapsto \epsilon + (1-\epsilon) f(x)$, it holds that
$$ \min \left( -\log \epsilon, \varepsilon_L (f) - \log (1 - \epsilon)\right) \ge \varepsilon_L (f_\epsilon) \ge (1 - \epsilon)\varepsilon_L (f)$$

\item[(ii)] For the mixture loss $L_\epsilon : (p,y)\mapsto -\log \left(\epsilon + (1-\epsilon) p(y)\right)$, it holds that
$$ \min \left( -\log \epsilon, \varepsilon_L (f) - \log (1 - \epsilon)\right) \ge \varepsilon_{L_\epsilon}(f) \ge (1-\epsilon)\varepsilon_L (f) $$
\end{itemize}
Note that $L_\epsilon(f(x),y)) = L(f_\epsilon(x),y)$ for all $x\in\calX,y\in\calY$, so both statements are equivalent.
\end{Prop}
\begin{proof}
The statement $L_\epsilon(f(x),y)) = L(f_\epsilon(x),y)$ follows from definitions, thus (i) and (ii) are equivalent statements since the middle terns are equal.\\
For the first inequality, note that $-\log\epsilon \ge L_\epsilon(p,y)$, and $-\log \left((1-\epsilon) p(y)\right) \ge L_\epsilon (p,y)$  for any $p,y$, by monotonicity of the logarithm.\\
For the second inequality, Lemma~\ref{Lem:functionjensencond}~(i.a) may be applied, yielding
$$\varepsilon_L(f) \ge \epsilon\varepsilon_L(u) + (1-\epsilon)\varepsilon_L (f) = (1-\epsilon)\varepsilon_L (f),$$
where $u = [x\mapsto [y\mapsto 1]]$ is the uninformed predictor always predicting the uniform on $[0,1]$.
\end{proof}

Proposition~\ref{Prop:lossapprox} may be read in two ways:
(i) suggests modifying every, or potentially outlier sensitive prediction strategies by mixing predictions with the uniform; (ii) generically suggests modifying the log-loss to ensure lower boundedness. From a scientific perspective, the view (i) is preferable as the modified log-loss is no longer proper.\\
Either view also ensures, through the bounding inequalities, that for very small $\epsilon$, the estimated log-losses of the mixture modified predictor differ from those of the original method only in two ways: (a) lower bounding by $-\log \epsilon$, and (b) absolute/relative perturbation of size in the order of $\epsilon$. Thus, loss capped model comparison for mixture modified prediction strategies is equivalent to inference with the modified predictors, unless changes in the order of $\epsilon$ change significances. This equivalence to capped comparison may be made more quantitative in terms of a cut-off loss:

\begin{Cor}\label{Cor:cutoffloss}
Let $f$ be a $[\calX\rightarrow \Distr(\calY)]$-valued prediction strategy, with $\calY = [0,1]$.
Consider:\\
the log-loss $L:(p,y)\mapsto -\log p(y)$.\\
The cut-off log-loss $\tilde{L}_\epsilon:(p,y)\mapsto \min\left(-\log\epsilon, -\log p(y)\right)$.\\
The mixture log-loss $L_\epsilon : (p,y)\mapsto -\log \left(\epsilon + (1-\epsilon) p(y)\right)$.\\
It holds that:
\begin{itemize}
\item[(i)] $\varepsilon_{L_\epsilon} (f)  \le \left( \varepsilon_{\tilde{L}_\epsilon} (f) \right)$

\item[(ii)] $\left\|\varepsilon_{L_\epsilon} (f) - \varepsilon_{\tilde{L}_\epsilon} (f)\right\|\le \max\left(-\log (1 - \epsilon), \epsilon |\varepsilon_L (f)|\right)$
\end{itemize}
For small $\epsilon$, the RHS in (ii) is approximately $\epsilon\max(1,|\varepsilon_L (f)|)$.
\end{Cor}

Corollary~\ref{Cor:cutoffloss} may be seen as a justification for the cut-off log-loss $\tilde{L}_\epsilon$, by expressing it as a perturbation of a strictly proper loss $L$, applied to a modified prediction strategy, the mixture strategy $f_\epsilon$ from Proposition~\ref{Prop:lossapprox}, thus making the cut-off loss more than simply a heuristic.\\

For univariate real predictions not in the range $[0,1]$, variable transformation results from Section~\ref{Sec:trafores} may be applied, by mapping the full real line to the unit interval. For this purpose, the sigmoid function is (not the only sensible choice but) appealing due to simplicity of the analytical expressions it involves:

\begin{Lem}\label{Lem:sigmoid}
Let $s:\RR\rightarrow [0,1], x\mapsto (\exp(-x)+1)^{-1}$ be the (logistic) sigmoid function. Write $s'$ for its first derivative. Then:
\begin{itemize}
\item[(i)] $s^{-1} = \logit,$ where $\logit: x\mapsto \log (x) - \log (1-x)$

\item[(ii)] $s'(x) = s(x)\cdot (1-s(x))$

\item[(iii)] $s'(s^{-1}(s)) = x\cdot (1-x)$

\item[(iv)] $\frac{\diff}{\diff x} \left[\log(s(x))\right] = 1-s(x)$
\end{itemize}
\end{Lem}
\begin{proof}
The statements are all results of elementary analytical computations.
\end{proof}

\begin{Cor}\label{Cor:lossapproxsigma}
Let $f$ be a $[\calX\rightarrow \Distr(\calY)]$-valued prediction strategy, with $\calY \subseteq \RR$.
Consider the log-loss $L:(p,y)\mapsto -\log p(y)$ and the uniform prediction strategy's pullback $s_\sharp u$, where $u:[x\mapsto [y\mapsto 1]]$.
\begin{itemize}
\item[(i)] Explicitly, it holds that $s_\sharp u: x\mapsto s'(x) = s(x)(1-s(x))$.

\item[(ii)] For the mixture strategy $f_\epsilon : x\mapsto \epsilon s_\sharp u + (1-\epsilon) f(x)$, it holds that
$$ \min \left( \varepsilon_L (s_\sharp u) -\log \epsilon, \varepsilon_L (f) - \log (1 - \epsilon)\right) \ge \varepsilon_L (f_\epsilon) \ge \epsilon \varepsilon_L (s_\sharp u) + (1 - \epsilon)\varepsilon_L (f)$$
\item[(iii)] For the mixture strategy $f_\epsilon : x\mapsto \epsilon s_\sharp u + (1-\epsilon) f(x)$, it holds that
$$ \min \left(-\log \epsilon, \varepsilon_L (s^\sharp f) - \log (1 - \epsilon)\right) \ge \varepsilon_L (s^\sharp f_\epsilon) \ge  (1 - \epsilon)\varepsilon_L (f),$$
where generalization losses are for predicting $s(Y)$.
\end{itemize}
\end{Cor}
\begin{proof}
(i) follows from Lemma~\ref{Lem:Tsharp}~(iii) which implies $s_\sharp = (s^{-1})^\sharp$.\\
(iii) follows from Proposition~\ref{Prop:lossapprox}, applied to $g:=s^\sharp f$ (taken as $f$ of Proposition~\ref{Prop:lossapprox}) and noticing that $g_\epsilon = s^\sharp f_\epsilon$.\\
(ii) follows from (iii) and Corollary~\ref{Cor:loglossmono}.
\end{proof}

Corollary~\ref{Cor:lossapproxsigma} states that the consequences of Proposition~\ref{Prop:lossapprox} still hold true, and usage of a capped log-loss, or perturbation mixtures - in the real case for example with the distribution $s(x)(1-s(x))$ - are still justified.

\subsection{Target adaptors and target composite strategies}
\label{sec:adaptors}
Depending on the native type of a prediction strategy $f$ t.v.i.~$\calX\rightarrow \calP$ for some $\calP\subseteq \Distr(\calY)$ in comparison to the task the strategy is to be evaluated for, the predicted type of $f$ may have to be changed as to allow fair evaluation with respect to the task of importance.

The most important cases are adaptors for conversion of:

\begin{itemize}
\item[(i)] probabilistic predictions to point predictions

\item[(ii)] point predictions to probabilistic predictions

\item[(iii)] empirical samples such as Bayesian predictive posterior samples to continuous distributions

\item[(iv)] mixed distributions to continuous distributions

\item[(v)] a continuous distribution to a specific parametric one

\end{itemize}

We discuss the different cases separately.
Many adaptors below take the obvious form of concatenation with a functional in $[\calP\rightarrow \calQ]$ for the desired $\calQ\subseteq \Distr(\calY)$, but we would like to point out that this does not necessarily need to be the case: adaptors may depend on the features and targets, and can be fitted to the training data as well; they also may combine multiple (not necessarily probabilistic) prediction strategies into a prediction strategy of the desired type.\\

\begin{Def}
We call functionals in $[\calP\rightarrow \calQ]$, as above, \emph{1:1 adaptor functionals}.\\
We call random variables t.v.i.~$[\calP\rightarrow \calQ]$ \emph{1:1 adaptor strategies}.
\end{Def}

\subsubsection{Conversion to point prediction}

While Section~\ref{sec:classprob} outlines how point prediction may be obtained as a sub-case for a fixed choice of distribution class, it may still be of interest to obtain a back-conversion to the classical supervised learning task.

For this, we briefly introduce some theory specific to the point prediction case, following~\cite{burkart2017predictive}.

\begin{Def}\label{Def:convex}
A functional $L:\calY\times \calY\rightarrow \RR$ is called point prediction loss functional, by convention we consider the first argument the predicted, and the second argument the observed label. Furthermore:\\
\begin{enumerate}
\item[(i)] $L$ is called convex if $L(\EE[Y],y)\le \EE[L(Y,y)]$ for all $y\in \calY$ and random variables $Y$ t.v.i.~$\calY$.
\item[(ii)] $L$ is called strictly convex if $L$ is convex, and $L(\EE[Y],y) = \EE[L(Y,y)]$ if and only if $Y$ is a constant random variable.
\end{enumerate}
\end{Def}

Point prediction loss functionals are further canonically paired with summary statistics of distributions:

\begin{Def}\label{Def:elicit}
Let $L:\calY\times \calY\rightarrow \RR$ be a convex point prediction loss functional.\\
For $p\in\Distr(\calY)$, define
$$\mu_L:\Distr(\calY)\rightarrow \calY, \; p\mapsto \underset{y\in\calY}{\argmin} \EE\left[L(y,Y)\right]\;\mbox{where}\; Y\sim p$$
(taking an arbitrary but fixed minimizer if there are multiple).
Following~\citet{gneiting2007strictly}, $\mu_L$ is called the \emph{eliciting functional} associated to $L$, and $\mu_L(p)$ is the \emph{summary} of $p$ \emph{elicited by} $L$.
\end{Def}

Well-definedness of $\mu_L$, i.e., existence of a minimizer, is ensured by convexity of $L$.
Well-known examples of elicited summaries are given in the following:

\begin{Lem}\label{Lem:elicit}
The following relations between losses and elicited statistics of real-valued random variables hold:
\begin{itemize}
\item[(i)] the convex squared loss $L_{sq}:(y,y_*)\mapsto (y-y_*)^2$ elicits the mean. That is, $\mu_{L_{sq}}(\calL(Z)) = \EE [Z]$ for any $\RR^n$-valued random variable $Z$.
\item[(ii)] the (convex but not strictly convex) absolute loss $L_{abs}:(y,y_*)\mapsto |y-y_*|$ elicits the median(s). That is, $\mu_{L_{abs}}(\calL(Z)) = \mbox{median} [Z]$ for any $\RR$-valued random variable $Z$.
\item[(iii)] the (convex but not strictly convex) quantile-loss (or short: Q-loss) $L(y,y_*)=\alpha\cdot m(y_*,y) + (1-\alpha)\cdot m(y,y_*)$, with $m(x,z)=\min(x-z,0)$, elicits the $\alpha$-quantile(s).
            That is, $\mu_L([Y]) = F^{-1}_Y(\alpha)$ for any $\RR$-valued random variable $Y$, where $F^{-1}_Y:[0,1] \rightarrow P(\RR)$ is the set-valued inverse c.d.f.~of $Y$ (with the convention that the full set of inverse values is returned at jump discontinuities rather than just an extremum).
\end{itemize}
\end{Lem}

The above implies two slightly different 1:1 adaptor functionals to point predictions:
\begin{enumerate}
\item[(i)] In practice, point prediction loss functionals are all convex. Thus, the predictive mean adaptor $\mu_{L_{sq}}: \calL(Z) \mapsto \EE[Z]$ is sensible, since by definition of convexity, it will always outperform a random sample from $Z$ in terms of the expected generalization loss, for any point prediction loss $L$ (not only the squared loss).
\item[(ii)] By Lemma~\ref{Lem:elicit}, the eliciting functional for the specific loss considered, $\mu_L$, is also a valid choice.
\end{enumerate}

It is interesting to note that among the 1:1 adaptor strategies, none is necessarily better than the other. (ii) will be best only if the predicted distributions' elicited statistics are sufficiently close to the conditionally elicited statistic. This is the case if the prediction is perfect, but for other predictions may arbitrarily deviate and leave (i) as the better option. In general, only systematic experiments can determine which option is better.\\

On a more general note, the more general class of 1:1 adaptor strategies convert predicted distributions into point predictions, while being allowed to use the features for prediction, and to train on the training data. As such, the task type for such 1:1 adaptor strategies falls in the field of distributional regression or classification, more generally distributional supervised learning, since it takes distributions as features.

\subsubsection{Adaptors for point prediction strategies}
\label{sec:adaptors.PPandResid}

We start with an important observation for the task of converting point predictions into probabilistic predictions: in general, there is no sensible way to translate a single point prediction into a single probabilistic prediction, as for a sensible specification of a distribution, at least two numbers must be specified\footnote{This statement is valid for all practical purposes, but it is not strictly mathematically true as there is a bijection $\RR^2\rightarrow \RR$, such as the Peano bijection, or mapping decimal expansion digits of the two input numbers to odd and even digits of the output number. However, all such bijections, in stylized practice, only add to the computational and theoretical burden rather than simplifying anything. Thus, validity of the statement that at least two numbers must be specified is in terms of necessity for all practical purposes - even though we are unaware of an argument from the theory of computation that would make this semi-qualitative statement mathematically and logically precise.}, usually corresponding to a location and a dispersion parameter. This is also mirrored by the discussion in Section~\ref{sec:classprob} where even for classical predictors turned probabilistic, the noise term $\epsilon$ needs to be specified, which at least requires a dispersion type parameter such as the noise variance, which together with the classical prediction implies at least two numbers to specify.

From the viewpoint of adaptors, this means that sensible point-prediction-to-probabilistic adaptor algorithms are necessarily composite.

The simplest form of such composite adaptors is obtained by composing point predictors which output the parameters of a parametric distribution - these can be obtained by choosing parameters which can be elicited in the sense of Definition~\ref{Def:elicit}, as detailed in Algorithm~\ref{alg:adapt-elicit}.

\begin{algorithm}[ht]
\caption{Elicitation-composition adaptor.\newline
\textit{Input:} Training data $\calD = \{(X_1,Y_1),\dots (X_N,Y_N)\}$. Data $(X_i,Y_i)$ t.v.i.~$\calX\times \calY$ (usually: $\calY\subseteq \RR$).
A family of parametric distributions $p(\theta_1,\dots,\theta_m)\in\Distr(\calY), \theta_i\in\calY$, together with point prediction losses $L_1,\dots,L_m:\calY\times\calY\rightarrow \RR$ such that $L_i$ elicits $\theta_i$, i.e., $\mu_{L_i}(p(\theta_1,\dots,\theta_m)) = \theta_i$. Point prediction strategies $f_i,i=1\dots m$ t.v.i.~$[\calX\rightarrow \calY]$, where $f_i$ is good w.r.t.~$L_i$.\newline
\textit{Output:} a composite probabilistic prediction strategy $f$ t.v.i.~$[\calX\rightarrow \Distr (\calY)$ \label{alg:adapt-elicit}}
\begin{algorithmic}[1]
    \State Fit $f_i$ to $\calD$, such that $f_i$ has small $L_i$-generalization loss.
    \State Return $f:=[x\mapsto p(f_1(x),\dots,f_m(x))].$
\end{algorithmic}
\end{algorithm}

One subtle thing to note is that the variance of a distribution cannot be elicited (see~\cite{Gneiting2011}), hence Algorithm~\ref{alg:adapt-elicit} is not directly applicable to many parametric output distributions parameterized in the standard way, e.g., to Gaussian distributions $p(\mu,\sigma^2)$ where $\mu,\sigma^2$ bear their usual meaning of mean and variance. While this may in-principle remedied through re-parameterization in terms of quantiles which are elicited by the quantile losses (see Lemma~\ref{Lem:elicit}), this may not be advisable as quantile estimation is known to be potentially unreliable due to potentially high variance, in the order of one over density squared at that quantile which may be very problematic when there is model mismatch.

One possible remedy are the non-elicitation based quantile estimation strategies as described in Section~\ref{sec:quantileregression}, which do not necessarily estimate predictive quantiles separately; another possible solution is the residual adaptor strategy described in Algorithm~\ref{alg:adapt-resid}, closely related to heteroscedastic prediction interval strategies as described in Section~\ref{sec:heteroscedastic-regression}, and primarily applicable to prediction of two-parameter, location-dispersion parametric distributions.

\begin{algorithm}[ht]
\caption{Residual-composition adaptor.\newline
\textit{Input:} Training data $\calD = \{(X_1,Y_1),\dots (X_N,Y_N)\}$. Data $(X_i,Y_i)$ t.v.i.~$\calX\times \calY$, where $\calY\subseteq \RR.$
A family of parametric distributions $p(\mu,\sigma)$, where $\mu$ is a location and $\sigma$ is a variance parameter.
Point prediction strategies $g,h$ t.v.i.~$[\calX\rightarrow \calY]$, with small generalization squared loss.\newline
\textit{Output:} a composite probabilistic prediction strategy $f$ t.v.i.~$[\calX\rightarrow \Distr (\calY)]$ \label{alg:adapt-resid}}
\begin{algorithmic}[1]
    \State Fit $g$ to $\calD$, such that $g$ has small squared generalization loss.
    \State Define $\rho^2_i:=\left(g(X_i) - Y_i\right)^2$
    \State Fit $h$ to training data $(X_1,\rho^2_1),\dots,(X_m,\rho^2_m)$.
    \State Return $f:=[x\mapsto p(g(x),h(x))]$.
\end{algorithmic}
\end{algorithm}

The residual adaptor Algorithm~\ref{alg:adapt-resid} is closely related to the point prediction type baseline resp.~its more general boosting type version given in Algorithm~\ref{alg:classiprob}, though both are slightly different. The residual adaptor Algorithm~\ref{alg:adapt-resid} combines two point prediction strategies, one for the location and one for the dispersion, into a probabilistic strategy, while Algorithm~\ref{alg:classiprob} takes as input one point prediction strategy for the location and a probabilistic density estimation strategy for the distributional shape. Both agree in the sub-case where the density estimator $h$ in Algorithm~\ref{alg:classiprob} has the form of $[x\mapsto p(0,h(x))]$ in Algorithm~\ref{alg:adapt-resid} for a centered family of distributions $p$.

In this practically very relevant sub-case we would like to make an important remark: empirical experiments indicate that it may be advantageous to lower bound the output of the point prediction functional $h$ in Algorithm~\ref{alg:adapt-resid} which predicts the dispersion. This lower bound itself may be tuned or predicted, but it can also be fixed (to a fixed fraction of the label variance, say). On an interesting note, this is itself a composition step around the point predictor, so with this bound, the final form of Algorithm~\ref{alg:adapt-resid} is a four-fold hierarchical composite.

Algorithm~\ref{alg:adapt-resid} also has a number of straightforward directions of generalization:
\begin{enumerate}
\item[(i)] Replacing squared residuals by absolute and log-residuals, and/or changing the dispersion parameter in terms of these residuals.
\item[(ii)] Replacing squared loss with other point prediction losses, concordant with the parametric form of $p$, as per Proposition~\ref{Prop:classic}.
\item[(iii)] For multivariate predictions, i.e., $\calY = \RR^m$, one may want to choose prediction strategies $g$ t.v.i.~$[\calX\rightarrow \calY]$ and $h$ t.v.i.~$\calX\rightarrow [\calY\otimes\calY]$, where $\calY\otimes \calY = \RR^{m\times m}$ is the tensor product space containing multivariate dispersion parameters such as covariance matrices. The residuals, in this case, need to be defined via the tensor square, e.g., $\rho^2_i := (g(X_i)-Y_i)\cdot (g(X_i)-Y_i)^\top$.
\end{enumerate}

\subsubsection{Adaptors for empirical samples such as from the Bayesian predictive posterior}

An important class of potential output distributions are empirical sample distributions, i.e., distributions which are averages of point masses. We introduce convenient notation for these.

\begin{Def}
For any set, denote by $S^\ast$ the set of tuples in $S$ of arbitrary length; for (boldface) $\bs\in\S^\ast$, denote by $s_i$ (non-boldface with sub-index) the $i$-th element, and by $\ell(s)$ the length of $s$.\\
\end{Def}

\begin{Def}
An empirical (sample) distribution on $\calY$ is a distribution of the form
$$\delta(\by):=\frac{1}{m}\sum_{i=1}^{\ell(\by)} \delta (y_i),$$
where for $y\in\calY$, we denote by $\delta(y)$ the delta distribution supported at $y$.\\
$\delta(\by)\in\Distr(\calY)$ is called the empirical distribution of $\by$, considered as a fixed ordered sample.
\end{Def}

Many Bayesian prediction strategies yield empirical distributions as outputs when sampling from a predictive posterior, even if the task is to predict a continuous distribution. A particularity of this setting is that the number of points in the posterior sample may be chosen by the experimenter, as a parameter of the method, and/or may vary.

An empirical-sample-to-continuous adaptor strategy will hence have the type $\calY^\ast\rightarrow \calP$, where we identify $\calY^\ast$ with empirical sample distributions, and where $\calP$ is the set of continuous distributions.
If a mixed loss such as in Section~\ref{sec:mixed} is used, no adaption is necessary, but transformation to continuous distributions may still benefit predictive performance.

Generally, it should also be noted that adaptors of the above type $[\calY^\ast\rightarrow \calP]$ superficially have the same core type as the density estimation baselines discussed in Section~\ref{Sec:uninfbase} which do not make use of training features, as both types of methods take as input a $\calY$-valued sample, and output a (potentially continuous) distribution. However, there is a subtle difference: the $\calY$-valued sample is the training set for the uninformed density estimation baseline, and the same learnt density is returned for every test feature. For the adaptor, a different $\calY$-valued sample, in $\calY^\ast$, is seen for each \emph{test data point} and is to be converted into a different continuous sample each. Despite this contextual difference, though, similar strategies may be used for both and are straightforward to convert to the other situation.

Two basic strategies are kernel and histogram density estimation, Algorithms~\ref{alg:adapt-kern} and~\ref{alg:adapt-hist} below.

\begin{algorithm}[ht]
\caption{Kernel density estimation adaptor.\newline
\textit{Input:} Empirical distribution $\delta(\by)$ with $\by\in\calY^\ast$. Continuous kernel distribution function $k_\sigma:\calY\rightarrow \RR^+$ with variance parameter $\sigma$, which is also a pdf. Bootstrap parameters $B,b$.\newline
\textit{Output:} A continuous distribution with pdf $p:\calY\rightarrow \RR^+$ \label{alg:adapt-kern}}
\begin{algorithmic}[1]
    \State Randomly sample $\by_1,\dots, \by_B\subseteq \by$, where $\card \by_i = b$.
    \State Estimate variances $\sigma_i$ of the samples $\by_i$, for all $i=1..B$ (e.g., by the empirical sample variance, or setting a pre-determined $\sigma$ if $\card\by_i = 1$).
    \State Return $p:=\left[x\mapsto \frac{1}{B}\sum_{i=1}^B \sum_{y\in\by_i} k_{\sigma_i}(y-x) \right]$
\end{algorithmic}
\end{algorithm}

\begin{algorithm}[ht]
\caption{Histogram density estimation adaptor.\newline
\textit{Input:} Empirical distribution $\delta(\by)$ with $\by\in\calY^\ast$. Disjoint bins $B_1,\dots, B_m\subseteq \calY$ with volume $V_1,\dots, V_m\in \RR$.\newline
\textit{Output:} A continuous distribution with pdf $p:\calY\rightarrow \RR^+$\label{alg:adapt-hist}}
\begin{algorithmic}[1]
    \State Define $V\leftarrow \sum_{i=1}^m V_m$
    \State Let $b_i\leftarrow \card (B_i\cap \by)/\card\by$ for all $i=1..m$
    \State Return $p:=\left[x\mapsto \sum_{i=1}^m \card(B_i\cap \{x\})\cdot \frac{b_i}{V_i} \right]$
\end{algorithmic}
\end{algorithm}

Classical kernel density estimation in Algorithm~\ref{alg:adapt-kern} is obtained for choosing $b=1,B=m$, taking $\by_i$ disjoint, and fixing $\sigma$ to be the same.
Histogram estimation via Algorithm~\ref{alg:adapt-hist} may have to be combined with target re-normalization in Section~\ref{sec:logmix} as histogram bars outside the bins $B_i$ will carry density zero which may lead to infinite generalization log-loss.
For both algorithms, hyper-parameters such as the choice of density kernel and histogram bin choice may be tuned on the training set.

\subsubsection{Adaptors for mixed distributions}
\label{sec:adaptors.mixed}
The case where a general mixed distribution is predicted is as rare as off-shelf algorithms which produce mixed predictions. Converting such a distribution to a continuous one may be part of the prediction strategy, required as the task is to predict a continuous distribution, or as part of the convolution loss presented in Section~\ref{sec:mixed.conv}.

In either case, convolution is (not the only but) a straightforward choice; Algorithm~\ref{alg:adapt-conv} provides a probabilistic algorithm for obtaining an approximation to a continuous convolution distribution.

Before stating it, we briefly introduce notation for easy notation of mixed distributions:

\begin{Def}
Denote by $\Delta(\RR)$ the elements of $(\RR^+)^\ast$ which sum to one, i.e., the probability simplex of arbitrary length.\\
For $\by\in\calY^\ast$ and $\bw\in\Delta(\RR)$ of equal length $\ell(\by) = \ell(\bw)$, define
$$\delta(\by,\bw) := \sum_{i=1}^{\ell(\by)} w_i\cdot \delta(y_i).$$
\end{Def}

As discussed in Section~\ref{sec:mixed.split}, every mixed decomposition $p$ may be decomposed as $p=\alpha_c q + \alpha_d\delta(\by,\bw)$ for $(\alpha_c,\alpha_d),\bw\in\Delta(\RR)$, $\by\in\calY^\ast$, and $q\in\Distr(\calY)$ continuous.

We are ready to state the convolution Algorithm~\ref{alg:adapt-conv}.

\begin{algorithm}[ht]
\caption{Convolution adaptor.\newline
\textit{Input:} Mixed distribution $p = \alpha_c q + \alpha_d\delta(\by)\in \Distr(\calY)$, with $q$ continuous and $\by\in\calY^\ast$. For convolution, a continuous random variable $Z$ (sampleable) with pdf $p_Z$. Sampling size parameter $m$.\newline
\textit{Output:} An approximation of $p\ast p_Z$. \label{alg:adapt-conv}}
\begin{algorithmic}[1]
    \State Return the function $f$ below as the convoluted density.
    \Function{$f$}{$y\in\calY$}
    \State Initialize $p_y\leftarrow \alpha_d\sum_{i=1}^m p_Z(y-y_i).$
    \State If $(\alpha_c,\alpha_d) = (0,1)$, return $p_y$ and terminate; the approximation is exact. Otherwise continue.
    \State Obtain an i.i.d. random sample $Z_1,\dots, Z_m\sim Z$.
    \State Set $p_y \leftarrow p_y + \alpha_c\frac{1}{m}\sum_{i=1}^m q(y-Z_i)$.
    \State Return $p_y$.
    \EndFunction
\end{algorithmic}
\end{algorithm}

We briefly show in which sense Algorithm~\ref{alg:adapt-conv} is correct and by which asymptotic it approximates:

\begin{Prop}
Keep the notation of Algorithm~\ref{alg:adapt-conv}.\\
Write $g:=[y\mapsto \frac{1}{m}\sum_{i=1}^m q(y-Z_i)],$ note that $g$ is a random variable t.v.i.~$[\calY\rightarrow \RR^+]$.\\
It holds that:
\begin{enumerate}
\item[(i)] $f = \alpha_c\cdot g + \alpha_d\cdot \delta(\by)\ast p_Z$,
\item[(ii.a)] $\EE[g] = q\ast p_Z.$
\item[(ii.b)] In particular, $\EE[f] = p\ast p_Z$.
\item[(iii)] $\Var[f(y)] = \alpha_c^2\Var[g(y)] = \frac{\alpha_c^2}{\sqrt{m}}\Var[q(y-Z)],$ for any $y\in\calY$.
\item[(iv)] If $\Var[f(y)]$ is finite for a fixed $y\in\calY$, then
$$\Var[f(y)] = \alpha_c^2\sqrt{m}\left(f(y) - \EE f(y)\right)\overset{d}{\rightarrow} \calN \left(0, \alpha_c^2 \Var  q(y-Z)\right),\;\mbox{as}\; m\rightarrow \infty.$$
\end{enumerate}
\end{Prop}
\begin{proof}
(i) follows from writing out definitions, and noting that $\delta(y_i)\ast p_Z = [y\mapsto p_Z(y-y_i)]$.\\
(ii.a) follows from observing that for all $i = 1..m$, it holds that
$$\EE\left[q(y-Z_i)\right] = \int_\calY p_Z(z)\cdot q(y-z)\;\diff z = [q\ast p_Z](y).$$
The full statement follows from this, linearity of expectation and definition of $\EE[g]$.\\
(ii.b) follows from (ii.a) and linearity of convolution.\\
(iii) follows from substituting definitions and observing that the summands $q(y-Z_i)$ are independent since the $Z_i$ are.\\
(iv) follows from the central limit theorem for the mean of i.i.d.~samples
$g(y) = \frac{1}{m}\sum_{i=1}^m q(y-Z_i).$
\end{proof}

\subsubsection{Conversion to parametric distributions}

It may also be beneficial to convert a general non-parametric predicted distribution into a distribution from a prescribed parametric family, especially if it is known that the conditionals will be similar to that form. This task is special in three respects:
\begin{enumerate}
\item[(i)] It has the format of distributional regression, i.e., the input is a distribution.
\item[(ii)] It is probabilistic, i.e., the output is also a distribution.
\item[(iii)] It is supervised, i.e., the output distribution sought is not necessarily the ``closest'' in a discrepancy sense, or ``best'' in an estimation sense, but the one with the lowest probabilistic generalization loss.
\end{enumerate}

Having said that, nevertheless parametric estimation and discrepancy estimation strategies apply in principle, and form a class of baseline adaptors. Both are classical topics in statistics, though studying modifications to cope with all the features (i)-(iii) above is beyond the scope of this manuscript, especially since the situations in which there is an empirical benefit to converting a non-parametric prediction back to a parametric distribution type is not entirely clear.

\subsection{Bagging and model averaging for probabilistic prediction strategies}
\label{sec:baggeraging}

Bagging~\cite{breiman1996bagging} and Bayesian model averaging~\cite{raftery1997bayesian,hoeting1999bayesian} are popular techniques to improve model fit and generalization ability, leading to some of the best performing algorithms in practice, such as Breiman's random forests~\cite{breiman2001random} or Bayesian additive regression trees~\cite{chipman2010bart}.

In the supervised prediction setting, these techniques may be seen as a sub-case of entropic variance reduction by probabilistic marginalization as discussed in Corollary~\ref{Cor:posterior} which
we will slightly reformulate to make the connection to the bagging/averaging settings clearer:

\begin{Cor}
\label{Cor:bagging}
Let $\Theta$ be some parameter set. For a fixed parameter $\theta\in \Theta$,
let $f_\theta$ be a prediction strategy, i.e., a $[\calX\rightarrow\Distr(\calY)]$-valued random variable.
Let $\calD$ be any random variable, e.g., the training data and/or the source of uncontrollable randomness in $f_\theta$.
The following are true:
\begin{itemize}
\item[(i)] For any $\Theta$-valued random variable $Z$, it holds that $\varepsilon(\EE[f_Z|\calD])\le \EE\left[ \varepsilon(f_Z)\right]$.
\item[(ii)] For any (random or deterministic) $\theta_1,\dots, \theta_B\in\Theta$, it holds that
$$ \varepsilon\left(\frac{1}{B}\sum_{i=1}^B \widehat{f}_{\theta_i} \right)\le \frac{1}{B}\sum_{i=1}^B \varepsilon\left(\widehat{f}_{\theta_i}\right)$$
(where the averaged predictor predicts the average of distributions).
\end{itemize}
(ii) is a special case of (i), where $Z$ is the uniform distribution over $\{\theta_1,\dots, \theta_B\}$.
The difference between left hand side and right hand side, in either case, is $\Var (f_Z|\calD)$.
\end{Cor}
\begin{proof}
Both statements follow from Corollary~\ref{Cor:posterior} where additionally expectations over $X$ are taken.
For (i), take the random variable $\theta$ in Corollary~\ref{Cor:posterior} to be $Z$ in this statement.
For (ii), substitute for $\theta$ in Corollary~\ref{Cor:posterior} the uniform distribution over $\{\theta_1,\dots, \theta_B\}$ into (i).
For the statement on the difference, note that expectations taken inside or outside only make a difference in case of the $\Var$ term in the bias-variance trade-off, since the others are deterministic.
\end{proof}

Corollary~\ref{Cor:bagging} hence yields strategies for constructing a stronger model out of a collection of weaker ones: select $Z$ in either case in a way such that
$\Var (f_Z | \calD)$ is large. Or, modify an initial model $f$ to one depending on $Z$, in a way that $\varepsilon(f_Z)$ is close to $\varepsilon(f)$ and $\Var (f_Z | \calD)$ is large.
In vernacular use, the strategy in Corollary~\ref{Cor:bagging}~(i) would be understood as ``model averaging'', especially if $\Theta$ parameterizes different, possibly nested models (or model classes); the strategy in Corollary~\ref{Cor:bagging}~(ii) would be understood as ``bagging'' a predictor, especially if the $\theta_i$ or $Z$ are some predictor $f$ learnt on independent re- or sub-samples of the training sample and/or feature set (and usually analyzed in the context of the classical squared loss or bias-variance trade-off, see Section~\ref{Sec:biasvariance} for the full correspondence).

The above does not indicate directly which types models profit from bagging or model averaging, though it implies that it should be exactly those which
have a high (entropic) variance dependence on hyper-parameter choices, such as decision trees in the classical setting.

It is also interesting to note that the probabilistic way of bagging appears to differ from the classical bagging strategy, as the probabilistic version involves mixing the output distributions, while the second involves averaging of the prediction itself, which by Section~\ref{sec:classprob} corresponds to averaging location parameters of parametric distributions rather than mixing them. This seeming discrepancy can be reconciled in the case of the deterministic squared loss (or the mean squared error) by noting that the predictive means, i.e., the expectations of the predicted distributions are the same in both cases, leading to the same point prediction strategy when considering the best point estimate as measured by the mean squared loss.\\
For deterministic losses other than the mean loss, the two bagging strategies in general need not agree with each other, however other parallels may be drawn, such as from the mean absolute loss to bragging (see e.g.~\cite{buhlmann2003bagging}), which is a parallel technique involving median aggregation.

\newpage
\subsection{Boosting of probabilistic predictors}
\label{sec:boost}

Many powerful meta-learning and ensembling strategies in supervised regression learning (i.e., in the situation $\calY\subseteq \RR$) rely on insightful use of residuals in the training phase. The most prominent example is boosting which is one of the most powerful meta-learning strategies in classical supervised learning, yielding well-known algorithms such as AdaBoost or the gradient boosting machines. Other examples include stacking and nesting strategies which, among others, leads to the back-propagation type algorithms for neural networks.

Boosting is a meta-learning strategy which is based on incrementally reducing the loss of a prediction strategy by modifying the prediction by a weak learner's prediction. Boosted predictors, in particular gradient boosting machines based on tree learners, have empirically proven to be superior in supervised
prediction tasks~\cite{friedman2000additive,friedman2001greedy,schapire2003boosting}.

Close connections of boosting to back-fitting, greedy-fitting, and incremental optimization of the log-likelihood have already been noted by~\citet{friedman2000additive} in the context of AdaBoost and generalized linear models.

Algorithm~\ref{alg:greedyboost} provides a probabilistic formulation of greedy-type residual boosting, based on Corollary~\ref{Cor:trafoboost}, which can be applied in the where case $\calY \subseteq \RR$, similar to classical back- or greedy-fitting. As in the classical case, a learning parameter $\alpha$ allows to regulate the strength of the boosting update, which is tunable, also dynamically.

\begin{algorithm}[h]
\caption{Greedy probabilistic residual boosting.\newline
\textit{Input:} Training data $\calD = \{(X_1,Y_1),\dots (X_N,Y_N)\}$. Data $(X_i,Y_i)$ take values in $\calX\times \calY$ where $\calY \subseteq \RR$.\newline
Weak learning strategies $f_1,\dots, f_k$, t.v.i.~$\calX\rightarrow \calY$, that can be trained on plug-in data.\newline Learning rate parameter $\alpha \in [0,1]$.\newline
\textit{Output:} a boosted predictor $f:\calX\rightarrow \Distr(\calY)$ \label{alg:greedyboost}}
\begin{algorithmic}[1]
    \State Fit an uninformed prediction functional $f^{(0)} : \calX\rightarrow \Distr (\calY)$ to $Y_1,\dots, Y_N$ (this could be ``always predict the sigmoid'',  as in Section~\ref{sec:logmix}).
    \State Let $F^{(0)}$ be the cdf corresponding to the unique predicted density of $f^{(0)}$.
    \State Set $\rho_{0j}\leftarrow F^{(0)}(Y_j)$ for all $j=1\dots N$.
    \For{$i=1\dots k$}
    \For{$j=1\dots N$}
    \State Denote by $F_{ij}$ the cdf to the pdf $f^{(i-1)}(X_j)$
    \State Compute residuals: set $\rho_{ij} \leftarrow F_{ij}(\rho_{{i-1},j})$
    \EndFor
    \State Fit strategy $f^{(i)}$ to data $(X_j,\rho_{ij}),j=1\dots N.$
    \State $f^{(i)}\leftarrow \alpha u + (1-\alpha) f^{(i)}$, where $u$ is the uniform uninformed predictor
    \EndFor
    \State Return $f:= f^{(0)}_\sharp f^{(1)}_\sharp f^{(2)}_\sharp  \dots f^{(N-1)}_\sharp f^{(N)}$, with abbreviating pull-back notation as in Definition~\ref{Def:compose}.
\end{algorithmic}
\end{algorithm}

\begin{algorithm}[h]
\caption{Gentle probabilistic boosting.\newline
\textit{Input:} Training data $\calD = \{(X_1,Y_1),\dots (X_N,Y_N)\}$. Data $(X_i,Y_i)$ take values in $\calX\times \calY$.
Weak learning strategies $1\dots M$ that can be trained on weighted training data. Learning rate parameters $\alpha,\gamma \in [0,\infty]$.\newline
\textit{Output:} a boosted predictor $f:\calX\rightarrow \Distr(\calY)$ \label{alg:gentleboost}}
\begin{algorithmic}[1]
    \State Initialize weights $w_1,\dots, w_N \leftarrow 1/N$.
    \State Learn an uninformed prediction functional $b_0 : \calX\rightarrow \Distr (\calY)$.
    \For{$i=1\dots M$}
    \State Fit strategy $i$ to $\calD$ using weight $w_j$ with data $(X_j,Y_j), j=1\dots N,$\\
     to obtain prediction functional $b_i: \calX\rightarrow \Distr (\calY)$
    \State Compute $\widehat{\beta} \leftarrow \argmin_\beta \sum_{j=1}^N L(f_{i-1}(X_j) + \beta b_i(X_j),Y_j)$, e.g., by gradient descent or by approximation
    \State Set $f_i\leftarrow (1-\gamma\cdot \widehat{\beta})\cdot f_{i-1} + \gamma\cdot\widehat{\beta}\cdot b_i$
    \State Update $w_i \leftarrow w_i + \alpha\cdot w_i\cdot L(f_i(X_i),Y_i)$
    \State Update $w_i \leftarrow w_i/\sum_{i=1}^N w_i$
    \EndFor
    \State Return $f:= f_m$.
\end{algorithmic}
\end{algorithm}

Algorithm~\ref{alg:gentleboost} provides an extension to the probabilistic case based on the Gentle AdaBoost idea~\cite[Algorithm 4]{friedman2000additive}
and observation~3 of~\cite[Section~4.2]{friedman2000additive} by which the AdaBoost update approximates a log-loss based update, which takes the place of the
AdaBoost update in Algorithm~\ref{alg:gentleboost}. Further, the probabilistic algorithm returns full probabilistic prediction which is a (linear) mixture of
the weak learner's probabilistic predictions.
Apart from that, Algorithm~\ref{alg:gentleboost} is a straightforward generalization of (a probabilistic variants of) Gentle AdaBoost. In accordance with modern implementations of boosting, tunable learning parameters  $\alpha,\gamma$ for the size of the model updates are included.

Corollary~\ref{Cor:bagging} indicates that for both algorithms, one would expect a performance that is at least not worse than of the unboosted variants.

\subsection{A multi-variate test for independence}
\label{Sec:hypotest}
With Theorem~\ref{Thm:entropy}, we have shown that statistical independence is equivalent to predictability.
This directly motivates a statistical test for independence, using model comparison of probabilistic prediction strategies.

More precisely, statistical independence can be tested directly by the scientifically necessary baseline comparison between the best uninformed predictor and a candidate prediction strategy.
By the remaining results in Section~\ref{sec:entropy}, this comparison may equally be interpreted as a hypothesis test for the entropy $\Ent (Y)$ and the conditional entropy $\Ent (Y/X)$ being different.

\subsubsection{Constructing a probabilistic predictive independence test}
We re-state previous equivalences which directly relate the independence testing task, the model validation task, and the entropy comparison task. In the following, notation is as in the generic probabilistic supervised learning setting task (Section~\ref{sec:propreset}), and we consider a fixed, strictly proper, strictly convex loss $L:\calP\times \calY\rightarrow \RR$ with $\calP\subseteq \Distr (\calY)$.

\begin{Prop} \label{Prop:estent}
In the above setting, the following are equivalent:
\begin{itemize}
\item[(i)] $X$ and $Y$ are not statistically independent.
\item[(ii)] There exists a prediction functional $f:\calX\rightarrow \Distr (\calY)$ such that $\varepsilon_L(f) \lneq \Ent_L(Y)$.
\item[(iii)] There exists a prediction functional $f:\calX\rightarrow \Distr (\calY)$ such that $\varepsilon_L (f) \lneq \varepsilon_L (g)$ for all uninformed prediction functionals $g$.
\item[(iv)] $\Ent_L(Y/X)\lneq \Ent_L(Y)$
\item[(v)] $\Bias_L (\varpi_Y)\gneq 0$
\end{itemize}
The validity of the above equivalence statement is unchanged by:
\begin{itemize}
\item[(a)] replacing any of the three mentions of prediction functionals in (ii) and (iii) by mentions of prediction strategies, t.v.i.~$[\calX\rightarrow \Distr (\calY)]$
\item[(b)] replacing quantification over ``all uninformed prediction functionals $g$'' in (iii) by the best uninformed predictor $g=\varpi_Y$.
\end{itemize}
\end{Prop}
\begin{proof}
This follows from the equivalences in Theorem~\ref{Thm:entropy}, Lemma~\ref{Lem:Entpre}, and Proposition~\ref{Prop:uninformed}.
\end{proof}

The equivalence in Proposition~\ref{Prop:estent} are directly applicable for hypothesis testing.
Namely, a prediction of $Y$ from $X$ which successfully outperforms all possible uninformed baselines (i.e., a ``better-than-uninformed prediction'') certifies for the alternative hypothesis, statistical dependence of $Y$ on $X$, in the following precise way:

\begin{Rem}
Let $\widehat{\varpi}_{Y|X}$ be a consistent (in the training size) estimate for the best predictor $\varpi_{Y|X}$, and let
$\widehat{\varpi}_{Y}$ be a consistent (in the training size) estimate for the best uninformed predictor $\varpi_{Y}$.
Let $\widehat{\varepsilon}_L$ be a consistent (in the test size) estimation procedure for the generalization loss $\varepsilon_L$,
e.g., the empirical test loss as described in Section~\ref{sec:modelvalidation.estim}. Then, one can argue that:
\begin{itemize}
\item[(i)] $\widehat{\varepsilon}_L(\widehat{\varpi}_{Y|X})$ is a consistent (in minimum of training and test size) estimate for $\Ent_L(Y/X)$
\item[(ii)] $\widehat{\varepsilon}_L(\widehat{\varpi}_{Y})$ is a consistent (in minimum of training and test size) estimate for $\Ent_L(Y)$
\item[(iii)] $\widehat{\varepsilon}_L(\widehat{\varpi}_{Y|X}) - \widehat{\varepsilon}_L(\widehat{\varpi}_{Y})$ is a consistent (in minimum of training and test size) estimate for $\Bias_L(\varpi_{Y|X})$
\item[(iv)] Any hypothesis test for whether $\varepsilon_L(\widehat{\varpi}_{Y|X}) \lneq \varepsilon_L(\widehat{\varpi}_{Y})$ is a hypothesis test for whether $\Ent_L(Y/X)\lneq \Ent_L(Y),$ or equivalently for whether $\Bias_L (\varpi_{Y|X})\gneq 0$.
\end{itemize}
More precisely, for an i.i.d.~sample $(X_1,Y_1),\dots,(X_N,Y_N)\sim (X,Y)$, any (possibly paired) test for whether the location of
$\calS_{Y|X}:= \{L(\widehat{\varpi}_{Y|X}(X_1),Y_1),\dots, L(\widehat{\varpi}_{Y|X}(X_N),Y_N)\}$ is smaller than the location of\\
$\calS_Y:= \{L(\widehat{\varpi}_{Y}(X_1),Y_1),\dots, L(\widehat{\varpi}_{Y}(X_N),Y_N)\}$ is, in a frequentist hypothesis testing sense, a test for whether $Y$ and $X$ are statistically dependent.
The null hypothesis would naturally be ``the location of $\calS_{Y|X}$ equals or is larger than the location of $\calS_{Y}$''.
\end{Rem}

Note that we have used estimates $\widehat{\varpi}_{Y|X}$ and $\widehat{\varpi}_{Y}$ for the best predictor and best uninformed
predictor - estimates which are naturally obtained from any benchmarking experiment, since the purpose of such an experiment is
to find the best predictor, and to ensure that it outperforms the best (uninformed) baseline.
Hence, such estimates may also be understood as the natural, and in fact best possible surrogates to the ``true'' best predictors
$\varpi_{Y|X}$ and $\varpi_{Y}$ obtainable in a benchmarking context. A similar reasoning holds if $\widehat{\varpi}_{Y|X}$ is not a consistent estimator for $\varpi_{Y|X}$ but merely for some ($L-$)better-than-uninformed prediction functional.

Algorithm~\ref{alg:indtest} codifies this procedure as a pseudo-code meta-algorithm, yielding a hypothesis test for
independence of $X$ and $Y$ from i.i.d.~samples $(X_1,Y_1),\dots,(X_N,Y_N)\sim (X,Y)$.
Two-sample independence testing may be seen as a special case
where $Y$ takes values in $\calY=\{-1,1\}$ indicating which sample the observation is in, for convenience pseudo-code is given as Algorithm~\ref{alg:RFtest} (with an explicit baseline performance estimate for the log-loss).

\begin{algorithm}[ht]
\caption{Meta-algorithm: a multi-variate probabilistic predictive independence test.\newline
\textit{Input:} Training data $\calD$, test data $\calT$, following i.i.d.~a joint distribution $(X,Y)$ t.v.i.~$\calX\times \calY$\newline
\textit{Output:} a (frequentist) $p$-value and test statistic on whether $X$ and $Y$ are independent\label{alg:indtest}}
\begin{algorithmic}[1]
    \State use $\calD$ (and a good model selection strategy) to compute\newline a good uninformed predictor $\widehat{\varpi}_{Y}:\calX\rightarrow \Distr(\calY)$
    \State use $\calD$ (and a good model selection strategy) to compute a good predictor $\widehat{\varpi}_{Y|X}:\calX\rightarrow \Distr(\calY)$
    \State Compute the sample of (informed) losses $\calS_{Y|X}:= \left\{L(\widehat{\varpi}_{Y|X}(x),y)\;:\;(x,y)\in\calT\right\}$
    \State Compute the sample of (uninformed) losses $\calS_Y:= \left\{L( \widehat{\varpi}_{Y}(x),y)\;:\;(x,y)\in\calT\right\}$
    \State Conduct a paired sample test for ``the location of $\calS_{Y|X}$ equals or is larger than the location of $\calS_{X}$'', as in Section~\ref{Sec:modelvalidation.tuning}.
    \State Return the resulting $p$-value and test statistic.
\end{algorithmic}
\end{algorithm}

\begin{algorithm}[ht]
\caption{Special case: a log-loss and random forest or [other favourite classifier] based two-sample test.\newline
\textit{Input:} Two i.i.d. samples $\calD_1,\calD_{-1}$ (sampled from two possibly distinct distributions), t.v.i.~$\calX$\newline
\textit{Output:} a $p$-value and test statistic whether $\calD_1$ and $\calD_{-1}$ are sampled from different distributions\label{alg:RFtest}}
\begin{algorithmic}[1]
    \State Merge the data samples with a ``source index'', i.e., define $\calD_{\cup} \leftarrow\{(x,i)\;:\;\;x\in \calD_i\}$
    \State Split $\calD_{\cup}$ randomly into a training set $\calD$ and test set $\calT$ with elements in $\calX\times \calY$, where $\calY = \{-1,1\}$
    \State On the training set $\calD$, train a probabilistic random forest (or other favourite classifier)
            to obtain the prediction rule $\widehat{\varpi}_{Y|X}:\calX\rightarrow \Distr(\calY)$
    \State Compute the sample of (informed) losses $\calS_{Y|X}:= \{-\log\widehat{\varpi}_{Y|X}(x)(y)\;:\;(x,y)\in\calT\}$
    \State Compute the empirical entropy of the best uninformed predictor $S_Y := - p_1\log p_1 - p_{-1} \log p_{-1},$ where $p_i:=\card\calD_i/\card\calD_{\cup}$
            (this is known explicitly: predicting as probabilities the observed class frequencies)
    \State Conduct a one-sample test for ``the location of $\calS_{Y|X}$ equals or is larger than $S_Y$''.
    \State Return the resulting $p$-value and test statistic.
\end{algorithmic}
\end{algorithm}

Note the major weaknesses of Algorithms~\ref{alg:indtest} and~\ref{alg:RFtest} which are: reliance on further unspecified model selection strategies,
and the choice of a paired (or unpaired) sample test on location. The latter is not severe, as one may choose to use a paired $t$-test for large sample sizes or
the non-parametric Wilcoxon signed-rank test (or rank-sum test).
The former is more an issue since model selection, especially in a probabilistic setting, remains a field which is quite open. However, model selection in a predictive setting is much better explored than test statistic selection in the field of multivariate independence testing.

\subsubsection{Workflow meta-interfaces in the probabilistic setting}

The strategy outlined above closely follows the rationale of~\citet{burkart2017predictive} in which an equivalent independence testing strategy using classical, point prediction losses and prediction strategies is discussed. As such, it shares its main properties and (dis-)advantages compared to literature (as discussed in~\cite{burkart2017predictive}, see there), including its main advantage which is the close link to the well-studied supervised learning workflow, including access to powerful off-shelf prediction methods as well as re-usability of automated model tuning and parameter selection strategies, which can all be plugged into the hypothesis testing strategy.

Similarly, conditional predictive independence tests and graphical model structure learning algorithms may be constructed by the same principles as outlined in~\cite{burkart2017predictive}, leveraging the probabilistic prediction workflow.

The main notable difference between using classical and probabilistic losses is that a single strictly proper loss is sufficient to certify for independence, as may be seen from Proposition~\ref{Prop:estent}, irrespective of the target domain $\calY$. This is in crass contrast to the deterministic case where an infinity of losses may be necessary for an equivalent certificate of independence. In the terminology of~\cite{burkart2017predictive}, any fixed strictly proper loss always forms a one-element faithful set in certifying for independence; re-stated, in contrast to the classical setting, there is never more than one strictly proper loss needed to certify for statistical independence.

This property of the probabilistic setting also allows to remove the issue with the deterministic predictive testing strategy which is the need to make a potential arbitrary choice of a (possibly infinite) set of losses, potentially with difficult estimation properties, to obtain an equivalent certificate for independence. Namely, by Proposition~\ref{Prop:logcanon}, the choice of log-loss is in a sense canonical.
Alternatively, Example~\ref{Ex:BiasDiv} relates the test statistic to a divergence estimate, which is the conditional KL-divergence for the log-loss and the mean integrated squared error of distributional prediction for the squared integrated loss.

\subsubsection{Relation to kernel discrepancy based testing procedures}

It is interesting to note that a number of procedures from the family of kernel discrepancy based independence tests may be seen as sub-cases of the probabilistic predictive independence test (Algorithm~\ref{alg:indtest}), obtained for specific choices of pairs of probabilistic prediction strategies, and instances of the kernel discrepancy loss in Section~\ref{sec:mixed.kernel}.

Most illustratively, we may obtain the two-sample testing procedure of~\citet{gretton2012kernel} by attempting to probabilistically predict what are usually the features from a binary class label:

\begin{Ex}
Consider a ``flipped'' binary classification setting, where the data generative process $(X,Y)$ takes values in $\calX\times \calY$ where $\calX=\{-1,1\}$. The data are obtained as i.i.d.~samples $(X_1,Y_1),\dots,(X_N,Y_N)\sim (X,Y)$, and we consider sub-sets $\bY(x) := \{Y_1^{(x)},\dots, Y_{N_x}^{(x)}\}$ such that $\bY(x)=\{Y_i\;:\; X_i = x\}$ for $x\in\calX$. That is, $\bY(1)$ is the (random) collection of all $Y_i$ with feature value $X_i=1$, and $\bY(-1)$ is the collection of all $Y_i$ with feature value $-1$.\\

As already noted by~\citet{sriperumbudur2009kernel}, the question of testing whether $X$ can be predicted from $Y$ (i.e., the ``common'' classification way rather the ``flipped'' way) is intimately related to the question whether the samples $\bY^{(1)},\bY^{(-1)}$ come from the same distribution, i.e., the two-sample testing task.

We show that we are even able to recover the MMD test statistic of~\citet{gretton2012kernel} by relating the two-sample testing task to the question of whether $Y$ can be predicted from $X$. For this, consider the kernel discrepancy loss $L_k$ from Section~\ref{sec:mixed.kernel}, and the following prediction strategy to certify for independence:
$$\widehat{\varpi}_{Y|X}: \calX\rightarrow \Distr(\calY)\;:\;x\mapsto \frac{1}{N_x}\sum_{y\in\bY (x)} \delta(y),$$
where $\delta(y)$ is the delta distribution located at $\calY$. More intuitively stated, $\widehat{\varpi}_{Y|X}$ predicts, for an input $x\in\calX = \{-1,1\}$ the empirical distribution of labels which have been observed with a feature $x$. As the uninformed baseline predictor estimate, we consider
$$\widehat{\varpi}_{Y} : \calX\rightarrow \Distr(\calY)\;:\; x\mapsto \frac{1}{N}\sum_{i=1}^N \delta(Y_i),$$
i.e., simply always predicting the empirical sample of the labels.

In order to estimate $\Bias_L(\varpi_{Y|X})$, we evaluate $\widehat{\varpi}_{Y|X}$ and $\widehat{\varpi}_{Y}$ in-sample, i.e., the test sample is \emph{the same sample} of $(X_i,Y_i)$ to which both predictors were fitted. An elementary computation yields
\begin{align*}
\widehat{\varepsilon}_{L_k}(\widehat{\varpi}_{Y|X}) - \widehat{\varepsilon}_{L_k}(\widehat{\varpi}_{Y}) &=
\frac{1}{N_{1}^2}\sum_{y,y'\in \bY(1)}k(y,y') - \frac{2}{N_1N_{-1}}\sum_{\substack{y\in \bY(1)\\y'\in\bY(-1)}}k(y,y') + \frac{1}{N_{-1}^2}\sum_{y,y'\in \bY(-1)}k(y,y'),
\end{align*}
where the right hand side is the kernel maximum mean discrepancy (MMD) statistic. Thus, two-sample testing via MMD is a particular sub-case of predictive independence testing with a kernel discrepancy loss and a pair of empirical prediction strategies (one conditional, one baseline).
\end{Ex}

\newpage
\section{Probabilistic supervised learning algorithms}
\label{sec:problearners}

We list some instances of algorithms that may be considered or adapted as probabilistic supervised learners.
These include off-shelf instances and descriptions of how existing modelling strategies may be adapted with minor modifications to the probabilistic supervised learning setting.

\subsection{Bayesian models}
\label{sec:bayesianmodels}
Probably the largest and most relevant class of models that may be instantiated for probabilistic supervised learning are Bayesian models.
However, not every Bayesian model is predictive, and even when predictive posteriors are obtained, there are subtle and possibly not immediately obvious conditions that should be met for applicability for the probabilistic prediction task.

Suppose we are in the usual supervised situation where $(X,Y)$ is a generative distribution t.v.i.~$\calX\times \calY$,
there is a training set $\calD = \{(X_1,Y_1),\dots, (X_N,Y_N)\}$ and a test set $\calT = \{(X^*_1,Y^*_1),\dots, (X^*_M,Y^*_M)\}$
(note the slight change in notation as we will need to talk about different test points individually).

In this scenario, Bayesian models may be invoked to compute, for any test index $i=1\dots M$, the posterior predictive distribution
$$p_i := p(y_i|X^*_i,\calD),$$
yielding probabilistic predictions $p_1,\dots, p_M$ which can then be checked against the $Y^*_1,\dots, Y^*_M$.

We would like to stress that this is the correct choice for supervised learning, because:

\begin{itemize}
\item[(i)] Any model parameters should be marginalized over, by Corollary~\ref{Cor:posterior}.
\item[(ii)] Predictions should be statistically independent when conditioned on the training data $\calD$, otherwise the out-of-sample log-loss will in general not be an estimate of the generalization loss as described in Section~\ref{sec:modelvalidation.estim}.
\end{itemize}

In particular, Bayesian models in a probabilistic supervised setting should \emph{not} be invoked as
\begin{itemize}
\item[(i)] Posteriors that may depend on, or are joint with model parameters; or
\item[(ii)] Predictive posteriors for $Y^*_i$ of the form $p_i := p(y_i|\calT,\calD),$ or $p_i := p(y_i|X^*_i,\calT\setminus \{(X^*_i,Y^*_i)\},\calD),$ say.
\end{itemize}

We would like to stress that we would do \emph{not} say that it wrong to obtain these posteriors in general, only that doing so prevents comparability with
non-Bayesian strategies and the uninformed baselines previously discussed.

\subsection{Probabilistic Classifiers}

For supervised classification, prediction of probabilities, even in semi-parametric, non-parametric, and non-Bayesian contexts, is quite standard.
More or less heuristic probabilistic variants of support vector machines, tree ensemble methods and neural networks (usually with a soft-max output layer, which includes logistic/multinomial regression) are routinely available in modern machine learning toolboxes.

However, it should be noted that the best uninformed baseline for this situation - always predicting as test probabilities the relative training frequencies (see the discussion in Section~\ref{Sec:uninfbase}, or, equivalently comparing to the entropy of the training label distribution) - does not seem to be part of the common workflow, at least it is not available in the usual off-shelf context.

Hence, if probabilistic classifiers are invoked, care should taken that an uninformed baseline comparison is conducted,
otherwise a solid statement that any method predicts better than (uninformed) ``guessing'' may not be made.

\subsection{Heteroscedastic regression and prediction intervals}
\label{sec:heteroscedastic-regression}

Heteroscedastic regression models aim at estimating regression models where the variance of the predictive distribution may depend on the
features.

Models of this type usually output an estimate for the predictive mean and variance, yielding so-called prediction intervals (sometimes: predictive intervals).
The more classical statistical models~\cite{park1966estimation, harvey1976estimating, stine1985bootstrap, davidian1987variance, vecchia1987simultaneous, patel1989prediction, welsh1994fitting} of this type
are usually based on either:
\begin{itemize}
\item[(i)] a two-step process, where first a classical regression model is fitted to predict the predictive mean,
and then another model is fitted to some monotone transform of the absolute residuals (e.g., the absolute residuals, squared residuals, log-residuals),
\item[(ii)] re-sample estimates of variance such as via bootstrapping or jackknifing the training set, though this will usually yield only
        an estimate for model stability and not the conditional predictive density.
\end{itemize}

A distributional prediction may be then obtained from combining the prediction for the location with a prediction of the variance.
It should be noted that most of these strategies may be applied to residuals of arbitrary black-box regression models,
similar to the more model-independent strategy suggested in Section~\ref{Sec:classicbase}.

More recently, the idea of heteroscedastic regression or variance prediction has re-surfaced in the machine learning community, mostly in the context of
neural networks with Gaussian target distributions~\cite{nix1994estimating,tibshirani1996comparison,heskes1997practical,hwang1997prediction,de1998prediction,papadopoulos2001confidence,
khosravi2011comprehensive}, but also for other popular non-linear supervised learning methods such as support vector machines~\cite{de2011approximate}
and bagged tree ensembles~\cite{wager2014confidence}.

\subsection{Conditional Density Estimation}

Conditional density estimation aims, in the probabilistic supervised regression setting, estimating the conditional densities $\calL (Y|X = x)$ for any given $x$.
This task is mathematically equivalent to the probabilistic supervised prediction task as any conditional density estimator may be interpreted as the probabilistic
predictor $x\mapsto \calL (Y|X = x)$, thus any implementation of a conditional density estimator may be directly interfaced.

Influential examples of this branch of literature, also containing further literature summaries, are~\cite{rosenblatt1969conditional,hyndman1996estimating, fan1996estimation, hall1999methods, hyndman2002nonparametric, hall2004cross, sugiyama2010conditional}.

However, even though the fitting part of the set-up is mathematically equivalent, it is worth noting some major features of conditional density estimation literature that are not quite congruent with the predictive set-up. Namely:
\begin{itemize}
\item[(i)] The goodness of conditional density estimation is, if at all, usually evaluated against an unknown distributional ground truth, for example via integrated squared error measures, as opposed to actual test observations. While the distribution-vs-distribution evaluation allows for proving approximation and learning theory results which are quite valuable, they cannot be used for external model validation proper as the true conditional distributions are unknown.
\item[(ii)] In most (but not all) of the conditional density estimation literature, the density estimate is not interpreted as a prediction, and fitting/prediction is not considered in the context of supervised learning. Consequently, parameter tuning of conditional density estimators is usually not done by predictive goodness estimation such as via cross-validation strategies, though there are exceptions~\cite{fan2004crossvalidation,hall2004cross}.
\item[(iii)] The fitting strategies considered under the term of conditional density estimation are usually (but not always) of the following specific kind: Nadaraya-Watson type kernel density estimators, kernel smoothing strategies, or variants of kernel basis function mixtures.
\end{itemize}

It is also worth noting, that the conditional density estimation literature is, to our knowledge, the only type of off-shelf density estimators that comes with rigorous convergence guarantees of the estimated conditional density to the (non-parametric) true density as the number of data points approaches infinity. Those guarantees are usually obtained as a natural generalization of guarantees of unconditional Nadaraya-Watson type estimators. In contrast, it is also interesting to note that the guarantees found in Bayesian literature are usually of different kind - namely, concerned with the posterior variance of model-specific parameters, or expected loss/utility.

\subsection{Gaussian Processes}
\label{sec:gaussianprocesses}
Gaussian processes are a class models where predictive densities at multiple points are jointly Gaussian~\cite{williams1996gaussian,rasmussen2006gaussian}.
The modelling framework is closely related to (and mathematically almost equivalent to) Kriging~\cite{matheron1963principles}, though the modern field of Gaussian processes
has substantially extended the expressivity in terms of non-parametric modelling beyond the original geo-spatial remit.

As a well-known connection to the field of kernel learning~\cite{scholkopf2002learning, shawe2004kernel}, it is noteworthy to observe that the predictive mean in Gaussian process regression is equivalent to the classical (point) prediction in kernel ridge regression.
The fitting of Gaussian process regression is, consequently, similar to the fitting of classical regression models, and the posterior variance may be interpreted as a prediction interval for kernel ridge regression.
Furthermore, the kernels correspondence can exploited to define Gaussian processes on features which may be objects of any kind~\cite{haussler1999convolution,lohdi02textclassification,bahlmann2002online}.
Recent work has seen extensions of Gaussian processes to explicitly incorporate heteroscedasticity mediated through features~\cite{kersting2007most,titsias2011variational}.

For non-expert practitioners, it is very important to note that the vanilla (GP: noise-free; Kriging: no nugget) formulation of Gaussian processes or Kriging yield a good estimate of the conditional density: the usual predictive distribution, while Gaussian, is not actually a valid predictive posterior (in Bayesian terminology), or (in frequentist terminology) a reasonable estimate for the predictive conditional density. Instead, the predictive distribution of vanilla GP/Kriging may be understood as follows: while located at the posterior predictive mean, its variance should be interpreted as a Bayesian uncertainty (or credibility) of a Bayesian predictive mean estimate, which approaches zero (and not the true conditional variance) as the number of training samples approaches infinity. We illustrate this very practical issue with an example:

\begin{Ex}\label{Ex:badGP}
Let $(X,Y)$ a pair of random variables t.v.i.~$\calX times \calY$ with $\calY\subseteq \RR$. Consider training data $(X_1,Y_1),\dots, (X_N,Y_N)$ to train a (vanilla) Gaussian process model, i.e., the prediction strategy $f$ t.v.i.~$[\calX\rightarrow \Distr (\calY)]$ where
\begin{align*}
f:& x\mapsto \calN(\mu(x),\sigma(x))\\
\mu(x) & := \kappa(x)^\top \cdot \left(K + \lambda I\right)^{-1}\cdot y\\
\sigma^2(x) & := k(x,x) - \kappa(x)^\top \left(K + \lambda I\right)^{-1} \kappa(x)\\
& \mbox{where}\; \kappa(x)_i = k(x,X_i)\;\mbox{and}\;K_{ij} = k(X_i,X_j),
\end{align*}
t.v.i.~$\RR^N$ and $\RR^{N\times N}$ respectively, where $\calN(a,b)$ denotes a univariate Gaussian distribution with mean $a$ and covariance $b$, and where $k:\calX\times \calX\rightarrow \RR$ is a kernel function (frequentist terminology) or covariance function (Bayesian terminology).\\
We further consider the toy case where all $X_i$ are equal to a single feature point $x\in \calX$, and where $Y|X=x$ is Gaussian with mean $\mu$ and variance $\sigma^2$.\\
If $f$ were a good probabilistic prediction strategy, then surely $f(x)$ should converge (in distribution) to the law of $Y|X=x$, i.e., to $\calN(\mu,\sigma)$, as $N\rightarrow \infty$.
However, an explicit computation (involving the Sherman-Morrison-Woodbury formula) shows that
\begin{align*}
\mu(x) = \frac{1}{N+\frac{\lambda}{k(x,x)}}\cdot \sum_{i=1}^N Y_i,\quad\mbox{and}\quad
\sigma^2(x) =  \frac{1}{\frac{N}{\lambda} + \frac{1}{k(x,x)}},
\end{align*}
so in fact $\sqrt{N}(f(x)-\mu)\overset{d}{\rightarrow} \calN(0, \lambda)$ as $N\rightarrow \infty$, which is the central limit asymptotic of a sample mean estimator rather than that of a density estimator for which $(f(x)-\mu)\overset{d}{\rightarrow} \calN(0, \sigma^2)$ would have to hold.\\

A less toy-like case (with random $X$) will see the posterior variance decreasing with the same asymptotic, though we do not carry out the calculation since the easier one is sufficient for the argumentation.\\

An even simpler (but slightly less mathematical) argument from which the inappropriateness of vanilla GPs for probabilistic supervised learning may be derived, is that $\sigma^2(x)$ does not depend on the label observations $Y_i$ at all, thus there is no means by which information on the true variance $\sigma^2$ could possibly enter the prediction $f(x)$ in general.\\

On the other hand, the Bayesian \emph{predictive posterior} in this example will be $\calN(\mu(x),\sigma^2(x) + \lambda)$, i.e., corresponding to taking the mean prediction as a point prediction, in a composite Gaussian prediction with constant variance $\lambda$.
\end{Ex}

Summarizing, Example~\ref{Ex:badGP} highlights a large caveat when attempting probabilistic predictions with (vanilla) Gaussian processes or via Kriging: even though this class of models usually outputs a predictive distribution which is a valid form of probabilistic prediction, it is not directly appropriate for probabilistic supervised learning as long as the predictive variance is not a (frequentist) estimate of the conditional label variance - as opposed to, for example, an estimate of the conditional mean's uncertainty, i.e., the posterior distribution of the predictive mean.
This also implies that in practice a plug-in Gaussian process model employed in the naive, seemingly ``obvious'' way for probabilistic supervised learning is not only going to perform badly, but is mathematically expected to perform badly, due to the above.\\
More precisely, for applicability in the supervised learning setting, Gaussian processes should be run in a specific way:
\begin{itemize}
\item[(i)] they should be queried for ``true'' predictive posteriors rather than the posterior distribution of the predicted mean. I.e., the prediction's variance should be the true predictive posterior variance rather than the predicted mean's variance.
\item[(ii)] for a genuinely probabilistic prediction, heteroscedastic Gaussian process models need to be considered. Otherwise, the posterior predictive variance converges to a constant for large sample size - which is also a valid probabilistic prediction, just not one that can model heteroscedasticity properly.
\end{itemize}

\subsection{Mixture density estimators, and mixtures of experts}
Mixtures of experts and (predictive) Bayesian networks produce conditional densities that are usually mixtures of parametric distributions~\cite{jacobs1991adaptive,bishop1994mixture,jordan1994hierarchical,hinton1995alternative,jensen1996introduction}.

More recent approaches include the case where the experts are Gaussian Process predictors~\cite{rasmussen1999infinite,tresp2000mixtures,rasmussen2002infinite,shi2005hierarchical,meeds2006alternative}.

While the model specification is usually Bayesian or at least fully parametric, a variety of fitting/prediction strategies such as least-squares, likelihood-based, expectation-maximization, or of course classical Bayesian updates have been proposed.

As a model class considered by the early machine learning community, it is usually considered in a supervised learning context, even in pure Bayesian formulations. Hence mixtures of experts are not (uniformly considered to be) a sub-class of ``Bayesian models'' or ``conditional density estimators'', and listed separately.

Accordingly, depending on the precise implementation, the exact relation to Gaussian processes, and/or which types of posteriors are considered, the same caveats as for the generic Bayesian strategies in Section~\ref{sec:bayesianmodels}, or Gaussian process models in Section~\ref{sec:gaussianprocesses}, may apply.

\subsection{Quantile Regression and Quantile Estimation}
\label{sec:quantileregression}
Quantile regression estimates aim at predicting not a full distribution, but a specific quantile or a set of pre-specified quantiles of a real-valued labels distribution~\cite{koenker1978regression}.
Modelling strategies considered in more recent literature are usually non-parametric and non-linear in nature and make use of different strategies, including kernel methodology and ensembling in the machine learning community~\cite{yu1998local, yu2003quantile,meinshausen2006quantile, takeuchi2006nonparametric, li2007quantile}.

While quantile regression is a task in which only an aspect of the conditional distribution is predicted (namely some of its quantiles), full distributional estimates
may be obtained by:
\begin{itemize}
\item[(i)] Using a quantile regression strategy which is simultaneously able to predict all quantiles. This is equivalent to a full probabilistic prediction as the prediction's density may be uniquely recovered from the cumulative density function and vice versa.
\item[(ii)] Predicting multiple (but not all) quantiles and then converting the quantile estimates into a histogram or kernel density or kernel smoothing estimate of the predictive distribution, for example via the method suggested by~\citet{takeuchi2009nonparametric}.
\item[(iii)] In either setting, it should be noted that in general quantile regression methods are not guaranteed to predict quantiles monotonously, i.e, to predict a value that is larger if the quantile is higher. This is known as the ``quantile crossing problem'' and has been addressed by various strategies and meta-strategies in literature~\cite{dette2008non,chernozhukov2010quantile}.
\end{itemize}

We would like to note a further connection of probabilistic supervised learning to quantile regression, already partly noted in the seminal paper of~\citet{koenker1978regression}. Namely, since the minimizer of the expected mean absolute error is the regression median, the classical supervised regression predictors with low mean absolute errors are, simultaneously, good predictors of the predictive median quantile. By the correspondence established in Section~\ref{sec:classprob}, these are in turn directly related to low-error probabilistic predictors where the predicted distribution is Laplace with constant variance. Similar relations between quantile predictors, combinations of classical supervised predictors and mean losses, and probabilistic predictors of a special structure hold for the other quantiles.

\newpage
\section{An API for probabilistic supervised learning}
\label{sec:API}

Based on the previous discussion, we propose an API for probabilistic supervised learning which formalizes the desirable properties of a joint inference framework in a structured algorithmic setting. We describe its general principles and discuss use cases as well as requirements, before providing an overview of the implemented classes and algorithms. We discuss a modular structure that allows for standardized prediction workflows, and its implementation in the skpro python package \cite{kiraly_skpro_2017}.

\subsection{Overview}

The main rationale of the presented API is to allow for the integration of arbitrary probabilistic prediction strategies into a unified interface that enables consistent and fair model assessment and comparison. A secondary objective is the simplification and standardisation of model validation such that defined prediction strategies can be reliably assessed using existing workflow components.
From there, the requirements of the API can be derived as a result of the previously described prediction setting. A crucial concept is the suggested approach to represent predicted probability distributions as an interface of the prediction strategy. Notably, the API structure extends the principal ideas and design objectives that developed in the scikit-learn project~\cite{pedregosa2011scikit} to the probabilistic setting.

\subsubsection{General design principles}\label{sec:api-principles}

To begin with, our API design follows five principles that turned out to provide a solid foundation for the popular scikit-learn machine learning framework (for a more detailed discussion see \cite{sklearn_api}).\\

{\bf Consistency of interface.} All objects of the same type (prediction methods, metrics, etc.) should share the same API to enforce consistency of use. In particular, this means that frequentist and Bayesian prediction methods need to share the same (or a compatible) interface.

{\bf Inspection of parameters.} Parameters and parameter values, such as non-data dependent choices of frequentist and Bayesian algorithms, should be accessible in the public part of the interface to allow for easy inspection. That includes regularization and learning parameters of frequentist machine learning methods and Bayesian priors or hyper-parameters.

{\bf Non-proliferation of classes.} The definition of base classes should be held at a minimum to foster maintainable, efficient, and universal design.

{\bf Composition.} Meta-algorithms should be modelled as modular compositions of simpler algorithms, e.g. meta-classes or wrappers to avoid code repetitions.

{\bf Sensible defaults.} When parameters are not set, defaults with a sensible behaviour should be available. Frequentist/point prediction methods that are called for probabilistic prediction should default to predicting a (homoscedastic) Gaussian around the point prediction, with constant residual estimated standard deviation, following the discussion in Sections~\ref{sec:classprob} and~\ref{Sec:classicbase}. If Bayesian methods exhibit priors as their model parameters, they should default to the uninformed baseline, an estimated density, following the discussion of section~\ref{Sec:uninfbase}.

\subsubsection{Main use cases}
\label{sec:API.overview.usecases}

As specified in section \ref{sec:propreset.probsupl}, the principal use cases we consider are supervised prediction tasks involving prediction functionals of the form  $f: \calX \rightarrow \Distr (\calY)$, trained and evaluated on feature-label-pairs in $(\calX\times \calY)$. More precisely:

\begin{itemize}
\item[(a)] Fitting and predicting via any prediction strategy described in Sections~\ref{sec:propreset} and~\ref{sec:problearners}, including Bayesian, non-Bayesian, composite, and baseline strategies. Optimally choosing and setting up a standard strategy is easy and via a uniform, consistent, and exchangeable interface.
\item[(b)] Applying meta-modelling strategies including automated hyper-parameter tuning, model selection, model composition, ensembling, and pipelining.
\item[(c)] Running model evaluation and model comparison experiments by the re-sampling plus out-of-sample principle as described in Section~\ref{sec:modelvalidation}, estimating one or multiple generalization losses.
\end{itemize}

\subsubsection{API design requirements}\label{sec:api-requirements}

Consideration of the use cases and general design principles implies further API requirements specific to the probabilistic prediction setting:\\

{\bf Fit-predict-parameters interface for probabilistic prediction strategies.}
Following the standard interface design in the supervised setting, an encapsulation object for probabilistic prediction strategies has to provide a training interface that takes as argument both training features $\bx = (x_1,\dots, x_N) \in \calX^N$ and corresponding training labels $\by = (y_1,\dots, y_N)\in\calY^N$ , e.g. a function of the signature \texttt{fit($\bx : \calX^N$, $\by : \calY^N$ ) $\rightarrow$ trainedmodel}; it must also give access to the learnt prediction functional $f:\calX\rightarrow \Distr (\calY)$, e.g. by a function with signature\\ \texttt{predict($x^* : \calX$, trainedmodel) $\rightarrow$ $f(x^*): \Distr (\calY)$} which returns a predicted distribution $f(x^*)$ for every test data point $x^*$. Furthermore, parameters of the learning strategy, which may affect the fit and predict interfaces, need to be accessible and settable via a unified interface point to allow for meta-strategies and tuning.\\

{\bf Stand-alone distribution type.}
To represent the predicted distributions $f(x^*)$, we argue that by the modularity principle it is reasonable to introduce a distinct distribution object type to store predictions, which allows for a natural separation of model fitting and prediction from interrogation of the prediction object (the distribution), say for strategy evaluation. This separation of model-fit and model-predict from prediction-evaluation is also computationally beneficial since it avoids unnecessary repetitions of bottleneck computations in the model-fit or model-predict part whenever distributions are evaluated. In the considered use cases, an distinct distribution type is particularly beneficial since the evaluations are likely to be repeated. The introduced probabilistic loss functions, for instance, rely on the evaluation of the predicted density (see Section~\ref{sec:probloss}). Furthermore, ensemble strategies such as normalized averaging of log-distributions imply repeated evaluations (see \ref{sec:meta-algorithms}). Similarly, wrappers such as post-prediction smoothing of densities, capping and re-normalization (see Section~\ref{sec:meta-algorithms}) make direct use of prediction evaluation. Therefore, from an API perspective, it is more advantageous to have an distinct interface to the predicted distributions.
The alternative, as found in Weka to model with the usual behaviour of Gaussian Process regression, compounds model-prediction with prediction-evaluation; it avoids introducing a new data type at the cost of returning ``a gorilla holding the banana and the entire jungle'' when only the banana would be of interest, to say it in the words of~\citet{sklearn_api} (i.e., the whole algorithm object in lieu of the distribution).\\

{\bf Explicit representation of distribution properties.}
Since the predicted distribution $f(x^*)$ at a feature point $x*$ is not a random variable but an explicit distribution, the object that represents the distribution $f(x^*)$ has to provide an interface to its distributional properties. Most notably, the API should provide access to the predicted density function (in case of existence), since it is used by the loss calculation (see Section~\ref{sec:probloss}), or an alternative suitable representation (such as the bag of sample points for an empirical sample) in the mixed case. We also suggest direct access to essential distributional properties such as mean and standard deviation for ease of use, as well as methods for accessing other properties such as the cumulative distribution function or function norm, in case of existence. The uninformed baseline strategy as introduced in section \ref{Sec:uninfbase}, for example, has to exhibit the estimated density.\\

{\bf Representation on the distribution level.}
The explicit representation of distributional properties implies that predictions need to be readily represented on a distributional level and not as a sample from the associated random variable. Some learning algorithms, such as Bayesian inference via Markov chain Monte Carlo (MCMC) sampling, however, do not yield an explicit interface to a predictive density that is needed for the comparison with non-Bayesian strategies and the uninformed baseline (see Section~\ref{sec:bayesianmodels}). Nonetheless, the API needs to integrate such algorithms using an estimate of the explicit distributional properties, such that predictions always come as represented on the distributional level. For example, prediction methods that allow access to an estimated predictive posterior distribution through posterior samples, may be interfaced through said posterior samples, by introducing an adaptor element such as directly interfacing the posterior sample as a mixed sample, or converting it to an absolutely continuous density.\\

{\bf Vectorized implementation.}
As prediction is usually required for a number of points, learning algorithms often allow for concurrent and efficient parallel processing, e.g.~Gaussian process models (see Section~\ref{sec:gaussianprocesses}). In such cases, vectorized implementation across test points is preferable for efficiency reasons. The API should hence allow for convenient vectorized implementation at the discretion of the user. Furthermore, for test feature points $\bx^* = (x_1^*,\dots, x_M^*)\in\calX^M$, the resulting vector of predicted distributions $f(\bx^*) = \left(f(x_1^*),\dots, f(x_M^*)\right)\in \Distr (\calY)^M$ should be represented in a vector-like interface that allows both point-wise and vector-wise access to the predicted distribution and its accessible properties.\\

{\bf Meta-modelling and composition.}
Hyper-parameter tuning, pipelining and ensembling are standard meta-motifs which also need to be supported by a probabilistic modelling platform, as meta-strategies which combine probabilistic modelling strategies into a higher-level probabilistic modelling strategy.\\
The meta-strategies of encapsulating classical prediction strategies, as described in Section~\ref{Sec:classicbase}, make use of a two-step prediction process that may involves deterministic as well as a probabilistic prediction strategies. The API structure that encapsulates learning algorithms should therefore provide a meta-interface for strategy composition with a potential meta-interface that may call both deterministic and probabilistic strategy encapsulation interfaces.\\

{\bf Automated validation workflow}
The framework should not only enable meaningful and domain-agnostic model comparison, but also standardize and simplify the workflow that leads to such comparisons. Specifically, the API structure should implement model comparison motifs as discussed in Section~\ref{Sec:modelvalidation.tuning}. While the interface should allow for a flexible implementation of arbitrary use cases, it should automate as many repetitive tasks as possible, for example, calculation of error estimates and their confidence intervals for multiple probabilistic prediction strategies, as described by algorithm~\ref{alg:validation-2}.

\newpage
\subsection{Pythonic skpro interface design}
\label{sec:api-suggested-approach}
Originating from these requirements, we specify a pythonic high-level API design, which is depicted in figure~\ref{fig:api-overview}, implemented in the skpro package, and elaborated below.\\

\begin{figure}
    \begin{center}
    \input{figures/api_overview.tex}
    \label{fig:api-overview}
    \caption{Overview of the skpro API: All prediction strategies are encapsulated by a probabilistic estimator class, returning a distribution object that exhibits properties like the predicted density function etc. The user can define a workflow involving different strategies and datasets while the results of a model assessment is aggregated automatically.}
    \end{center}
\end{figure}
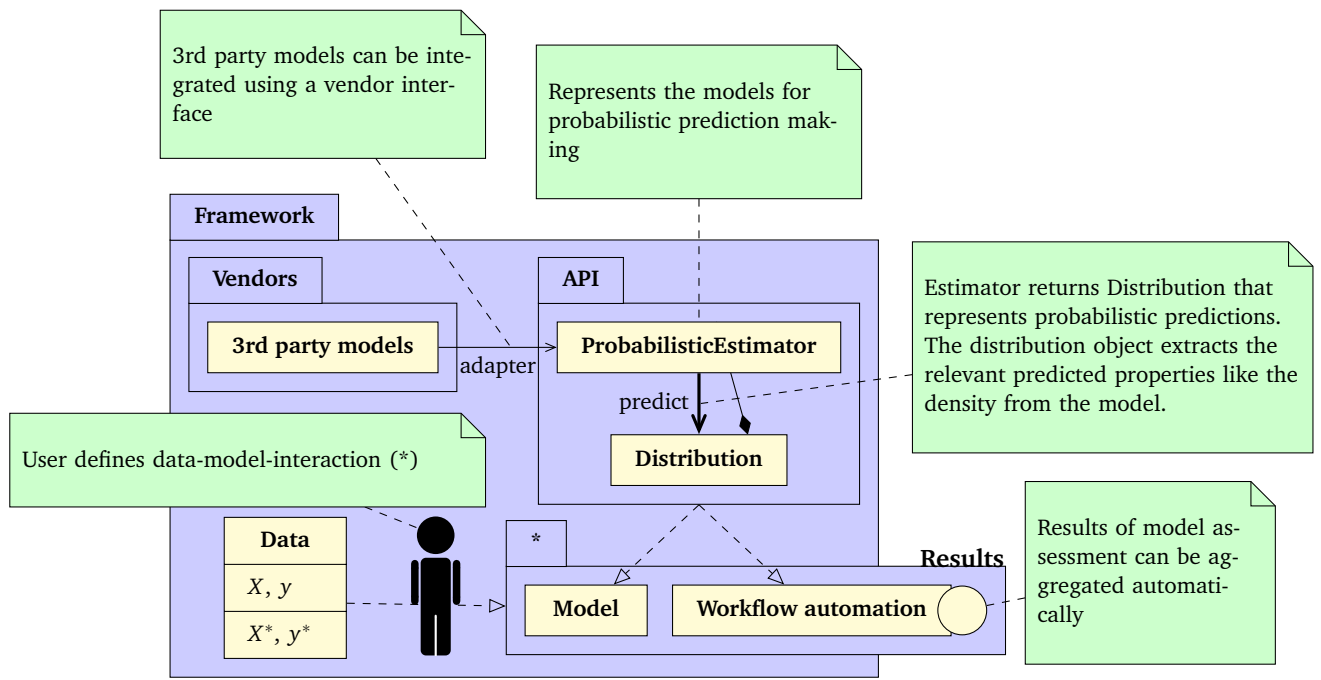

{\bf Distribution type interface.}\\
Firstly, much depends on the structure that allows for a sufficient and flexible representation of the probabilistic predictions, i.e. probability distributions. While the classical predictions are numbers (or arrays of numbers) and hence primitive data types (see Section~\ref{sec:classical-learning}), the probabilistic counterparts can be highly general and complex structures. There is not only a vast collection of distribution types, which can substantially differ (e.g. continuous/discrete or with finite/with infinite support) but also complex distributional properties that need to be represented (either directly within the object or computable based on the object structure).

Our suggested approach for modelling distributions in line with the requirements is to refrain from a stand-alone representation of the predicted distributions; rather, we represent predicted distributions as an interface of the trained model. The prediction strategy provides -- by definition -- the information that is needed to obtain relevant properties of the predicted distribution. The intuitive solution to encode the relevant properties into a distinct general distribution object that is then returned as prediction seemed, however, cumbersome as such a general return object would have been hard to define. Instead, we suggest that an interface object should encode the way of how these properties can be obtained from the referenced strategy. In this way, the interface object provides the desired properties of the predicted distribution by extracting them from the strategy that returned it, and can therefore serve as probabilistic prediction, i.e. a predicted distribution object.

It is worth noting that there is no superficial difference between returning a distinct distribution object that holds the desired properties and the alternative of returning an interface object that calculates the desired distribution properties on demand based on the prediction strategies' properties. The interface object could in fact be used to calculate all distribution properties to generate a distinct full-blown distribution object.

At first glance, using distribution-like interfaces has the disadvantage that the returned predictions still hold the need for extended computations. Evaluating the predicted density function, for instance, includes the computational cost of extracting the density from the fitted model. Repeated evaluations of the density with different arguments can hence result in inefficient recalculations. Moreover, the interfaces cannot be stored for later use without storing the entire referenced prediction strategy as well.

These downsides can, however, be addressed through an implementation of automatic caching of the interface which ensures that operations are not needlessly repeated. The caching also allows for the continued use of the distribution interface even if the referenced strategy is no longer available. Given the caching, the deferred calculation in an interface object, in fact, turns out to be rather advantageous.

Firstly, computational efficiency is improved since the distribution properties are only computed if they are actually requested. The density function, for example, does not need to be retrieved if one's sole interest is the distribution's mean. This becomes increasingly valuable when ensemble methods like bagging are used and multiple predictions are averaged. More importantly, the interface approach also reduces the complexity of model development since a distinct data representation in the distribution object becomes unnecessary. The data is already represented by the prediction strategy and the extraction of the desired distribution properties like the mean can be implemented by exploiting the estimator object's data structure. The mentioned challenge to find a general representation of distributions therefore boils down to a model-specific interface implementation that provides access to the prediction's distributional properties. The interface approach is therefore flexible enough to represent probabilistic predictions by arbitrary strategies, on the distributional level.

The distribution object also provides an efficient realisation of vectorized prediction across test points. Rather than returning a vector of distribution objects, the distribution object itself can represent multiple predictions in a vector-like interface. As vectorization is managed internally by the object, the user can easily switch between vectorized and point-wise implementation, where in the latter case, vectorization is automatically carried out behind the scenes.\\

{\bf Probabilistic estimators.}\\
Given the distribution object type, the prediction mechanisms can be implemented as a function that takes test data as argument and returns a distribution object. Consequently, an obvious structure to encapsulate probabilistic learning algorithms in is scikit-learn's estimator object \cite{sklearn_api}. Estimators implement a fit function to train the model and provide a predict function that returns the predictions, i.e. distribution objects. The hyper-parameters of the prediction strategy are accessible as public variable of the estimator object. Non-proliferation implies that there is only one such estimator base class, namely the ProbabilisticEstimator class, which all learning algorithms must implement. Note that a probabilistic estimator shares the API of an classic scikit-learn estimator except for the fact that it returns a distribution object.\\

{\bf Adaptors}\\
As we seek to represent the probabilistic predictions on the distributional level, the API provides so called adaptors that can transform prediction outputs into required distribution-level properties. The \obj{DensityAdapter}\footnote{This spelling is correct: the class names in the skpro package implementation use of the US English spelling ``adapter'' rather than the UK English spelling ``adaptor'' used in the manuscript otherwise, both referring in a metaphorical sense to the same meaning of ``a device that connects two different, specific parts''.} class, for example, transforms a given input into an estimated density function, e.g. a posterior sample into kernel estimated density. Hence, if a prediction strategy does not provide the required distributional properties, a suitable adaptor can be plugged in to complete the distribution interface.\\

{\bf Vendor integration}\\
To integrate existing prediction algorithms into the framework, the API provides a vendor interface with the entry points \obj{on\_fit(X, y)} and \obj{on\_predict(X)}. The functions can be overridden to implement the fitting and prediction process of the vendor strategy that generates the prediction output, i.e. a posterior sample. Combined with an adaptor, the third-party strategy can then integrate with the wider framework infrastructure.\\

{\bf Workflow automation: Model-view-controller structure}\\
To automate and standardise prediction workflows, we suggest three fundamental components: workflow model, workflow controller, workflow view.

The \emph{workflow model object} simplifies the management of learning algorithms and contains the actual prediction algorithm that was defined by the user (i.e. a probabilistic estimator object encapsulating a probabilistic prediction strategy). It allows to store meta-information and configuration for the algorithm it contains, e.g. a name or a range of hyper-parameters that should be optimized. In future, it might support storage of trained prediction strategies for later use.

A \emph{workflow controller object} represents an action or task that can be done with a workflow model to obtain certain information. A scoring controller, for instance, might take a data set and loss function and return the loss score of the model on this data set. The controller can save the obtained data for later use.

Finally, the \emph{workflow view object} takes what a controller returns to present it to the user. A model validation view, for example, could take prediction-true-label-pairs and compute the sample of loss residuals. The separation of workflow controller and view level is advantageous since controller tasks, such as training of a prediction strategy, can be computationally expensive. Thus, mere reformatting of an output on the workflow view level should not require the re-evaluation of the workflow controller task. Moreover, if a view only displays a part of the information it is yet useful to store the full information the controller returned.

The model-view-controller structure (MVC) encapsulates the central workflow of predictive supervised learning: evaluating a collection of candidate strategies on one or multiple tasks, then displaying model comparison results and diagnostics. While arbitrary controllers and views can be implemented, they inherit the API of their base classes and thus share a consistent interface. From there, the MVC modules can be further aggregated using existing, standardized workflow components of the framework, for example, for composing result tables or visualising model performances. While the user can implement specific tasks, the heavy lifting of running experiments and compute-aggregation of results can be left to the framework.

We believe that collaborative extension of such a common but flexible code base can greatly benefit and accelerate model development and foster reproducibility of obtained results.\\

{\bf Technical implementation}\\
While the realisation of the outlined framework is not bound to a specific technical stack, our implementation is being carried out in Python and is hence based on the widely used scikit-learn library and its seminal interface design~\cite{pedregosa2011scikit}. Whenever possible, we re-utilised sklearn's existing code base with the aim to extend the library into the probabilistic scope while staying close to its syntax and logic whenever possible.

\subsection{Core API}
In the following, we detail the skpro classes and their role and interaction in the prediction-validation workflow API. Figure~\ref{fig:api-base} summarizes the major classes and their inheritance relations.

\begin{figure}
    \vspace*{-1.5cm}
    \begin{center}
    \input{figures/api_base.tex}
    \end{center}
    \caption{Overview of the framework's core API where abstract classes are denoted in italic font and inheritance and dependence are indicated by arrows: The seminal probabilistic estimator object directly inherits from scikit-learn's base estimator object and thus implements the fit-predict logic that produce probabilistic predictions in form of a distribution object. A vendor estimator allows for the integration of 3rd party strategies such as the Bayesian PyMC prediction algorithms.}
    \label{fig:api-base}
\end{figure}
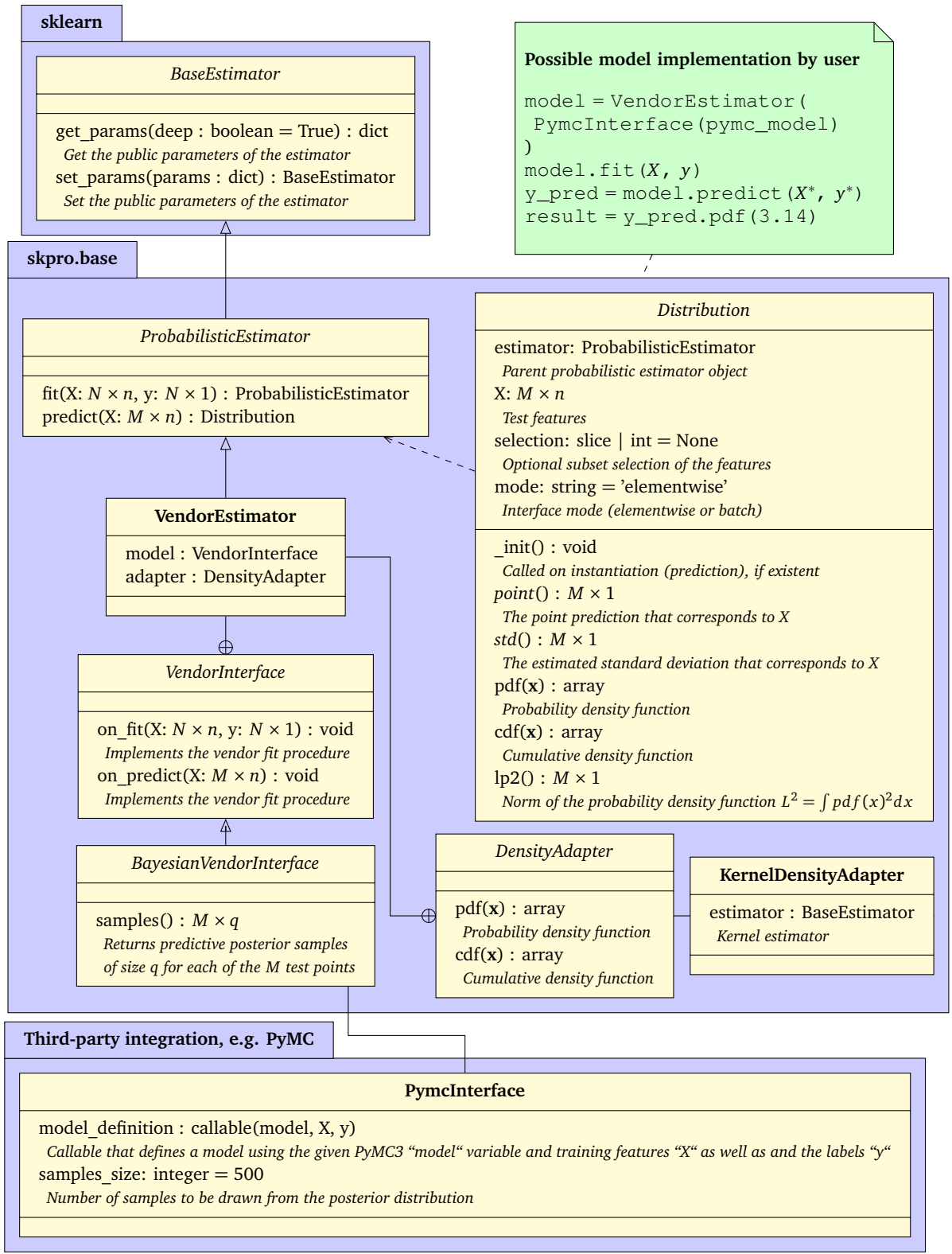

\subsubsection{Probabilistic estimators and distributions}

All prediction strategies are implemented as a subclass of the seminal\footnote{The word ``seminal'' is used here in its usual meaning in the context of API design, as in ``seminal class'' = ``most other classes are derived from this class, e.g., by inheritance, mixin or instantiation''. We expressly do not attach the meaning often found in scientific literature review, where ``seminal concept'' = ``most other concepts in its field (will) derive from it or relate to it'', which would not only be presumptuous at this point in time, but also nonsensical in the given API design context - similarly to the third, medical meaning of the word ``seminal''.} abstract \obj{ProbabilisticEstimator} class that directly inherits from scikit-learn's \obj{BaseEstimator}. Thus, they implement the mentioned fit-predict API as well as methods to get and set the public parameters of the estimator class. Crucially, each probabilistic estimator also comes with an associated sub-class of \obj{Distribution} which defines the distribution interface that represents the predictions. When \obj{predict} is called, the estimator simply instantiates and returns its corresponding distribution object. Note that the distribution object takes a reference of the estimator and the test features and has to implement the access to the actual predictions based on the strategy (see Section~ \ref{sec:api-suggested-approach}). To set up distribution properties, the distribution instance can override the \obj{\_init} method, which is automatically called after the object is being instantiated. Since the distribution object represents multiple test feature predictions $p_1=f(x_1^*),..., p_M=f(x_M^*)$, it provides an optional selection parameter that allows to sub-select predictions. For example, \obj{selection=0} or \obj{dist[0]} would refer to $p_1$, the prediction distribution of the first test point only. The distribution object also implements two different evaluation modes for given function arguments $\vec{y}^* = y_1^*,...,y_k^*$, namely element-wise and batch-wise evaluation.

For default \emph{element-wise evaluation}, repeated sequences of $\vec{y}^*$ are paired up with repeated sequences of the $p_i$. More precisely, the sequence of $p_i$ is repeated until there are $k$, i.e., $p_1,...,p_m, p_1, p_2,..., p_m, p_1,...,p_{k/M}$ where $k/M$ is the integer remainder of dividing $k$ by $M$.
In the alternative \emph{batch evaluation} mode, every distribution $p_i$ is evaluated on all test labels $y^*_1, ..., y^*_k$, resulting in a matrix $k$ columns and $m$ rows.

\subsubsection{Integration of third-party prediction strategies}

To interface existing prediction strategies from the framework, users can implement their own sub-classes of the probabilistic estimator base class. The API, however, also offers default semi-automated strategy integration, using the derived \obj{VendorEstimator} object by specifying a \obj{VendorInterface} and a \obj{DensityAdapter}. The vendor interface must only define \texttt{on\_fit} and \texttt{on\_predict} events that are invoked automatically. The results of the fit-predict procedure are exposed as public variables of the interface. The adaptor, on the other hand, then describes how distributional properties are generated from the interfaced vendor strategy, i.e. the VendorInterface's public properties. Given a vendor interface and appropriate adaptor, a vendor estimator can be used like any other probabilistic estimator of the framework.

\subsubsection{Bayesian and probabilistic programming interfaces}

A notable example of the model integration API is the case of Bayesian and/or probabilistic programming strategies. To integrate such prediction strategies, one can implement the \obj{BayesianVendorInterface} and its \obj{samples} method which should return a predictive posterior sample. Combined with a \texttt{DensityAdapter} like the \texttt{KernelDensityAdapter} that transforms the sample into estimated densities, the Bayesian prediction strategy can then be used as a (absolutely continuous) probabilistic estimator. As a special case, the Bayesian API is used for the integration of prediction strategies from the PyMC3 library~\cite{salvatier2016probabilistic} (see Section~\ref{sec:api-pymc-integration}).\\
An alternative is conversion to a mixed probabilistic prediction interface and use with a mixed probabilistic loss (currently in beta).

\subsubsection{Model selection and performance metrics}

To evaluate the accuracy of the learn prediction functionals, the skpro API provides number of probabilistic loss metrics in the \obj{metrics} module (see Section~\ref{sec:probloss}). In line with definition~\ref{Def:losses}, the framework implements, as continuous losses, the log-loss (\obj{log\_loss}) and the integrated squared loss (\obj{gneiting\_loss}). The module also introduces an experimental implementation of the continuous rank probability score (CRPS) for mixed predictions, as discussed in section~\ref{sec:mixed.split}. Thereby, the integration of the probability function is carried out numerically using the \obj{quad} method of \texttt{scipy}'s integrate module which itself uses a technique from the Fortran library \texttt{QUADPACK}~\cite{jonesscipy2017}. In order to maximize efficiency, the method automatically adjusts the step size it uses for numerical integration by judging by the flatness and slope of the integrated function. As a result, the outcome might not be satisfactory for arbitrary probability functions. In our experiments, however, the loss metrics showed satisfying precision compared to analytical solutions. Notably, it was ensured that all metrics have a unified signature:\\

\begin{tabular}{ll}
    y\_true : np.array &
        The true labels \\
    dist\_pred: Distribution &
        The predicted distribution \\
    sample: boolean = True &
        If true, loss will be averaged across the test point sample \\
    return\_std: boolean = False &
        If true, returns average loss of sample and its standard deviation\\
    \hline
    \textit{Returns} mixed type & \\
\end{tabular}\\

\noindent
In accordance with scikit-learn's model validation API, we furthermore implement a cross-validation procedure out-of-the-box. In fact, the implemented \obj{cross\_val\_score} uses the existing sklearn method of the same name but introduces support for the standard deviation of the loss values (as in Algorithm~\ref{alg:validation-2}). The relevant parameters are:\\

\begin{tabular}{ll}
    estimator : estimator object &
        The object to use to fit the data\\
    X : array-like &
        The data to fit  \\
    y: array-like &
        The target variable\\
    scoring : callable or None &
        Used scoring function\\
    \hline
    \textit{Returns} array &
        \parbox[t]{10cm}{
        Array of scores of the estimator for each run of the cross validation
        with their corresponding uncertainty.
        } \\
\end{tabular}\\

\subsubsection{Code example}

The following code example illustrates a possible usage of the framework's core API for prediction of Boston housing prices:

\begin{lstlisting}[language=Python, basicstyle=\small]
class ThirdPartyModel:
    on_fit(X, y):
        self.model = ... # load and fit some 3rd party model
    on_predict(X):
        self.posterior_sample = self.model.predict(X)
    sample():
        return self.posterior_sample

X, y = load_boston(return_X_y=True) # Load boston housing data
model = VendorEstimator(ThirdPartyModel(), adapter=KernelDensityAdapter())
scoring = make_scorer(log_loss) # transform loss metric into scoring function
results = cross_val_score(model, X, y, scoring)
>>> results
[[-3.47921375  0.04544327]
 [-4.21840238  0.14861681]
 [-3.94242129  0.07013955]]
\end{lstlisting}

\subsection{Prediction strategies}

Like sklearn, the skpro framework comes with ready-to-use prediction strategies that are built using the described API. Currently (first release version), the implemented algorithms include a baseline, an estimator that interfaces the PyMC3 package, as well as composite meta-strategies that leverage arbitrary scikit-learn models for probabilistic prediction by predicting parameters of common parametric densities.

\subsubsection{Baselines}

Naturally, the framework includes an uninformed baseline as discussed in section \ref{Sec:uninfbase}. Technically, the \obj{DensityBaseline} object that implements the baseline takes a density adaptor (see Section~\ref{sec:api-suggested-approach}) that is being used to estimate a density from the training labels (discarding the training features). This constant density is then returned as distributional prediction for any given test features.

\subsubsection{PyMC-interface}
\label{sec:api-pymc-integration}

As mentioned earlier, the API supports the integration of Bayesian prediction methods that produce a predictive posterior sample (see Figure~\ref{fig:api-base}). In particular, the framework comes with a specific vendor interface for the PyMC3 library for Bayesian modelling \cite{salvatier2016probabilistic}. The user can write out a model definition by specifying a callable (e.g. a function) with the arguments \obj{model}, \obj{X}, \obj{y}, where model represents the PyMC model and \obj{X} and \obj{y} the features and their respective labels, for instance:
\begin{lstlisting}[language=Python, basicstyle=\small]
def pymc_model(model, X, y):
    ... # define priors alpha, betas, sigma etc. using PyMC's syntax
    y_pred = pm.Normal("y_pred", # prediction variable
        mu = alpha + pm.math.dot(betas, X.T), sd=sigma, observed=y
    )
\end{lstlisting}
The callable is passed as an argument to the \obj{PymcInterface}, that itself becomes an argument of the \obj{BayesianVendorEstimator}. Using the estimator then automatically performs the inference of the predictive posterior sample using PyMC's standard estimation algorithms. The interface thus provides a direct bridge to integrate PyMC models with the wider framework.

\subsubsection{Composite parametric prediction}
\label{sec:composite}

The framework also offers a range non-Bayesian composite prediction strategy which use classical estimators to predict defining parameters of parametric distributions, see Algorithm~\ref{alg:adapt-resid} and~\ref{sec:heteroscedastic-regression}. For example, any prediction of normal distributions is equivalent to a prediction of its defining parameters mean $\mu$ and standard deviation $\sigma$, plus specification of the distribution type being normal (or, say, prediction of location $\mu$ and scale $b$ while specifying a Laplacian distribution etc.). More generally, classical prediction algorithms which already return \emph{point} and \emph{variance} predictive estimates may be composed by distribution shape types (e.g. Normal, Laplace etc.) which the point and variance estimates are plugged into, yielding a probabilistic prediction strategy. Appropriate distributional types may be empirically chosen based on the data, for instance by choosing the type that minimizes a given empirical probabilistic loss.

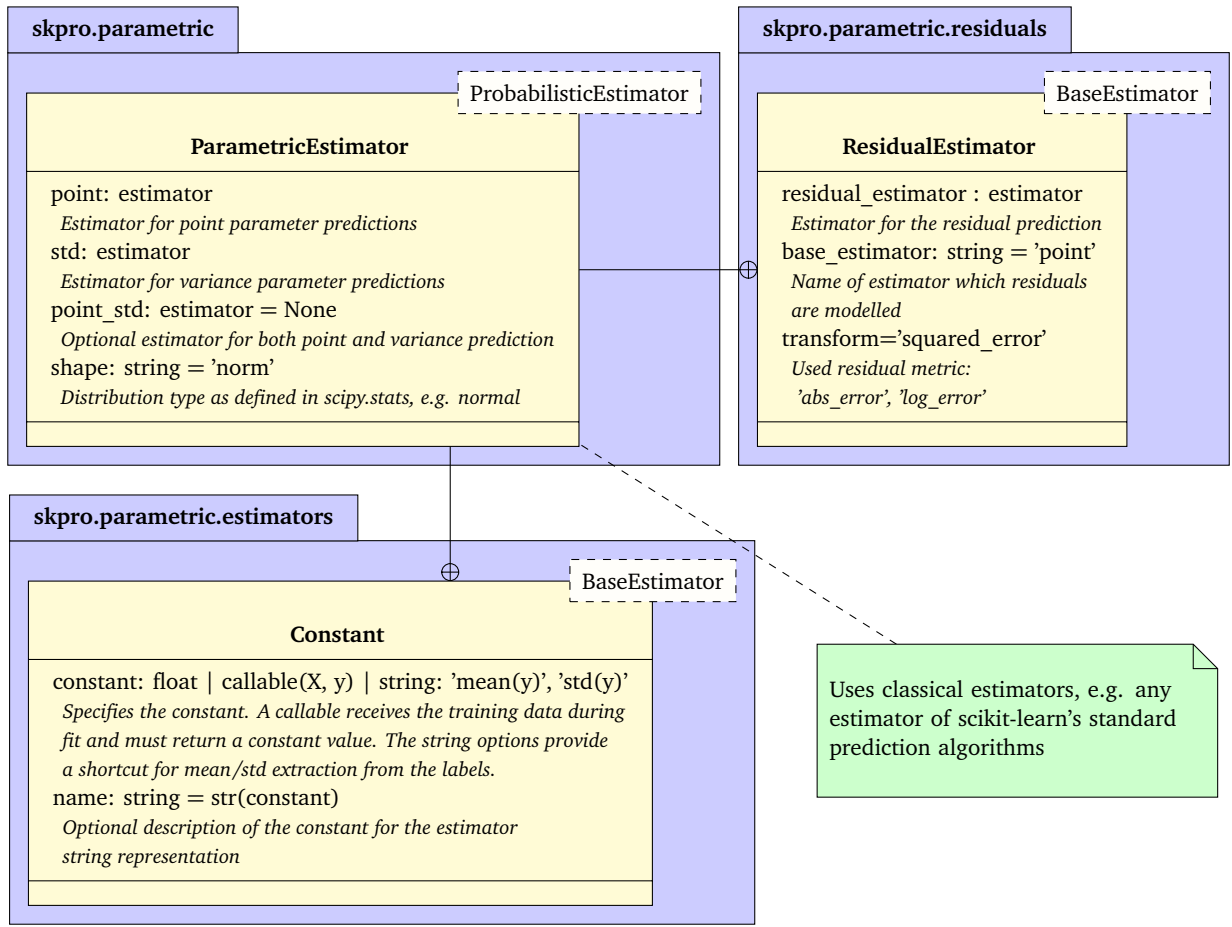
\begin{figure}
    \vspace*{-1.5cm}
    \begin{center}
    \input{figures/api_parametric.tex}
    \end{center}
    \caption{Overview of the parametric estimation strategy (base classes are annotated in the dashed boxes): The probabilistic parametric estimator takes classical estimators that predict the defining parameters (\obj{point} and \obj{std}) of the given distribution type (e.g. mean and standard deviation for the normal distribution). In particular, the residual estimator can be used for the \obj{std} prediction since it takes another classical estimator to predict residuals or variances of the \obj{point} prediction. To implement simple strategies, the module offers a constant estimator that produces pre-defined constant predictions.}
    \label{fig:api-parametric}
\end{figure}

The composite parametric prediction strategy is implemented by the \obj{ParametricEstimator} object that currently supports two-parametric continuous distributions as depicted in figure \ref{fig:api-parametric}. It takes a point estimator (\obj{point}), a variance estimator (\obj{std}) and a \obj{shape} parameter to define the assumed distribution form (e.g. 'norm' or 'laplace'). During fitting (\obj{fit}), the parametric estimator automatically fits the provided point and variance estimators; accordingly, on prediction (\obj{predict}), it retrieves their estimations to compose the overall predicted distribution interface of the specified shape. The composite parametric prediction strategy also supports combined estimation in which the same estimator instance is used to obtain both point and variance prediction. The combined estimator has to be passed to the optional \obj{point\_std} parameter while the \obj{point} and \obj{std} parameter can be used to specify how point and variance estimation should be retrieved from it. Hence, the parametric estimator can be considered a first-order strategy that maps the distribution interface onto the actual learning algorithms of the provided estimators. Since the implementation follows the estimator API of sklearn, sklearn's classical estimators may be employed, in any combination, as part of a composite probabilistic prediction strategies.

In this paradigm, a classical algorithm can be employed to obtain the point prediction, e.g. a housing price that then represents the mean of the predicted price distribution of that house. To obtain variance predictions, on the other hand, it is an intuitive idea to use the residuals of the point estimations, since they represent the magnitude of error or uncertainty associated with the point prediction. Precisely, using the training labels $y_i$ we can obtain the absolute training residuals $\varepsilon_{\text{train}, i} = |\hat{y}_i - y_i$| of the point predictions $\hat{y}_i$. Since training and test data are assumed to be i.i.d. sampled from the same generative distribution, we can estimate the test residuals based on the training residuals. Specifically, after the first-pass predictions have been obtained, we fit a second point prediction strategy - the residual prediction strategy - to the training features and the training residuals ($x_i$, $\varepsilon_{\text{train}, i}$) of the first-pass strategy. The residual prediction strategy then provides an estimate of the test residuals $\hat{\varepsilon}_{\text{test}, j}$ for given test features $x_j^*$. Consequently, we regard $\hat{\varepsilon}_{\text{test}, j}$ as prediction of the distribution's deviation parameter (e.g.\ $\sigma$ in $\calN(\mu, \sigma)$), i.e. the variance prediction of the overall strategy. The residual estimation strategy is implemented by the \obj{ResidualEstimator} (RE) object as documented in figure \ref{fig:api-parametric}.

Note that we calculated the absolute residuals to account for the non-negativity of the variance. Alternatively, the strategy can be modified by fitting the squared or logarithmic training residuals to the residual prediction strategy and back transforming the estimated test residuals using the square root and exponential function respectively. Such a residuals transformations can, for instance, be useful to emphasize or depreciate larger residuals, e.g. the influence of outliers in the data. Note further, that the overall residual strategy involves two estimators, for the point and the residual prediction, which are not necessarily of the same type. One could, for example, use a linear regression to obtain the point predictions while choosing a more sophisticated strategy to model the residuals of that regression.

With the given parametric API, classical estimators can be employed for the purposes of probabilistic prediction making. In addition to the estimators in the sklearn library, the skpro framework provides a \obj{Constant} (C) estimator that predicts a constant value which is pre-defined or calculated from the training data. The estimator is particularly useful for control strategies, e.g. a baseline that omits the training data features and makes an uninformed guess by calculating the constant mean of the dependent variable. The following code example illustrates a resulting overall syntax that defines a baseline model \obj{baseline} using the parametric estimator. The predictions of such model would be normal distributions with mean zero, and a standard deviation that equals the mean of the absolute training residuals.

\begin{lstlisting}[language=Python, basicstyle=\small]
baseline = ParametricEstimator(
    point=Constant(0),     # Point estimator
    std=ResidualEstimator(  # Variance estimator
        'point',            # Base estimator
        Constant('mean(y)'),# Residual estimator
        'abs_error'         # Calculation method
    ),
    shape='norm'            # Distribution type
)
\end{lstlisting}

\subsection{Meta-modelling API}
\label{sec:API.pred.tuning}

Meta-modelling API concepts are required for a modular implementation of higher-order modelling features such as hyper-parameter optimization or ensemble methods.

Following the principle of composition the meta-algorithms are implemented as modular compositions of the simpler algorithms (see Section~\ref{sec:api-principles}). Technically, the meta-strategies are realised by meta-estimators that are estimator-like objects which perform certain methods with a given estimators. They hence take estimator type objects as some of their initializing inputs, and when initialized exhibit the fit-predict logic that implements the meta-algorithm when instantiated on the wrapped estimators (for an extended discussion see~\cite{sklearn_api}, section~3.1).\\

{\bf Grid and random hyper-parameter optimization}\\
Most prediction strategies exhibit so-called hyper-parameters that are externally settable. Besides manual tuning, which is discouraged due to the problematic reproducibility, there are different meta-strategies to algorithmically optimize or tune the hyper-parameters of a learning algorithm for a given prediction problem. Two commonly used default strategies are grid search and random search. In grid search, a set of hyper-parameters that will be searched is explicitly defined. On the contrary, for random search, a distribution over possible hyper-parameters is given from which the search set is sampled. In either case, the model is evaluated for each of the hyper-parameters in the search set and the parameter with lowest estimated generalization error, estimated by nested re-sampling within the training set(s), is chosen.

The two mentioned tuning meta-strategies are implemented by scikit-learn's meta-estimators \obj{GridSearchCV} for exhaustive grid search and \obj{RandomizedSearchCV} for randomized parameter optimization. Both take an instance of the model which parameters are ought to be optimized and a list of the to-be-searched parameter range; the fit function then performs the optimization while the predict function uses the optimized estimator's predict function.

Since the implemented probabilistic estimator and all of its subclasses are implementing the API of an sklearn estimator, it is possible to directly use the above, or any other of sklearn's meta-estimators to tune hyper-parameters of the probabilistic prediction strategies. Accordingly, it is also possible to employ sklearn's pipeline meta-estimator to combine multiple estimators, including transformation or prediction strategies, into a single prediction strategy. This can, for instance, be useful to conveniently combine data pre-processing an a probabilistic prediction algorithm in a single estimator interface.\\

{\bf Ensemble methods}\\
The skpro package provides (at release) experimental support for ensemble methods. Currently, this includes bagging in a regression setting which is implemented by the \obj{BaggingRegressor} estimator in the ensemble module. The meta-estimator fits base regressors (i.e. probabilistic estimators) on random subsets of the original dataset and then aggregates their individual predictions in a distribution interface to form a final prediction (see Section~\ref{sec:baggeraging}). The implementation is based on scikit's meta-estimator of the same name but introduces support for the probabilistic setting\footnote{A detailed description of the employed ensemble algorithms is available in scikit-learn's \href{http://scikit-learn.org/stable/modules/generated/sklearn.ensemble.BaggingRegressor.html}{online documentation}}.
The relevant parameters of the bagging estimator that can be used like a standard probabilistic estimator, are described below:
\vspace{0.5em}\\
\begin{tabular}{ll}
    base\_estimator: estimator &
        The base estimator to fit on random subsets of the dataset\\
    n\_estimators: int = 10 &
        The number of base estimators in the ensemble\\
    max\_samples: int or float = 1.0 &
        The (relative or absolute) number of samples ...\\
    max\_features: int or float = 1.0 &
        ... features to draw from X to train each base estimator\\
    bootstrap: boolean = True &
        Whether samples are drawn with replacement\\
    bootstrap\_features: boolean = False &
        Whether features are drawn with replacement\\
    oob\_score: boolean = False &
        Whether to use out-of-bag samples to estimate the generalization error.\\
\end{tabular}\\

\subsection{Workflow automation}

In terms of the workflow automation API that was described in section \ref{sec:api-suggested-approach}, the framework currently implements one major controller, the cross-validation controller (CV), and multiple views to display evaluation scores and model information. The CV controller encapsulates the common k-fold cross-validation procedure that can be used to estimate a prediction strategy's generalization error. It takes a data set and loss function, and returns the fold-losses as well as the overall loss, with estimated confidence intervals, for a given prediction strategy (see Algorithm~\ref{alg:validation-2}). If the workflow specifies a range of hyper-parameters for tuning, the controller automatically optimizes the hyper-parameters in a nested cross-validation procedure and additionally returns the best hyper-parameters found.\\

{\bf Result aggregation and comparison}\\
The workflow framework supports a simple way of results aggregation and comparison, namely the generation of result tables. A table can be easily defined by providing controller-view-pairs as columns and models as rows.
The framework will then evaluate the table cells by running the controller task for the respective models and render the results table using the specified views. Note that the evaluation of the controller tasks and the process of rendering the table is decoupled. It is therefore possible to access the raw table with all the information each controller returned and then render the table with the reduced information that is actually needed. Furthermore, the decoupling allows for manipulation or enhancement of the raw data before rendering. The raw table data can, for example, be sorted by the model performances. Notably, the table supports so-called rank sorting. Rank sorting is, for instance, useful if models are compared on different data sets and ought to be sorted by their overall performance. In this case, it is unsuitable to simply average the data sets' performance scores since the value ranges might differ considerably between the different data sets. Instead, it is useful to rank the performances on each data set and then average the model's rank on each data set to obtain the overall rank. Table~\ref{tbl:results-table-example} gives an example of a rank sorted result table that is typically generated by the workflow framework. The workflow module hence allows for the fast generation of fair, reliable and best-practice model comparison with minimal development overhead for the user.\\

\noindent
\textbf{Source code}

The skpro package has been released open-source under a BSD license on \href{https://github.com/alan-turing-institute/skpro}{GitHub.com/alan-turing-institute/skpro}. The repository includes a comprehensive online documentation and code examples. We welcome contributions and extensions of the skpro project that extend the possibilities of probabilistic supervised learning (please read the contribution guidelines on the skpro project page~\cite{kiraly_skpro_2017} and contact the authors in case of interest).

\begin{table}
    \centering
    \begin{tabular}{rlllll}
    \hline
       \# & Model                    & CV(Dataset A, loss function)   & CV(Dataset B, loss function)    \\
    \hline
       0 & Example model 1 & (2) 12$\pm$1*  & (1) 3$\pm$2*    \\
       1 & Example model 2 & (1) 5$\pm$0.5* & (2) 9$\pm$1*  \\
       2 & Example model 3 & (3) 28$\pm$3*  & (3) 29$\pm$4*  \\
    \hline
    \end{tabular}
\caption{Example and explanation of a rank-sorted results table that can be easily created in the workflow framework: prediction strategies are listed in the rows of the table while the columns present the cross-validated performance of a certain dataset and loss function. Entries are as follows: the numbers after the parenthesis are performances, plusminus estimated standard errors - so for 95\% CI, the standard error ought to be multiplied by 1.96. The numbers in parentheses denote the model's performance rank in the respective column. An asterisk indicates that the prediction strategy was auto-tuned. The strategies are sorted by the average performance rank across all columns, displaying models with the best performances (that is: the lowest losses) at the top of the table.}
\label{tbl:results-table-example}
\end{table}

\newpage
\section{Experiments}\label{sec:experiments}

We conducted a number of numerical experiments to test and showcase the described probabilistic prediction framework.

Our choice of set-up was guided by the development and testing of the framework rather than focused on the implementation of complex state-of-the-art models. A priority was therefore that the implemented models reliably performed in either loss metric above a suitable baseline on a variety of different datasets. Secondary focus is the investigation of composite probabilistic modelling strategies built from off-shelf point prediction strategies.

The exact set-up of datasets, prediction strategies and mode of model comparison is described below.

\subsection{Datasets}\label{datasets}

For the initial experiments, we selected a number of relatively small, tabular datasets frequently used in benchmarking studies. Since probabilistic classification is relatively well studied (see Section~\ref{sec:priorart.assessment}~(c) for this), we focused exclusively on regression prediction tasks. The four datasets of our preliminary study are publicly available at the UCI Machine Learning Repository~\cite{UCIrepo}.

\textbf{Boston Housing} The Boston Housing dataset is our main dataset since it has been extensively studied in practical machine learning literature. It provides a sample of 506 houses, each described by 12 features and the dependant variable of the house-price in thousands of dollar ranging from 5 to 50.

\textbf{Diabetes} The diabetes dataset includes a sample of 442 patients and 10 physiological variables with an indication of disease progression a year later. In our version of the dataset, the features in the provided dataset are already mean-centred and normalized ($-0.2 < x < 0.2$) while the indication targets range from 25 to 346.

\textbf{Energy} The energy dataset reports heating load or energy efficiency in a sample of 768 buildings as a function of five different building parameters (surface, wall, roof, glazing area and overall height). The targets are all real-valued and range from 6 to 44.

\textbf{Bikes} The bike sharing dataset contains a sample of 17389 days on which the amount of rented bicycles and weather data are recorded. The data were collected between 2011 and 2012 and exhibit a wide variation in the targets (22 to 8714 rented bicycles). It should be noted that due to the consecutive/temporal sampling, unlike in the other three datasets, the i.i.d.~assumption is not plausible here - though the bike sharing dataset is nevertheless used in literature for benchmarking in an i.i.d.~setting.

\subsection{Validation set-up}

As losses, we use the squared integrated loss or Gneiting loss $L_{Gn}$ (Definition~\ref{Def:losses}), and the log-loss capped at $\epsilon = 10^{-10}\approx \exp(-23)$. (Definition~\ref{Def:losses} and Section~\ref{sec:logmix}).
For reading convenience, tables involving the capped log-loss will be captioned with the symbol for the uncapped log-loss $L_\ell$, with the capping implicitly understood.

Estimation of performance metrics follows the discussion in Sections~\ref{sec:modelvalidation.estim} and \ref{Sec:modelvalidation.tuning}:
The expected generalization losses are estimated by the common 5-fold cross-validation scheme, using the same splits for all models. Empirical generalization error estimates are aggregated by taking the mean of the fold-wise estimates. Standard errors for the expected generalization loss are estimated by the standard error of the loss residual sample mean, for each of the five test folds; standard error estimates are aggregated from the fold-wise standard errors by taking the mean. For each prediction strategy, the aggregate empirical generalization losses and standard errors are reported.

Automated hyper-parameter tuning, if invoked, is always done as part of the prediction strategy, i.e., conducted in a nested way\footnote{In terms of the API design, the inner, tuning loop of nested tuning is always encapsulated in the prediction strategy estimator object, rather than being run as part of the validation set-up, see Section~\ref{sec:API.pred.tuning}. This way it is also ensured that the user doesn't inadvertently tune on the test set, i.e., it is ensured that evaluation by the workflow orchestrator is always conducted on an entirely unseen test set which was not used for hyper-parameter tuning.}.

In all these cases, tuning is done using the 5-fold nested cross-validation method, using the same grid for tuning within the training fold, across all outer error estimation splits.

\subsection{Prediction strategies}

We describe the probabilistic prediction strategies we consider for the experiments. Most are composite and use off-shelf point predictions in their construction. We describe below the primitive strategies used in composition, the composition motifs, and the strategies used.

\subsubsection{Primitive point prediction strategies}
\label{sec:experiments.strat.prim}

We use a number of point prediction strategies as part of composite predictions as components in composite strategies. These are, with abbreviations in brackets:\\

Uninformed/constant predictor (C) - a constant point prediction. This can be set to a fixed constant, for example 42 which is included as a framework sanity check, as in general a constant prediction of 42 (``C(42)'') should perform badly. Best uninformed baselines (see Section~\ref{Sec:uninfbase}) may be obtained by choosing a constant that depends on the training data, such as the mean label (``mean(y)''), or the standard deviation of the labels (``std(y)'', best uninformed baseline for residual predictions).\\

Linear Regression (LR) - estimated via Ordinary Least Squares.\\

Random Forest (RF) - the implementation of Breiman's random forests~\cite{breiman2001random} from sklearn. We tune hyper-parameters (n\_estimators: $5$, $20$, $30$) and (max\_depth: $10$, $20$, None).\\

Kernel ridge regression (GP) - the implementation of kernel ridge regression from sklearn, based on Rasmussen and Williams' Algorithm 2.1 \cite{rasmussen2006gaussian}. We tune hyper-parameters ($\alpha$: $10^{-10}$, $0.1$, $0.01$). Not enabled by default (but usually needed in practice unless features are already scaled) is normalization of input features, which is another hyper-parameter that can be set. This is denoted as \{Scale, GP\} in what follows.\\

A more complex prediction model should outperform simple baseline strategies, e.g. random guessing. A reasonable simple baseline strategy for our purposes is a model that does not use any of the provided features to make its predictions. Instead, a constant like the mean of the training labels is returned. Such a baseline will most likely predict better-than-chance and do specifically well if the true population labels are uniformly distributed. The corresponding probabilistic strategy, which will be used as baseline in our experiments, calculates the mean and standard deviation of the training labels as an estimation of the point and variance parameter of the predicted distribution.

\subsubsection{Composition strategies}
\label{sec:experiments.strat.comp}

We describe a number of composition meta-strategies used for building the strategies for our experiments. The main motif is the composite parametric prediction strategy, see Section~\ref{sec:composite}.\\

In line with Sections~\ref{sec:classprob} and~\ref{Sec:classicbase}, the class of classical point prediction estimators as measured by the mean squared or mean absolute error may be embedded through parametric predicted distributions with location and dispersion parameter coming from possibly different point prediction strategies.\\
The two main cases we consider in these preliminary experiments are a parametric Gaussian prediction, denoted $\calN$(p=P1, s=P2) and a parametric Laplace prediction, denoted Laplace(p=P1, s=P2), where P1 and P2 are (possibly but not necessarily distinct) point prediction strategies for location and dispersion parameter (both having the unit of the label, i.e., the dispersion having the unit of standard deviation even in the case of the normal). We also implemented a uniform prediction Uniform(p=P1, s=P1) and test it, even though we would not expect good performance on the given datasets - this is for potential latter use in combination with target re-normalization such as in Section~\ref{sec:logmix}.\\

The classical baselines, following Sections~\ref{sec:classprob} and~\ref{Sec:classicbase} come in two types:\\
(a) where the dispersion is estimated from the raw labels, in the below P2 estimates the training label standard deviation. This is denoted s=C(std(y)).\\
(b) where the dispersion is estimated from the residuals of P1. This is a special case of the residual composite adaptor in Algorithm~\ref{alg:adapt-resid}, denoted as s=RE(p, C(std(y))) as a special case of the residual adaptor strategy below.\\

Non-baseline composite strategies are obtained from composing a location point predictor P1 with a residual point estimator PRes by Algorithm~\ref{alg:adapt-resid}. This is denoted (in the example of a Gaussian) as $\calN$(p=P1, s=RE(p, PRes)). For example, a random forest location prediction with an ordinary least squares residual prediction is denoted $\calN$(p=RF, s=RE(p, LR)).\\
In later experiments, the type of residual predicted may also be changed/tuned. This is indicated by an additional argument being one of (squared, log, abs) with a default of squared, for example $\calN$(p=RF, s=RE(p, LR, log)) where the linear regressor predicts logarithmic and not squared residuals. Note that the change in residual type does not affect the domain of the dispersion parameter as the prediction is scaled back to the dispersion parameter domain.\\

As part of the experiments below, it also became apparent that some strategies made zero dispersion predictions which led the log-loss to be infinite and the capped log-loss to be large. A natural idea is therefore to introduce a lower boundary $\kappa$ to obtain a modified dispersion estimate of the type $\hat{\sigma}_{\text{min}} = \max(\kappa, \hat{\sigma})$, note that $\kappa$ does not need to agree with the capping constant $\epsilon$ for the log-loss.\\
The lower bounding is implemented through a wrapper ``Min'' on the natural level of point prediction strategies rather than on the probabilistic strategy level, so that Min(PX) is the lower bounded variant of a point prediction strategy PX. For example, $\calN$(p=GP, s=Min(RE(p, RF))) denotes a normal prediction where location is by Gaussian process posterior mean, dispersion is by a random forest predicting squared residuals, where the prediction is lower bounded by $\kappa$ - i.e., if the random forest on residuals predicts anything smaller than $\kappa$, then the predicted dispersion is set to $\kappa$.\\
The Min-wrapper is tuned, by standard, over an exponential range $\kappa = \{0, 2^0, 2^1, 2^2, 2^3, 2^5\}$.

\subsubsection{Investigated prediction strategies}
\label{sec:experiments.strat.quest}

Of particular interest were a number of composite strategies, corresponding to the following principal questions of interest:

\begin{itemize}
\item[(1)] Are there any genuine composite probabilistic predictors which outperform the classical baselines?
\item[(2)] Is it helpful to tune/change the type of residual which is predicted in the composite parametric residual strategy?
\item[(3)] Is it helpful to tune/change the shape of distribution which is predicted in the composite parametric residual strategy?
\item[(4)] How do off-shelf probabilistic prediction strategies compare to composite strategies?
\end{itemize}

For (1)-(3), we considered combinations of primitive strategies in Section~\ref{sec:experiments.strat.prim} via the composition modes described in Section~\ref{sec:experiments.strat.comp}.

For (4), we considered only Gaussian process regression, as they were the only off-shelf probabilistic strategy readily available through sklearn, see Section~\ref{sec:gaussianprocesses}. Despite our vendor interface allowing access to Bayesian strategies from probabilistic programming packages, we were unable to complete experiments with off-shelf Bayesian methods due to their long run-time (see Section~\ref{sec:experiments.future}; currently, the experiments are still underway).

\subsection{Results}
\label{sec:experiments.results}

The results show that the skpro toolbox is functional for the major use cases outlined in Section~\ref{sec:API.overview.usecases}.\\

The preliminary answers to the questions raised in Section~\ref{sec:experiments.strat.quest} are (1) yes, definitely; (2) maybe; (3) maybe; (4) badly, but not unexpectedly so. Further investigation may be necessary on (2)-(4).

The exact experiments for obtaining these answers are described below.

\subsubsection{Point prediction baselines}
\label{sec:experiments.results.baselines}

The first experiment establishes performance of the point prediction baselines, i.e., the type (a) baselines in Section~\ref{sec:experiments.strat.comp}, with a normal distribution shape which corresponds to the mean squared error via Sections~\ref{sec:classprob}. The prediction strategy whose mean and standard deviation are constant is the best classical uninformed baseline. We also include a very badly chosen uninformed model that predicts a constant $C = 42$ for both location and variance as a sanity check of the framework.

Table~\ref{tbl:baseline} summarizes the results.

\begin{table}[h]
\begin{adjustbox}{center}
    \scriptsize
	\input{data/tex/baseline}
\end{adjustbox}
\caption{Performance of type (a) point prediction baselines: tuned GP, tuned RF and LR estimator are compared to the best uninformed point prediction strategy (3) which they all outperform at a 95\% confidence level. Note that a larger number for the Gneiting loss is better, due to the negative sign. Rows of uninformed baseline choices are grey. All methods, including the best uninformed point predictor, outperform the badly chosen uninformed baseline.}
\label{tbl:baseline}
\end{table}

As expected, the informed classical models outperform the best uninformed baseline, while the badly chosen uninformed predictor produces by far the highest loss values. The model ranking is consistent across the different datasets and for both loss functions and finds the GP and RF strategies performing on par as the best, except on the bike sharing dataset where the tuned GP is best.

\subsubsection{Residual estimation}

As a next step, we ran a series of experiments using the \emph{ResidualEstimator} composite parametric prediction strategy.

As mentioned before, as we encountered the small variance problem for composition, we first investigated the effect of the Min-wrapper on the goodness of the residual composite prediction.

Table \ref{tbl:minimum} shows a comparison of residual estimation models with and without the minimum strategy on the Boston Housing data for both losses.

\begin{table}[h]
	\scriptsize
    \begin{adjustbox}{center}
	       \input{data/tex/minimum}
    \end{adjustbox}
\caption{Effectiveness of the minimum wrapper on composite predictions: Min($\cdot$)-strategies replace standard deviations below a tuned threshold with the threshold value. It is apparent that in terms of both losses, i.e., not only the log-loss, the Min-wrapped predictions generally improve the model's performances above the respective bound-free counterparts (in grey)}
\label{tbl:minimum}
\end{table}

Generally, it can be said, that the minimum variance boundary improves the residual prediction performance of all three estimators GP, RF and LR. The Min-wrapper is hence used in all subsequent residual estimation models.

In the full experiment, we compare the ResidualEstimator models to the best classical and best uninformed baselines (see Section~\ref{sec:experiments.results.baselines}) to answer question (1) of Section~\ref{sec:experiments.strat.quest}, namely whether we can find evidence for genuine probabilistic prediction improving above point prediction type strategies.
Table~\ref{tbl:residual} presents the results.

\begin{table}[h]
	\scriptsize
    \begin{adjustbox}{center}
	       \input{data/tex/residual}
    \end{adjustbox}
\caption{Performance of residual composite prediction strategies: composite predictors use tuned RF, LR and GP estimators for dispersion predictions, using squared residuals. Note that we do \emph{not} use the GP's in-built variance, but instead we fit a second GP to squared residuals. Generally, use of a ResidualEstimator leads to improved model performances against the two baselines with constant (=homoscedastic) variance predictions (grey). Note that the residual estimators with constant residual variance are not baselines, as they use the ResidualEstimator composition strategy.}
\label{tbl:residual}
\end{table}

Different composite strategies are best on different datasets: the best composite strategy for Boston Housing is $\calN$(p=RF, s=Min(RE(p, GP))), i.e., random forest for the mean and GP (posterior mean of a model fitted to residuals) for the variance. The best strategy for diabetes is $\calN$(p=GP, s=Min(RE(p, GP))), with two GP, one for mean and variance each. The best strategy for the energy dataset consist of $\calN$(p=P1, s=Min(RE(p, P2))), with P1,P2$\in\{\mbox{GP},\mbox{RF}\}$, i.e., random forest and GP for mean and variance prediction chosen in any combination. For bike sharing, a constant variance prediction strategy with mean given by random forest is best - however, the variance is obtained from the residual composite estimator rather from the training labels alone, thus this is also a non-trivial composite strategy (and not a classical baseline).

Thus, question (1) may be answered positively - as measured by the log-loss, use of composite strategies is beneficial.

\subsubsection{Impact of residual transformations}

As discussed in Section~\ref{sec:adaptors.PPandResid}, residual composite strategies may be modified by fitting the absolute or logarithmic training residuals to the residual model, then back transforming predicted residuals accordingly.
We investigate the impact of three residual transformations on residual composite strategies where the same point estimation strategy, GP, RF or LR, is used twice - which includes predictors which were on par best in the previous experiment.
Table~\ref{tbl:residual-transformers} reports the results.

\begin{table}[h]
	\scriptsize
    \begin{adjustbox}{center}
	       \input{data/tex/residual_transformers.tex}
    \end{adjustbox}
\caption{Comparison of different residual transformations for composite residual prediction. Three transformations are compared: absolute (abs), logarithmic (log) and squared (squared) residuals. Different transformations lead to different performance. Overall, however, the squared residual transformation leads to highest model performances.}
\label{tbl:residual-transformers}
\end{table}

One observes that different transformations lead to different performance, which validates that the shape does have an effect on predictive performance. Overall, however, the squared residual transformation leads to highest performances which is somewhat surprising, and does not provide evidence that switching the shape may be beneficial (unless it is \emph{to} squared residuals, of course).
Though only a small range of composite strategies were investigated, so further investigation as to question (2) may be necessary. For example, the choice of parametric form for the predicted distribution may have to be matched with the residual type, e.g., Laplace to absolute and Cauchy to logarithmic.

\subsubsection{Impact of the parametric distribution type}

The experiments so far used the parametric form of a normal distribution $\calN$  in composite models. Ultimately, however, it may be beneficial to choose a parametric shape based on the data, or even make non-parametric predictions. To scope the impact of the distributional form, we investigate the performance of the classical, constant variance, baselines for different forms of predicted distributions: Gaussian, Laplace and uniform distribution (located at the interval center).
Table \ref{tbl:shape_comparison} presents the results for location parameters predicted by GP, RF and LR-based models, while the dispersion parameter is obtained from the overall training label dispersion.

\begin{table}[h]
	\scriptsize
    \begin{adjustbox}{center}
	       \input{data/tex/shape_comparison.tex}
    \end{adjustbox}
\caption{Performance of different assumption on distribution shape: compared prediction strategies use GP, RF and LR point prediction strategy for location/mean prediction and training label dispersion for predicted dispersion. The three mean prediction strategies are combined with a predicted distribution shape which is either Gaussian ($\calN$), Laplace, or uniform, resulting in nine prediction strategies. Depending on the dataset, an assumption of Laplace or Gaussian form performs best.}
\label{tbl:shape_comparison}
\end{table}

While the uniform distribution is consistently the lowest performing parametric assumption, the relative performance of the Laplace and Gaussian parametric form depend on the dataset and used point prediction strategy. The best classical Laplace strategy outperforms the best classical Gaussian strategy on the Boston housing and energy datasets, while the situation is reversed for the diabetes and bike sharing datasets.

Hence, question (3) of Section~\ref{sec:experiments.strat.quest}, on whether there is evidence for a parametric distribution choice to measurably impact predictive performance, may be answered positively with a caveat. The caveat is that a more systematic study of distribution type with composition strategies is necessary for a full assessment.

\subsubsection{Direct use of an off-shelf Gaussian process strategy}

sklearn's GP regression estimator \texttt{GaussianProcessRegressor} provides its own variance prediction interface via the \texttt{return\_std} flag - though it is not explicitly stated (as of version 0.19.1) which distribution's variance this exactly provides (being of especial relevance in the context of the discussion in Section~\ref{sec:gaussianprocesses}).

We investigate the performance of a model which uses both mean and variance prediction from the full Gaussian process regression strategy. As for the residual composite prediction strategies, we also implemented a version that includes a minimum cut-off for the variance predictions of the combined GP strategy.
Table \ref{tbl:gp} summarizes the performance, in comparison against the best ResidualEstimator-based models from the previous experiments.

 \begin{table}[h]
 	\scriptsize
    \begin{adjustbox}{center}
 	      \input{data/tex/gp}
    \end{adjustbox}
 \caption{Performance of full GP prediction strategy: The models with ``p=s'' use the output of a full Gaussian process mean/variance prediction, one with and one without Min-wrapping for the predicted variance. While the Min-wrapping helps performance of the GP strategy, it generally performs worse than the ResidualEstimator-based models (denoted with B, =/+/- for equal/better/worse performance), and sometimes worse than any of the naive baselines.}
 \label{tbl:gp}
 \end{table}

Noticeable, the cut-off optimization proves to be useful since it is apparent that the minimum variance model clearly outperforms its unoptimized counterpart. The full Gaussian process strategy performs badly, indicating that the variance prediction is possibly not obtained from the true predictive posterior but instead from the posterior distribution of the predictive mean, in line with the discussion at the end of Section~\ref{sec:gaussianprocesses}.\\

More off-shelf probabilistic prediction models will be studied in further experiments.

\subsection{Further experiments}
\label{sec:experiments.future}

As mentioned above, further systematic investigation of combinations for composite strategies may be beneficial to understand the effect and usefulness of different choices. We hope to have further results on this early to mid 2018.\\

Similarly, the performance of most off-shelf methodology, in particular of Bayesian methodology including probabilistic programming, is still open.
A larger series of benchmark experiments is currently running on high-performance computing facilities, with some of the off-shelf methods being the computational bottleneck, especially Bayesian MCMC, but also of the other Bayesian denominations. Due to the high runtime of off-shelf Bayesian methodology, the occasional break-on-error, and the number of iterations required for validatory benchmarking of multiple methods on multiple datasets, we expect conclusion of experiments sometime mid or late 2018, after which this section will be updated with the full set of results.\\

As already mentioned, we very much welcome any volunteer effort in extending the skpro framework, for example with interfaces to popular methods, and/or in conducting the benchmarking experiments (please read the contribution guidelines on the skpro project page~\cite{kiraly_skpro_2017} and contact the authors in case of interest).

\newpage

\pagestyle{empty}
\bibliographystyle{plainnat}
\bibliography{propre}

\end{document}

%% file: macros/macros.tex
\usepackage[utf8]{inputenc}

\usepackage{tikz,ifthen,xstring,calc,pgfkeys,pgfopts}
\usepackage{tikz-uml}
\usepackage{courier}

\usepackage{multirow}
\usepackage{array}
\usepackage{rotating}

\usepackage{hyperref}
\usepackage{url}

\usepackage{algorithm,algpseudocode}
\usepackage{listings}

\usepackage{mathtools}
\usepackage{amsmath,amssymb,amsthm}
\usepackage{thmtools}

\usepackage{graphicx}
\usepackage{adjustbox}

\usepackage{todonotes}

\usepackage[charter]{mathdesign}
\usepackage[mathcal]{eucal}
\usepackage{bbm}
\usepackage{CJKutf8}

\usepackage{color, colortbl}
\definecolor{Gray}{gray}{0.9}
\definecolor{LightCyan}{rgb}{0.88,1,1}

\usepackage[shortlabels]{enumitem}
\usepackage{scrextend}
\addtokomafont{labelinglabel}{\sffamily}

\usepackage[margin=1.2 in]{geometry}
\usepackage{pdflscape}

\usepackage[numbers]{natbib}
\usepackage{authblk}

\usepackage{titlesec}
\titleformat{\section}
	{\normalfont\Large\bfseries\filcenter}{\thesection.}{1 ex}{}
\titleformat{\subsection}
	{\normalfont\normalsize\bfseries}{\thesubsection.}{1 ex}{}
\titleformat{\subsubsection}
	{\normalfont\normalsize\bfseries}{\thesubsubsection.}{1 ex}{}

\declaretheorem[name=Theorem]{Thm}
\declaretheorem[within=section,name=Lemma]{Lem}
\declaretheorem[sibling=Lem,name=Definition]{Def}

\declaretheorem[sibling=Lem,name=Proposition]{Prop}
\declaretheorem[sibling=Lem,name=Remark]{Rem}
\declaretheorem[sibling=Lem,name=Example]{Ex}
\declaretheorem[sibling=Lem,name=Corollary]{Cor}

\newcommand{\obj}[1]{\texttt{\footnotesize{#1}}}

\renewcommand{\vec}[1]{\mathbf{#1}}

\newcommand{\argmin}{\operatornamewithlimits{argmin}}

\newcommand{\id}{\operatorname{id}}

\newcommand{\Var}{\operatorname{Var}}
\newcommand{\Bias}{\operatorname{Bias}}

\newcommand{\PBias}{\operatorname{PBias}}
\newcommand{\DBias}{\operatorname{DBias}}
\newcommand{\Ent}{\operatorname{H}}
\newcommand{\Err}{\operatorname{Err}}

\newcommand{\logit}{\operatorname{logit}}

\newcommand{\Pure}{\operatorname{Pure}}

\newcommand{\card}{\#}

\newcommand{\Bin}{\operatorname{Bin}}

\newcommand{\calA}{\mathcal{A}}

\newcommand{\calD}{\mathcal{D}}

\newcommand{\calF}{\mathcal{F}}

\newcommand{\calH}{\mathcal{H}}

\newcommand{\calL}{\mathcal{L}}
\newcommand{\calM}{\mathcal{M}}
\newcommand{\calN}{\mathcal{N}}

\newcommand{\calP}{\mathcal{P}}
\newcommand{\calQ}{\mathcal{Q}}

\newcommand{\calS}{\mathcal{S}}
\newcommand{\calT}{\mathcal{T}}
\newcommand{\calU}{\mathcal{U}}

\newcommand{\calX}{\mathcal{X}}
\newcommand{\calY}{\mathcal{Y}}
\newcommand{\calZ}{\mathcal{Z}}

\newcommand{\BB}{\ensuremath{\mathbb{B}}}

\newcommand{\EE}{\ensuremath{\mathbb{E}}}

\newcommand{\NN}{\ensuremath{\mathbb{N}}}

\newcommand{\RR}{\ensuremath{\mathbb{R}}}

\newcommand{\bs}{\boldsymbol{s}}

\newcommand{\bw}{\boldsymbol{w}}
\newcommand{\bx}{\boldsymbol{x}}
\newcommand{\by}{\boldsymbol{y}}

\newcommand{\bY}{\boldsymbol{Y}}

\makeatletter
\providecommand*{\diff}%
        {\@ifnextchar^{\DIfF}{\DIfF^{}}}
\def\DIfF^#1{%
        \mathop{\mathrm{\mathstrut d}}%
                \nolimits^{#1}\gobblespace
}
\def\gobblespace{%
        \futurelet\diffarg\opspace}
\def\opspace{%
        \let\DiffSpace\!%
        \ifx\diffarg(%
                \let\DiffSpace\relax
        \else
                \ifx\diffarg\[%
                        \let\DiffSpace\relax
                \else
                        \ifx\diffarg\{%
                                \let\DiffSpace\relax
                        \fi\fi\fi\DiffSpace}
\makeatother

\newcommand{\Distr}{\operatorname{Distr}}

\newcommand{\docstr}[1]{\footnotesize{\hspace{0.5em}\emph{#1}}\\}


%% file: figures/api_overview.tex

\begin{tikzpicture}
    \begin{umlpackage}[x=0, y=0]{Framework}
        \umlclass[x=-5.5, y=-0.2, name=data]{Data}{
        $X$, $y$
        }{
        $X^*$, $y^*$
        }

        \begin{umlpackage}[x=-5, y=3]{Vendors}
            \umlsimpleclass{3rd party models}

        \end{umlpackage}

        \begin{umlpackage}[x=0, y=3]{API}
            \umlsimpleclass[x=0, y=0]{ProbabilisticEstimator}
            \umlnote[x=0,y=3, width=4cm]{ProbabilisticEstimator}{Represents the models for probabilistic prediction making}
            \umlsimpleclass[x=0, y=-1.5]{Distribution}

            \umluniassoc[mult=predict, pos=0.5, name=predict, very thick]{ProbabilisticEstimator}{Distribution}
            \umlnote[x=5.5,y=0, width=5cm]{predict-1}{Estimator returns Distribution that represents probabilistic predictions. The distribution object extracts the relevant predicted properties like the density from the model.}

            \umlunicompo[anchors=1cm and 2cm]{Distribution}{ProbabilisticEstimator}
        \end{umlpackage}

        \umluniassoc[mult=adapter, pos=0.9, align=right, name=integration]{3rd party models}{ProbabilisticEstimator}
        \umlnote[x=-5,y=6.5, width=4cm]{integration-1}{3rd party models can be integrated using a vendor interface}

    \end{umlpackage}

    \begin{umlpackage}[x=-3.5, y=-0.5]{*}
        \umlsimpleclass[x=2, y=0]{Model}
        \umlreal[anchor1=south]{API}{Model}
        \umlsimpleclass[x=5, y=0]{Workflow automation}
        \umlreal[anchor1=south]{API}{Workflow automation}
        \umlsimpleinterface[x=7, y=0]{Results}

        \node[circle,fill,minimum size=5mm] (head) at (0,1) {};\node[rounded corners=2pt,minimum height=1.3cm,minimum width=0.4cm,fill,below = 1pt of head] (body) {};\draw[line width=1mm,round cap-round cap] ([shift={(2pt,-1pt)}]body.north east) --++(-90:6mm); \draw[line width=1mm,round cap-round cap] ([shift={(-2pt,-1pt)}]body.north west)--++(-90:6mm);\draw[thick,white,-round cap] (body.south) --++(90:5.5mm);

        \umlnote[x=-2.5, y=2, width=6cm]{head}{User defines data-model-interaction (*)}
        \umlnote[x=9.5, y=0.5]{Results}{Results of model assessment can be aggregated automatically}
    \end{umlpackage}

    \umlreal[anchor1=-0.5cm]{Data}{*}
\end{tikzpicture}

%% file: figures/api_base.tex

\begin{tikzpicture}

    \begin{umlpackage}[x=0, y=0]{sklearn}
    \umlclass[type=abstract]{BaseEstimator}{

    }{
    get\_params(deep : boolean = True) : dict\\
        \docstr{Get the public parameters of the estimator}
    set\_params(params : dict) : BaseEstimator\\
        \docstr{Set the public parameters of the estimator}
    }

    \end{umlpackage}

    \begin{umlpackage}[x=0, y=-7, name=skpro-base]{skpro.base}

       \umlclass[type=abstract, x=0, y=3]{ProbabilisticEstimator}{

        }{
            fit(X: $N\times n$, y: $N\times 1$) : ProbabilisticEstimator \\
            predict(X: $M\times n$) : Distribution
        }
        \umlVHVinherit{ProbabilisticEstimator}{BaseEstimator}

        \umlclass[x=8,y=0, type=abstract]{Distribution}{
            estimator: ProbabilisticEstimator\\
                \docstr{Parent probabilistic estimator object}
            X: $M\times n$\\
                \docstr{Test features}
            selection: slice | int = None\\
                \docstr{Optional subset selection of the features}
            mode: string = 'elementwise'\\
                \docstr{Interface mode (elementwise or batch)}
        }{
            \_init() : void\\
             \docstr{Called on instantiation (prediction), if existent}
            \textit{point}() : $M\times 1$\\
                \docstr{The point prediction that corresponds to X}
            \textit{std}() : $M\times 1$\\
                \docstr{The estimated standard deviation that corresponds to X}
            pdf($\vec{x}$) : array\\
                \docstr{Probability density function}
            cdf($\vec{x}$) : array\\
                \docstr{Cumulative density function}
            lp2() : $M\times 1$\\
                \docstr{Norm of the probability density function $L^2 = \int pdf(x)^2 dx$}
        }
        \umldep[name=dist]{Distribution}{ProbabilisticEstimator}

        \umlclass[x=0,y=0]{VendorEstimator}{
            model : VendorInterface\\
            adapter : DensityAdapter
        }{

        }
        \umlVHVinherit{VendorEstimator}{ProbabilisticEstimator}

        \umlclass[x=0,y=-3, type=abstract]{VendorInterface}{

        }{
            on\_fit(X: $N\times n$, y: $N\times 1$) : void \\
                \docstr{Implements the vendor fit procedure}
            on\_predict(X: $M\times n$) : void\\
                \docstr{Implements the vendor fit procedure}
        }
        \umlnest{VendorInterface}{VendorEstimator}

        \umlclass[x=0,y=-6, type=abstract]{BayesianVendorInterface}{

        }{
            samples() : $M \times q$\\
            \docstr{Returns predictive posterior samples}
            \docstr{of size $q$ for each of the $M$ test points}
        }
        \umlVHVinherit{BayesianVendorInterface}{VendorInterface}

        \umlclass[x=5.5,y=-6, type=abstract]{DensityAdapter}{

        }{
            pdf($\vec{x}$) : array\\
                \docstr{Probability density function}
            cdf($\vec{x}$) : array\\
                \docstr{Cumulative density function}
        }
        \umlnest[geometry=-|-]{DensityAdapter}{VendorEstimator}

        \umlclass[x=9.8,y=-6]{KernelDensityAdapter}{
            estimator : BaseEstimator\\
            \docstr{Kernel estimator}
        }{

        }
        \umlVHVinherit{KernelDensityAdapter}{DensityAdapter}

    \end{umlpackage}

    \begin{umlpackage}[x=4, y=-17, name=thirdparty]{Third-party integration, e.g. PyMC}
        \umlclass[x=0,y=0]{PymcInterface}{
            model\_definition : callable(model, X, y)\\
                \docstr{Callable that defines a model using the given PyMC3 ``model`` variable and training features ``X`` as well as and the labels ``y``}
            samples\_size: integer = 500\\
                \docstr{Number of samples to be drawn from the posterior distribution}
        }{

        }
        \umlVHVinherit[anchor2=30]{PymcInterface}{BayesianVendorInterface}
    \end{umlpackage}

    \umlnote[width=6cm, x=8, y=0]{skpro-base}{
        \textbf{Possible model implementation by user}\\
        \vspace{1em}
        \texttt{model = VendorEstimator(}\\
            \hspace{0.5em}\texttt{PymcInterface(pymc\_model)}
            \\)\\
        \texttt{model.fit($X$, $y$)}\\
        \texttt{y\_pred = model.predict($X^*$, $y^*$)}\\
        \texttt{result = y\_pred.pdf(3.14)}
    }
\end{tikzpicture}

%% file: figures/api_parametric.tex

\begin{tikzpicture}

    \begin{umlpackage}[x=0, y=0, name=parametric]{skpro.parametric}

       \umlclass[x=0, y=3, template=ProbabilisticEstimator]{ParametricEstimator}{
           point: estimator\\
              \docstr{Estimator for point parameter predictions}
           std: estimator\\
              \docstr{Estimator for variance parameter predictions}
           point\_std: estimator = None\\
              \docstr{Optional estimator for both point and variance prediction}
           shape: string = 'norm'\\
              \docstr{Distribution type as defined in scipy.stats, e.g. normal}
        }{
        }

    \end{umlpackage}

    \begin{umlpackage}[x=8.5, y=3, name=residuals]{skpro.parametric.residuals}
        \umlclass[x=0,y=0, template=BaseEstimator]{ResidualEstimator}{
            residual\_estimator : estimator\\
                \docstr{Estimator for the residual prediction}
            base\_estimator: string = 'point'\\
                \docstr{Name of estimator which residuals}
                \docstr{are modelled}
            transform='squared\_error'\\
                \docstr{Used residual metric:}
                \docstr{ 'abs\_error', 'log\_error'}
        }{

        }
        \umlnest[]{ResidualEstimator}{ParametricEstimator}
    \end{umlpackage}

    \begin{umlpackage}[x=0.5, y=-3.3, name=estimators]{skpro.parametric.estimators}

        \umlclass[template=BaseEstimator]{Constant}{
        constant: float | callable(X, y) | string: 'mean(y)', 'std(y)'\\
            \docstr{Specifies the constant. A callable receives the training data during}
            \docstr{fit and must return a constant value. The string options provide}
            \docstr{a shortcut for mean/std extraction from the labels.}
        name: string = str(constant)\\
            \docstr{Optional description of the constant for the estimator}
            \docstr{string representation}
        }{

        }
        \umlnest[geometry=|-, anchor1=2cm]{Constant}{ParametricEstimator}
    \end{umlpackage}

    \umlnote[width=5cm, x=9.5, y=-3]{ParametricEstimator}{
        Uses classical estimators, e.g. any estimator of scikit-learn's standard prediction algorithms\\
    }
\end{tikzpicture}

%% file: data/tex/baseline.tex
\begin{tabular}{rlllll}
\hline
   \# & Model                               & CV(Boston, $L_{\ell}$)   & CV(Diabetes, $L_{\ell}$)   & CV(Bikes, $L_{\ell}$)   & CV(Energy, $L_{\ell}$)   \\
\hline
   0 & $\mathcal{N}$(p=\{Scale, GP\}, s=C(std(y))) & (2) 3.247$\pm$0.007*                & (1) 5.512$\pm$0.017*                  & (1) 8.685$\pm$0.008*               & (1) 3.231$\pm$0.004*                \\
   1 & $\mathcal{N}$(p=RF, s=C(std(y)))            & (1) 3.216$\pm$0.025*                & (3) 5.554$\pm$0.017*                  & (2) 8.724$\pm$0.006*               & (2) 3.2309$\pm$0.0035*              \\
   2 & $\mathcal{N}$(p=LR, s=C(std(y)))            & (3) 3.279$\pm$0.011                 & (2) 5.516$\pm$0.009                   & (3) 8.763$\pm$0.011                & (3) 3.2755$\pm$0.0027               \\
\rowcolor{Gray}
   3 & $\mathcal{N}$(p=C(mean(y)), s=C(std(y)))    & (4) 3.65$\pm$0.05                   & (4) 5.767$\pm$0.021                   & (4) 8.989$\pm$0.023                & (4) 3.731$\pm$0.019                 \\
\rowcolor{Gray}
   4 & $\mathcal{N}$(p=C(42), s=C(42))             & (5) 4.788$\pm$0.005                 & (5) 9.7$\pm$0.4                       & (5) 23.999$\pm$0.024               & (5) 4.7954$\pm$0.0028               \\
\hline
\hline
   \# & Model                               & CV(Boston, $L_{g}$)   & CV(Diabetes, $L_{g}$)   & CV(Bikes, $L_{g}$)    & CV(Energy, $L_{g}$)   \\
\hline
   0 & $\mathcal{N}$(p=\{Scale, GP\}, s=C(std(y))) & (2) -0.04955$\pm$0.00022*     & (2) -0.00472$\pm$0.00005*       & (1) -0.0001991$\pm$0.0000023* & (1) -0.05108$\pm$0.00015*     \\
   1 & $\mathcal{N}$(p=RF, s=C(std(y)))            & (1) -0.0516$\pm$0.0006*       & (3) -0.00457$\pm$0.00005*       & (2) -0.0001921$\pm$0.0000015* & (2) -0.05105$\pm$0.00014*     \\
   2 & $\mathcal{N}$(p=LR, s=C(std(y)))            & (3) -0.0478$\pm$0.0005        & (1) -0.00475$\pm$0.00013        & (3) -0.0001836$\pm$0.0000013  & (3) -0.04784$\pm$0.00017      \\
\rowcolor{Gray}
   3 & $\mathcal{N}$(p=C(mean(y)), s=C(std(y)))    & (4) -0.0342$\pm$0.0006        & (4) -0.00331$\pm$0.00016        & (4) -0.000136$\pm$0.000006    & (4) -0.0237$\pm$0.0006        \\
\rowcolor{Gray}
   4 & $\mathcal{N}$(p=C(42), s=C(42))             & (5) -0.01000$\pm$0.00005      & (5) 0.00125$\pm$0.00015         & (5) 0.006693$\pm$0.000021     & (5) -0.00991$\pm$0.00004      \\
\hline
\end{tabular}

%% file: data/tex/minimum.tex
\begin{tabular}{rlll}
\hline
   \# & Model                                                              & CV(Boston, $L_{\ell}$)   & CV(Boston, $L_{Gn}$)   \\
\hline
   0 & $\mathcal{N}$(p=RF, s=Min(RE(p, RF)))            & (2) 2.59$\pm$0.09*                  & (1) -0.1177$\pm$0.0018*       \\
   1 & $\mathcal{N}$(p=RF, s=Min(RE(p, LR)))            & (1) 2.57$\pm$0.07*                  & (3) -0.106$\pm$0.006*         \\
   2 & $\mathcal{N}$(p=RF, s=Min(RE(p, \{Scale, GP\}))) & (3) 2.74$\pm$0.09*                  & (2) -0.112$\pm$0.004*         \\
\rowcolor{Gray}
   3 & $\mathcal{N}$(p=RF, s=RE(p, \{Scale, GP\}))             & (4) 2.84$\pm$0.10*                  & (4) -0.104$\pm$0.006*         \\
\rowcolor{Gray}
   4 & $\mathcal{N}$(p=RF, s=RE(p, C(std(y))))                 & (5) 2.97$\pm$0.12*                  & (6) -0.077$\pm$0.011*         \\
\rowcolor{Gray}
   5 & $\mathcal{N}$(p=RF, s=RE(p, RF))                        & (7) 3.29$\pm$0.14*                  & (5) -0.082$\pm$0.016*         \\
   6 & $\mathcal{N}$(p=RF, s=Min(RE(p, C(std(y)))))     & (6) 2.98$\pm$0.17*                  & (7) -0.064$\pm$0.005*         \\
\rowcolor{Gray}
   7 & $\mathcal{N}$(p=RF, s=RE(p, LR))                        & (8) 3.97$\pm$0.20*                  & (8) 0.09$\pm$0.07*            \\
\hline

\hline
\hline
   0 & $\mathcal{N}$(p=\{Scale, GP\}, s=Min(RE(p, LR)))            & (2) 2.94$\pm$0.06*                  & (2) -0.1047$\pm$0.0035*       \\
   1 & $\mathcal{N}$(p=\{Scale, GP\}, s=Min(RE(p, \{Scale, GP\}))) & (4) 3.00$\pm$0.13*                  & (1) -0.1070$\pm$0.0026*       \\
   2 & $\mathcal{N}$(p=\{Scale, GP\}, s=Min(RE(p, C(std(y)))))     & (1) 2.92$\pm$0.12*                  & (5) -0.1007$\pm$0.0024*       \\
   3 & $\mathcal{N}$(p=\{Scale, GP\}, s=Min(RE(p, RF)))            & (3) 2.98$\pm$0.13*                  & (3) -0.104$\pm$0.004*         \\
\rowcolor{Gray}
   4 & $\mathcal{N}$(p=\{Scale, GP\}, s=RE(p, C(std(y))))                 & (5) 3.53$\pm$0.17*                  & (4) -0.103$\pm$0.005*         \\
\rowcolor{Gray}
   5 & $\mathcal{N}$(p=\{Scale, GP\}, s=RE(p, \{Scale, GP\}))             & (6) 6.6$\pm$0.4*                    & (6) -0.006$\pm$0.007*         \\
\rowcolor{Gray}
   6 & $\mathcal{N}$(p=\{Scale, GP\}, s=RE(p, LR))                        & (7) 7.6$\pm$0.5*                    & (8) 0.60$\pm$0.12*            \\
\rowcolor{Gray}
   7 & $\mathcal{N}$(p=\{Scale, GP\}, s=RE(p, RF))                        & (8) 9.77$\pm$0.28*                  & (7) 0.39$\pm$0.04*            \\
\hline

\hline
\hline
   0 & $\mathcal{N}$(p=LR, s=Min(RE(p, \{Scale, GP\}))) & (1) 3.39$\pm$0.05*                  & (1) -0.0599$\pm$0.0013*       \\
\rowcolor{Gray}
   1 & $\mathcal{N}$(p=LR, s=RE(p, RF))                        & (2) 3.466$\pm$0.032*                & (2) -0.0518$\pm$0.0027*       \\
   2 & $\mathcal{N}$(p=LR, s=Min(RE(p, RF)))            & (3) 3.52$\pm$0.06*                  & (3) -0.0501$\pm$0.0017*       \\
   3 & $\mathcal{N}$(p=LR, s=Min(RE(p, LR)))            & (5) 3.827$\pm$0.025*                & (5) -0.0365$\pm$0.0011*       \\
\rowcolor{Gray}
   4 & $\mathcal{N}$(p=LR, s=RE(p, \{Scale, GP\}))             & (4) 3.67$\pm$0.09*                  & (6) -0.0282$\pm$0.0031*       \\
\rowcolor{Gray}
   5 & $\mathcal{N}$(p=LR, s=RE(p, LR))                        & (6) 4.17$\pm$0.06                   & (4) -0.039$\pm$0.009          \\
\rowcolor{Gray}
   6 & $\mathcal{N}$(p=LR, s=RE(p, C(std(y))))                 & (7) 4.97$\pm$0.04                   & (7) -0.00889$\pm$0.00032      \\
   7 & $\mathcal{N}$(p=LR, s=Min(RE(p, C(std(y)))))     & (8) 4.979$\pm$0.024*                & (8) -0.00885$\pm$0.00013*     \\
\hline
\end{tabular}

%% file: data/tex/residual.tex

\begin{tabular}{rlllll}
\hline
   \# & Model                                                              & CV(Boston, $L_{\ell}$)   & CV(Diabetes, $L_{\ell}$)   & CV(Bikes, $L_{\ell}$)   & CV(Energy, $L_{\ell}$)   \\
\hline
\rowcolor{Gray}
   0 & $\mathcal{N}$(p=RF, s=C(std(y)))                                           & (5) 3.198$\pm$0.013*                & (1) 5.560$\pm$0.015*                  & (1) 8.724$\pm$0.008*               & (5) 3.231$\pm$0.004*                \\
   1 & $\mathcal{N}$(p=RF, s=Min(RE(p, \{Scale, GP\}))) & (1) 2.646$\pm$0.035*                & (4) 5.92$\pm$0.12*                    & (5) 9.378$\pm$0.026*               & (2) 0.51$\pm$0.07*                  \\
   2 & $\mathcal{N}$(p=RF, s=Min(RE(p, C(std(y)))))     & (4) 2.79$\pm$0.08*                  & (3) 5.86$\pm$0.06*                    & (2) 8.773$\pm$0.031*               & (4) 0.872$\pm$0.030*                \\
   3 & $\mathcal{N}$(p=RF, s=Min(RE(p, RF)))            & (2) 2.71$\pm$0.06*                  & (5) 6.04$\pm$0.12*                    & (6) 9.61$\pm$0.09*                 & (1) 0.47$\pm$0.05*                  \\
   4 & $\mathcal{N}$(p=RF, s=Min(RE(p, LR)))            & (3) 2.77$\pm$0.08*                  & (6) 6.14$\pm$0.09*                    & (4) 9.31$\pm$0.08*                 & (3) 0.60$\pm$0.05*                  \\
\rowcolor{Gray}
   5 & $\mathcal{N}$(p=C(mean(y)), s=C(std(y)))                                   & (6) 3.640$\pm$0.030                 & (2) 5.767$\pm$0.024                   & (3) 8.990$\pm$0.021                & (6) 3.731$\pm$0.009                 \\
\hline


\hline
\hline
   0 & $\mathcal{N}$(p=\{Scale, GP\}, s=Min(RE(p, \{Scale, GP\}))) & (1) 2.92$\pm$0.08*                  & (1) 5.379$\pm$0.024*                  & (4) 8.549$\pm$0.030*               & (1) 0.462$\pm$0.033*                \\
   1 & $\mathcal{N}$(p=\{Scale, GP\}, s=Min(RE(p, LR)))            & (2) 2.98$\pm$0.08*                  & (2) 5.411$\pm$0.025*                  & (1) 8.525$\pm$0.029*               & (3) 0.553$\pm$0.020*                \\
   2 & $\mathcal{N}$(p=\{Scale, GP\}, s=Min(RE(p, C(std(y)))))     & (3) 3.01$\pm$0.13*                  & (3) 5.424$\pm$0.014*                  & (2) 8.534$\pm$0.022*               & (4) 0.90$\pm$0.11*                  \\
   3 & $\mathcal{N}$(p=\{Scale, GP\}, s=Min(RE(p, RF)))            & (4) 3.04$\pm$0.11*                  & (4) 5.47$\pm$0.06*                    & (3) 8.543$\pm$0.022*               & (2) 0.49$\pm$0.06*                  \\
\rowcolor{Gray}
   4 & $\mathcal{N}$(p=\{Scale, GP\}, s=C(std(y)))                                           & (5) 3.250$\pm$0.014*                & (5) 5.515$\pm$0.016*                  & (5) 8.685$\pm$0.007*               & (5) 3.231$\pm$0.004*                \\
\rowcolor{Gray}
   5 & $\mathcal{N}$(p=C(mean(y)), s=C(std(y)))                                              & (6) 3.639$\pm$0.025                 & (6) 5.774$\pm$0.030                   & (6) 8.990$\pm$0.008                & (6) 3.731$\pm$0.015                 \\
\hline


\hline
\hline
   0 & $\mathcal{N}$(p=LR, s=Min(RE(p, \{Scale, GP\}))) & (2) 2.75$\pm$0.06*                  & (1) 5.402$\pm$0.026*                  & (3) 8.705$\pm$0.032*               & (1) 1.896$\pm$0.013*                \\
   1 & $\mathcal{N}$(p=LR, s=Min(RE(p, LR)))            & (3) 2.905$\pm$0.014*                & (2) 5.403$\pm$0.027*                  & (1) 8.682$\pm$0.030*               & (3) 2.39$\pm$0.05*                  \\
   2 & $\mathcal{N}$(p=LR, s=Min(RE(p, RF)))            & (1) 2.73$\pm$0.04*                  & (4) 5.454$\pm$0.034*                  & (2) 8.69$\pm$0.04*                 & (2) 1.907$\pm$0.027*                \\
   3 & $\mathcal{N}$(p=LR, s=Min(RE(p, C(std(y)))))     & (4) 3.154$\pm$0.018*                & (3) 5.438$\pm$0.023*                  & (4) 8.711$\pm$0.022*               & (4) 2.605$\pm$0.006*                \\
\rowcolor{Gray}
   4 & $\mathcal{N}$(p=LR, s=C(std(y)))                                           & (5) 3.277$\pm$0.005                 & (5) 5.516$\pm$0.019                   & (5) 8.764$\pm$0.016                & (5) 3.275$\pm$0.005                 \\
\rowcolor{Gray}
   5 & $\mathcal{N}$(p=C(mean(y)), s=C(std(y)))                                   & (6) 3.643$\pm$0.034                 & (6) 5.765$\pm$0.017                   & (6) 8.994$\pm$0.023                & (6) 3.732$\pm$0.016                 \\
\hline
\end{tabular}

%% file: data/tex/residual_transformers.tex
\begin{tabular}{rlllll}
\hline
   \# & Model                                                   & CV(Boston, $L_{\ell}$)   & CV(Diabetes, $L_{\ell}$)   & CV(Bikes, $L_{\ell}$)   & CV(Energy, $L_{\ell}$)   \\
\hline
   0 & $\mathcal{N}$(p=RF, s=Min(RE(p, RF, squared))) & (2) 2.71$\pm$0.04*                  & (1) 6.08$\pm$0.10*                    & (1) 9.73$\pm$0.23*                 & (1) 0.492$\pm$0.027*                \\
   1 & $\mathcal{N}$(p=RF, s=Min(RE(p, RF, log)))     & (1) 2.70$\pm$0.04*                  & (2) 6.08$\pm$0.12*                    & (3) 12.44$\pm$0.14*                & (3) 1.047$\pm$0.006*                \\
   2 & $\mathcal{N}$(p=RF, s=Min(RE(p, RF, abs)))     & (3) 2.72$\pm$0.09*                  & (3) 6.10$\pm$0.13*                    & (2) 10.46$\pm$0.10*                & (2) 0.61$\pm$0.11*                  \\
\hline

\hline
\hline
   0 & $\mathcal{N}$(p=LR, s=Min(RE(p, LR, squared))) & (1) 2.900$\pm$0.013*                & (1) 5.400$\pm$0.034*                  & (1) 8.698$\pm$0.025*               & (2) 2.364$\pm$0.024*                \\
   1 & $\mathcal{N}$(p=LR, s=Min(RE(p, LR, abs)))     & (2) 2.92$\pm$0.05*                  & (2) 5.48$\pm$0.06*                    & (2) 8.73$\pm$0.04*                 & (1) 2.33$\pm$0.04*                  \\
   2 & $\mathcal{N}$(p=LR, s=Min(RE(p, LR, log)))     & (3) 3.02$\pm$0.05*                  & (3) 5.77$\pm$0.11*                    & (3) 9.29$\pm$0.07*                 & (3) 2.37$\pm$0.05*                  \\
\hline

\hline
\hline
   0 & $\mathcal{N}$(p=\{Scale, GP\}, s=Min(RE(p, \{Scale, GP\}, squared))) & (2) 2.98$\pm$0.10*                  & (1) 5.390$\pm$0.015*                  & (1) 8.551$\pm$0.022*               & (1) 0.47$\pm$0.05*                  \\
   1 & $\mathcal{N}$(p=\{Scale, GP\}, s=Min(RE(p, \{Scale, GP\}, abs)))     & (3) 3.01$\pm$0.13*                  & (2) 5.445$\pm$0.033*                  & (2) 8.568$\pm$0.018*               & (2) 0.58$\pm$0.05*                  \\
   2 & $\mathcal{N}$(p=\{Scale, GP\}, s=Min(RE(p, \{Scale, GP\}, log)))     & (1) 2.889$\pm$0.029*                & (3) 5.63$\pm$0.04*                    & (3) 8.918$\pm$0.035*               & (3) 1.045$\pm$0.006*                \\
\hline
\end{tabular}

%% file: data/tex/shape_comparison.tex
\begin{tabular}{rlllll}
\hline
   \# & Model                       & CV(Boston, $L_{\ell}$)   & CV(Diabetes, $L_{\ell}$)   & CV(Bikes, $L_{\ell}$)   & CV(Energy, $L_{\ell}$)   \\
\hline
   0 & $\mathcal{N}$(p=LR, s=C(std(y)))    & (2) 3.284$\pm$0.021                 & (1) 5.517$\pm$0.015                   & (1) 8.768$\pm$0.019                & (2) 3.2752$\pm$0.0034               \\
   1 & Laplace(p=LR, s=C(std(y))) & (1) 3.278$\pm$0.008                 & (2) 5.612$\pm$0.026                   & (2) 8.866$\pm$0.015                & (1) 3.2243$\pm$0.0035               \\
   2 & Uniform(p=LR, s=C(std(y))) & (3) 16.2$\pm$0.4                    & (3) 16.1$\pm$0.7                      & (3) 17.95$\pm$0.33                 & (3) 12.72$\pm$0.15                  \\
\hline

\hline
\hline
   0 & $\mathcal{N}$(p=RF, s=C(std(y)))    & (2) 3.200$\pm$0.016*                & (1) 5.553$\pm$0.019*                  & (1) 8.724$\pm$0.009*               & (2) 3.2310$\pm$0.0018*              \\
   1 & Laplace(p=RF, s=C(std(y))) & (1) 3.160$\pm$0.015*                & (2) 5.654$\pm$0.013*                  & (2) 8.828$\pm$0.008*               & (1) 3.041$\pm$0.005*                \\
   2 & Uniform(p=RF, s=C(std(y))) & (3) 13.9$\pm$0.5*                   & (3) 17.2$\pm$0.4*                     & (3) 17.52$\pm$0.20*                & (3) 12.9$\pm$0.5*                   \\
\hline

\hline
\hline
   0 & Laplace(p=\{Scale, GP\}, s=C(std(y))) & (1) 3.198$\pm$0.006*                & (2) 5.610$\pm$0.017*                  & (2) 8.799$\pm$0.017*               & (1) 3.0385$\pm$0.0030*              \\
   1 & $\mathcal{N}$(p=\{Scale, GP\}, s=C(std(y)))    & (2) 3.252$\pm$0.015*                & (1) 5.516$\pm$0.018*                  & (1) 8.687$\pm$0.014*               & (2) 3.231$\pm$0.006*                \\
   2 & Uniform(p=\{Scale, GP\}, s=C(std(y))) & (3) 14.67$\pm$0.31*                 & (3) 16.50$\pm$0.28*                   & (3) 17.36$\pm$0.16*                & (3) 9.6$\pm$0.4*                    \\
\hline
\end{tabular}

%% file: data/tex/gp.tex
\begin{tabular}{rlllll}
\hline
   \# & Model                    & CV(Boston, $L_{\ell}$)   & CV(Diabetes, $L_{\ell}$)   & CV(Bikes, $L_{\ell}$)                & CV(Energy, $L_{\ell}$)   \\
\hline
B & $\mathcal{N}$(p=\{Scale, GP\}, s=Min(RE(p, \{Scale, GP\}))) & ($=$) 2.92$\pm$0.08*                  & ($+$) 5.379$\pm$0.024*                  & ($+$) 8.549$\pm$0.030*               & ($+$) 0.462$\pm$0.033*                \\
B & $\mathcal{N}$(p=LR, s=Min(RE(p, \{Scale, GP\}))) & ($-$) 2.75$\pm$0.06*                  & ($+$) 5.402$\pm$0.026*                  & ($+$) 8.705$\pm$0.032*               & ($-$) 1.896$\pm$0.013*                \\
B & $\mathcal{N}$(p=RF, s=C(std(y)))                                           & ($-$) 3.198$\pm$0.013*                & ($+$) 5.560$\pm$0.015*                  & ($+$) 8.724$\pm$0.008*               & ($-$) 3.231$\pm$0.004*                \\
\hline
   0 & $\mathcal{N}$(p=s=Min(\{Scale, GP\})) & (1) 2.92$\pm$0.10*                  & (1) 5.83$\pm$0.09*                    & (1) 22.79$\pm$0.11*                             & (1) 1.037$\pm$0.011*                \\
   1 & $\mathcal{N}$(p=s=\{Scale, GP\}) & (2) 9.5$\pm$0.4*                    & (2) 22.91$\pm$0.20*                   & (2) 24.03$\pm$0.0* & (2) 11.36$\pm$0.13*                 \\
\hline
\end{tabular}